%% file: main.tex
\documentclass[pmlr]{jmlr}

\usepackage[utf8]{inputenc}
\usepackage[T1]{fontenc}

\usepackage{graphicx}
\usepackage{amsmath,amssymb}
\usepackage{thmtools}
\usepackage{thm-restate}
\usepackage{hyperref}

\usepackage{booktabs}
\usepackage{acronym}
\usepackage{csquotes}
\usepackage{tikz}
\usetikzlibrary{arrows,arrows.meta,automata,backgrounds,calc,patterns,positioning,shapes,shadows}
\usepackage{pgfplots}
\pgfplotsset{compat=1.16}
\usepackage{algorithm}
\makeatletter
\let\If\@undefined
\let\Else\@undefined
\let\ElsIf\@undefined
\let\EndIf\@undefined
\let\For\@undefined
\let\ForAll\@undefined
\let\ForEach\@undefined
\let\EndFor\@undefined
\let\While\@undefined
\let\EndWhile\@undefined
\let\Repeat\@undefined
\let\Until\@undefined
\let\State\@undefined
\let\Return\@undefined
\let\Continue\@undefined
\let\Break\@undefined
\makeatother
\usepackage[noend]{algpseudocode}
\algnewcommand{\algorithmicbreak}{\textbf{break}}
\algnewcommand{\algorithmiccontinue}{\textbf{continue}}
\algnewcommand{\algorithmicforeach}{\textbf{for each}}
\algnewcommand\Break{\algorithmicbreak}
\algnewcommand\Continue{\algorithmiccontinue}
\algnewcommand{\LeftComment}[1]{\Statex \(\triangleright\) #1}
\algdef{SE}[FOR]{ForEach}{EndForEach}[1]
	{\algorithmicforeach\ #1\ \algorithmicdo}
	{\algorithmicend\ \algorithmicforeach}
\algtext*{EndForEach} 
\algdef{SE}[DOWHILE]{DoWhile}{EndDoWhile}{\algorithmicdo}[1]{\algorithmicwhile\ #1}
\algnewcommand{\IfThen}[2]{\State \algorithmicif\ #1\ \algorithmicthen\ #2}
\usepackage[capitalise]{cleveref}
\usepackage{acronym}

\input{defs.tex}

\jmlrvolume{TBD}
\jmlryear{2026}
\jmlrworkshop{Probabilistic Graphical Models (PGM)}

\title[On the Detection of Commutative Factors in Factor Graphs]{On the Detection of Commutative Factors in Factor Graphs: Necessary and Sufficient Conditions}

\author{%
	\Name{Malte Luttermann} \Email{malte.luttermann@uni-hamburg.de}\\
	\addr Institute for Humanities-Centered Artificial Intelligence,
	University of Hamburg,
	Germany
	\AND
	\Name{Ralf Möller} \Email{ralf.moeller@uni-hamburg.de}\\
	\addr Institute for Humanities-Centered Artificial Intelligence,
	University of Hamburg,
	Germany
	\AND
	\Name{Marcel Gehrke} \Email{marcel.gehrke@uni-hamburg.de}\\
	\addr Institute for Humanities-Centered Artificial Intelligence,
	University of Hamburg,
	Germany
}

\editor{Gustau Camps-Valls, Manuele Leonelli and Gherardo Varando}

\begin{document}

\maketitle

\begin{abstract}
	Exploiting the indistinguishability of objects in a probabilistic graphical model such as a \acl{fg} is key to lifted probabilistic inference algorithms and allows for tractable probabilistic inference problems with respect to domain sizes.
	A central building block for the exploitation of indistinguishable objects in \aclp{fg} is the identification of commutative factors, i.e., factors whose output values are invariant under permutations of input values assigned to a subset of their arguments.
	In this paper, we revisit the theoretical foundations underlying the state-of-the-art algorithm to detect commutative factors.
	Specifically, we show that in its current form, the state-of-the-art algorithm relies on a central theorem that is mistakenly regarded as a sufficient condition to identify commutative factors, while it actually only implies necessary condition.
	Consequently, the state of the art might, as we show in this paper, deliver incorrect results.
	To fix the flaws currently present in the state of the art, we prove a slightly modified version of the aforementioned theorem, which serves as a necessary condition to identify commutative factors.
	Moreover, we present a corrected version of the state-of-the-art algorithm, which keeps its efficiency while ensuring correctness and introduce a complementary algorithm with tighter worst-case bounds.
\end{abstract}

\begin{keywords}
	colour passing; factor graphs; lifted inference.
\end{keywords}

\section{Introduction} \label{sec:decor_revised_intro}
Reasoning under uncertainty is a core challenge in artificial intelligence, and probabilistic graphical models offer a principled framework for this task by factorising a joint probability distribution into a product of local functions (which we call factors).
Computing marginal and conditional distributions---the central inference problem---is, however, intractable in general as the computational cost grows exponentially with the number of \acp{rv} in propositional models such as \aclp{bn}, \aclp{mn}, or \acp{fg}~\citep{Cooper1990a}.
This intractability even persists for approximate inference~\citep{Dagum1993a}.
Lifted inference algorithms exploit the indistinguishability of objects to allow for tractable probabilistic inference problems (i.e., inference problems that are solvable in polynomial time) with respect to domain sizes while maintaining exact answers~\citep{Niepert2014a}.
To perform lifted inference, however, a lifted representation encoding equivalent semantics as the initially given propositional probabilistic model must be constructed.
Constructing such a lifted representation requires the identification of indistinguishable objects, which also includes the identification of commutative factors.
A commutative factor is a factor (i.e., a function) that outputs the same value regardless of the order of the input values of a subset of its arguments.
In this paper, we solve the problem of efficiently detecting commutative factors by revisiting the theoretical foundations underlying the efficient detection of commutative factors and providing efficient algorithms with correctness guarantees that are missing in the current state of the art.

\paragraph{Previous work.}
\citet{Poole2003a} has introduced the notion of a \ac{pfg} as a lifted representation combining concepts from first-order logic with probabilistic modelling.
To perform lifted inference, \citet{Poole2003a} has proposed \ac{lve} as a lifted inference algorithm that operates on a \ac{pfg}, and since its introduction, \ac{lve} has continuously been refined by many researchers to reach its current form~\citep{Taghipour2013a}.
The current state of the art to construct a lifted representation such as a \ac{pfg} from a given propositional model is the \ac{acp} algorithm~\citep{Luttermann2024a}, which has also been generalised to enable approximate lifted model construction~\citep{Luttermann2025c,Speller2025b}.
An essential subroutine of lifted model construction (and hence of \ac{acp}) is to compute subsets of commutative arguments of a factor to be able to exploit the indistinguishability of objects.
In its original form, \ac{acp} applies a brute-force approach to compute subsets of commutative arguments of a factor.
Such a brute-force approach iterates over all possible subsets of arguments and checks whether a subset is actually commutative, thus requiring $O(2^n)$ sets to check for a factor with $n$ arguments in the worst case.
To avoid a computationally expensive brute-force approach, \citet{Luttermann2024f} have presented the \ac{decor} algorithm, which applies a pruning strategy to heavily restrict the search space of subsets of arguments to consider.
However, we found that a central theorem to the \ac{decor} algorithm does not hold in its current form and hence, the correctness of \ac{decor} is not guaranteed.

\paragraph{Our contributions.}
\acused{decorplus}
We first show that the central theorem underlying the \ac{decor} algorithm provides only a necessary, not a sufficient, condition for the identification of commutative arguments, and we prove a corrected version of the theorem.
Second, based on the corrected theorem, we present the \ac{decorplus} algorithm, which is guaranteed to return a correct solution.
We formally prove the correctness of \ac{decorplus} and show that it maintains the efficiency of \ac{decor}.
Third, we introduce a complementary bottom-up algorithm, called \ac{decorapriori}, which is  inspired by the Apriori algorithm~\citep{Agrawal1994a} and exploits the downward closure and transitivity of pairwise commutativity, thereby yielding an algorithm with tighter worst-case bounds.

\paragraph{Structure of this paper.}
The remaining part of this paper is organised as follows.
In \cref{sec:decor_revised_prelim}, we introduce the necessary background on \acp{fg} and commutative factors.
Then, in \cref{sec:decor_revised_theory}, we investigate the \ac{decor} algorithm and provide a counterexample showing that \ac{decor} does not guarantee correctness in its current form.
We further present corrected theoretical results that clarify the necessary conditions for the identification of commutative factors.
Thereafter, in \cref{sec:decor_revised_algo}, we show how our new theoretical results can be incorporated into the \ac{decor} algorithm to ensure correctness while maintaining efficiency.
In \cref{sec:decor_revised_apriori}, we introduce the \ac{decorapriori} algorithm, which is a complementary bottom-up algorithm with tighter worst-case bounds before we empirically evaluate \ac{decorplus} and \ac{decorapriori} \cref{sec:decor_revised_eval}.

\section{Background} \label{sec:decor_revised_prelim}
An \ac{fg} represents a joint probability distribution over a set of \acp{rv} by decomposing it into a product of local functions called factors~\citep{Frey1997a,Kschischang2001a}.
In the following, we write $\range{R}$ to denote the values a \ac{rv} $R$ can take (called range).

\begin{definition}[Factor Graph]
	An \emph{\ac{fg}} is a tuple $M = (\boldsymbol R, \allowbreak \boldsymbol F, \allowbreak \boldsymbol E, \allowbreak \boldsymbol \Phi)$ where $\boldsymbol R = \{R_1, \ldots, \allowbreak R_n\}$ is a set of \acp{rv}, $\boldsymbol F = \{f_1, \ldots, \allowbreak f_m\}$ is a set of factor names, and $\boldsymbol \Phi = \{\phi_1, \ldots, \allowbreak \phi_m\}$ is a set of function definitions (called factors).
	Each $\phi_j(\mathcal R_j)$ defines a function $\phi_j \colon \times_{R \in \mathcal R_j} \range{R} \to
	 \mathbb{R}_{\geq 0}$ defined over an argument sequence $\mathcal R_j$ of \acp{rv} from $\boldsymbol R$ that outputs positive real numbers, called potentials, of which at least one is non-zero.
	For each pair of variable node $R_i \in \boldsymbol R$ and factor node $f_j \in \boldsymbol F$, there is an edge $\{ R_i, \allowbreak f_j \} \in \boldsymbol E \subseteq \{ \{R, f\} \mid R \in \boldsymbol R \land f \in \boldsymbol F \}$ if $R_i \in \mathcal R_j$.
	The full joint distribution encoded by $M$ for an assignment $(R_1 = r_1, \allowbreak \ldots, \allowbreak R_n = r_n)$ to the \acp{rv} in $\boldsymbol R$, abbreviated as $\boldsymbol R = \boldsymbol r$, is defined as
	\begin{align*}
		P_M(\boldsymbol R = \boldsymbol r)
		&= \frac{1}{Z} \prod_{j = 1}^{m} \phi_j(\mathcal R_j = \boldsymbol r_j),
	\end{align*}
	where $\boldsymbol r_j$ is a projection of $\boldsymbol r$ to the argument list of $\phi_j$ and $Z$ is the normalisation constant, defined as $Z = \sum_{\boldsymbol r \in \range{R_1} \times \ldots \times \range{R_n}} \prod_{j=1}^{m} \phi_j(\mathcal R_j = \boldsymbol r_j)$.
\end{definition}

\begin{example}
	\Cref{fig:decor_revised_example_fg_graph} displays an \ac{fg} $M = (\boldsymbol R, \allowbreak \boldsymbol F, \allowbreak \boldsymbol E, \allowbreak \boldsymbol \Phi)$ that models the interplay between the competences $ComA$ and $ComB$ of two employees $Alice$ and $Bob$, the revenue $Rev$ of the company they work for, and their corresponding salaries $SalA$ and $SalB$.
	The range of each \ac{rv} is $\{\high, \allowbreak \low\}$.
	An exemplary function definition for $\phi_3$ is given in \cref{fig:decor_revised_example_fg_table}, where $\varphi_1, \ldots, \varphi_6 \in \mathbb{R}_{\geq 0}$.
	We omit the remaining function definitions for brevity.
\end{example}

\begin{figure}
	\centering
	\subfigure[\label{fig:decor_revised_example_fg_graph}]{
		\input{files/decor_revised_example_fg.tex}
	}\hfill
	\subfigure[\label{fig:decor_revised_example_fg_table}]{
		\resizebox{0.48\textwidth}{!}{\input{files/decor_revised_example_potential_table.tex}}
	}
	\caption{($a$) An example for an \ac{fg} modelling the interplay between the competences of two employees $Alice$ and $Bob$, the revenue of the company they work for, and their salaries. ($b$) An exemplary function definition for $\phi_3$, where $\varphi_1, \ldots, \varphi_6 \in \mathbb{R}_{\geq 0}$.}
	\label{fig:decor_revised_example_fg}
\end{figure}
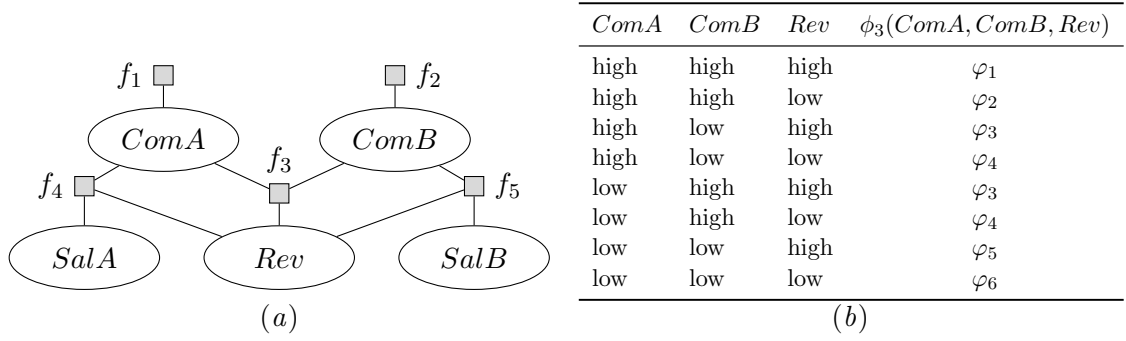

In this particular example, the revenue of the company equally depends on the competences of $Alice$ and $Bob$.
Consider the factor $\phi_3(ComA, \allowbreak ComB, \allowbreak Rev)$.
Each assignment of range values to $ComA$, $ComB$, and $Rev$ is mapped to a potential, which indicates the compatibility of the assigned values (i.e., a higher potential indicates a higher probability for the assigned values to occur together).
Specifically, we can verify that the revenue of the company equally depends on the competences of $Alice$ and $Bob$ because it holds that
\begin{alignat*}{3}
	&\phi_3(\high, \low, \high) &&= \phi_3(\low, \high, \high) &&= \varphi_5,~\text{and} \\
	&\phi_3(\high, \low, \low)  &&= \phi_3(\low, \high, \low) &&= \varphi_6.
\end{alignat*}
In other words, the potential (and hence also the probability) for a high (low, respectively) revenue of the company is the same if one of the employees is highly competent and the other is not, regardless of which employee is the highly competent one.
This property is referred to as the \emph{commutativity} of a factor---here, $\phi_3$ is commutative with respect to $ComA$ and $ComB$, i.e., $\phi_3(ComA, \allowbreak ComB, \allowbreak Rev) = \phi_3(ComB, \allowbreak ComA, \allowbreak Rev)$.

\begin{definition}[Commutative Factor, \citealp{Luttermann2024a}] \label[definition]{def:decor_revised_commutative}
	Let $\phi(R_1, \ldots, R_n)$ denote a factor and let $\boldsymbol C \subseteq \{ R_1, \allowbreak \ldots, \allowbreak R_n \}$ be a subset of $\phi$'s arguments with $\abs{\boldsymbol C} > 1$.
	Then, $\phi$ is \emph{commutative} with respect to $\boldsymbol C$ if for all assignments $(r_1, \ldots, \allowbreak r_n) \in \times_{i = 1}^{n} \range{R_i}$ it holds that $\phi(r_1, \ldots, r_n) = \phi(r_{\pi(1)}, \ldots, r_{\pi(n)})$ for all permutations $\pi$ of $\{ 1, \ldots, n \}$ with $\pi(i) = i$ for all $R_i \notin \boldsymbol C$.
	The arguments in $\boldsymbol C$ are called commutative arguments.
\end{definition}

Since in our example, the revenue only depends on the number of highly competent employees and not on knowing which specific employee is highly competent, it is possible to compress $\phi_3$ by grouping $ComA$ and $ComB$ together and counting how many $\high$ (and $\low$, respectively) values there are, which is an essential part of the current state-of-the-art algorithm \ac{acp} to construct a lifted representation~\citep{Luttermann2024a}.
Specifically, commutative arguments might be grouped and represented by a so-called \acl{crv}~\citep{Milch2008a}, which can then efficiently be handled by lifted inference algorithms.
A \acl{crv} counts the occurrences of specific range values in an assignment for a subset of a factor's arguments, and hence, it allows us to compactly encode a factor where only the number of \acp{rv} having a particular range value is relevant for the output but not which specific \acp{rv} are assigned that range value.
We provide additional details about the concept of a \acl{crv} in \cref{appendix:decor_revised_crv}.
To obtain the most compression possible, our goal is to find a subset $\boldsymbol C$ of commutative arguments of \emph{maximum size}.

\section{Efficient Detection of Commutative Factors with Buckets} \label{sec:decor_revised_theory}
To prune the search space of argument subsets when searching for a maximum sized subset of commutative arguments, the \ac{decor} algorithm partitions the potentials a factor maps its arguments to into so-called \emph{buckets}.
A bucket counts the occurrences of specific range values in an assignment to a subset of arguments.

\begin{definition}[Bucket, \citealp{Luttermann2024f}]
	Let $\phi(R_1, \ldots, R_n)$ denote a factor and let $\boldsymbol C \subseteq \{R_1, \ldots, R_n\}$ denote a subset of $\phi$'s arguments such that all arguments in $\boldsymbol C$ have the same range $\boldsymbol V$.
	A \emph{bucket} $b$ entailed by $\boldsymbol C$ is a set of tuples $\{(v_i, n_i)\}_{i = 1}^{\abs{\boldsymbol V}}$ such that $n_i \in \mathbb{N}$ specifies the number of occurrences of range value $v_i \in \boldsymbol V$ in an assignment for all \acp{rv} in $\boldsymbol C$ (i.e., $\sum_{i = 1}^{\abs{\boldsymbol V}} n_i = \abs{\boldsymbol C}$).
	A shorthand notation for $\{(v_i, n_i)\}_{i = 1}^{\abs{\boldsymbol V}}$ is $[n_1, \dots, n_{\abs{\boldsymbol V}}]$.
	In abuse of notation, we denote by $\phi^{\succ}(b)$ the ordered multiset (with respect to their order in the potential table of $\phi$) of potentials the assignments represented by $b$ are mapped to by $\phi$.
	The set of all buckets entailed by the arguments of $\phi$ is denoted as $\mathcal B(\phi)$.
\end{definition}

\begin{example}
	Consider the factor $\phi_3(ComA, \allowbreak ComB, \allowbreak Rev)$ depicted in \cref{tab:decor_revised_example_buckets} and let $\boldsymbol C = \{ComA, \allowbreak ComB, \allowbreak Rev\}$ with $\range{ComA} = \range{ComB} = \range{Rev} = \{\high, \low\}$.
	Then, $\boldsymbol C$ entails the four buckets $\mathcal B(\phi_3) = \{(\high, \allowbreak 3), \allowbreak (\low, \allowbreak 0)\}, \{(\high, \allowbreak 2), \allowbreak (\low, \allowbreak 1)\}, \{(\high, \allowbreak 1), \allowbreak (\low, \allowbreak 2)\}, \{(\high, 0), (\low, 3)\}$ (or $\mathcal B(\phi_3) = \{ [3,0], [2,1], [1,2], [0,3] \}$ in shorthand notation).
	According to the potential table of $\phi_3$, it holds that $\phi_3^{\succ}([3,0]) = \langle \varphi_1 \rangle$, $\phi_3^{\succ}([2,1]) = \langle \varphi_2, \varphi_3, \varphi_3 \rangle$, $\phi_3^{\succ}([1,2]) = \langle \varphi_4, \varphi_4, \varphi_5 \rangle$, and $\phi_3^{\succ}([0,3]) = \langle \varphi_6 \rangle$, where the order of the potentials in each ordered multiset is given by their order of appearance in $\phi_3$'s potential table.
\end{example}

\begin{table}
	\centering
	\input{files/decor_revised_example_buckets.tex}
	\caption{A factor $\phi_3(ComA, ComB, Rev)$ and the partition of its potentials into buckets.}
	\label{tab:decor_revised_example_buckets}
\end{table}

Buckets induce helpful properties that allow for pruning the search space of possible subsets of commutative arguments, as we investigate next.

\subsection{Theoretical Foundations of Commutativity Detection with Buckets}
\citet{Luttermann2024f} show that the number of duplicate potentials in the buckets induces an upper bound for the size of a subset of commutative arguments of a factor.

\begin{proposition}[\citealp{Luttermann2024f}] \label[proposition]{th:decor_revised_duplicates_in_buckets}
	Let $\phi(R_1, \allowbreak \ldots, \allowbreak R_n)$ be a factor that is commutative with respect to $\boldsymbol C \subseteq \{ R_1, \allowbreak \ldots, \allowbreak R_n \}$.
	Then, for every bucket $b \in \mathcal B(\phi)$ with $\abs{\phi^{\succ}(b)} > 1$, there exists a potential $\varphi$ that occurs at least $\abs{\boldsymbol C}$ times in $\phi^{\succ}(b)$.
\end{proposition}

Hence, we can conclude that any commutative subset $\boldsymbol C$ contains at most as many arguments as the minimum number of duplicate potentials over all buckets containing more than one potential.
In addition, buckets allow us to directly restrict candidates of possibly commutative arguments.
Originally, the following claim has been established as the central theorem of \ac{decor} to identify subsets of commutative arguments by looking at the element-wise intersection of the assignments corresponding to identical potentials in a bucket.

\begin{claim}[\citealp{Luttermann2024f}] \label[claim]{th:decor_revised_args_by_buckets}
	Let $\phi(R_1, \ldots, R_n)$ denote a factor and let $b \in \mathcal B(\phi)$ with $\abs{\phi^{\succ}(b)} > 1$ denote a bucket entailed by $\phi$.
	Further, let $\boldsymbol \Psi = \{\varphi_1, \ldots, \varphi_{\ell}\}$ with $\abs{\boldsymbol \Psi} > 2$ denote an arbitrary maximal set of identical potentials in $\phi^{\succ}(b)$ such that $\phi(r_1^1, \ldots, r_n^1) = \varphi_1, \ldots, \phi(r_1^{\ell}, \ldots, r_n^{\ell}) = \varphi_{\ell}$.
	Then, computing the element-wise intersection $(r_1^1, \ldots, r_n^1) \cap \ldots \cap (r_1^{\ell}, \ldots, r_n^{\ell}) = (\{r_1^1\} \cap \ldots \cap \{r_1^{\ell}\}, \ldots, \{r_n^1\} \cap \ldots \cap \{r_n^{\ell}\})$ and adding all arguments $R_i$ for which $\{r_i^1\} \cap \ldots \cap \{r_i^{\ell}\} = \emptyset$ to a set $\boldsymbol C_b$ yields a set of candidate commutative arguments for bucket $b$ such that $\boldsymbol C = \bigcap_{b} \boldsymbol C_b$ is a subset of commutative arguments of $\phi$.
\end{claim}

We now show that this claim (reprinted from \citep{Luttermann2024f}), which is central to \ac{decor}, is flawed in its current form, as it only establishes a necessary condition for the identification of commutative arguments, but not a sufficient one as originally claimed.

\begin{restatable}{theorem}{decorRevisedFlawByBucketsTheorem} \label{th:decor_revised_flaw_in_args_by_buckets}
	The application of \cref{th:decor_revised_args_by_buckets} to compute a subset of arguments $\boldsymbol C$ does not guarantee that the computed set $\boldsymbol C$ is actually a subset of commutative arguments.
\end{restatable}
\begin{proof}
	Consider $\phi(R_1, \allowbreak R_2, \allowbreak R_3, \allowbreak R_4)$ with $\range{R_i} = \{ \high, \allowbreak \low \}$ for all $i \in \{1, \ldots, 4\}$ and let $\phi(\low, \allowbreak \low, \allowbreak \high, \allowbreak \high) = \varphi_1$, $\phi(\high, \allowbreak \high, \allowbreak \low, \allowbreak \low) = \varphi_1$, and $\phi(\boldsymbol r) = \varphi_2$ with $\varphi_1 \neq \varphi_2$ for all remaining assignments $\boldsymbol r$.
	Applying \cref{th:decor_revised_args_by_buckets} yields $\boldsymbol C = \{R_1, \allowbreak R_2, \allowbreak R_3, \allowbreak R_4\}$ although $\phi$ is not commutative with respect to $\boldsymbol C$ (a detailed computation of $\boldsymbol C$ is given in \cref{appendix:decor_revised_proofs}).
\end{proof}

\Cref{th:decor_revised_flaw_in_args_by_buckets} shows that applying \cref{th:decor_revised_args_by_buckets} is not guaranteed to yield a valid subset of commutative arguments.
To solve this problem, we next rephrase \cref{th:decor_revised_args_by_buckets} as a necessary condition---because the subset obtained from applying \cref{th:decor_revised_args_by_buckets} does not miss any truly commutative arguments, but it may happen that too many arguments are included.

\subsection{Necessary Condition for the Detection of Commutative Arguments}
The upcoming theorem presents a corrected version of \cref{th:decor_revised_args_by_buckets} that captures the necessary condition for the identification of commutative arguments by looking at the element-wise intersection of the assignments corresponding to groups of identical potentials in a bucket.

\begin{theorem} \label{th:decor_revised_args_by_buckets_fixed}
	Let $\phi(R_1, \allowbreak \ldots, \allowbreak R_n)$ denote a factor, let $\boldsymbol C \subseteq \{ R_1, \allowbreak \ldots, \allowbreak R_n \}$ with $\abs{\boldsymbol C} > 1$ denote a subset of $\phi$'s arguments of maximum size such that $\phi$ is commutative with respect to $\boldsymbol C$, and let $b \in \mathcal B(\phi)$ denote any bucket entailed by the arguments of $\phi$ such that $\abs{\phi^{\succ}(b)} > 1$ holds.
	Then, there exists a maximal set of identical potentials $\boldsymbol \Psi = \{ \varphi_1, \allowbreak \ldots, \allowbreak \varphi_{\ell} \} \subseteq \phi^{\succ}(b)$ in $\phi^{\succ}(b)$ with $\abs{\boldsymbol \Psi} > 1$ and $\phi(r_1^1, \allowbreak \ldots, \allowbreak r_n^1) = \varphi_1, \allowbreak \ldots, \allowbreak \phi(r_1^{\ell}, \allowbreak \ldots, \allowbreak r_n^{\ell}) = \varphi_{\ell}$ such that for all $R_i \in \boldsymbol C$ the element-wise intersection $\big( r_1^1, \allowbreak \ldots, \allowbreak r_n^1 \big) \cap \ldots \cap \big( r_1^{\ell}, \allowbreak \ldots, \allowbreak r_n^{\ell} \big) = \big( \{ r_1^1 \} \cap \ldots \cap \{ r_1^{\ell} \}, \allowbreak \ldots, \allowbreak \{ r_n^1 \} \cap \ldots \cap \{ r_n^{\ell} \} \big)$ contains an empty set at position $i \in \{ 1, \allowbreak \ldots, \allowbreak n \}$.
\end{theorem}
\begin{proof}
	Since we consider only buckets $b$ with at least two elements in $\phi^{\succ}(b)$, all corresponding assignments contain at least two distinct assigned values because all assignments that contain only a single value are those that belong to a bucket that is mapped to a single potential.
	According to \cref{def:decor_revised_commutative}, for all assignments $(r_1, \ldots, r_n) \in \times_{i = 1}^{n} \range{R_i}$ it holds that $\phi(r_1, \allowbreak \ldots, \allowbreak r_n) = \phi(r_{\pi(1)}, \allowbreak \ldots, \allowbreak r_{\pi(n)})$ for all permutations $\pi$ of $\{ 1, \allowbreak \ldots, \allowbreak n \}$ with $\pi(i) = i$ for all $R_i \notin \boldsymbol C$.
	Let $(r_1^1, \allowbreak \ldots, \allowbreak r_n^1)$ denote any assignment that belongs to bucket $b$, that is, $\phi(r_1^1, \allowbreak \ldots, \allowbreak r_n^1) = \varphi$ with $\varphi \in \phi^{\succ}(b)$, such that $(r_1^1, \allowbreak \ldots, \allowbreak r_n^1)$ contains at least two distinct assigned values, say $r_j^1$ and $r_k^1$ with $r_j^1 \neq r_k^1$, at any two positions $j, k \in \{ 1, \allowbreak \ldots, \allowbreak n \}$ for which it holds that $R_j \in \boldsymbol C$ and $R_k \in \boldsymbol C$ are commutative arguments of $\phi$.
	Further, let $\boldsymbol \Psi \subseteq \phi^{\succ}(b)$ denote the set of all potentials $\varphi = \phi(r_{\pi(1)}, \allowbreak \ldots, \allowbreak r_{\pi(n)})$ resulting from applying all permutations $\pi$ of $\{ 1, \allowbreak \ldots, \allowbreak n \}$ with $\pi(i) = i$ for all $R_i \notin \boldsymbol C$ to the assignment $(r_1^1, \allowbreak \ldots, \allowbreak r_n^1)$.
	By construction, all potentials in $\boldsymbol \Psi$ are equal to $\varphi$ and hence $\boldsymbol \Psi$ is a set of identical potentials.
	As it holds that $\abs{\boldsymbol C} > 1$, there are at least two positions that are not fixed in each permutation $\pi$ and that contain distinct assigned values in the assignment $(r_1^1, \allowbreak \ldots, \allowbreak r_n^1)$, which implies that there exist at least two distinct assignments, say $\boldsymbol r_1$ and $\boldsymbol r_2$ with $\boldsymbol r_1 \neq \boldsymbol r_2$, such that $\phi(\boldsymbol r_1) = \varphi \in \boldsymbol \Psi$ and $\phi(\boldsymbol r_2) = \varphi \in \boldsymbol \Psi$ and thus, it holds that $\abs{\boldsymbol \Psi} > 1$.
	Now, let $(r_1^1, \allowbreak \ldots, \allowbreak r_n^1), \allowbreak \ldots, \allowbreak (r_1^{\ell}, \allowbreak \ldots, \allowbreak r_n^{\ell})$ denote the assignments corresponding to the potentials in $\boldsymbol \Psi$, which are obtained by applying all permutations $\pi$ of $\{ 1, \allowbreak \ldots, \allowbreak n \}$ with $\pi(i) = i$ for all $R_i \notin \boldsymbol C$ to the assignment $(r_1^1, \allowbreak \ldots, \allowbreak r_n^1)$.
	As each position $i \in \{ 1, \allowbreak \ldots, \allowbreak n \}$ with $R_i \in \boldsymbol C$ is not fixed in the permutations and there exist at least two distinct assigned values that are permuted over all positions belonging to arguments in $\boldsymbol C$, there are at least two assignments among $(r_1^1, \allowbreak \ldots, \allowbreak r_n^1), \allowbreak \ldots, \allowbreak (r_1^{\ell}, \allowbreak \ldots, \allowbreak r_n^{\ell})$ such that their assigned values at each such position $i$ differ.
	Thus, computing the intersection of the assigned values at any such position $i$ yields an empty set.
	In case there are further potentials in $\phi^{\succ}(b)$ that are equal to $\varphi$ but not contained in $\boldsymbol \Psi$ so far (that is, if $\boldsymbol \Psi$ is not maximal), we extend $\boldsymbol \Psi$ by adding these potentials to $\boldsymbol \Psi$ such that $\boldsymbol \Psi$ becomes a maximal set of identical potentials.
	The element-wise intersection of all assignments corresponding to the potentials in $\boldsymbol \Psi$ then still yields an empty set at each position $i \in \{ 1, \allowbreak \ldots, \allowbreak n \}$ for which it holds that $R_i \in \boldsymbol C$ because the intersection has already been empty at these positions before extending $\boldsymbol \Psi$ and the extension of $\boldsymbol \Psi$ only adds additional sets to be intersected.
\end{proof}

We next apply \cref{th:decor_revised_args_by_buckets_fixed} to the factor $\phi_3(ComA, \allowbreak ComB, \allowbreak Rev)$ depicted in \cref{tab:decor_revised_example_buckets}.

\begin{example}
	Consider $\phi_3(ComA, \allowbreak ComB, \allowbreak Rev)$ from \cref{tab:decor_revised_example_buckets} and let $b = [2,1]$.
	There are two maximal groups of identical potentials in $\phi_3^{\succ}(b)$: $\boldsymbol \Psi_1 = \{\varphi_2\}$ and $\boldsymbol \Psi_2 = \{\varphi_3, \allowbreak \varphi_3\}$.
	Since $\abs{\boldsymbol \Psi_1} = 1$, no candidates of commutative arguments are induced by $\boldsymbol \Psi_1$.
	However, $\boldsymbol \Psi_2$ contains the potential $\varphi_3$ twice and the corresponding assignments are $(\high, \allowbreak \low, \allowbreak \high)$ and $(\low, \allowbreak \high, \allowbreak \high)$.
	The element-wise intersection is then given by $(\{\high\}, \allowbreak \{\low\}, \allowbreak \{\high\}) \cap (\{\low\}, \allowbreak \{\high\}, \allowbreak \{\high\}) = (\emptyset, \allowbreak \emptyset, \allowbreak \{\high\})$.
	As the element-wise intersection is empty at positions one and two, the set $\{ComA, \allowbreak ComB\}$ is a candidate subset of commutative arguments.
\end{example}

By applying \cref{th:decor_revised_args_by_buckets_fixed}, we are able to significantly prune the search space of possible subsets of commutative arguments.
Nevertheless, as \cref{th:decor_revised_args_by_buckets_fixed} provides only a necessary condition, we need an additional verification step in the end to check whether the computed candidate subset is indeed a subset of commutative arguments.
We next show how \cref{th:decor_revised_args_by_buckets_fixed} is applied in practice to efficiently identify maximum sized subsets of commutative arguments.

\section{The \acs{decorplus} Algorithm} \label{sec:decor_revised_algo}
To enable the efficient computation of maximum sized subsets of commutative arguments in practice, we introduce a revised version of the \ac{decor} algorithm (called \ac{decorplus}) based on the theoretical results presented in the previous section.
\Cref{alg:decor_revised} presents the entire \ac{decorplus} algorithm, which we now describe in detail.
\begin{algorithm}[t]
	\caption{Detection of Commutative Factors Revised (DECOR+)}
	\label{alg:decor_revised}
	\alginput{A factor $\phi(R_1, \ldots, R_n)$.} \\
	\algoutput{A set containing all maximum sized subsets of commutative arguments of $\phi$.}
	\begin{algorithmic}[1]
		\State $\boldsymbol C_{cand} \gets \{\{R_1, \ldots, R_n\}\}$ \label{line:decor_revised_init}
		\ForEach{bucket $b \in \mathcal B(\phi)$ \label{line:decor_revised_loop_buckets}}
			\IfThen{$\abs{\phi^{\succ}(b)} < 2$ \label{line:decor_revised_if_single_potential_bucket}}{
				\Continue \label{line:decor_revised_continue_single_potential_bucket}
			}
			\State $\boldsymbol \Psi_1, \ldots, \boldsymbol \Psi_{\ell} \gets$ Partition of $\phi^{\succ}(b)$ into maximal groups of identical potentials such that $\abs{\boldsymbol \Psi_i} \geq 2$ holds for all $i \in \{1, \ldots, \ell\}$ \label{line:decor_revised_partition_max_groups}
			\IfThen{$\ell < 1$ \label{line:decor_revised_if_no_groups}}{
				\algorithmicreturn\ $\emptyset$ \label{line:decor_revised_return_empty_no_groups}
			}
			\State $\boldsymbol C' \gets \emptyset$ \label{line:decor_revised_init_c_prime}
			\ForEach{group $\boldsymbol \Psi_i \in \{ \boldsymbol \Psi_1, \ldots, \boldsymbol \Psi_{\ell} \}$ \label{line:decor_revised_loop_groups}}
				\State $\boldsymbol P_1, \ldots, \boldsymbol P_n \gets$ Intersection of positions corresponding to all potentials in $\boldsymbol \Psi_i$ \label{line:decor_revised_intersection_positions}
				\State $\boldsymbol C_i \gets \{ R_i \mid i \in \{ 1, \ldots, n \} \land \boldsymbol P_i = \emptyset \}$ \label{line:decor_revised_set_ci}
				\IfThen{$\nexists \boldsymbol C_j \in \boldsymbol C': \boldsymbol C_i \subseteq \boldsymbol C_j$ \label{line:decor_revised_if_ci_not_subsumed}}{
					$\boldsymbol C' \gets \boldsymbol C' \cup \{ \boldsymbol C_i \}$ \label{line:decor_revised_update_c_prime}
				}
			\EndForEach
			\State $\boldsymbol C_{\cap} \gets \emptyset$ \label{line:decor_revised_init_c_cap}
			\ForEach{candidate set $\boldsymbol C_i \in \boldsymbol C_{cand}$ \label{line:decor_revised_loop_intersect_outer}}
				\ForEach{candidate set $\boldsymbol C_j \in \boldsymbol C'$ \label{line:decor_revised_loop_intersect_inner}}
					\IfThen{$\abs{\boldsymbol C_i \cap \boldsymbol C_j} \geq 2 \land \nexists \boldsymbol C'' \in \boldsymbol C_{\cap}: (\boldsymbol C_i \cap \boldsymbol C_j) \subseteq \boldsymbol C''$ \label{line:decor_revised_if_ci_cj_not_subsumed}}{
						$\boldsymbol C_{\cap} \gets \boldsymbol C_{\cap} \cup \{ \boldsymbol C_i \cap \boldsymbol C_j \}$\label{line:decor_revised_update_c_cap}
					}
				\EndForEach
			\EndForEach
			\State $\boldsymbol C_{cand} \gets \boldsymbol C_{\cap}$ \label{line:decor_revised_update_c}
			\IfThen{$\boldsymbol C_{cand} = \emptyset$ \label{line:decor_revised_if_c_empty}}{
				\algorithmicreturn\ $\emptyset$ \label{line:decor_revised_return_empty_c}
			}
		\EndForEach
		\State \Return $verify(\boldsymbol C_{cand})$ \label{line:decor_revised_return_c}
	\end{algorithmic}
\end{algorithm}
\Ac{decorplus} starts with the candidate subset of all arguments and then iterates over all buckets containing at least two potentials to compute partitions of maximal groups of identical potentials in each bucket, which are then used to prune the search space according to \cref{th:decor_revised_args_by_buckets_fixed}.
Specifically, if there are no duplicate potentials in a bucket at all, there cannot be any commutative arguments (\cref{th:decor_revised_duplicates_in_buckets}) and \ac{decorplus} immediately returns (\cref{line:decor_revised_return_empty_no_groups}).
Otherwise, \ac{decorplus} computes a subset $\boldsymbol C_i$ of candidate arguments permitted by each maximal group $\boldsymbol \Psi_i$ of identical potentials in the currently considered bucket.
The subset $\boldsymbol C_i$ contains all arguments $R_i$ for which the element-wise intersection of the assignments corresponding to the potentials in $\boldsymbol \Psi_i$ is empty at position $i$.
The computed subsets $\boldsymbol C_i$ of candidate arguments for each group of identical potentials in the current bucket $b$ are then collected into a set $\boldsymbol C'$ of candidate subsets for the current bucket $b$ (such that no candidate subset is subsumed by another candidate subset).
The collected candidate subsets permitted by a specific bucket are then intersected with the candidate subsets in $\boldsymbol C_{cand}$ collected from the previous buckets to obtain candidate subsets that are permitted by all buckets considered so far (again, subsets that are subsumed by another candidate subset after the intersection are removed).
If there are no candidate subsets left at some point, \ac{decorplus} immediately returns.
Finally, if there are candidate subsets left after iterating over all buckets, \ac{decorplus} runs a verification step to check whether the candidate subsets are indeed subsets of commutative arguments.
The verification step considers every candidate subset $\boldsymbol C \in \boldsymbol C_{cand}$ and checks whether the factor $\phi$ is indeed commutative with respect to $\boldsymbol C$ (by directly applying \cref{def:decor_revised_commutative}).
If the check succeeds, $\boldsymbol C$ is kept as a confirmed subset of commutative arguments.
Otherwise, $\boldsymbol C$ might still contain a strict subset of commutative arguments and \ac{decorplus} continues by checking all $(\abs{\boldsymbol C}-1)$-element subsets of $\boldsymbol C$, then all $(\abs{\boldsymbol C}-2)$-element subsets of $\boldsymbol C$, and so on, until either a commutative subset is found or no subsets of size at least two remain.
Since we are only interested in a single maximum sized subset of commutative arguments, \ac{decorplus} might stop the verification as soon as the first commutative subset is found (we nevertheless formulate \ac{decorplus} as a general algorithm in \cref{alg:decor_revised} that is able to return all maximum sized subsets of commutative arguments if desired).
An example run of \ac{decorplus} is given in \cref{appendix:decor_revised_example_run}.
We also emphasise that \ac{decorplus} can handle factors with arguments having arbitrary ranges.
In this paper, we consider identical ranges of all arguments for brevity, however, it is possible to apply \ac{decorplus} to factors with arguments having arbitrary ranges $\boldsymbol V_1, \allowbreak \ldots, \allowbreak \boldsymbol V_k$ by considering the buckets for all arguments with range $\boldsymbol V_i$ separately for all $i \in \{1, \allowbreak \ldots, \allowbreak k\}$ as only arguments with the same range can be in the same set of commutative arguments.
Next, we prove the correctness of \ac{decorplus}, i.e., we show that the pruning step in \ac{decorplus} does not remove any truly commutative arguments such that the verification step is guaranteed to find a maximum sized subset of commutative arguments.

\begin{restatable}{theorem}{decorplusCorrectnessTheorem} \label{th:decor_revised_correctness}
	Let $\phi(R_1, \allowbreak \ldots, \allowbreak R_n)$ denote a factor that is given as an input to \ac{decorplus} (\cref{alg:decor_revised}).
	Then, the set returned by \ac{decorplus} contains exactly the maximum sized subsets of commutative arguments of $\phi$'s arguments, or the empty set if no such subset exists.
\end{restatable}
\begin{proof}
	Let $\boldsymbol C^{*} \subseteq \{R_1, \allowbreak \ldots, \allowbreak R_n\}$ with $\abs{\boldsymbol C^{*}} \geq 2$ denote any maximum sized subset of commutative arguments of $\phi$.
	By \cref{th:decor_revised_args_by_buckets_fixed}, for every bucket $b \in \mathcal B(\phi)$ with $\abs{\phi^{\succ}(b)} > 1$, there exists a maximal set of identical potentials $\boldsymbol \Psi_i \subseteq \phi^{\succ}(b)$ with $\abs{\boldsymbol \Psi_i} > 1$ such that the element-wise intersection of the assignments corresponding to the potentials in $\boldsymbol \Psi_i$ contains an empty set at every position $j \in \{1, \allowbreak \ldots, \allowbreak n\}$ for which $R_j \in \boldsymbol C^{*}$.
	Consequently, the set $\boldsymbol C_i$ computed for $\boldsymbol \Psi_i$ in the inner loop of \cref{alg:decor_revised} (\cref{line:decor_revised_set_ci}) satisfies $\boldsymbol C^{*} \subseteq \boldsymbol C_i$, and hence either $\boldsymbol C_i$ itself is added to $\boldsymbol C'$ or it is subsumed by another candidate subset already in $\boldsymbol C'$ that contains $\boldsymbol C^{*}$.
	The cross-intersection step preserves $\boldsymbol C^{*}$ as a subset of some entry in $\boldsymbol C_{cand}$ at every iteration, because the intersection of two supersets of $\boldsymbol C^{*}$ is itself a superset of $\boldsymbol C^{*}$ and the subsumption check only removes entries that are themselves subsumed by a superset.
	By induction over the buckets, every maximum sized subset $\boldsymbol C^{*}$ of commutative arguments of $\phi$ is contained in some entry of $\boldsymbol C_{cand}$ after the bucket loop terminates.
	The correctness then immediately follows as the verification step is exact (it directly applies \cref{def:decor_revised_commutative}) and considers the subsets in order of decreasing size (if verification fails, an empty set is returned).
\end{proof}

Even though \ac{decorplus} might have to check all subsets of a candidate subset in the worst case during the verification step, a major advantage of \ac{decorplus} is that the number of candidate subsets computed by \ac{decorplus} depends on the number of maximal groups of identical potentials in the buckets.
In most practical settings, potentials are not identical unless they stem from a subset of commutative arguments.
Therefore, \ac{decorplus} heavily restricts the search space in most practical settings, which we also confirm later in our experiments.
Specifically, the bucket structure of a factor $\phi$ provides an even tighter bound than the bound implied by \cref{th:decor_revised_duplicates_in_buckets} on the size of any commutative subset of $\phi$.

\begin{proposition} \label[proposition]{prop:decor_revised_bound_on_commutative_subset}
	Let $\phi(R_1, \allowbreak \ldots, \allowbreak R_n)$ denote a factor that is commutative with respect to $\boldsymbol C \subseteq \{R_1, \allowbreak \ldots, \allowbreak R_n\}$ and let $\boldsymbol C'_b$ denote the set of candidate subsets computed by \ac{decorplus} for a bucket $b \in \mathcal B(\phi)$ (\cref{line:decor_revised_init_c_prime,line:decor_revised_loop_groups,line:decor_revised_intersection_positions,line:decor_revised_set_ci,line:decor_revised_if_ci_not_subsumed,line:decor_revised_update_c_prime} of \cref{alg:decor_revised}).
	Then, $\abs{\boldsymbol C} \leq \min_{b \in \mathcal B(\phi), \abs{\phi^{\succ}(b)} > 1} \,\max_{\boldsymbol C_i \in \boldsymbol C'_b} \abs{\boldsymbol C_i}$.
\end{proposition}
\begin{proof}
	Let $b \in \mathcal B(\phi)$ with $\abs{\phi^{\succ}(b)} > 1$.
	By \cref{th:decor_revised_args_by_buckets_fixed}, there exists a maximal set of identical potentials $\boldsymbol \Psi \subseteq \phi^{\succ}(b)$ with $\abs{\boldsymbol \Psi} > 1$ such that the element-wise intersection over the assignments corresponding to the potentials in $\boldsymbol \Psi$ contains an empty set at every position $i \in \{1, \allowbreak \ldots, \allowbreak n\}$ with $R_i \in \boldsymbol C$.
	Hence, the candidate subset $\boldsymbol C_{\boldsymbol \Psi}$ computed by \ac{decorplus} for $\boldsymbol \Psi$ in \cref{line:decor_revised_set_ci} satisfies $\boldsymbol C \subseteq \boldsymbol C_{\boldsymbol \Psi}$.
	Due to the subsumption check in \cref{line:decor_revised_if_ci_not_subsumed}, either $\boldsymbol C_{\boldsymbol \Psi} \in \boldsymbol C'_b$ or $\boldsymbol C_{\boldsymbol \Psi}$ is subsumed by some $\boldsymbol C_j \in \boldsymbol C'_b$ with $\boldsymbol C_{\boldsymbol \Psi} \subseteq \boldsymbol C_j$.
	In both cases, there exists $\boldsymbol C_i \in \boldsymbol C'_b$ with $\boldsymbol C \subseteq \boldsymbol C_i$, and thus $\abs{\boldsymbol C} \leq \abs{\boldsymbol C_i} \leq \max_{\boldsymbol C_i \in \boldsymbol C'_b} \abs{\boldsymbol C_i}$.
	As this bound holds for each $b \in \mathcal B(\phi)$ with $\abs{\phi^{\succ}(b)} > 1$, the claim follows by taking the minimum over all such buckets.
\end{proof}

Motivated by the example given in the proof of \cref{th:decor_revised_flaw_in_args_by_buckets} showing that the candidate subsets passed to the verification step of \ac{decorplus} may strictly over-approximate the actual subsets of commutative arguments, we next present a complementary approach that avoids the verification step entirely and hence requires less commutativity checks in the worst case.

\section{An Apriori-Style Algorithm for the Detection of Commutative Factors} \label{sec:decor_revised_apriori}
The \ac{decorplus} algorithm approaches the search for a maximum sized subset of commutative arguments \emph{top-down}: It starts with a candidate set of all arguments and aggregates evidence from the buckets to gradually narrow down the candidate subsets.
We now explore a \emph{bottom-up} approach inspired by the well-known \emph{Apriori} algorithm for association rule mining~\citep{Agrawal1994a}.
The idea is to enumerate commutative subsets level-by-level, starting from pairs of arguments and extending them step by step.
The underlying key observation for this approach is that commutativity is a \emph{downward-closed} property.

\begin{proposition} \label[proposition]{prop:decor_revised_apriori_downward_closure}
	Let $\phi(R_1, \allowbreak \ldots, \allowbreak R_n)$ denote a factor and let $\boldsymbol C \subseteq \{R_1, \allowbreak \ldots, \allowbreak R_n\}$ with $\abs{\boldsymbol C} \geq 2$ denote a subset of arguments such that $\phi$ is commutative with respect to $\boldsymbol C$.
	Then, for every subset $\boldsymbol C' \subseteq \boldsymbol C$ with $\abs{\boldsymbol C'} \geq 2$, $\phi$ is also commutative with respect to $\boldsymbol C'$.
\end{proposition}
\begin{proof}
	Let $\boldsymbol C' \subseteq \boldsymbol C$ with $\abs{\boldsymbol C'} \geq 2$ and let $\pi'$ denote any permutation of $\{1, \allowbreak \ldots, \allowbreak n\}$ with $\pi'(i) = i$ for all $R_i \notin \boldsymbol C'$.
	Since $\boldsymbol C' \subseteq \boldsymbol C$, every $R_i \notin \boldsymbol C$ also satisfies $R_i \notin \boldsymbol C'$, and hence the condition $\pi'(i) = i$ for all $R_i \notin \boldsymbol C'$ implies $\pi'(i) = i$ for all $R_i \notin \boldsymbol C$.
	Consequently, $\pi'$ is also a valid permutation for the commutativity of $\phi$ with respect to $\boldsymbol C$ such that $\phi(r_1, \allowbreak \ldots, \allowbreak r_n) = \phi(r_{\pi'(1)}, \allowbreak \ldots, \allowbreak r_{\pi'(n)})$ holds for every assignment $(r_1, \allowbreak \ldots, \allowbreak r_n)$.
\end{proof}

The contrapositive of \cref{prop:decor_revised_apriori_downward_closure} states that the set of non-commutative subsets is upward-closed: Once a subset $\boldsymbol C$ has been found to be non-commutative, every superset of $\boldsymbol C$ is also non-commutative.
A second structural property that further accelerates the bottom-up search is that the commutativity of a subset of arguments is fully determined by the commutativity of its constituent pairs, that is, pairwise commutativity is transitive.

\begin{theorem} \label{th:decor_revised_apriori_transitivity}
	Let $\phi(R_1, \allowbreak \ldots, \allowbreak R_n)$ denote a factor and let $\boldsymbol C \subseteq \{R_1, \allowbreak \ldots, \allowbreak R_n\}$ with $\abs{\boldsymbol C} \geq 2$ denote a subset of arguments.
	Then, $\phi$ is commutative with respect to $\boldsymbol C$ if and only if $\phi$ is commutative with respect to every pair $\{R_i, R_j\} \subseteq \boldsymbol C$ with $i \neq j$.
\end{theorem}
\begin{proof}
	The forward direction follows directly from \cref{prop:decor_revised_apriori_downward_closure}.
	For the converse direction, assume that $\phi$ is commutative with respect to every pair $\{R_i, R_j\} \subseteq \boldsymbol C$ with $i \neq j$, and let $\pi$ denote any permutation of $\{1, \allowbreak \ldots, \allowbreak n\}$ with $\pi(i) = i$ for all $R_i \notin \boldsymbol C$.
	Then, $\pi$ permutes only indices of arguments in $\boldsymbol C$, and hence can be written as a composition $\pi = \tau_1 \circ \tau_2 \circ \cdots \circ \tau_m$ of transpositions, where each transposition $\tau_\ell$ swaps two indices $i_\ell$ and $j_\ell$ with $R_{i_\ell}, R_{j_\ell} \in \boldsymbol C$.
	By assumption, $\phi$ is commutative with respect to $\{R_{i_\ell}, R_{j_\ell}\}$ for every $\ell \in \{1, \allowbreak \ldots, \allowbreak m\}$, that is, $\phi$ is invariant under each transposition $\tau_\ell$.
	Consequently, $\phi$ is invariant under the composition $\pi = \tau_1 \circ \tau_2 \circ \cdots \circ \tau_m$, that is, $\phi(r_1, \allowbreak \ldots, \allowbreak r_n) = \phi(r_{\pi(1)}, \allowbreak \ldots, \allowbreak r_{\pi(n)})$ holds for every assignment $(r_1, \allowbreak \ldots, \allowbreak r_n)$ and hence, by \cref{def:decor_revised_commutative}, $\phi$ is commutative with respect to $\boldsymbol C$.
\end{proof}

Together, \cref{prop:decor_revised_apriori_downward_closure,th:decor_revised_apriori_transitivity} enable an efficient pruning strategy, which builds the basis of the \ac{decorapriori} algorithm, depicted in \cref{alg:decor_revised_apriori}.
\begin{algorithm}[t]
	\caption{Apriori-Style Detection of Commutative Factors (A-DECOR)}
	\label{alg:decor_revised_apriori}
	\alginput{A factor $\phi(R_1, \ldots, R_n)$.} \\
	\algoutput{A set containing all maximum sized subsets of commutative arguments of $\phi$.}
	\begin{algorithmic}[1]
		\State $\boldsymbol L_2 \gets \{ \{R_i, R_j\} \mid 1 \leq i < j \leq n \land \phi \text{ is commutative with respect to } \{R_i, R_j\} \}$ \label{line:decor_revised_apriori_init}
		\State $\boldsymbol L \gets \boldsymbol L_2$, $\quad k \gets 2$ \label{line:decor_revised_apriori_init_l}
		\While{$\boldsymbol L_k \neq \emptyset$} \label{line:decor_revised_apriori_while}
			\State $\boldsymbol L_{k+1} \gets \emptyset$ \label{line:decor_revised_apriori_init_lk1}
			\ForEach{$\boldsymbol C \in \boldsymbol L_k$} \label{line:decor_revised_apriori_loop_lk}
				\ForEach{$R_i \notin \boldsymbol C$} \label{line:decor_revised_apriori_loop_args}
					\IfThen{$\forall R_j \in \boldsymbol C \colon \{R_i, R_j\} \in \boldsymbol L_2$ \label{line:decor_revised_apriori_prune}}{
						$\boldsymbol L_{k+1} \gets \boldsymbol L_{k+1} \cup \{\boldsymbol C \cup \{R_i\}\}$ \label{line:decor_revised_apriori_add}
					}
				\EndForEach
			\EndForEach
			\IfThen{$\boldsymbol L_{k+1} \neq \emptyset$ \label{line:decor_revised_apriori_update_if}}{
				$\boldsymbol L \gets \boldsymbol L_{k+1}$ \label{line:decor_revised_apriori_update}
			}
			\State $k \gets k+1$ \label{line:decor_revised_apriori_increment}
		\EndWhile
		\State \Return $\boldsymbol L$ \label{line:decor_revised_apriori_return}
	\end{algorithmic}
\end{algorithm}
\Ac{decorapriori} starts by computing the set $\boldsymbol L_2$ of all pairs of commutative arguments in \cref{line:decor_revised_apriori_init} by checking commutativity for every pair $\{R_i, R_j\}$ of arguments individually.
Subsequently, \ac{decorapriori} enters a loop in which it extends each subset $\boldsymbol C \in \boldsymbol L_k$ by every argument $R_i \notin \boldsymbol C$ to form a candidate $\boldsymbol C' = \boldsymbol C \cup \{R_i\}$ of size $k+1$.
By \cref{th:decor_revised_apriori_transitivity}, $\boldsymbol C'$ is commutative if and only if every pair of arguments in $\boldsymbol C'$ is commutative.
The pairs internal to $\boldsymbol C$ are commutative since $\boldsymbol C \in \boldsymbol L_k$ and $\boldsymbol L_k$ contains only commutative subsets, so it suffices to verify the $\abs{\boldsymbol C}$ pairs $\{R_i, R_j\}$ with $R_j \in \boldsymbol C$ are commutative (i.e., are in $\boldsymbol L_2$).
This is exactly the test performed in \cref{line:decor_revised_apriori_prune}, which obviates the need for any additional commutativity check on $\boldsymbol C'$ itself.
Throughout the bottom-up search, \ac{decorapriori} keeps track of the most recent non-empty layer $\boldsymbol L_k$ in $\boldsymbol L$ (\cref{line:decor_revised_apriori_update}).
\Ac{decorapriori} terminates as soon as no commutative subset of the current size $k$ is found, and returns the largest commutative subsets found thus far (stored in $\boldsymbol L$).

The correctness of \ac{decorapriori} follows directly from \cref{prop:decor_revised_apriori_downward_closure,th:decor_revised_apriori_transitivity} and we give a formal correctness proof in \cref{sec:decor_revised_cc}, where we also show how \ac{decorapriori} can be implemented efficiently.
It is noteworthy that \ac{decorapriori} is guaranteed to only perform $O(n^2)$ commutativity checks in the worst case (compared to $O(2^n)$ checks for \ac{decor} and \ac{decorplus}).

\section{Experiments} \label{sec:decor_revised_eval}
In addition to our theoretical results, we empirically evaluate \ac{decorplus} and \ac{decorapriori} to assess their performance in practice.
We compare the run times of the original \ac{decor} algorithm (which is not guaranteed to return a correct solution), \ac{decorplus}, \ac{decorapriori}, and the brute-force algorithm (which checks every possible subset of arguments in order of decreasing size for commutativity).
To keep the comparison to \ac{decor} fair, we use the same input instances as in the original \ac{decor} paper~\citep{Luttermann2024f} and also add new instances with varying distributions of subsets of commutative arguments.
The input factors used in this evaluation contain $n \in \{2, \allowbreak 4, \allowbreak 6, \allowbreak 8, \allowbreak 10, \allowbreak 12, \allowbreak 14, \allowbreak 16\}$ Boolean arguments and we vary the number $k$ of commutative arguments between $0$ and $n$.
We run each algorithm with a timeout of five minutes per instance and report the average run time over all instances for each choice of $n$.
In \cref{appendix:decor_revised_further_results}, we also report results for individual scenarios separately.

\begin{figure}
	\centering
	\subfigure[\label{fig:decor_revised_plot_avg_times}]{%
		\input{files/decor_revised_plot_avg_times.tex}%
	}\hfill
	\subfigure[\label{fig:decor_revised_plot_avg_phases_decorplus}]{%
		\input{files/decor_revised_plot_avg_phases_decorplus.tex}%
	}
	\caption{($a$) Average run times of \ac{decor}, \ac{decorplus}, \ac{decorapriori}, and the brute-force algorithm for input factors with $n$ arguments. ($b$) Proportion of the run time of \ac{decorplus} spent in the candidate generation loop and in the verification step.}
	\label{fig:decor_revised_plot_avg}
\end{figure}
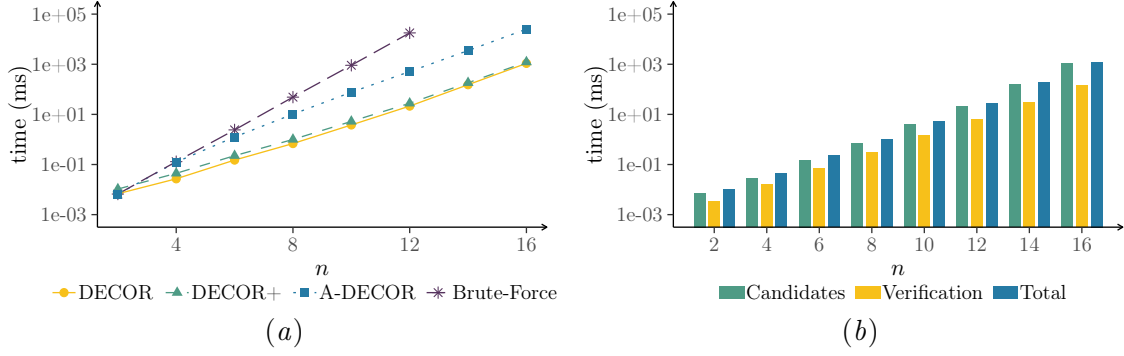

\Cref{fig:decor_revised_plot_avg_times} displays for each choice of $n$ the average run times of the algorithms over all choices of numbers of commutative arguments $k$, where $k$ is varied between $0$ and $n$.
While the brute-force algorithm only scales to small instances and runs into the timeout as $n$ grows, \ac{decorplus} solves all instances comfortably within the timeout and matches the run time of the (unsound) original \ac{decor} algorithm---the correctness guarantee provided by \ac{decorplus} thus comes essentially for free.
It also becomes evident that \ac{decorapriori} is not able to match the performance of \ac{decorplus}, even though it comes with a tighter worst case bound, which can be explained by the fact that \ac{decorapriori} always performs $\Omega(n^2)$ commutativity checks, each iterating over all entries in the factor's potential table (which has exponential size in the number of arguments) while \ac{decorplus} requires mostly just a constant number of iterations over the potential table in practice.
\Cref{fig:decor_revised_plot_avg_phases_decorplus} decomposes the total run time of \ac{decorplus} into the candidate generation loop (\cref{line:decor_revised_loop_buckets}) and the verification step (\cref{line:decor_revised_return_c}) at the end.
We note that in our experiments, there is at most one candidate subset to be verified by \ac{decorplus} for each instance, which highlights the effectiveness of the pruning strategy in the candidate generation loop of \ac{decorplus} and also explains the low run times for the verification step (note that the y-axis is log-scaled and hence, the \enquote{Candidates} and \enquote{Verification} bars do not visually add up to the \enquote{Total} bar, whereas their numbers do).

\section{Conclusion} \label{sec:decor_revised_conclusion}
We have revisited the theoretical foundations of detecting commutative factors in an \ac{fg} and identified a flaw in the central theorem of the original \ac{decor} algorithm.
Based on a corrected version of the theorem, we have presented \ac{decorplus}, a revised algorithm that augments the search space pruning with an explicit verification step and is guaranteed to return a correct solution.
We also introduced a complementary bottom-up algorithm \ac{decorapriori} with tighter worst-case bounds than \ac{decorplus}.
Our experimental evaluation confirms the practical efficiency of \ac{decorplus} and the effectiveness of its pruning strategy.

\section*{Declaration on Generative AI}
During the preparation of this work, the authors used Claude Code (Anthropic) to discuss ideas and to polish text passages.
After using this tool, the authors reviewed and edited the content as needed and take full responsibility for the publication's content.

\bibliography{references.bib}

\clearpage
\appendix
\crefalias{section}{appendix}

\section{Counting Random Variables} \label{appendix:decor_revised_crv}
In this section, we briefly discuss the concept of a \ac{crv}~\citep{Milch2008a}, which allows us to compactly encode a factor where it does not matter which specific individual \acp{rv} have a certain range value but instead only the number of \acp{rv} having that range value is of interest.
The range of a \ac{crv} is the space of histograms, that is, each range value is a histogram indicating how many of the \acp{rv} represented the \ac{crv} have a certain range value.
Since a formal definition of a \ac{crv} requires concepts not directly related to the main topic of this paper, we refrain from giving a formal definition and instead illustrate the concept of a \ac{crv} by an example.

\begin{example} \label{ex:decor_revised_crv}
	Consider the factor $\phi_3(ComA, \allowbreak ComB, \allowbreak Rev)$ from \cref{fig:decor_revised_example_fg_table}.
	Using $h$ and $l$ as shorthand notations for $\high$ and $\low$, respectively, $\phi_3$ encodes the following potential table:
	\begin{alignat*}{4}
		\phi_3(h, h, h) &= \varphi_1, \quad
		\phi_3(h, h, l) &&= \varphi_2, \quad
		\phi_3(h, l, h) &&= \varphi_3, \quad
		\phi_3(h, l, l) &&= \varphi_4, \\
		\phi_3(l, h, h) &= \varphi_3, \quad
		\phi_3(l, h, l) &&= \varphi_4, \quad
		\phi_3(l, l, h) &&= \varphi_5, \quad
		\phi_3(l, l, l) &&= \varphi_6.
	\end{alignat*}
	It becomes evident that for a specific value of $Rev$, it does not matter which specific employees have a certain competence but only how many employees have a certain competence: For $Rev = \high$, two $\high$ values for the competences are mapped to $\varphi_1$, one $\high$ and one $\low$ value are mapped to $\varphi_3$, and two $\low$ values are mapped to $\varphi_5$.
	Analogously, for $Rev = \low$, two $\high$ values for the competences are mapped to $\varphi_2$, one $\high$ and one $\low$ value are mapped to $\varphi_4$, and two $\low$ values are mapped to $\varphi_6$.
	Now, consider a \ac{crv} that counts over the competences of $Alice$ and $Bob$, denoted as $\#_{E}[Com(E)]$ with $\range{Com(E)} = \{ \low, \high \}$ and $\domain{E} = \{ Alice, Bob \}$ being the domain of the \acl{lv} $E$ representing the employees.
	Then, there are $m = \abs{\range{Com(E)}} = 2$ possible range values and $n = 2$ represented instances ($Com(Alice)$, or $ComA$ for short, and $Com(Bob)$, or $ComB$ for short).
	Hence, the histograms are $[0, 2]$, $[1, 1]$, and $[2, 0]$ (corresponding to $\{ (\high, 0), \allowbreak (\low, 2) \}$, $\{ (\high, 1), \allowbreak (\low, 1) \}$, and $\{ (\high, 2), \allowbreak (\low, 0) \}$ in set notation, respectively).
	By replacing $ComA$ and $ComB$ with $\#_{E}[Com(E)]$, the factor $\phi_3$ can be rewritten as $\phi_3(\#_{E}[Com(E)], \allowbreak Rev)$, where
	\begin{alignat*}{4}
		\phi_3([2, 0], h) &= \varphi_1, \quad
		\phi_3([2, 0], l) &&= \varphi_2, \quad
		\phi_3([1, 1], h) &&= \varphi_3, \\
		\phi_3([1, 1], l) &= \varphi_4, \quad
		\phi_3([0, 2], h) &&= \varphi_5, \quad
		\phi_3([0, 2], l) &&= \varphi_6.
	\end{alignat*}
	A tabular illustration of $\phi_3(\#_{E}[Com(E)], \allowbreak Rev)$ is given in \cref{tab:decor_revised_example_crv_encoding}.
	Both $\phi_3(ComA, \allowbreak ComB, \allowbreak Rev)$ and $\phi_3(\#_{E}[Com(E)], \allowbreak Rev)$ encode the same information.
	It is important to note that a representation using a \ac{crv} such as $\phi_3(\#_{E}[Com(E)], \allowbreak Rev)$ can efficiently be handled by common lifted inference algorithms~\citep{Milch2008a,Taghipour2013c}.
\end{example}

\begin{table}
	\centering
	\input{files/decor_revised_example_crv_encoding.tex}
	\caption{A compressed representation of the factor $\phi_3(ComA, \allowbreak ComB, \allowbreak Rev)$ from \cref{ex:decor_revised_crv}. The \acp{rv} $ComA$ and $ComB$ have now been replaced by a \ac{crv} $\#_{E}[Com(E)]$, which uses histograms to represent $ComA$ and $ComB$ simultaneously. Note that, when replacing \acp{rv} by a \ac{crv}, the underlying semantics of $\phi_3$ (and hence of the full joint probability distribution encoded by the \ac{fg} containing $\phi_3$) remains unchanged while at the same time the size of $\phi_3$'s potential table no longer exponentially depends on the number of \acp{rv} in the argument list of $\phi_3$.}
	\label{tab:decor_revised_example_crv_encoding}
\end{table}

\section{Detailed Proofs} \label{appendix:decor_revised_proofs}
In this section, we provide extended versions of some of the proofs given in the main paper.

\decorRevisedFlawByBucketsTheorem*
\begin{proof}
	We prove the claim by counterexample.
	Consider a factor $\phi(R_1, \allowbreak R_2, \allowbreak R_3, \allowbreak R_4)$ with $\range{R_1} = \range{R_2} = \range{R_3} = \range{R_4} = \{ \high, \allowbreak \low \}$.
	Further, let $\phi(\low, \allowbreak \low, \allowbreak \high, \allowbreak \high) = \varphi_1$, $\phi(\high, \allowbreak \high, \allowbreak \low, \allowbreak \low) = \varphi_1$, and $\phi(\boldsymbol r) = \varphi_2$ with $\varphi_1 \neq \varphi_2$ for all remaining assignments $\boldsymbol r$.
	A tabular illustration of the factor $\phi$ together with its buckets is given in \cref{tab:decor_revised_counterexample}.
	Clearly, $\phi$ is not commutative with respect to $\{R_1, \allowbreak R_2, \allowbreak R_3, \allowbreak R_4\}$ because, e.g., $\phi(\low, \allowbreak \high, \allowbreak \high, \allowbreak \low) = \varphi_2 \neq \varphi_1 = \phi(\high, \allowbreak \high, \allowbreak \low, \allowbreak \low)$.
	However, applying \cref{th:decor_revised_args_by_buckets} to $\phi$ yields the subset $\boldsymbol C = \{R_1, \allowbreak R_2, \allowbreak R_3, \allowbreak R_4\}$ as a result:
	The buckets $[4,0]$ and $[0,4]$ each contain a single potential and are thus skipped.
	In the bucket $[3,1]$, all four assignments are mapped to $\varphi_2$, forming a single group $\boldsymbol \Psi = \{\varphi_2, \allowbreak \varphi_2, \allowbreak \varphi_2, \allowbreak \varphi_2\}$.
	The element-wise intersection of the assignments corresponding to the potentials in $\boldsymbol \Psi$ is then given by
	\begin{align*}
		&(\{\high\}, \allowbreak \{\high\}, \allowbreak \{\high\}, \allowbreak \{\low\}) \cap (\{\high\}, \allowbreak \{\high\}, \allowbreak \{\low\}, \allowbreak \{\high\}) \\
		\cap~ &(\{\high\}, \allowbreak \{\low\}, \allowbreak \{\high\}, \allowbreak \{\high\}) \cap (\{\low\}, \allowbreak \{\high\}, \allowbreak \{\high\}, \allowbreak \{\high\}) \\
		=~ &(\emptyset, \allowbreak \emptyset, \allowbreak \emptyset, \allowbreak \emptyset)
	\end{align*}
	and hence empty at every position, yielding $\boldsymbol C_{[3,1]} = \{R_1, \allowbreak R_2, \allowbreak R_3, \allowbreak R_4\}$.
	Analogously, the bucket $[1,3]$ yields $\boldsymbol C_{[1,3]} = \{R_1, \allowbreak R_2, \allowbreak R_3, \allowbreak R_4\}$.
	In the bucket $[2,2]$, the group $\boldsymbol \Psi = \{\varphi_1, \allowbreak \varphi_1\}$ (with corresponding assignments $(\low, \allowbreak \low, \allowbreak \high, \allowbreak \high)$ and $(\high, \allowbreak \high, \allowbreak \low, \allowbreak \low)$) yields the element-wise intersection $(\emptyset, \allowbreak \emptyset, \allowbreak \emptyset, \allowbreak \emptyset)$ and hence $\boldsymbol C_{[2,2]} = \{R_1, \allowbreak R_2, \allowbreak R_3, \allowbreak R_4\}$ (the same holds for the group  $\boldsymbol \Psi = \{\varphi_2, \allowbreak \varphi_2, \allowbreak \varphi_2, \allowbreak \varphi_2\}$ in bucket $[2,2]$).
	Thus, $\boldsymbol C = \boldsymbol C_{[3,1]} \cap \boldsymbol C_{[2,2]} \cap \boldsymbol C_{[1,3]} = \{R_1, \allowbreak R_2, \allowbreak R_3, \allowbreak R_4\}$, yet $\phi$ is not commutative with respect to $\boldsymbol C = \{R_1, \allowbreak R_2, \allowbreak R_3, \allowbreak R_4\}$.
\end{proof}

\begin{table}[t]
	\centering
	\input{files/decor_revised_counterexample.tex}
	\caption{The potential table of the factor $\phi(R_1, \allowbreak R_2, \allowbreak R_3, \allowbreak R_4)$ used in the proof of \cref{th:decor_revised_flaw_in_args_by_buckets} (left) and the multiset of potentials $\phi^{\succ}(b)$ associated with each bucket $b \in \mathcal B(\phi)$ (right). Only the assignments $(\low, \low, \high, \high)$ and $(\high, \high, \low, \low)$ are mapped to $\varphi_1$, while all remaining assignments are mapped to $\varphi_2 \neq \varphi_1$.}
	\label{tab:decor_revised_counterexample}
\end{table}

\decorplusCorrectnessTheorem*
\begin{proof}
	We split the proof into two parts.
	First, we show that the bucket loop in \cref{alg:decor_revised} computes a set $\boldsymbol C_{cand}$ of candidate subsets that is guaranteed to contain every maximum sized subset of commutative arguments of $\phi$.
	Second, we show that the subsequent verification step identifies a maximum sized subset of commutative arguments among the candidates.

	\emph{Correctness of the bucket loop.}
	Let $\boldsymbol C^{*} \subseteq \{R_1, \allowbreak \ldots, \allowbreak R_n\}$ with $\abs{\boldsymbol C^{*}} \geq 2$ denote any maximum sized subset of commutative arguments of $\phi$.
	By \cref{th:decor_revised_args_by_buckets_fixed}, for every bucket $b \in \mathcal B(\phi)$ with $\abs{\phi^{\succ}(b)} > 1$, there exists a maximal set of identical potentials $\boldsymbol \Psi_i \subseteq \phi^{\succ}(b)$ with $\abs{\boldsymbol \Psi_i} > 1$ such that the element-wise intersection of the assignments corresponding to the potentials in $\boldsymbol \Psi_i$ contains an empty set at every position $j \in \{1, \allowbreak \ldots, \allowbreak n\}$ for which $R_j \in \boldsymbol C^{*}$.
	Consequently, the set $\boldsymbol C_i$ computed for $\boldsymbol \Psi_i$ in the inner loop of \cref{alg:decor_revised} (\cref{line:decor_revised_set_ci}) satisfies $\boldsymbol C^{*} \subseteq \boldsymbol C_i$, and hence either $\boldsymbol C_i$ itself is added to $\boldsymbol C'$ or it is subsumed by another candidate subset already in $\boldsymbol C'$ that contains $\boldsymbol C^{*}$.
	The cross-intersection step preserves $\boldsymbol C^{*}$ as a subset of some entry in $\boldsymbol C_{cand}$ at every iteration, because the intersection of two supersets of $\boldsymbol C^{*}$ is itself a superset of $\boldsymbol C^{*}$ and the subsumption check only removes entries that are themselves subsumed by a superset.
	By induction over the buckets, every maximum sized subset $\boldsymbol C^{*}$ of commutative arguments of $\phi$ is contained in some entry of $\boldsymbol C_{cand}$ after the bucket loop terminates.

	\emph{Correctness of the verification step.}
	After the bucket loop, the verification step iterates over the candidate subsets in $\boldsymbol C_{cand}$ in order of decreasing size and, for each candidate $\boldsymbol C \in \boldsymbol C_{cand}$, checks whether $\phi$ is commutative with respect to $\boldsymbol C$ and, if not, recursively checks all $(\abs{\boldsymbol C} - 1)$-element subsets of $\boldsymbol C$, then all $(\abs{\boldsymbol C} - 2)$-element subsets of $\boldsymbol C$, and so on, until either a commutative subset is found or no subset of size at least two remains.
	By the first part of the proof, every maximum sized subset of commutative arguments of $\phi$ is contained in some entry of $\boldsymbol C_{cand}$, and the verification step thus eventually checks every such maximum sized subset.
	The check itself is exact (it directly applies \cref{def:decor_revised_commutative}), so the first commutative subset returned by the verification step is necessarily a subset of commutative arguments of $\phi$.
	Moreover, as the verification step considers candidates in order of decreasing size and stops as soon as a commutative subset is found, the returned subset is maximum sized---any larger commutative subset would have been considered (and accepted) earlier.
	If no candidate passes the commutativity check at any size, then by the first part of the proof (each maximum sized subset of commutative arguments is contained in $\boldsymbol C_{cand}$), $\phi$ admits no commutative subset of size at least two, and \cref{alg:decor_revised} correctly returns the empty set.
\end{proof}

\section{Example Run of the \acs{decorplus} Algorithm} \label{appendix:decor_revised_example_run}
The following example showcases an example run of the \ac{decorplus} algorithm (\cref{alg:decor_revised}) on the factor $\phi_3(ComA, \allowbreak ComB, \allowbreak Rev)$ from \cref{tab:decor_revised_example_buckets}.
\begin{example}
	Let us take a look at how \ac{decorplus} computes a maximum sized subset of commutative arguments for the factor $\phi_3(ComA, \allowbreak ComB, \allowbreak Rev)$ from \cref{tab:decor_revised_example_buckets}.
	Initially, \ac{decorplus} starts with the set $\boldsymbol C_{cand} = \{\{ComA, ComB, Rev\}\}$ containing the single candidate subset of all arguments.
	The buckets $[3,0]$ and $[0,3]$ each contain only a single potential and are thus skipped by \ac{decorplus}.
	For the bucket $[2,1]$ with $\phi_3^{\succ}([2,1]) = \langle \varphi_2, \allowbreak \varphi_3, \allowbreak \varphi_3 \rangle$, \ac{decorplus} finds the maximal group of identical potentials $\boldsymbol \Psi_1 = \{\varphi_3, \allowbreak \varphi_3\}$ corresponding to the assignments $(\high, \allowbreak \low, \allowbreak \high)$ and $(\low, \allowbreak \high, \allowbreak \high)$.
	The remaining group $\{\varphi_2\}$ contains less than two elements and is thus ignored.
	\Ac{decorplus} then computes the element-wise intersection of the assignments corresponding to the potentials in $\boldsymbol \Psi_1$, which is given by $(\{\high\}, \allowbreak \{\low\}, \allowbreak \{\high\}) \cap (\{\low\}, \allowbreak \{\high\}, \allowbreak \{\high\}) = (\emptyset, \allowbreak \emptyset, \allowbreak \{\high\})$.
	As the element-wise intersection is empty at positions one and two, \ac{decorplus} adds the set $\boldsymbol C_1 = \{ComA, \allowbreak ComB\}$ to the set $\boldsymbol C'$ of candidate subsets for the bucket $[2,1]$.
	Afterwards, \ac{decorplus} computes the intersection of $\{ComA, \allowbreak ComB, \allowbreak Rev\}$ (as $\boldsymbol C_{cand} = \{\{ComA, \allowbreak ComB, \allowbreak Rev\}\}$) and $\{ComA, \allowbreak ComB\}$ (as $\boldsymbol C' = \{\{ComA, \allowbreak ComB\}\}$), which yields $\boldsymbol C_{\cap} = \{\{ComA, \allowbreak ComB\}\}$ and thus also $\boldsymbol C_{cand} = \{\{ComA, \allowbreak ComB\}\}$ for the next iteration.
	For the bucket $[1,2]$ with $\phi_3^{\succ}([1,2]) = \langle \varphi_4, \allowbreak \varphi_4, \allowbreak \varphi_5 \rangle$, \ac{decorplus} finds the maximal group of identical potentials $\boldsymbol \Psi_1 = \{\varphi_4, \allowbreak \varphi_4\}$ corresponding to the assignments $(\high, \allowbreak \low, \allowbreak \low)$ and $(\low, \allowbreak \high, \allowbreak \low)$, and computes the element-wise intersection $(\{\high\}, \allowbreak \{\low\}, \allowbreak \{\low\}) \cap (\{\low\}, \allowbreak \{\high\}, \allowbreak \{\low\}) = (\emptyset, \allowbreak \emptyset, \allowbreak \{\low\})$.
	Again, positions one and two are empty and thus \ac{decorplus} adds the set $\boldsymbol C_1 = \{ComA, \allowbreak ComB\}$ to the set $\boldsymbol C'$ of candidate subsets for the bucket $[1,2]$.
	Thereafter, \ac{decorplus} computes the intersection of $\{ComA, \allowbreak ComB\}$ (as $\boldsymbol C_{cand} = \{\{ComA, \allowbreak ComB\}\}$) and $\{ComA, \allowbreak ComB\}$ (as $\boldsymbol C' = \{\{ComA, \allowbreak ComB\}\}$), which yields $\boldsymbol C_{cand} = \boldsymbol C_{\cap} = \{\{ComA, \allowbreak ComB\}\}$.
	As there are no further buckets left, \ac{decorplus} proceeds to the verification step and checks whether $\phi_3$ is commutative with respect to $\{ComA, \allowbreak ComB\}$.
	The check succeeds and \ac{decorplus} returns the set $\{\{ComA, \allowbreak ComB\}\}$ as a set containing all maximum sized subsets of commutative arguments of $\phi_3$.
\end{example}

\section{Correctness and Efficiency Improvements of \acs{decorapriori}} \label{sec:decor_revised_cc}
In this section, we first formally state and prove the correctness of the \ac{decorapriori} algorithm (\cref{alg:decor_revised_apriori}), and afterwards also show how \ac{decorapriori} can be implemented efficiently by computing connected components in a commutativity graph.

\subsection{Correctness of the \acs{decorapriori} Algorithm}
We now show that for a given factor $\phi$, \ac{decorapriori} (\cref{alg:decor_revised_apriori}) returns exactly the set of all commutative subsets of $\phi$'s arguments of maximum size.

\begin{theorem}
	Let $\phi(R_1, \allowbreak \ldots, \allowbreak R_n)$ denote a factor that is given as an input to \ac{decorapriori} (\cref{alg:decor_revised_apriori}).
	Then, the set returned by \ac{decorapriori} contains exactly the maximum sized subsets of commutative arguments of $\phi$, or the empty set if no such subset exists.
\end{theorem}
\begin{proof}
	We first show by induction on $k$ that $\boldsymbol L_k$ contains exactly the commutative subsets of $\{R_1, \allowbreak \ldots, \allowbreak R_n\}$ of size $k$.

	\emph{Base case} ($k = 2$): $\boldsymbol L_2$ is constructed in \cref{line:decor_revised_apriori_init} by directly checking commutativity for every pair of arguments and contains exactly the commutative pairs.

	\emph{Inductive step} ($k \to k+1$): Assume that $\boldsymbol L_k$ contains exactly the commutative subsets of size $k$.
	Let $\boldsymbol C'$ denote any commutative subset of size $k+1$.
	We first show that any truly commutative subset $\boldsymbol C'$ of size $k+1$ is added to $\boldsymbol L_{k+1}$ in \cref{line:decor_revised_apriori_add}.
	By \cref{prop:decor_revised_apriori_downward_closure}, every subset of $\boldsymbol C'$ of size at least two is also commutative; in particular, $\boldsymbol C' \setminus \{R_i\} \in \boldsymbol L_k$ for any $R_i \in \boldsymbol C'$ (by the induction hypothesis), and every pair $\{R_i, R_j\} \subseteq \boldsymbol C'$ with $i \neq j$ is contained in $\boldsymbol L_2$.
	Picking any $R_i \in \boldsymbol C'$ and setting $\boldsymbol C = \boldsymbol C' \setminus \{R_i\}$, the iteration of \cref{line:decor_revised_apriori_loop_lk,line:decor_revised_apriori_loop_args} considers the candidate $\boldsymbol C' = \boldsymbol C \cup \{R_i\}$ (which is indeed commutative), and the pruning condition in \cref{line:decor_revised_apriori_prune} is satisfied because every pair $\{R_i, R_j\}$ with $R_j \in \boldsymbol C$ is in $\boldsymbol L_2$.
	Hence $\boldsymbol C'$ is added to $\boldsymbol L_{k+1}$.
	Conversely, suppose a subset $\boldsymbol C' = \boldsymbol C \cup \{R_i\}$ is added to $\boldsymbol L_{k+1}$ in \cref{line:decor_revised_apriori_add}.
	We next show that $\boldsymbol C'$ is indeed commutative.
	In particular, every pair $\{R_i, R_j\}$ with $R_j \in \boldsymbol C$ is in $\boldsymbol L_2$ by the pruning condition in \cref{line:decor_revised_apriori_prune}, and every pair $\{R_j, R_\ell\} \subseteq \boldsymbol C$ with $j \neq \ell$ is in $\boldsymbol L_2$ by \cref{prop:decor_revised_apriori_downward_closure} (since $\boldsymbol C \in \boldsymbol L_k$ is commutative by the induction hypothesis).
	Thus, every pair of arguments in $\boldsymbol C'$ is in $\boldsymbol L_2$, and \cref{th:decor_revised_apriori_transitivity} implies that $\boldsymbol C'$ is commutative.
	Consequently, $\boldsymbol L_{k+1}$ contains exactly the commutative subsets of size $k+1$.

	Hence, for every $k \geq 2$, the layer $\boldsymbol L_k$ computed inside the loop coincides with the set of all commutative subsets of $\{R_1, \allowbreak \ldots, \allowbreak R_n\}$ of size $k$.
	By \cref{line:decor_revised_apriori_update_if,line:decor_revised_apriori_update}, the variable $\boldsymbol L$ is updated to $\boldsymbol L_{k+1}$ whenever $\boldsymbol L_{k+1}$ is non-empty.
	Consequently, when the loop in \cref{line:decor_revised_apriori_while} terminates with $\boldsymbol L_k = \emptyset$, the variable $\boldsymbol L$ holds the largest non-empty layer $\boldsymbol L_{k-1}$, which contains exactly the commutative subsets of maximum size.
	If no pair of arguments is commutative, then $\boldsymbol L_2 = \emptyset$, the loop in \cref{line:decor_revised_apriori_while} does not execute, $\boldsymbol L$ remains the empty set as initialised in \cref{line:decor_revised_apriori_init_l}, and \cref{alg:decor_revised_apriori} returns the empty set.
\end{proof}

\subsection{A Connected-Component Algorithm for Detecting Commutative Factors}
\Ac{decorapriori} terminates only after iterating up to $n - 1$ levels in the worst case, even though every level $k \geq 3$ relies exclusively on pair information already contained in $\boldsymbol L_2$ (cf.\ \cref{th:decor_revised_apriori_transitivity}).
A stronger consequence of transitivity allows us to skip the level-by-level enumeration entirely: Any two commutative subsets that share at least one argument can be merged into a single (larger) commutative subset.

\begin{theorem} \label{th:decor_revised_cc_overlap}
	Let $\phi(R_1, \allowbreak \ldots, \allowbreak R_n)$ denote a factor and let $\boldsymbol C_1, \boldsymbol C_2 \subseteq \{R_1, \allowbreak \ldots, \allowbreak R_n\}$ with $\abs{\boldsymbol C_1}, \abs{\boldsymbol C_2} \geq 2$ denote two subsets of arguments such that $\phi$ is commutative with respect to both $\boldsymbol C_1$ and $\boldsymbol C_2$, and $\boldsymbol C_1 \cap \boldsymbol C_2 \neq \emptyset$.
	Then, $\phi$ is commutative with respect to $\boldsymbol C_1 \cup \boldsymbol C_2$.
\end{theorem}
\begin{proof}
	Let $R_k \in \boldsymbol C_1 \cap \boldsymbol C_2$ denote a shared argument and let $\{R_i, R_j\} \subseteq \boldsymbol C_1 \cup \boldsymbol C_2$ denote any pair of arguments with $i \neq j$.
	If $\{R_i, R_j\} \subseteq \boldsymbol C_1$ or $\{R_i, R_j\} \subseteq \boldsymbol C_2$, then $\phi$ is commutative with respect to $\{R_i, R_j\}$ by \cref{prop:decor_revised_apriori_downward_closure}.
	Otherwise, $R_i$ and $R_j$ lie in different subsets, and we may assume without loss of generality that $R_i \in \boldsymbol C_1 \setminus \boldsymbol C_2$ and $R_j \in \boldsymbol C_2 \setminus \boldsymbol C_1$.
	Then, $\{R_i, R_k\} \subseteq \boldsymbol C_1$ and $\{R_k, R_j\} \subseteq \boldsymbol C_2$, so $\phi$ is commutative with respect to both $\{R_i, R_k\}$ and $\{R_k, R_j\}$ by \cref{prop:decor_revised_apriori_downward_closure}.
	Hence, $\phi$ is invariant under both transpositions $\tau_{ik}$ and $\tau_{kj}$ (i.e., the transpositions swapping indices $i$ and $k$ and indices $k$ and $j$, respectively) and consequently, $\phi$ is invariant under $\tau_{ij} = \tau_{ik} \circ \tau_{kj} \circ \tau_{ik}$ (where the composition sends $i \mapsto k \mapsto j \mapsto j$, $j \mapsto j \mapsto k \mapsto i$, and $k \mapsto i \mapsto i \mapsto k$, leaving $i$ and $j$ swapped while $k$ returns to its original position) because $\phi$ is invariant under each individual transposition on the right-hand side.
	Thus, $\phi$ is commutative with respect to $\{R_i, R_j\}$ and as $\{R_i, R_j\}$ is an arbitrary pair of arguments in $\boldsymbol C_1 \cup \boldsymbol C_2$, $\phi$ is commutative with respect to every pair of arguments in $\boldsymbol C_1 \cup \boldsymbol C_2$.
	\Cref{th:decor_revised_apriori_transitivity} then implies that $\phi$ is commutative with respect to $\boldsymbol C_1 \cup \boldsymbol C_2$.
\end{proof}

\Cref{th:decor_revised_cc_overlap} suggests an alternative algorithm, which we call \ac{decorcc}, that operates in only two phases: It first computes the set $\boldsymbol L_2$ of all commutative pairs of arguments (analogously to \ac{decorapriori}) and then iteratively merges overlapping subsets of arguments until a fixed point is reached.

\subsection{The \acs{decorcc} Algorithm}
\begin{algorithm}[t]
	\caption{Connected-Component Detection of Commutative Factors (CC-DECOR)}
	\label{alg:decor_revised_cc}
	\alginput{A factor $\phi(R_1, \ldots, R_n)$.} \\
	\algoutput{A set containing all maximum sized subsets of commutative arguments of $\phi$.}
	\begin{algorithmic}[1]
		\State $\boldsymbol L_2 \gets \{ \{R_i, R_j\} \mid 1 \leq i < j \leq n \land \phi \text{ is commutative with respect to } \{R_i, R_j\} \}$ \label{line:decor_revised_cc_init}
		\State $\boldsymbol M \gets \boldsymbol L_2$ \label{line:decor_revised_cc_init_m}
		\While{$\exists \boldsymbol C_1, \boldsymbol C_2 \in \boldsymbol M \colon \boldsymbol C_1 \neq \boldsymbol C_2 \land \boldsymbol C_1 \cap \boldsymbol C_2 \neq \emptyset$} \label{line:decor_revised_cc_while}
			\State $\boldsymbol M \gets (\boldsymbol M \setminus \{\boldsymbol C_1, \boldsymbol C_2\}) \cup \{\boldsymbol C_1 \cup \boldsymbol C_2\}$ \label{line:decor_revised_cc_merge}
		\EndWhile
		\State \Return $\{ \boldsymbol C \in \boldsymbol M \mid \forall \boldsymbol C' \in \boldsymbol M \colon \abs{\boldsymbol C} \geq \abs{\boldsymbol C'} \}$ \label{line:decor_revised_cc_return}
	\end{algorithmic}
\end{algorithm}

The \ac{decorcc} algorithm is presented in \cref{alg:decor_revised_cc} and starts by computing the set $\boldsymbol L_2$ of all commutative pairs of arguments in \cref{line:decor_revised_cc_init} (analogously to \ac{decorapriori}).
The set $\boldsymbol M$ is then initialised with the commutative pairs in \cref{line:decor_revised_cc_init_m} and iteratively updated by replacing any two distinct overlapping subsets $\boldsymbol C_1, \boldsymbol C_2 \in \boldsymbol M$ by their union $\boldsymbol C_1 \cup \boldsymbol C_2$ (\cref{line:decor_revised_cc_while,line:decor_revised_cc_merge}).
By \cref{th:decor_revised_cc_overlap}, every set added to $\boldsymbol M$ is a commutative subset of arguments, and the merging procedure terminates as soon as the subsets in $\boldsymbol M$ are pairwise disjoint.
Finally, \ac{decorcc} returns the subsets of maximum cardinality among the disjoint subsets in $\boldsymbol M$ in \cref{line:decor_revised_cc_return}.

The merging procedure in \cref{line:decor_revised_cc_while,line:decor_revised_cc_merge} can equivalently be viewed as computing the connected components of the \emph{commutativity graph} of $\phi$, that is, the graph whose vertices are the arguments of $\phi$ and whose edges are the pairs in $\boldsymbol L_2$.
A standard union-find data structure on $\{R_1, \allowbreak \ldots, \allowbreak R_n\}$ that calls $union(R_i, R_j)$ for every pair $\{R_i, R_j\} \in \boldsymbol L_2$ thus yields an implementation in which the merging phase runs in $O(n^2 \cdot \alpha(n))$ time, where $\alpha$ denotes the inverse Ackermann function~\citep{Tarjan1979a}.
The dominant cost of \ac{decorcc} is therefore the computation of $\boldsymbol L_2$ in \cref{line:decor_revised_cc_init}, which requires $\binom{n}{2}$ commutativity checks on pairs of arguments---the same cost as the initialisation step of \ac{decorapriori}.

\subsection{Correctness of the \acs{decorcc} Algorithm}
We now show that for a given factor $\phi$, \ac{decorcc} (\cref{alg:decor_revised_cc}) returns exactly the set of all commutative subsets of $\phi$'s arguments of maximum size.

\begin{theorem}
	Let $\phi(R_1, \allowbreak \ldots, \allowbreak R_n)$ denote a factor that is given as an input to \ac{decorcc} (\cref{alg:decor_revised_cc}).
	Then, the set returned by \ac{decorcc} contains exactly the maximum sized subsets of commutative arguments of $\phi$, or the empty set if no such subset exists.
\end{theorem}
\begin{proof}
	We first show that the following invariant holds throughout the execution of the loop in \cref{line:decor_revised_cc_while}: Every subset in $\boldsymbol M$ is a commutative subset of arguments of $\phi$ of size at least two.
	Initially, $\boldsymbol M = \boldsymbol L_2$ contains exactly the commutative pairs of arguments by construction, so the invariant holds.
	At each iteration, two distinct overlapping subsets $\boldsymbol C_1, \boldsymbol C_2 \in \boldsymbol M$ are replaced by their union $\boldsymbol C_1 \cup \boldsymbol C_2$ in \cref{line:decor_revised_cc_merge}, which is a commutative subset of size at least two by \cref{th:decor_revised_cc_overlap}.
	Hence, the invariant is preserved.
	The loop terminates because each iteration strictly decreases $\abs{\boldsymbol M}$ (two subsets are removed and one is added) and $\boldsymbol M$ is finite, and upon termination, the subsets in $\boldsymbol M$ are pairwise disjoint.

	Let $K = \max_{\boldsymbol C \in \boldsymbol M} \abs{\boldsymbol C}$ denote the size of the largest subset in $\boldsymbol M$ at termination (with $K = 0$ if $\boldsymbol M = \emptyset$).
	By the invariant, every $\boldsymbol C \in \boldsymbol M$ is a commutative subset of size at least two whenever $\boldsymbol M \neq \emptyset$, so the maximum size of any commutative subset of $\phi$ is at least $K$.
	Conversely, let $\boldsymbol C^{*} \subseteq \{R_1, \allowbreak \ldots, \allowbreak R_n\}$ with $\abs{\boldsymbol C^{*}} \geq 2$ denote any commutative subset of $\phi$.
	By \cref{prop:decor_revised_apriori_downward_closure}, every pair $\{R_i, R_j\} \subseteq \boldsymbol C^{*}$ with $i \neq j$ is contained in $\boldsymbol L_2$, that is, in $\boldsymbol M$ at the start of the loop.
	Since any two such pairs that share an argument are eventually merged, induction over the merging steps yields that there exists $\boldsymbol C \in \boldsymbol M$ at termination with $\boldsymbol C^{*} \subseteq \boldsymbol C$, and hence $\abs{\boldsymbol C^{*}} \leq \abs{\boldsymbol C} \leq K$.
	Consequently, the maximum size of any commutative subset of $\phi$ equals $K$, and the maximum sized subsets in $\boldsymbol M$ are exactly the maximum sized commutative subsets of $\phi$.
	The filtering step in \cref{line:decor_revised_cc_return} returns exactly those subsets of maximum cardinality in $\boldsymbol M$ (or the empty set if $\boldsymbol M = \emptyset$, that is, if no commutative subset of size at least two exists), which completes the proof.
\end{proof}

\section{Additional Experimental Results} \label{appendix:decor_revised_further_results}
In addition to the experimental results from \cref{sec:decor_revised_eval}, we provide further experimental results in this section.
While the plots given in \cref{fig:decor_revised_plot_avg} show averages of multiple runs for different choices of the number of commutative arguments $k$, we now present separate plots for each of the choices of the number of commutative arguments $k$ to highlight the influence of $k$ on the performance on the algorithms.
We again compare the run times of the original \ac{decor} algorithm, \ac{decorplus}, \ac{decorapriori}, and the brute-force algorithm on factors with $n \in \{2, \allowbreak 4, \allowbreak 6, \allowbreak 8, \allowbreak 10, \allowbreak 12, \allowbreak 14, \allowbreak 16\}$ Boolean arguments and set a timeout of five minutes per instance.

\begin{figure}[t]
	\centering
	\subfigure[\label{fig:decor_revised_plot_k=0_times}]{%
		\input{files/decor_revised_plot_k=0_times.tex}%
	}\hfill
	\subfigure[\label{fig:decor_revised_plot_k=0_phases}]{%
		\input{files/decor_revised_plot_k=0_phases_decorplus.tex}%
	}
	\caption{($a$) Run times of \ac{decor}, \ac{decorplus}, \ac{decorapriori}, and the brute-force algorithm for input factors with $n$ arguments, of which $k = 0$ arguments are commutative. ($b$) Proportion of the run time of \ac{decorplus} spent in the bucket loop and the verification step for input factors with $k = 0$ commutative arguments.}
	\label{fig:decor_revised_plot_k=0}
\end{figure}
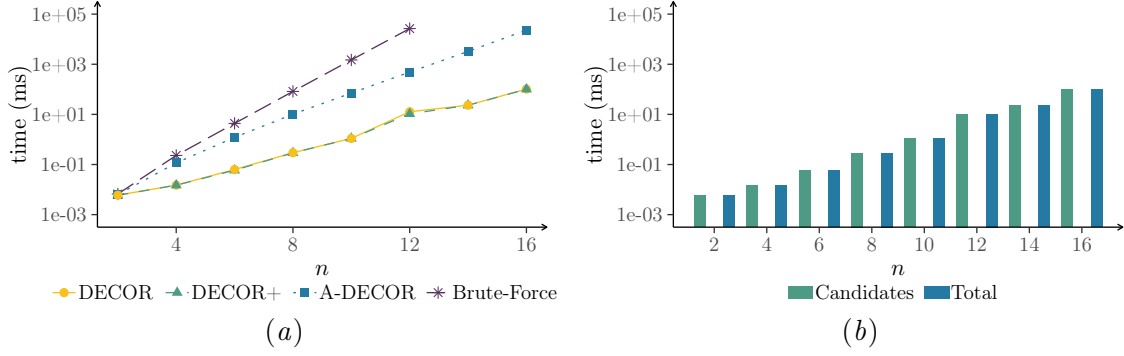
\begin{figure}[t]
	\centering
	\subfigure[\label{fig:decor_revised_plot_k=2_times}]{%
		\input{files/decor_revised_plot_k=2_times.tex}%
	}\hfill
	\subfigure[\label{fig:decor_revised_plot_k=2_phases}]{%
		\input{files/decor_revised_plot_k=2_phases_decorplus.tex}%
	}
	\caption{($a$) Run times of \ac{decor}, \ac{decorplus}, \ac{decorapriori}, and the brute-force algorithm for input factors with $n$ arguments, of which $k = 2$ arguments are commutative. ($b$) Proportion of the run time of \ac{decorplus} spent in the bucket loop and the verification step for input factors with $k = 2$ commutative arguments.}
	\label{fig:decor_revised_plot_k=2}
\end{figure}
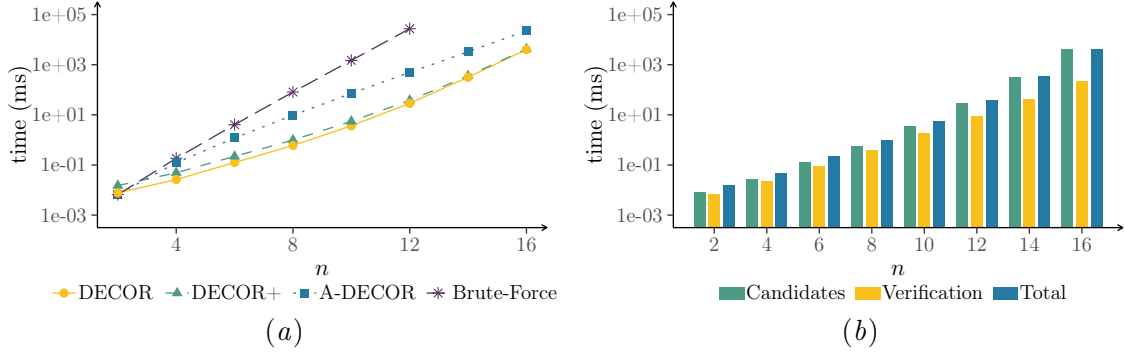
\begin{figure}[t]
	\centering
	\subfigure[\label{fig:decor_revised_plot_k=log2n_times}]{%
		\input{files/decor_revised_plot_k=log2n_times.tex}%
	}\hfill
	\subfigure[\label{fig:decor_revised_plot_k=log2n_phases}]{%
		\input{files/decor_revised_plot_k=log2n_phases_decorplus.tex}%
	}
	\caption{($a$) Run times of \ac{decor}, \ac{decorplus}, \ac{decorapriori}, and the brute-force algorithm for input factors with $n$ arguments, of which $k = \lfloor \log_2 n \rfloor$ arguments are commutative. ($b$) Proportion of the run time of \ac{decorplus} spent in the bucket loop and the verification step for input factors with $k = \lfloor \log_2 n \rfloor$ commutative arguments.}
	\label{fig:decor_revised_plot_k=log2n}
\end{figure}
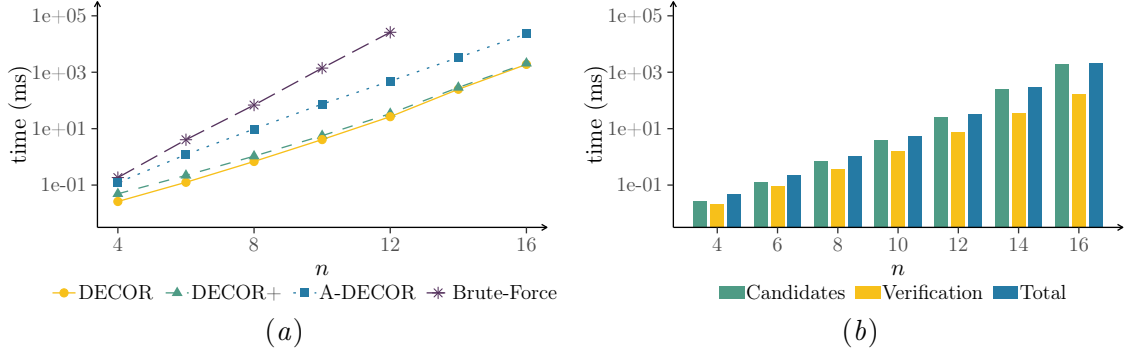
\begin{figure}[t]
	\centering
	\subfigure[\label{fig:decor_revised_plot_k=ndiv2_times}]{%
		\input{files/decor_revised_plot_k=ndiv2_times.tex}%
	}\hfill
	\subfigure[\label{fig:decor_revised_plot_k=ndiv2_phases}]{%
		\input{files/decor_revised_plot_k=ndiv2_phases_decorplus.tex}%
	}
	\caption{($a$) Run times of \ac{decor}, \ac{decorplus}, \ac{decorapriori}, and the brute-force algorithm for input factors with $n$ arguments, of which $k = \lfloor n \mathbin{/} 2 \rfloor$ arguments are commutative. ($b$) Proportion of the run time of \ac{decorplus} spent in the bucket loop and the verification step for input factors with $k = \lfloor n \mathbin{/} 2 \rfloor$ commutative arguments.}
	\label{fig:decor_revised_plot_k=ndiv2}
\end{figure}
\begin{figure}[t]
	\centering
	\subfigure[\label{fig:decor_revised_plot_k=nsub1_times}]{%
		\input{files/decor_revised_plot_k=nsub1_times.tex}%
	}\hfill
	\subfigure[\label{fig:decor_revised_plot_k=nsub1_phases}]{%
		\input{files/decor_revised_plot_k=nsub1_phases_decorplus.tex}%
	}
	\caption{($a$) Run times of \ac{decor}, \ac{decorplus}, \ac{decorapriori}, and the brute-force algorithm for input factors with $n$ arguments, of which $k = n - 1$ arguments are commutative. ($b$) Proportion of the run time of \ac{decorplus} spent in the bucket loop and the verification step for input factors with $k = n - 1$ commutative arguments.}
	\label{fig:decor_revised_plot_k=nsub1}
\end{figure}
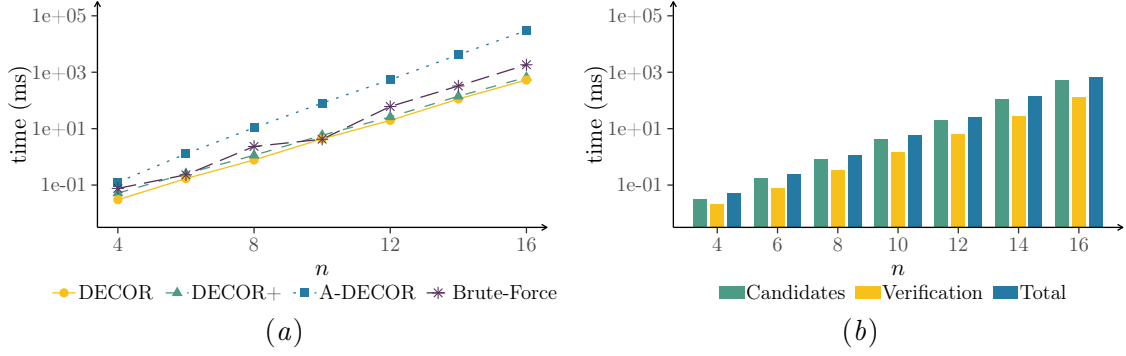
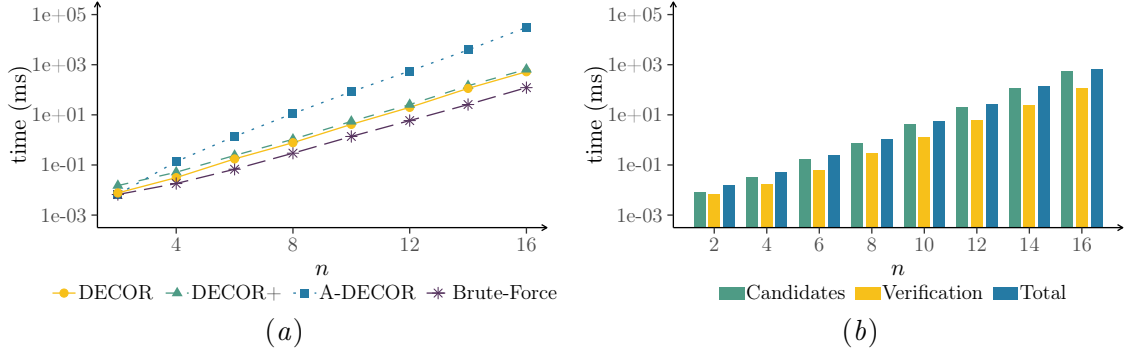
\begin{figure}[t]
	\centering
	\subfigure[\label{fig:decor_revised_plot_k=n_times}]{%
		\input{files/decor_revised_plot_k=n_times.tex}%
	}\hfill
	\subfigure[\label{fig:decor_revised_plot_k=n_phases}]{%
		\input{files/decor_revised_plot_k=n_phases_decorplus.tex}%
	}
	\caption{($a$) Run times of \ac{decor}, \ac{decorplus}, \ac{decorapriori}, and the brute-force algorithm for input factors with $n$ arguments, of which $k = n$ arguments are commutative. ($b$) Proportion of the run time of \ac{decorplus} spent in the bucket loop and the verification step for input factors with $k = n$ commutative arguments.}
	\label{fig:decor_revised_plot_k=n}
\end{figure}

The plots for individual choices of $k$ in \cref{fig:decor_revised_plot_k=0,fig:decor_revised_plot_k=2,fig:decor_revised_plot_k=log2n,fig:decor_revised_plot_k=ndiv2,fig:decor_revised_plot_k=nsub1,fig:decor_revised_plot_k=n} confirm the general picture observed in \cref{fig:decor_revised_plot_avg} across all values of $k$: \Ac{decorplus} consistently outperforms \ac{decorapriori} and the brute-force algorithm runs into the timeout for $n > 12$---with the exception of $k = n - 1$ and $k = n$ (\cref{fig:decor_revised_plot_k=nsub1,fig:decor_revised_plot_k=n}), where its decreasing-size search finds a commutative subset within a few iterations (which also explains its superior performance for the choice of $k = n$, where the brute-force algorithm finds the unique commutative subset spanning all arguments immediately).
For $k = 0$ (\cref{fig:decor_revised_plot_k=0}), there exists no subset of commutative arguments at all, and the bucket loop of \ac{decorplus} rules out every candidate subset immediately before even reaching the verification step, which explains the missing bars for the verification step in \cref{fig:decor_revised_plot_k=0_phases}.
For the remaining values of $k$, the verification step still contributes only a small fraction of the total run time of \ac{decorplus}, in line with the observation in \cref{sec:decor_revised_eval} (again, note that due to the logarithmic scale on the $y$-axis, the \enquote{Candidates} and \enquote{Verification} bars do not visually add up to the \enquote{Total} bar).

\begin{figure}
	\centering
	\subfigure[\label{fig:decor_revised_plot_groups_vary_g}]{%
		\input{files/decor_revised_plot_groups_vary_g.tex}%
	}\hfill
	\subfigure[\label{fig:decor_revised_plot_groups_vary_s}]{%
		\input{files/decor_revised_plot_groups_vary_s.tex}%
	}
	\caption{Run times of \ac{decor}, \ac{decorplus}, \ac{decorapriori}, and the brute-force algorithm on input factors with $n$ arguments partitioned into $g$ disjoint commutative groups of size $s = n \mathbin{/} g$ each. ($a$) varies $g$ at the smallest available $s$ for each $n$; ($b$) varies $s$ at the smallest available $g$ for each $n$.}
	\label{fig:decor_revised_plot_groups}
\end{figure}
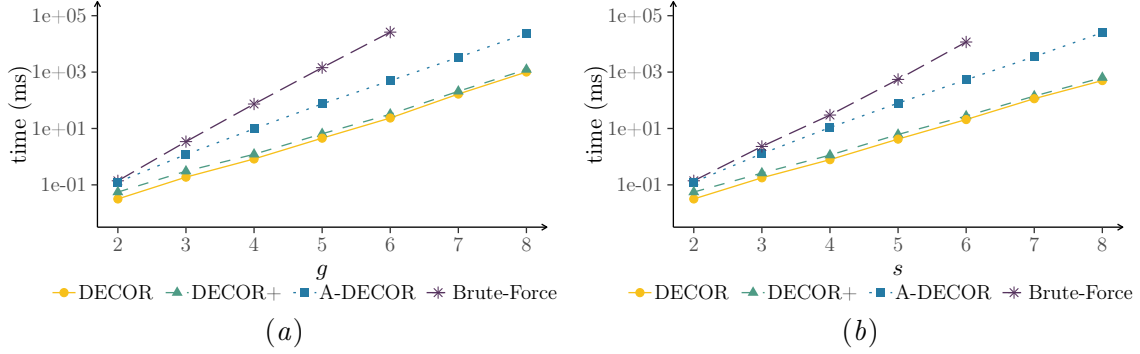

So far, all considered input factors had a single subset of commutative arguments of size $k$.
Next, we investigate how the algorithms behave when there are multiple disjoint subsets of commutative arguments (which is unlikely in practice---nevertheless an evaluation of this setting has been missing in the original \ac{decor} paper~\citep{Luttermann2024f}).
In particular, we now consider input factors with $n \in \{4, \allowbreak 6, \allowbreak 8, \allowbreak 10, \allowbreak 12, \allowbreak 14, \allowbreak 16\}$ Boolean arguments and partition the $n$ arguments into $g$ disjoint commutative groups of equal size $s = \lfloor n \mathbin{/} g \rfloor$ for every divisor $g$ of $n$ satisfying $2 \leq g \leq n \mathbin{/} 2$ (that is, there are $k = g \cdot s$ commutative arguments in total), and we again set a timeout of five minutes per instance.
\Cref{fig:decor_revised_plot_groups} examines how the algorithms behave when the commutative arguments are split into $g$ disjoint groups of size $s$ each.
The run times of \ac{decorplus} in \cref{fig:decor_revised_plot_groups} are essentially indistinguishable from the run times observed for a single subset of commutative arguments in the previous plots, regardless of how the commutative arguments are distributed across groups: \Ac{decorplus} consistently outperforms \ac{decorapriori} for every individual choice of $g$ and $s$, respectively, confirming that the bucket-based pruning of \ac{decorplus} maintains its high effectiveness independent of the group structure of the commutative arguments.

Before we close this section, we briefly compare the \ac{decorapriori} algorithm to \ac{decorcc} and discuss a selection of heuristics that seem promising to improve the performance of \ac{decorplus} in practice.
We have seen that \cref{prop:decor_revised_bound_on_commutative_subset} provides an a-priori upper bound on the size of any commutative subset of $\phi$, given by the smallest intermediate result of the candidate subsets.
Since the candidate set $\boldsymbol C_{cand}$ is monotonically refined by intersection with $\boldsymbol C'$ at every bucket, the order in which the buckets are processed does not affect the final result returned by \ac{decorplus}, but it might affect the size of the intermediate candidate set $\boldsymbol C_{cand}$ and hence the overall run time.
Concretely, a bucket that produces a small, restrictive $\boldsymbol C'$ collapses $\boldsymbol C_{cand}$ early, which in turn makes every subsequent cross-intersection step cheap.
Our goal hence is to find restrictive buckets early to keep intermediate results for the candidate subsets small.

To exploit this observation, we investigate the effect of a preliminary bucket-ordering pass that sorts the buckets entailed by $\phi$ before computing the intersections between the candidate subsets according to a heuristic that prioritises buckets that are likely to collapse $\boldsymbol C_{cand}$ early.
We consider four heuristics, ordered by increasing pre-processing cost:
\begin{enumerate}
	\item \textbf{Smallest bucket first (SBF).} Sort the buckets in $\mathcal B(\phi)$ by $\abs{\phi^{\succ}(b)}$ in ascending order.
	\item \textbf{Least groups first (LGF).} Sort the buckets by the number of maximal groups of identical potentials of size at least two in ascending order.
	\item \textbf{Smallest candidate set first (SCSF).} Sort the buckets by the size of the candidate set $\boldsymbol C'$ in ascending order.
	\item \textbf{Smallest minimal candidate first (SMCF).} Sort the buckets by the size of the smallest minimal candidate $\min_{\boldsymbol C_i \in \boldsymbol C'} \abs{\boldsymbol C_i}$ in ascending order.
\end{enumerate}

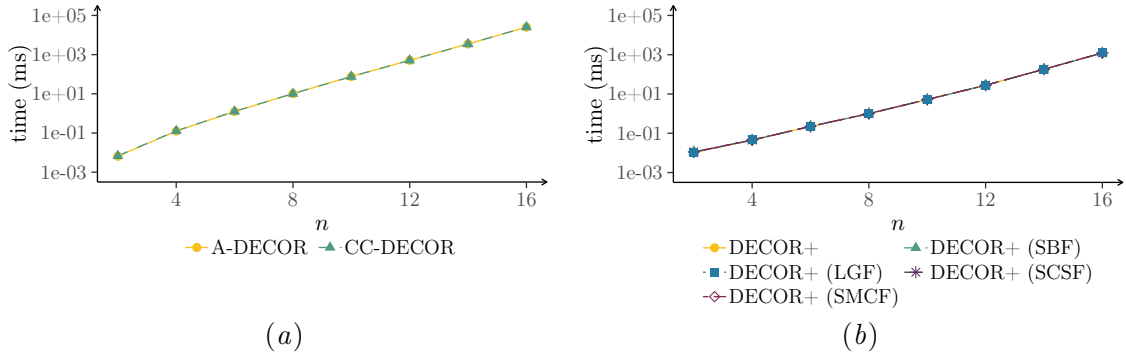
\begin{figure}
	\centering
	\subfigure[\label{fig:decor_revised_plot_avg_times_apriori_cc}]{%
		\input{files/decor_revised_plot_avg_times_apriori.tex}%
	}\hfill
	\subfigure[\label{fig:decor_revised_plot_avg_times_heuristics}]{%
		\input{files/decor_revised_plot_avg_times_heuristics.tex}%
	}
	\caption{($a$) Average run times of \ac{decorapriori} and \ac{decorcc} for input factors with $n$ arguments. ($b$) Average run times of \ac{decorplus} and its four bucket-ordering heuristics for input factors with $n$ arguments.}
	\label{fig:decor_revised_plot_avg_times_apriori_heuristic}
\end{figure}

\Cref{fig:decor_revised_plot_avg_times_apriori_cc} compares \ac{decorapriori} to \ac{decorcc}.
The run times of both algorithms are essentially indistinguishable, indicating that the merging procedure plays a minor role since the dominant cost of both algorithms is the initial computation of all pairwise commutative arguments.
Furthermore, \cref{fig:decor_revised_plot_avg_times_heuristics} compares the run times of \ac{decorplus} and its four bucket-ordering heuristics.
Here, again, the four heuristic variants of \ac{decorplus} achieve run times indistinguishable from the run time of plain \ac{decorplus}, showcasing that the bucket order does not have a significant influence on the overall run time of \ac{decorplus} (probably because the pruning step produces intermediate results independent of the bucket order).
\end{document}

%% file: defs.tex
\acrodef{acp}[ACP]{advanced colour passing}
\acrodef{apricot}[APRICOT]{Apriori-style detection of commutative factors}
\acrodef{bn}[BN]{Bayesian network}
\acrodef{cp}[CP]{colour passing}
\acrodef{crv}[CRV]{counting randvar}
\acrodef{decor}[DECOR]{detection of commutative factors}
\acrodef{decorapriori}[A-DECOR]{Apriori-style detection of commutative factors}
\acrodef{decorcc}[CC-DECOR]{connected-component detection of commutative factors}
\acrodef{decorplus}[DECOR+]{detection of commutative factors revised}
\acrodef{deft}[DEFT]{detection of exchangeable factors}
\acrodef{fg}[FG]{factor graph}
\acrodef{ljt}[LJT]{lifted junction tree}
\acrodef{lv}[logvar]{logical variable}
\acrodef{lve}[LVE]{lifted variable elimination}
\acrodef{mn}[MN]{Markov network}
\acrodef{pcrv}[PCRV]{parameterised CRV}
\acrodef{pf}[parfactor]{parametric factor}
\acrodef{pfg}[PFG]{parametric factor graph}
\acrodef{prv}[PRV]{parameterised randvar}
\acrodef{rv}[randvar]{random variable}

\crefname{algorithm}{Alg.}{Algs.}
\Crefname{algorithm}{Algorithm}{Algorithms}
\crefname{appendix}{App.}{Apps.}
\Crefname{appendix}{Appendix}{Appendices}
\crefname{claim}{Claim}{Claims}
\Crefname{claim}{Claim}{Claims}
\crefname{corollary}{Cor.}{Cors.}
\Crefname{corollary}{Corollary}{Corollaries}
\crefname{definition}{Def.}{Defs.}
\Crefname{definition}{Definition}{Definitions}
\crefname{example}{Ex.}{Exs.}
\Crefname{example}{Example}{Examples}
\crefname{proposition}{Prop.}{Props.}
\Crefname{proposition}{Proposition}{Propositions}
\crefname{section}{Sec.}{Secs.}
\Crefname{section}{Section}{Sections}
\crefname{theorem}{Thm.}{Thms.}
\Crefname{theorem}{Theorem}{Theorems}

\newtheorem{claim}[theorem]{Claim}

\newcommand{\alginput}[1]{\hspace*{\algorithmicindent} \textbf{Input:} #1}
\newcommand{\algoutput}[1]{\hspace*{\algorithmicindent} \textbf{Output:} #1}

\newcommand{\abs}[1]{\ensuremath{\lvert #1 \rvert}}

\newcommand{\domain}[1]{\ensuremath{\mathrm{dom}(#1)}}

\newcommand{\range}[1]{\ensuremath{\mathrm{range}(#1)}}

\newcommand{\low}{\ensuremath{\mathrm{low}}}

\newcommand{\high}{\ensuremath{\mathrm{high}}}

\definecolor{myyellow}{RGB}{247,192,26}
\definecolor{myblue}{RGB}{37,122,164}
\definecolor{mygreen}{RGB}{78,155,133}
\definecolor{mypurple}{RGB}{86,51,94}

\definecolor{newblue}{RGB}{50,113,173}
\definecolor{newred}{RGB}{222,32,36}
\definecolor{newgreen}{RGB}{70,165,69}
\definecolor{newpurple}{RGB}{140,69,152}

\definecolor{cborange}{RGB}{230,159,0}
\definecolor{cbblue}{RGB}{30,136,229}
\definecolor{cbbluedark}{RGB}{46,37,133}
\definecolor{cbpurple}{RGB}{170,68,153}
\definecolor{cbgreen}{RGB}{0,77,64}
\definecolor{cbgreenlight}{RGB}{93,168,153}
\definecolor{cbbrown}{RGB}{126,41,84}

\pgfdeclarelayer{bg}
\pgfsetlayers{bg,main}

\tikzset{
	rv/.style={draw, ellipse},
	pf/.style={draw, rectangle, fill = gray!30},
	arc/.style = {->, >={[round,sep]Stealth}},
	reversearc/.style = {<-, >={[round,sep]Stealth}},
	doublearc/.style = {<->, >={[round,sep]Stealth}},
}

\newcommand\factor[6]{
	\node[pf, #1=#3 of #2, label={#4:{#5}}](#6) {};
}

%% file: files/decor_revised_example_fg.tex
\begin{tikzpicture}[rv/.append style={minimum height=2.2em, minimum width=5.2em}]
	\node[rv, draw] (rev) {$Rev$};
	\node[rv, draw, above left = 2.5em and 0.3em of rev] (ca) {$ComA$};
	\node[rv, draw, above right = 2.5em and 0.3em of rev] (cb) {$ComB$};
	\node[rv, draw, left = 1.5em of rev] (sa) {$SalA$};
	\node[rv, draw, right = 1.5em of rev] (sb) {$SalB$};

	\factor{above}{ca}{0.75em}{180}{$f_1$}{f1}
	\factor{above}{cb}{0.75em}{0}{$f_2$}{f2}
	\factor{above}{rev}{0.75em}{90}{$f_3$}{f3}
	\factor{above}{sa}{1.0em}{180}{$f_4$}{f4}
	\factor{above}{sb}{1.0em}{0}{$f_5$}{f5}

	\draw (f1) -- (ca);
	\draw (f2) -- (cb);
	\draw (ca) -- (f3);
	\draw (cb) -- (f3);
	\draw (f3) -- (rev);
	\draw (rev) -- (f4);
	\draw (rev) -- (f5);
	\draw (ca) -- (f4);
	\draw (cb) -- (f5);
	\draw (f4) -- (sa);
	\draw (f5) -- (sb);
\end{tikzpicture}

%% file: files/decor_revised_example_potential_table.tex
\begin{tabular}{lllc}
	\toprule
	$ComA$ & $ComB$ & $Rev$ & $\phi_3(ComA, ComB, Rev)$ \\ \midrule
	\high  & \high  & \high & $\varphi_{1}$             \\
	\high  & \high  & \low  & $\varphi_{2}$             \\
	\high  & \low   & \high & $\varphi_{3}$             \\
	\high  & \low   & \low  & $\varphi_{4}$             \\
	\low   & \high  & \high & $\varphi_{3}$             \\
	\low   & \high  & \low  & $\varphi_{4}$             \\
	\low   & \low   & \high & $\varphi_{5}$             \\
	\low   & \low   & \low  & $\varphi_{6}$             \\ \bottomrule
\end{tabular}

%% file: files/decor_revised_example_buckets.tex
\begin{tikzpicture}
	\node (tab_f1) {
		\begin{tabular}{lllcc}
			\toprule
			$ComA$ & $ComB$ & $Rev$ & $\phi_3(ComA, ComB, Rev)$ & $b$ \\ \midrule
			\high  & \high  & \high & $\varphi_{1}$             & {\color{cborange}$[3,0]$} \\
			\high  & \high  & \low  & $\varphi_{2}$             & {\color{cbbluedark}$[2,1]$} \\
			\high  & \low   & \high & $\varphi_{3}$             & {\color{cbbluedark}$[2,1]$} \\
			\high  & \low   & \low  & $\varphi_{4}$             & {\color{cbpurple}$[1,2]$} \\
			\low   & \high  & \high & $\varphi_{3}$             & {\color{cbbluedark}$[2,1]$} \\
			\low   & \high  & \low  & $\varphi_{4}$             & {\color{cbpurple}$[1,2]$} \\
			\low   & \low   & \high & $\varphi_{5}$             & {\color{cbpurple}$[1,2]$} \\
			\low   & \low   & \low  & $\varphi_{6}$             & {\color{cbgreen}$[0,3]$} \\ \bottomrule
		\end{tabular}
	};

	\node[right = 0.2cm of tab_f1] (buckets) {
		\begin{tabular}{cc}
			\toprule
			$b$                         & $\phi_3^{\succ}(b)$                               \\ \midrule
			{\color{cborange}$[3,0]$}   & $\langle \varphi_1 \rangle$                       \\
			{\color{cbbluedark}$[2,1]$} & $\langle \varphi_2, \varphi_3, \varphi_3 \rangle$ \\
			{\color{cbpurple}$[1,2]$}   & $\langle \varphi_4, \varphi_4, \varphi_5 \rangle$ \\
			{\color{cbgreen}$[0,3]$}    & $\langle \varphi_6 \rangle$                       \\ \bottomrule
		\end{tabular}
	};
\end{tikzpicture}

%% file: files/decor_revised_plot_avg_times.tex
\begin{tikzpicture}[x=1pt,y=1pt]
\definecolor{fillColor}{RGB}{255,255,255}
\path[use as bounding box,fill=fillColor,fill opacity=0.00] (0,0) rectangle (209.58,115.63);
\begin{scope}
\path[clip] (  0.00,  0.00) rectangle (209.58,115.63);
\definecolor{drawColor}{RGB}{255,255,255}
\definecolor{fillColor}{RGB}{255,255,255}

\path[draw=drawColor,line width= 0.5pt,line join=round,line cap=round,fill=fillColor] (  0.00,  0.00) rectangle (209.58,115.63);
\end{scope}
\begin{scope}
\path[clip] ( 34.42, 28.73) rectangle (203.58,113.63);
\definecolor{fillColor}{RGB}{255,255,255}

\path[fill=fillColor] ( 34.42, 28.73) rectangle (203.58,113.63);
\definecolor{drawColor}{RGB}{247,192,26}

\path[draw=drawColor,line width= 0.5pt,line join=round] ( 42.11, 41.31) --
	( 64.07, 46.93) --
	( 86.04, 53.98) --
	(108.01, 60.20) --
	(129.98, 67.20) --
	(151.95, 74.30) --
	(173.92, 82.34) --
	(195.89, 90.38);
\definecolor{drawColor}{RGB}{78,155,133}

\path[draw=drawColor,line width= 0.5pt,dash pattern=on 4pt off 4pt ,line join=round] ( 42.11, 43.08) --
	( 64.07, 49.00) --
	( 86.04, 55.63) --
	(108.01, 61.74) --
	(129.98, 68.53) --
	(151.95, 75.36) --
	(173.92, 83.08) --
	(195.89, 90.88);
\definecolor{drawColor}{RGB}{37,122,164}

\path[draw=drawColor,line width= 0.5pt,dash pattern=on 1pt off 3pt ,line join=round] ( 42.11, 41.13) --
	( 64.07, 53.24) --
	( 86.04, 62.66) --
	(108.01, 71.26) --
	(129.98, 79.52) --
	(151.95, 87.29) --
	(173.92, 95.20) --
	(195.89,103.27);
\definecolor{drawColor}{RGB}{86,51,94}

\path[draw=drawColor,line width= 0.5pt,dash pattern=on 7pt off 3pt ,line join=round] ( 42.11, 41.20) --
	( 64.07, 53.38) --
	( 86.04, 65.34) --
	(108.01, 77.67) --
	(129.98, 89.66) --
	(151.95,101.85);
\definecolor{drawColor}{RGB}{247,192,26}
\definecolor{fillColor}{RGB}{247,192,26}

\path[draw=drawColor,line width= 0.4pt,line join=round,line cap=round,fill=fillColor] (173.92, 82.34) circle (  1.53);

\path[draw=drawColor,line width= 0.4pt,line join=round,line cap=round,fill=fillColor] ( 64.07, 46.93) circle (  1.53);

\path[draw=drawColor,line width= 0.4pt,line join=round,line cap=round,fill=fillColor] (151.95, 74.30) circle (  1.53);

\path[draw=drawColor,line width= 0.4pt,line join=round,line cap=round,fill=fillColor] (129.98, 67.20) circle (  1.53);

\path[draw=drawColor,line width= 0.4pt,line join=round,line cap=round,fill=fillColor] ( 42.11, 41.31) circle (  1.53);

\path[draw=drawColor,line width= 0.4pt,line join=round,line cap=round,fill=fillColor] ( 86.04, 53.98) circle (  1.53);

\path[draw=drawColor,line width= 0.4pt,line join=round,line cap=round,fill=fillColor] (195.89, 90.38) circle (  1.53);

\path[draw=drawColor,line width= 0.4pt,line join=round,line cap=round,fill=fillColor] (108.01, 60.20) circle (  1.53);
\definecolor{drawColor}{RGB}{86,51,94}

\path[draw=drawColor,line width= 0.4pt,line join=round,line cap=round] ( 62.54, 51.84) -- ( 65.61, 54.91);

\path[draw=drawColor,line width= 0.4pt,line join=round,line cap=round] ( 62.54, 54.91) -- ( 65.61, 51.84);

\path[draw=drawColor,line width= 0.4pt,line join=round,line cap=round] ( 61.91, 53.38) -- ( 66.24, 53.38);

\path[draw=drawColor,line width= 0.4pt,line join=round,line cap=round] ( 64.07, 51.21) -- ( 64.07, 55.55);

\path[draw=drawColor,line width= 0.4pt,line join=round,line cap=round] (150.42,100.32) -- (153.49,103.38);

\path[draw=drawColor,line width= 0.4pt,line join=round,line cap=round] (150.42,103.38) -- (153.49,100.32);

\path[draw=drawColor,line width= 0.4pt,line join=round,line cap=round] (149.79,101.85) -- (154.12,101.85);

\path[draw=drawColor,line width= 0.4pt,line join=round,line cap=round] (151.95, 99.68) -- (151.95,104.02);

\path[draw=drawColor,line width= 0.4pt,line join=round,line cap=round] (128.45, 88.13) -- (131.52, 91.20);

\path[draw=drawColor,line width= 0.4pt,line join=round,line cap=round] (128.45, 91.20) -- (131.52, 88.13);

\path[draw=drawColor,line width= 0.4pt,line join=round,line cap=round] (127.82, 89.66) -- (132.15, 89.66);

\path[draw=drawColor,line width= 0.4pt,line join=round,line cap=round] (129.98, 87.49) -- (129.98, 91.83);

\path[draw=drawColor,line width= 0.4pt,line join=round,line cap=round] ( 40.57, 39.67) -- ( 43.64, 42.73);

\path[draw=drawColor,line width= 0.4pt,line join=round,line cap=round] ( 40.57, 42.73) -- ( 43.64, 39.67);

\path[draw=drawColor,line width= 0.4pt,line join=round,line cap=round] ( 39.94, 41.20) -- ( 44.27, 41.20);

\path[draw=drawColor,line width= 0.4pt,line join=round,line cap=round] ( 42.11, 39.03) -- ( 42.11, 43.37);

\path[draw=drawColor,line width= 0.4pt,line join=round,line cap=round] ( 84.51, 63.81) -- ( 87.58, 66.88);

\path[draw=drawColor,line width= 0.4pt,line join=round,line cap=round] ( 84.51, 66.88) -- ( 87.58, 63.81);

\path[draw=drawColor,line width= 0.4pt,line join=round,line cap=round] ( 83.88, 65.34) -- ( 88.21, 65.34);

\path[draw=drawColor,line width= 0.4pt,line join=round,line cap=round] ( 86.04, 63.17) -- ( 86.04, 67.51);

\path[draw=drawColor,line width= 0.4pt,line join=round,line cap=round] (106.48, 76.13) -- (109.55, 79.20);

\path[draw=drawColor,line width= 0.4pt,line join=round,line cap=round] (106.48, 79.20) -- (109.55, 76.13);

\path[draw=drawColor,line width= 0.4pt,line join=round,line cap=round] (105.85, 77.67) -- (110.18, 77.67);

\path[draw=drawColor,line width= 0.4pt,line join=round,line cap=round] (108.01, 75.50) -- (108.01, 79.84);
\definecolor{fillColor}{RGB}{78,155,133}

\path[fill=fillColor] (173.92, 85.47) --
	(175.99, 81.89) --
	(171.86, 81.89) --
	cycle;

\path[fill=fillColor] ( 64.07, 51.39) --
	( 66.14, 47.81) --
	( 62.01, 47.81) --
	cycle;

\path[fill=fillColor] (151.95, 77.74) --
	(154.02, 74.17) --
	(149.89, 74.17) --
	cycle;

\path[fill=fillColor] (129.98, 70.92) --
	(132.05, 67.34) --
	(127.92, 67.34) --
	cycle;

\path[fill=fillColor] ( 42.11, 45.46) --
	( 44.17, 41.88) --
	( 40.04, 41.88) --
	cycle;

\path[fill=fillColor] ( 86.04, 58.02) --
	( 88.11, 54.44) --
	( 83.98, 54.44) --
	cycle;

\path[fill=fillColor] (195.89, 93.27) --
	(197.96, 89.69) --
	(193.83, 89.69) --
	cycle;

\path[fill=fillColor] (108.01, 64.13) --
	(110.08, 60.55) --
	(105.95, 60.55) --
	cycle;
\definecolor{fillColor}{RGB}{37,122,164}

\path[fill=fillColor] (172.39, 93.66) --
	(175.46, 93.66) --
	(175.46, 96.73) --
	(172.39, 96.73) --
	cycle;

\path[fill=fillColor] ( 62.54, 51.71) --
	( 65.61, 51.71) --
	( 65.61, 54.78) --
	( 62.54, 54.78) --
	cycle;

\path[fill=fillColor] (150.42, 85.75) --
	(153.49, 85.75) --
	(153.49, 88.82) --
	(150.42, 88.82) --
	cycle;

\path[fill=fillColor] (128.45, 77.98) --
	(131.52, 77.98) --
	(131.52, 81.05) --
	(128.45, 81.05) --
	cycle;

\path[fill=fillColor] ( 40.57, 39.60) --
	( 43.64, 39.60) --
	( 43.64, 42.67) --
	( 40.57, 42.67) --
	cycle;

\path[fill=fillColor] ( 84.51, 61.13) --
	( 87.58, 61.13) --
	( 87.58, 64.20) --
	( 84.51, 64.20) --
	cycle;

\path[fill=fillColor] (194.36,101.74) --
	(197.43,101.74) --
	(197.43,104.81) --
	(194.36,104.81) --
	cycle;

\path[fill=fillColor] (106.48, 69.73) --
	(109.55, 69.73) --
	(109.55, 72.80) --
	(106.48, 72.80) --
	cycle;
\end{scope}
\begin{scope}
\path[clip] (  0.00,  0.00) rectangle (209.58,115.63);
\definecolor{drawColor}{RGB}{0,0,0}

\path[draw=drawColor,line width= 0.5pt,line join=round] ( 34.42, 28.73) --
	( 34.42,113.63);

\path[draw=drawColor,line width= 0.5pt,line join=round] ( 35.55,111.66) --
	( 34.42,113.63) --
	( 33.28,111.66);
\end{scope}
\begin{scope}
\path[clip] (  0.00,  0.00) rectangle (209.58,115.63);
\definecolor{drawColor}{gray}{0.30}

\node[text=drawColor,anchor=base east,inner sep=0pt, outer sep=0pt, scale=  0.70] at ( 30.37, 31.03) {1e-03};

\node[text=drawColor,anchor=base east,inner sep=0pt, outer sep=0pt, scale=  0.70] at ( 30.37, 49.90) {1e-01};

\node[text=drawColor,anchor=base east,inner sep=0pt, outer sep=0pt, scale=  0.70] at ( 30.37, 68.77) {1e+01};

\node[text=drawColor,anchor=base east,inner sep=0pt, outer sep=0pt, scale=  0.70] at ( 30.37, 87.64) {1e+03};

\node[text=drawColor,anchor=base east,inner sep=0pt, outer sep=0pt, scale=  0.70] at ( 30.37,106.50) {1e+05};
\end{scope}
\begin{scope}
\path[clip] (  0.00,  0.00) rectangle (209.58,115.63);
\definecolor{drawColor}{gray}{0.20}

\path[draw=drawColor,line width= 0.5pt,line join=round] ( 32.17, 33.45) --
	( 34.42, 33.45);

\path[draw=drawColor,line width= 0.5pt,line join=round] ( 32.17, 52.31) --
	( 34.42, 52.31);

\path[draw=drawColor,line width= 0.5pt,line join=round] ( 32.17, 71.18) --
	( 34.42, 71.18);

\path[draw=drawColor,line width= 0.5pt,line join=round] ( 32.17, 90.05) --
	( 34.42, 90.05);

\path[draw=drawColor,line width= 0.5pt,line join=round] ( 32.17,108.92) --
	( 34.42,108.92);
\end{scope}
\begin{scope}
\path[clip] (  0.00,  0.00) rectangle (209.58,115.63);
\definecolor{drawColor}{RGB}{0,0,0}

\path[draw=drawColor,line width= 0.5pt,line join=round] ( 34.42, 28.73) --
	(203.58, 28.73);

\path[draw=drawColor,line width= 0.5pt,line join=round] (201.61, 27.59) --
	(203.58, 28.73) --
	(201.61, 29.87);
\end{scope}
\begin{scope}
\path[clip] (  0.00,  0.00) rectangle (209.58,115.63);
\definecolor{drawColor}{gray}{0.20}

\path[draw=drawColor,line width= 0.5pt,line join=round] ( 64.07, 26.48) --
	( 64.07, 28.73);

\path[draw=drawColor,line width= 0.5pt,line join=round] (108.01, 26.48) --
	(108.01, 28.73);

\path[draw=drawColor,line width= 0.5pt,line join=round] (151.95, 26.48) --
	(151.95, 28.73);

\path[draw=drawColor,line width= 0.5pt,line join=round] (195.89, 26.48) --
	(195.89, 28.73);
\end{scope}
\begin{scope}
\path[clip] (  0.00,  0.00) rectangle (209.58,115.63);
\definecolor{drawColor}{gray}{0.30}

\node[text=drawColor,anchor=base,inner sep=0pt, outer sep=0pt, scale=  0.70] at ( 64.07, 19.86) {4};

\node[text=drawColor,anchor=base,inner sep=0pt, outer sep=0pt, scale=  0.70] at (108.01, 19.86) {8};

\node[text=drawColor,anchor=base,inner sep=0pt, outer sep=0pt, scale=  0.70] at (151.95, 19.86) {12};

\node[text=drawColor,anchor=base,inner sep=0pt, outer sep=0pt, scale=  0.70] at (195.89, 19.86) {16};
\end{scope}
\begin{scope}
\path[clip] (  0.00,  0.00) rectangle (209.58,115.63);
\definecolor{drawColor}{RGB}{0,0,0}

\node[text=drawColor,anchor=base,inner sep=0pt, outer sep=0pt, scale=  0.80] at (119.00, 10.74) {$n$};
\end{scope}
\begin{scope}
\path[clip] (  0.00,  0.00) rectangle (209.58,115.63);
\definecolor{drawColor}{RGB}{0,0,0}

\node[text=drawColor,rotate= 90.00,anchor=base,inner sep=0pt, outer sep=0pt, scale=  0.80] at (  7.51, 71.18) {time (ms)};
\end{scope}
\begin{scope}
\path[clip] (  0.00,  0.00) rectangle (209.58,115.63);

\path[] ( 34.42,  1.00) rectangle (210.27,  7.18);
\end{scope}
\begin{scope}
\path[clip] (  0.00,  0.00) rectangle (209.58,115.63);
\definecolor{drawColor}{RGB}{247,192,26}

\path[draw=drawColor,line width= 0.5pt,line join=round] ( 17.50,  4.09) -- ( 26.17,  4.09);
\end{scope}
\begin{scope}
\path[clip] (  0.00,  0.00) rectangle (209.58,115.63);
\definecolor{drawColor}{RGB}{247,192,26}
\definecolor{fillColor}{RGB}{247,192,26}

\path[draw=drawColor,line width= 0.4pt,line join=round,line cap=round,fill=fillColor] ( 21.84,  4.09) circle (  1.53);
\end{scope}
\begin{scope}
\path[clip] (  0.00,  0.00) rectangle (209.58,115.63);
\definecolor{drawColor}{RGB}{78,155,133}

\path[draw=drawColor,line width= 0.5pt,dash pattern=on 4pt off 4pt ,line join=round] ( 60.10,  4.09) -- ( 68.77,  4.09);
\end{scope}
\begin{scope}
\path[clip] (  0.00,  0.00) rectangle (209.58,115.63);
\definecolor{fillColor}{RGB}{78,155,133}

\path[fill=fillColor] ( 64.43,  6.48) --
	( 66.50,  2.90) --
	( 62.37,  2.90) --
	cycle;
\end{scope}
\begin{scope}
\path[clip] (  0.00,  0.00) rectangle (209.58,115.63);
\definecolor{drawColor}{RGB}{37,122,164}

\path[draw=drawColor,line width= 0.5pt,dash pattern=on 1pt off 3pt ,line join=round] (108.14,  4.09) -- (116.81,  4.09);
\end{scope}
\begin{scope}
\path[clip] (  0.00,  0.00) rectangle (209.58,115.63);
\definecolor{fillColor}{RGB}{37,122,164}

\path[fill=fillColor] (110.94,  2.56) --
	(114.01,  2.56) --
	(114.01,  5.62) --
	(110.94,  5.62) --
	cycle;
\end{scope}
\begin{scope}
\path[clip] (  0.00,  0.00) rectangle (209.58,115.63);
\definecolor{drawColor}{RGB}{86,51,94}

\path[draw=drawColor,line width= 0.5pt,dash pattern=on 7pt off 3pt ,line join=round] (158.32,  4.09) -- (166.99,  4.09);
\end{scope}
\begin{scope}
\path[clip] (  0.00,  0.00) rectangle (209.58,115.63);
\definecolor{drawColor}{RGB}{86,51,94}

\path[draw=drawColor,line width= 0.4pt,line join=round,line cap=round] (161.12,  2.56) -- (164.19,  5.62);

\path[draw=drawColor,line width= 0.4pt,line join=round,line cap=round] (161.12,  5.62) -- (164.19,  2.56);

\path[draw=drawColor,line width= 0.4pt,line join=round,line cap=round] (160.49,  4.09) -- (164.82,  4.09);

\path[draw=drawColor,line width= 0.4pt,line join=round,line cap=round] (162.65,  1.92) -- (162.65,  6.26);
\end{scope}
\begin{scope}
\path[clip] (  0.00,  0.00) rectangle (209.58,115.63);
\definecolor{drawColor}{RGB}{0,0,0}

\node[text=drawColor,anchor=base west,inner sep=0pt, outer sep=0pt, scale=  0.70] at ( 27.26,  1.68) {DECOR};
\end{scope}
\begin{scope}
\path[clip] (  0.00,  0.00) rectangle (209.58,115.63);
\definecolor{drawColor}{RGB}{0,0,0}

\node[text=drawColor,anchor=base west,inner sep=0pt, outer sep=0pt, scale=  0.70] at ( 69.85,  1.68) {DECOR+};
\end{scope}
\begin{scope}
\path[clip] (  0.00,  0.00) rectangle (209.58,115.63);
\definecolor{drawColor}{RGB}{0,0,0}

\node[text=drawColor,anchor=base west,inner sep=0pt, outer sep=0pt, scale=  0.70] at (117.90,  1.68) {A-DECOR};
\end{scope}
\begin{scope}
\path[clip] (  0.00,  0.00) rectangle (209.58,115.63);
\definecolor{drawColor}{RGB}{0,0,0}

\node[text=drawColor,anchor=base west,inner sep=0pt, outer sep=0pt, scale=  0.70] at (168.07,  1.68) {Brute-Force};
\end{scope}
\end{tikzpicture}

%% file: files/decor_revised_plot_avg_phases_decorplus.tex
\begin{tikzpicture}[x=1pt,y=1pt]
\definecolor{fillColor}{RGB}{255,255,255}
\path[use as bounding box,fill=fillColor,fill opacity=0.00] (0,0) rectangle (209.58,115.63);
\begin{scope}
\path[clip] (  0.00,  0.00) rectangle (209.58,115.63);
\definecolor{drawColor}{RGB}{255,255,255}
\definecolor{fillColor}{RGB}{255,255,255}

\path[draw=drawColor,line width= 0.5pt,line join=round,line cap=round,fill=fillColor] (  0.00,  0.00) rectangle (209.58,115.63);
\end{scope}
\begin{scope}
\path[clip] ( 34.42, 28.73) rectangle (203.58,113.63);
\definecolor{fillColor}{RGB}{255,255,255}

\path[fill=fillColor] ( 34.42, 28.73) rectangle (203.58,113.63);
\definecolor{fillColor}{RGB}{78,155,133}

\path[fill=fillColor] (160.71, 28.73) rectangle (165.06, 82.38);

\path[fill=fillColor] ( 61.87, 28.73) rectangle ( 66.22, 47.07);

\path[fill=fillColor] (140.94, 28.73) rectangle (145.29, 74.25);

\path[fill=fillColor] (121.17, 28.73) rectangle (125.52, 67.27);

\path[fill=fillColor] ( 42.11, 28.73) rectangle ( 46.45, 41.53);

\path[fill=fillColor] ( 81.64, 28.73) rectangle ( 85.99, 54.06);

\path[fill=fillColor] (180.48, 28.73) rectangle (184.82, 90.37);

\path[fill=fillColor] (101.41, 28.73) rectangle (105.76, 60.21);
\definecolor{fillColor}{RGB}{247,192,26}

\path[fill=fillColor] (166.24, 28.73) rectangle (170.59, 75.53);

\path[fill=fillColor] ( 67.41, 28.73) rectangle ( 71.76, 45.00);

\path[fill=fillColor] (146.48, 28.73) rectangle (150.82, 69.47);

\path[fill=fillColor] (126.71, 28.73) rectangle (131.06, 63.11);

\path[fill=fillColor] ( 47.64, 28.73) rectangle ( 51.99, 38.33);

\path[fill=fillColor] ( 87.17, 28.73) rectangle ( 91.52, 50.94);

\path[fill=fillColor] (186.01, 28.73) rectangle (190.36, 82.13);

\path[fill=fillColor] (106.94, 28.73) rectangle (111.29, 56.98);
\definecolor{fillColor}{RGB}{37,122,164}

\path[fill=fillColor] (171.78, 28.73) rectangle (176.13, 83.08);

\path[fill=fillColor] ( 72.94, 28.73) rectangle ( 77.29, 49.00);

\path[fill=fillColor] (152.01, 28.73) rectangle (156.36, 75.36);

\path[fill=fillColor] (132.24, 28.73) rectangle (136.59, 68.53);

\path[fill=fillColor] ( 53.17, 28.73) rectangle ( 57.52, 43.08);

\path[fill=fillColor] ( 92.71, 28.73) rectangle ( 97.06, 55.63);

\path[fill=fillColor] (191.54, 28.73) rectangle (195.89, 90.88);

\path[fill=fillColor] (112.48, 28.73) rectangle (116.82, 61.74);
\end{scope}
\begin{scope}
\path[clip] (  0.00,  0.00) rectangle (209.58,115.63);
\definecolor{drawColor}{RGB}{0,0,0}

\path[draw=drawColor,line width= 0.5pt,line join=round] ( 34.42, 28.73) --
	( 34.42,113.63);

\path[draw=drawColor,line width= 0.5pt,line join=round] ( 35.55,111.66) --
	( 34.42,113.63) --
	( 33.28,111.66);
\end{scope}
\begin{scope}
\path[clip] (  0.00,  0.00) rectangle (209.58,115.63);
\definecolor{drawColor}{gray}{0.30}

\node[text=drawColor,anchor=base east,inner sep=0pt, outer sep=0pt, scale=  0.70] at ( 30.37, 31.03) {1e-03};

\node[text=drawColor,anchor=base east,inner sep=0pt, outer sep=0pt, scale=  0.70] at ( 30.37, 49.90) {1e-01};

\node[text=drawColor,anchor=base east,inner sep=0pt, outer sep=0pt, scale=  0.70] at ( 30.37, 68.77) {1e+01};

\node[text=drawColor,anchor=base east,inner sep=0pt, outer sep=0pt, scale=  0.70] at ( 30.37, 87.64) {1e+03};

\node[text=drawColor,anchor=base east,inner sep=0pt, outer sep=0pt, scale=  0.70] at ( 30.37,106.50) {1e+05};
\end{scope}
\begin{scope}
\path[clip] (  0.00,  0.00) rectangle (209.58,115.63);
\definecolor{drawColor}{gray}{0.20}

\path[draw=drawColor,line width= 0.5pt,line join=round] ( 32.17, 33.45) --
	( 34.42, 33.45);

\path[draw=drawColor,line width= 0.5pt,line join=round] ( 32.17, 52.31) --
	( 34.42, 52.31);

\path[draw=drawColor,line width= 0.5pt,line join=round] ( 32.17, 71.18) --
	( 34.42, 71.18);

\path[draw=drawColor,line width= 0.5pt,line join=round] ( 32.17, 90.05) --
	( 34.42, 90.05);

\path[draw=drawColor,line width= 0.5pt,line join=round] ( 32.17,108.92) --
	( 34.42,108.92);
\end{scope}
\begin{scope}
\path[clip] (  0.00,  0.00) rectangle (209.58,115.63);
\definecolor{drawColor}{RGB}{0,0,0}

\path[draw=drawColor,line width= 0.5pt,line join=round] ( 34.42, 28.73) --
	(203.58, 28.73);

\path[draw=drawColor,line width= 0.5pt,line join=round] (201.61, 27.59) --
	(203.58, 28.73) --
	(201.61, 29.87);
\end{scope}
\begin{scope}
\path[clip] (  0.00,  0.00) rectangle (209.58,115.63);
\definecolor{drawColor}{gray}{0.20}

\path[draw=drawColor,line width= 0.5pt,line join=round] ( 49.81, 26.48) --
	( 49.81, 28.73);

\path[draw=drawColor,line width= 0.5pt,line join=round] ( 69.58, 26.48) --
	( 69.58, 28.73);

\path[draw=drawColor,line width= 0.5pt,line join=round] ( 89.35, 26.48) --
	( 89.35, 28.73);

\path[draw=drawColor,line width= 0.5pt,line join=round] (109.12, 26.48) --
	(109.12, 28.73);

\path[draw=drawColor,line width= 0.5pt,line join=round] (128.88, 26.48) --
	(128.88, 28.73);

\path[draw=drawColor,line width= 0.5pt,line join=round] (148.65, 26.48) --
	(148.65, 28.73);

\path[draw=drawColor,line width= 0.5pt,line join=round] (168.42, 26.48) --
	(168.42, 28.73);

\path[draw=drawColor,line width= 0.5pt,line join=round] (188.18, 26.48) --
	(188.18, 28.73);
\end{scope}
\begin{scope}
\path[clip] (  0.00,  0.00) rectangle (209.58,115.63);
\definecolor{drawColor}{gray}{0.30}

\node[text=drawColor,anchor=base,inner sep=0pt, outer sep=0pt, scale=  0.70] at ( 49.81, 19.86) {2};

\node[text=drawColor,anchor=base,inner sep=0pt, outer sep=0pt, scale=  0.70] at ( 69.58, 19.86) {4};

\node[text=drawColor,anchor=base,inner sep=0pt, outer sep=0pt, scale=  0.70] at ( 89.35, 19.86) {6};

\node[text=drawColor,anchor=base,inner sep=0pt, outer sep=0pt, scale=  0.70] at (109.12, 19.86) {8};

\node[text=drawColor,anchor=base,inner sep=0pt, outer sep=0pt, scale=  0.70] at (128.88, 19.86) {10};

\node[text=drawColor,anchor=base,inner sep=0pt, outer sep=0pt, scale=  0.70] at (148.65, 19.86) {12};

\node[text=drawColor,anchor=base,inner sep=0pt, outer sep=0pt, scale=  0.70] at (168.42, 19.86) {14};

\node[text=drawColor,anchor=base,inner sep=0pt, outer sep=0pt, scale=  0.70] at (188.18, 19.86) {16};
\end{scope}
\begin{scope}
\path[clip] (  0.00,  0.00) rectangle (209.58,115.63);
\definecolor{drawColor}{RGB}{0,0,0}

\node[text=drawColor,anchor=base,inner sep=0pt, outer sep=0pt, scale=  0.80] at (119.00, 10.74) {$n$};
\end{scope}
\begin{scope}
\path[clip] (  0.00,  0.00) rectangle (209.58,115.63);
\definecolor{drawColor}{RGB}{0,0,0}

\node[text=drawColor,rotate= 90.00,anchor=base,inner sep=0pt, outer sep=0pt, scale=  0.80] at (  7.51, 71.18) {time (ms)};
\end{scope}
\begin{scope}
\path[clip] (  0.00,  0.00) rectangle (209.58,115.63);

\path[] ( 51.04,  1.00) rectangle (186.96,  7.18);
\end{scope}
\begin{scope}
\path[clip] (  0.00,  0.00) rectangle (209.58,115.63);
\definecolor{fillColor}{RGB}{78,155,133}

\path[fill=fillColor] ( 51.75,  1.71) rectangle ( 61.17,  6.47);
\end{scope}
\begin{scope}
\path[clip] (  0.00,  0.00) rectangle (209.58,115.63);
\definecolor{fillColor}{RGB}{78,155,133}

\path[fill=fillColor] ( 51.75,  1.71) rectangle ( 61.17,  6.47);
\end{scope}
\begin{scope}
\path[clip] (  0.00,  0.00) rectangle (209.58,115.63);
\definecolor{fillColor}{RGB}{78,155,133}

\path[fill=fillColor] ( 51.75,  1.71) rectangle ( 61.17,  6.47);
\end{scope}
\begin{scope}
\path[clip] (  0.00,  0.00) rectangle (209.58,115.63);
\definecolor{fillColor}{RGB}{247,192,26}

\path[fill=fillColor] (102.84,  1.71) rectangle (112.26,  6.47);
\end{scope}
\begin{scope}
\path[clip] (  0.00,  0.00) rectangle (209.58,115.63);
\definecolor{fillColor}{RGB}{247,192,26}

\path[fill=fillColor] (102.84,  1.71) rectangle (112.26,  6.47);
\end{scope}
\begin{scope}
\path[clip] (  0.00,  0.00) rectangle (209.58,115.63);
\definecolor{fillColor}{RGB}{247,192,26}

\path[fill=fillColor] (102.84,  1.71) rectangle (112.26,  6.47);
\end{scope}
\begin{scope}
\path[clip] (  0.00,  0.00) rectangle (209.58,115.63);
\definecolor{fillColor}{RGB}{37,122,164}

\path[fill=fillColor] (154.70,  1.71) rectangle (164.11,  6.47);
\end{scope}
\begin{scope}
\path[clip] (  0.00,  0.00) rectangle (209.58,115.63);
\definecolor{fillColor}{RGB}{37,122,164}

\path[fill=fillColor] (154.70,  1.71) rectangle (164.11,  6.47);
\end{scope}
\begin{scope}
\path[clip] (  0.00,  0.00) rectangle (209.58,115.63);
\definecolor{fillColor}{RGB}{37,122,164}

\path[fill=fillColor] (154.70,  1.71) rectangle (164.11,  6.47);
\end{scope}
\begin{scope}
\path[clip] (  0.00,  0.00) rectangle (209.58,115.63);
\definecolor{drawColor}{RGB}{0,0,0}

\node[text=drawColor,anchor=base west,inner sep=0pt, outer sep=0pt, scale=  0.70] at ( 61.88,  1.68) {Candidates};
\end{scope}
\begin{scope}
\path[clip] (  0.00,  0.00) rectangle (209.58,115.63);
\definecolor{drawColor}{RGB}{0,0,0}

\node[text=drawColor,anchor=base west,inner sep=0pt, outer sep=0pt, scale=  0.70] at (112.97,  1.68) {Verification};
\end{scope}
\begin{scope}
\path[clip] (  0.00,  0.00) rectangle (209.58,115.63);
\definecolor{drawColor}{RGB}{0,0,0}

\node[text=drawColor,anchor=base west,inner sep=0pt, outer sep=0pt, scale=  0.70] at (164.82,  1.68) {Total};
\end{scope}
\end{tikzpicture}

%% file: files/decor_revised_example_crv_encoding.tex
\begin{tikzpicture}
	\node (tab_f1) {
		\begin{tabular}{clc}
			\toprule
			$\#_E[Com(E)]$ & $Rev$ & $\phi_3(\#_E[Com(E)], Rev)$ \\ \midrule
			$[2,0]$        & \high & $\varphi_{1}$ \\
			$[2,0]$        & \low  & $\varphi_{2}$ \\
			$[1,1]$        & \high & $\varphi_{3}$ \\
			$[1,1]$        & \low  & $\varphi_{4}$ \\
			$[0,2]$        & \high & $\varphi_{5}$ \\
			$[0,2]$        & \low  & $\varphi_{6}$ \\ \bottomrule
		\end{tabular}
	};
\end{tikzpicture}

%% file: files/decor_revised_counterexample.tex
\begin{tikzpicture}
	\node (tab_f1) {
		\begin{tabular}{llllcc}
			\toprule
			$R_1$ & $R_2$ & $R_3$ & $R_4$ & $\phi(R_1, R_2, R_3, R_4)$ & $b$ \\ \midrule
			\high & \high & \high & \high & $\varphi_{2}$              & {\color{cborange}$[4,0]$}   \\
			\high & \high & \high & \low  & $\varphi_{2}$              & {\color{cbbluedark}$[3,1]$} \\
			\high & \high & \low  & \high & $\varphi_{2}$              & {\color{cbbluedark}$[3,1]$} \\
			\high & \high & \low  & \low  & $\varphi_{1}$              & {\color{cbpurple}$[2,2]$}   \\
			\high & \low  & \high & \high & $\varphi_{2}$              & {\color{cbbluedark}$[3,1]$} \\
			\high & \low  & \high & \low  & $\varphi_{2}$              & {\color{cbpurple}$[2,2]$}   \\
			\high & \low  & \low  & \high & $\varphi_{2}$              & {\color{cbpurple}$[2,2]$}   \\
			\high & \low  & \low  & \low  & $\varphi_{2}$              & {\color{cbblue}$[1,3]$}     \\
			\low  & \high & \high & \high & $\varphi_{2}$              & {\color{cbbluedark}$[3,1]$} \\
			\low  & \high & \high & \low  & $\varphi_{2}$              & {\color{cbpurple}$[2,2]$}   \\
			\low  & \high & \low  & \high & $\varphi_{2}$              & {\color{cbpurple}$[2,2]$}   \\
			\low  & \high & \low  & \low  & $\varphi_{2}$              & {\color{cbblue}$[1,3]$}     \\
			\low  & \low  & \high & \high & $\varphi_{1}$              & {\color{cbpurple}$[2,2]$}   \\
			\low  & \low  & \high & \low  & $\varphi_{2}$              & {\color{cbblue}$[1,3]$}     \\
			\low  & \low  & \low  & \high & $\varphi_{2}$              & {\color{cbblue}$[1,3]$}     \\
			\low  & \low  & \low  & \low  & $\varphi_{2}$              & {\color{cbgreen}$[0,4]$}    \\ \bottomrule
		\end{tabular}
	};

	\node[right = 0.2cm of tab_f1] (buckets) {
		\begin{tabular}{cc}
			\toprule
			$b$                         & $\phi^{\succ}(b)$                                                                  \\ \midrule
			{\color{cborange}$[4,0]$}   & $\langle \varphi_2 \rangle$                                                        \\
			{\color{cbbluedark}$[3,1]$} & $\langle \varphi_2, \varphi_2, \varphi_2, \varphi_2 \rangle$                       \\
			{\color{cbpurple}$[2,2]$}   & $\langle \varphi_1, \varphi_2, \varphi_2, \varphi_2, \varphi_2, \varphi_1 \rangle$ \\
			{\color{cbblue}$[1,3]$}     & $\langle \varphi_2, \varphi_2, \varphi_2, \varphi_2 \rangle$                       \\
			{\color{cbgreen}$[0,4]$}    & $\langle \varphi_2 \rangle$                                                        \\ \bottomrule
		\end{tabular}
	};
\end{tikzpicture}

%% file: files/decor_revised_plot_k=0_times.tex
\begin{tikzpicture}[x=1pt,y=1pt]
\definecolor{fillColor}{RGB}{255,255,255}
\path[use as bounding box,fill=fillColor,fill opacity=0.00] (0,0) rectangle (209.58,115.63);
\begin{scope}
\path[clip] (  0.00,  0.00) rectangle (209.58,115.63);
\definecolor{drawColor}{RGB}{255,255,255}
\definecolor{fillColor}{RGB}{255,255,255}

\path[draw=drawColor,line width= 0.5pt,line join=round,line cap=round,fill=fillColor] (  0.00,  0.00) rectangle (209.58,115.63);
\end{scope}
\begin{scope}
\path[clip] ( 34.42, 28.73) rectangle (203.58,113.63);
\definecolor{fillColor}{RGB}{255,255,255}

\path[fill=fillColor] ( 34.42, 28.73) rectangle (203.58,113.63);
\definecolor{drawColor}{RGB}{247,192,26}

\path[draw=drawColor,line width= 0.5pt,line join=round] ( 42.11, 40.77) --
	( 64.07, 44.55) --
	( 86.04, 50.43) --
	(108.01, 56.75) --
	(129.98, 62.14) --
	(151.95, 72.16) --
	(173.92, 74.65) --
	(195.89, 80.70);
\definecolor{drawColor}{RGB}{78,155,133}

\path[draw=drawColor,line width= 0.5pt,dash pattern=on 4pt off 4pt ,line join=round] ( 42.11, 40.74) --
	( 64.07, 44.53) --
	( 86.04, 50.16) --
	(108.01, 56.70) --
	(129.98, 62.12) --
	(151.95, 71.37) --
	(173.92, 74.61) --
	(195.89, 80.65);
\definecolor{drawColor}{RGB}{37,122,164}

\path[draw=drawColor,line width= 0.5pt,dash pattern=on 1pt off 3pt ,line join=round] ( 42.11, 41.02) --
	( 64.07, 52.92) --
	( 86.04, 62.52) --
	(108.01, 71.23) --
	(129.98, 79.20) --
	(151.95, 87.09) --
	(173.92, 94.88) --
	(195.89,102.88);
\definecolor{drawColor}{RGB}{86,51,94}

\path[draw=drawColor,line width= 0.5pt,dash pattern=on 7pt off 3pt ,line join=round] ( 42.11, 41.21) --
	( 64.07, 55.73) --
	( 86.04, 67.79) --
	(108.01, 79.82) --
	(129.98, 91.68) --
	(151.95,103.47);

\path[draw=drawColor,line width= 0.4pt,line join=round,line cap=round] ( 40.57, 39.67) -- ( 43.64, 42.74);

\path[draw=drawColor,line width= 0.4pt,line join=round,line cap=round] ( 40.57, 42.74) -- ( 43.64, 39.67);

\path[draw=drawColor,line width= 0.4pt,line join=round,line cap=round] ( 39.94, 41.21) -- ( 44.27, 41.21);

\path[draw=drawColor,line width= 0.4pt,line join=round,line cap=round] ( 42.11, 39.04) -- ( 42.11, 43.37);
\definecolor{fillColor}{RGB}{37,122,164}

\path[fill=fillColor] (128.45, 77.67) --
	(131.52, 77.67) --
	(131.52, 80.74) --
	(128.45, 80.74) --
	cycle;
\definecolor{fillColor}{RGB}{78,155,133}

\path[fill=fillColor] ( 42.11, 43.12) --
	( 44.17, 39.54) --
	( 40.04, 39.54) --
	cycle;

\path[fill=fillColor] (173.92, 76.99) --
	(175.99, 73.42) --
	(171.86, 73.42) --
	cycle;

\path[fill=fillColor] (129.98, 64.50) --
	(132.05, 60.93) --
	(127.92, 60.93) --
	cycle;
\definecolor{drawColor}{RGB}{247,192,26}
\definecolor{fillColor}{RGB}{247,192,26}

\path[draw=drawColor,line width= 0.4pt,line join=round,line cap=round,fill=fillColor] (151.95, 72.16) circle (  1.53);
\definecolor{fillColor}{RGB}{37,122,164}

\path[fill=fillColor] (194.36,101.35) --
	(197.43,101.35) --
	(197.43,104.41) --
	(194.36,104.41) --
	cycle;
\definecolor{drawColor}{RGB}{86,51,94}

\path[draw=drawColor,line width= 0.4pt,line join=round,line cap=round] (106.48, 78.29) -- (109.55, 81.35);

\path[draw=drawColor,line width= 0.4pt,line join=round,line cap=round] (106.48, 81.35) -- (109.55, 78.29);

\path[draw=drawColor,line width= 0.4pt,line join=round,line cap=round] (105.85, 79.82) -- (110.18, 79.82);

\path[draw=drawColor,line width= 0.4pt,line join=round,line cap=round] (108.01, 77.65) -- (108.01, 81.99);

\path[fill=fillColor] (106.48, 69.70) --
	(109.55, 69.70) --
	(109.55, 72.77) --
	(106.48, 72.77) --
	cycle;

\path[draw=drawColor,line width= 0.4pt,line join=round,line cap=round] (128.45, 90.15) -- (131.52, 93.22);

\path[draw=drawColor,line width= 0.4pt,line join=round,line cap=round] (128.45, 93.22) -- (131.52, 90.15);

\path[draw=drawColor,line width= 0.4pt,line join=round,line cap=round] (127.82, 91.68) -- (132.15, 91.68);

\path[draw=drawColor,line width= 0.4pt,line join=round,line cap=round] (129.98, 89.51) -- (129.98, 93.85);

\path[fill=fillColor] ( 40.57, 39.49) --
	( 43.64, 39.49) --
	( 43.64, 42.56) --
	( 40.57, 42.56) --
	cycle;
\definecolor{fillColor}{RGB}{78,155,133}

\path[fill=fillColor] ( 86.04, 52.55) --
	( 88.11, 48.97) --
	( 83.98, 48.97) --
	cycle;
\definecolor{fillColor}{RGB}{37,122,164}

\path[fill=fillColor] (172.39, 93.35) --
	(175.46, 93.35) --
	(175.46, 96.42) --
	(172.39, 96.42) --
	cycle;

\path[draw=drawColor,line width= 0.4pt,line join=round,line cap=round] ( 84.51, 66.26) -- ( 87.58, 69.33);

\path[draw=drawColor,line width= 0.4pt,line join=round,line cap=round] ( 84.51, 69.33) -- ( 87.58, 66.26);

\path[draw=drawColor,line width= 0.4pt,line join=round,line cap=round] ( 83.88, 67.79) -- ( 88.21, 67.79);

\path[draw=drawColor,line width= 0.4pt,line join=round,line cap=round] ( 86.04, 65.63) -- ( 86.04, 69.96);
\definecolor{drawColor}{RGB}{247,192,26}
\definecolor{fillColor}{RGB}{247,192,26}

\path[draw=drawColor,line width= 0.4pt,line join=round,line cap=round,fill=fillColor] (195.89, 80.70) circle (  1.53);
\definecolor{drawColor}{RGB}{86,51,94}

\path[draw=drawColor,line width= 0.4pt,line join=round,line cap=round] (150.42,101.94) -- (153.49,105.01);

\path[draw=drawColor,line width= 0.4pt,line join=round,line cap=round] (150.42,105.01) -- (153.49,101.94);

\path[draw=drawColor,line width= 0.4pt,line join=round,line cap=round] (149.79,103.47) -- (154.12,103.47);

\path[draw=drawColor,line width= 0.4pt,line join=round,line cap=round] (151.95,101.30) -- (151.95,105.64);
\definecolor{fillColor}{RGB}{37,122,164}

\path[fill=fillColor] ( 84.51, 60.98) --
	( 87.58, 60.98) --
	( 87.58, 64.05) --
	( 84.51, 64.05) --
	cycle;
\definecolor{drawColor}{RGB}{247,192,26}
\definecolor{fillColor}{RGB}{247,192,26}

\path[draw=drawColor,line width= 0.4pt,line join=round,line cap=round,fill=fillColor] ( 64.07, 44.55) circle (  1.53);

\path[draw=drawColor,line width= 0.4pt,line join=round,line cap=round,fill=fillColor] ( 86.04, 50.43) circle (  1.53);

\path[draw=drawColor,line width= 0.4pt,line join=round,line cap=round,fill=fillColor] (173.92, 74.65) circle (  1.53);

\path[draw=drawColor,line width= 0.4pt,line join=round,line cap=round,fill=fillColor] ( 42.11, 40.77) circle (  1.53);
\definecolor{fillColor}{RGB}{78,155,133}

\path[fill=fillColor] (108.01, 59.08) --
	(110.08, 55.51) --
	(105.95, 55.51) --
	cycle;
\definecolor{fillColor}{RGB}{37,122,164}

\path[fill=fillColor] (150.42, 85.56) --
	(153.49, 85.56) --
	(153.49, 88.63) --
	(150.42, 88.63) --
	cycle;
\definecolor{fillColor}{RGB}{78,155,133}

\path[fill=fillColor] (195.89, 83.04) --
	(197.96, 79.46) --
	(193.83, 79.46) --
	cycle;
\definecolor{drawColor}{RGB}{86,51,94}

\path[draw=drawColor,line width= 0.4pt,line join=round,line cap=round] ( 62.54, 54.20) -- ( 65.61, 57.26);

\path[draw=drawColor,line width= 0.4pt,line join=round,line cap=round] ( 62.54, 57.26) -- ( 65.61, 54.20);

\path[draw=drawColor,line width= 0.4pt,line join=round,line cap=round] ( 61.91, 55.73) -- ( 66.24, 55.73);

\path[draw=drawColor,line width= 0.4pt,line join=round,line cap=round] ( 64.07, 53.56) -- ( 64.07, 57.90);
\definecolor{drawColor}{RGB}{247,192,26}
\definecolor{fillColor}{RGB}{247,192,26}

\path[draw=drawColor,line width= 0.4pt,line join=round,line cap=round,fill=fillColor] (129.98, 62.14) circle (  1.53);
\definecolor{fillColor}{RGB}{78,155,133}

\path[fill=fillColor] ( 64.07, 46.91) --
	( 66.14, 43.33) --
	( 62.01, 43.33) --
	cycle;
\definecolor{fillColor}{RGB}{247,192,26}

\path[draw=drawColor,line width= 0.4pt,line join=round,line cap=round,fill=fillColor] (108.01, 56.75) circle (  1.53);
\definecolor{fillColor}{RGB}{78,155,133}

\path[fill=fillColor] (151.95, 73.76) --
	(154.02, 70.18) --
	(149.89, 70.18) --
	cycle;
\definecolor{fillColor}{RGB}{37,122,164}

\path[fill=fillColor] ( 62.54, 51.39) --
	( 65.61, 51.39) --
	( 65.61, 54.45) --
	( 62.54, 54.45) --
	cycle;
\end{scope}
\begin{scope}
\path[clip] (  0.00,  0.00) rectangle (209.58,115.63);
\definecolor{drawColor}{RGB}{0,0,0}

\path[draw=drawColor,line width= 0.5pt,line join=round] ( 34.42, 28.73) --
	( 34.42,113.63);

\path[draw=drawColor,line width= 0.5pt,line join=round] ( 35.55,111.66) --
	( 34.42,113.63) --
	( 33.28,111.66);
\end{scope}
\begin{scope}
\path[clip] (  0.00,  0.00) rectangle (209.58,115.63);
\definecolor{drawColor}{gray}{0.30}

\node[text=drawColor,anchor=base east,inner sep=0pt, outer sep=0pt, scale=  0.70] at ( 30.37, 31.03) {1e-03};

\node[text=drawColor,anchor=base east,inner sep=0pt, outer sep=0pt, scale=  0.70] at ( 30.37, 49.90) {1e-01};

\node[text=drawColor,anchor=base east,inner sep=0pt, outer sep=0pt, scale=  0.70] at ( 30.37, 68.77) {1e+01};

\node[text=drawColor,anchor=base east,inner sep=0pt, outer sep=0pt, scale=  0.70] at ( 30.37, 87.64) {1e+03};

\node[text=drawColor,anchor=base east,inner sep=0pt, outer sep=0pt, scale=  0.70] at ( 30.37,106.50) {1e+05};
\end{scope}
\begin{scope}
\path[clip] (  0.00,  0.00) rectangle (209.58,115.63);
\definecolor{drawColor}{gray}{0.20}

\path[draw=drawColor,line width= 0.5pt,line join=round] ( 32.17, 33.45) --
	( 34.42, 33.45);

\path[draw=drawColor,line width= 0.5pt,line join=round] ( 32.17, 52.31) --
	( 34.42, 52.31);

\path[draw=drawColor,line width= 0.5pt,line join=round] ( 32.17, 71.18) --
	( 34.42, 71.18);

\path[draw=drawColor,line width= 0.5pt,line join=round] ( 32.17, 90.05) --
	( 34.42, 90.05);

\path[draw=drawColor,line width= 0.5pt,line join=round] ( 32.17,108.92) --
	( 34.42,108.92);
\end{scope}
\begin{scope}
\path[clip] (  0.00,  0.00) rectangle (209.58,115.63);
\definecolor{drawColor}{RGB}{0,0,0}

\path[draw=drawColor,line width= 0.5pt,line join=round] ( 34.42, 28.73) --
	(203.58, 28.73);

\path[draw=drawColor,line width= 0.5pt,line join=round] (201.61, 27.59) --
	(203.58, 28.73) --
	(201.61, 29.87);
\end{scope}
\begin{scope}
\path[clip] (  0.00,  0.00) rectangle (209.58,115.63);
\definecolor{drawColor}{gray}{0.20}

\path[draw=drawColor,line width= 0.5pt,line join=round] ( 64.07, 26.48) --
	( 64.07, 28.73);

\path[draw=drawColor,line width= 0.5pt,line join=round] (108.01, 26.48) --
	(108.01, 28.73);

\path[draw=drawColor,line width= 0.5pt,line join=round] (151.95, 26.48) --
	(151.95, 28.73);

\path[draw=drawColor,line width= 0.5pt,line join=round] (195.89, 26.48) --
	(195.89, 28.73);
\end{scope}
\begin{scope}
\path[clip] (  0.00,  0.00) rectangle (209.58,115.63);
\definecolor{drawColor}{gray}{0.30}

\node[text=drawColor,anchor=base,inner sep=0pt, outer sep=0pt, scale=  0.70] at ( 64.07, 19.86) {4};

\node[text=drawColor,anchor=base,inner sep=0pt, outer sep=0pt, scale=  0.70] at (108.01, 19.86) {8};

\node[text=drawColor,anchor=base,inner sep=0pt, outer sep=0pt, scale=  0.70] at (151.95, 19.86) {12};

\node[text=drawColor,anchor=base,inner sep=0pt, outer sep=0pt, scale=  0.70] at (195.89, 19.86) {16};
\end{scope}
\begin{scope}
\path[clip] (  0.00,  0.00) rectangle (209.58,115.63);
\definecolor{drawColor}{RGB}{0,0,0}

\node[text=drawColor,anchor=base,inner sep=0pt, outer sep=0pt, scale=  0.80] at (119.00, 10.74) {$n$};
\end{scope}
\begin{scope}
\path[clip] (  0.00,  0.00) rectangle (209.58,115.63);
\definecolor{drawColor}{RGB}{0,0,0}

\node[text=drawColor,rotate= 90.00,anchor=base,inner sep=0pt, outer sep=0pt, scale=  0.80] at (  7.51, 71.18) {time (ms)};
\end{scope}
\begin{scope}
\path[clip] (  0.00,  0.00) rectangle (209.58,115.63);

\path[] ( 34.42,  1.00) rectangle (210.27,  7.18);
\end{scope}
\begin{scope}
\path[clip] (  0.00,  0.00) rectangle (209.58,115.63);
\definecolor{drawColor}{RGB}{247,192,26}

\path[draw=drawColor,line width= 0.5pt,line join=round] ( 17.50,  4.09) -- ( 26.17,  4.09);
\end{scope}
\begin{scope}
\path[clip] (  0.00,  0.00) rectangle (209.58,115.63);
\definecolor{drawColor}{RGB}{247,192,26}
\definecolor{fillColor}{RGB}{247,192,26}

\path[draw=drawColor,line width= 0.4pt,line join=round,line cap=round,fill=fillColor] ( 21.84,  4.09) circle (  1.53);
\end{scope}
\begin{scope}
\path[clip] (  0.00,  0.00) rectangle (209.58,115.63);
\definecolor{drawColor}{RGB}{78,155,133}

\path[draw=drawColor,line width= 0.5pt,dash pattern=on 4pt off 4pt ,line join=round] ( 60.10,  4.09) -- ( 68.77,  4.09);
\end{scope}
\begin{scope}
\path[clip] (  0.00,  0.00) rectangle (209.58,115.63);
\definecolor{fillColor}{RGB}{78,155,133}

\path[fill=fillColor] ( 64.43,  6.48) --
	( 66.50,  2.90) --
	( 62.37,  2.90) --
	cycle;
\end{scope}
\begin{scope}
\path[clip] (  0.00,  0.00) rectangle (209.58,115.63);
\definecolor{drawColor}{RGB}{37,122,164}

\path[draw=drawColor,line width= 0.5pt,dash pattern=on 1pt off 3pt ,line join=round] (108.14,  4.09) -- (116.81,  4.09);
\end{scope}
\begin{scope}
\path[clip] (  0.00,  0.00) rectangle (209.58,115.63);
\definecolor{fillColor}{RGB}{37,122,164}

\path[fill=fillColor] (110.94,  2.56) --
	(114.01,  2.56) --
	(114.01,  5.62) --
	(110.94,  5.62) --
	cycle;
\end{scope}
\begin{scope}
\path[clip] (  0.00,  0.00) rectangle (209.58,115.63);
\definecolor{drawColor}{RGB}{86,51,94}

\path[draw=drawColor,line width= 0.5pt,dash pattern=on 7pt off 3pt ,line join=round] (158.32,  4.09) -- (166.99,  4.09);
\end{scope}
\begin{scope}
\path[clip] (  0.00,  0.00) rectangle (209.58,115.63);
\definecolor{drawColor}{RGB}{86,51,94}

\path[draw=drawColor,line width= 0.4pt,line join=round,line cap=round] (161.12,  2.56) -- (164.19,  5.62);

\path[draw=drawColor,line width= 0.4pt,line join=round,line cap=round] (161.12,  5.62) -- (164.19,  2.56);

\path[draw=drawColor,line width= 0.4pt,line join=round,line cap=round] (160.49,  4.09) -- (164.82,  4.09);

\path[draw=drawColor,line width= 0.4pt,line join=round,line cap=round] (162.65,  1.92) -- (162.65,  6.26);
\end{scope}
\begin{scope}
\path[clip] (  0.00,  0.00) rectangle (209.58,115.63);
\definecolor{drawColor}{RGB}{0,0,0}

\node[text=drawColor,anchor=base west,inner sep=0pt, outer sep=0pt, scale=  0.70] at ( 27.26,  1.68) {DECOR};
\end{scope}
\begin{scope}
\path[clip] (  0.00,  0.00) rectangle (209.58,115.63);
\definecolor{drawColor}{RGB}{0,0,0}

\node[text=drawColor,anchor=base west,inner sep=0pt, outer sep=0pt, scale=  0.70] at ( 69.85,  1.68) {DECOR+};
\end{scope}
\begin{scope}
\path[clip] (  0.00,  0.00) rectangle (209.58,115.63);
\definecolor{drawColor}{RGB}{0,0,0}

\node[text=drawColor,anchor=base west,inner sep=0pt, outer sep=0pt, scale=  0.70] at (117.90,  1.68) {A-DECOR};
\end{scope}
\begin{scope}
\path[clip] (  0.00,  0.00) rectangle (209.58,115.63);
\definecolor{drawColor}{RGB}{0,0,0}

\node[text=drawColor,anchor=base west,inner sep=0pt, outer sep=0pt, scale=  0.70] at (168.07,  1.68) {Brute-Force};
\end{scope}
\end{tikzpicture}

%% file: files/decor_revised_plot_k=0_phases_decorplus.tex
\begin{tikzpicture}[x=1pt,y=1pt]
\definecolor{fillColor}{RGB}{255,255,255}
\path[use as bounding box,fill=fillColor,fill opacity=0.00] (0,0) rectangle (209.58,115.63);
\begin{scope}
\path[clip] (  0.00,  0.00) rectangle (209.58,115.63);
\definecolor{drawColor}{RGB}{255,255,255}
\definecolor{fillColor}{RGB}{255,255,255}

\path[draw=drawColor,line width= 0.5pt,line join=round,line cap=round,fill=fillColor] (  0.00,  0.00) rectangle (209.58,115.63);
\end{scope}
\begin{scope}
\path[clip] ( 34.42, 28.73) rectangle (203.58,113.63);
\definecolor{fillColor}{RGB}{255,255,255}

\path[fill=fillColor] ( 34.42, 28.73) rectangle (203.58,113.63);
\definecolor{fillColor}{RGB}{78,155,133}

\path[fill=fillColor] ( 42.11, 28.73) rectangle ( 46.45, 40.73);

\path[fill=fillColor] (160.71, 28.73) rectangle (165.06, 74.61);

\path[fill=fillColor] (121.17, 28.73) rectangle (125.52, 62.12);

\path[fill=fillColor] ( 81.64, 28.73) rectangle ( 85.99, 50.16);

\path[fill=fillColor] (101.41, 28.73) rectangle (105.76, 56.70);

\path[fill=fillColor] (180.48, 28.73) rectangle (184.82, 80.65);

\path[fill=fillColor] ( 61.87, 28.73) rectangle ( 66.22, 44.52);

\path[fill=fillColor] (140.94, 28.73) rectangle (145.29, 71.37);
\definecolor{fillColor}{RGB}{37,122,164}

\path[fill=fillColor] ( 53.17, 28.73) rectangle ( 57.52, 40.74);

\path[fill=fillColor] (171.78, 28.73) rectangle (176.13, 74.61);

\path[fill=fillColor] (132.24, 28.73) rectangle (136.59, 62.12);

\path[fill=fillColor] ( 92.71, 28.73) rectangle ( 97.06, 50.16);

\path[fill=fillColor] (112.48, 28.73) rectangle (116.82, 56.70);

\path[fill=fillColor] (191.54, 28.73) rectangle (195.89, 80.65);

\path[fill=fillColor] ( 72.94, 28.73) rectangle ( 77.29, 44.53);

\path[fill=fillColor] (152.01, 28.73) rectangle (156.36, 71.37);
\end{scope}
\begin{scope}
\path[clip] (  0.00,  0.00) rectangle (209.58,115.63);
\definecolor{drawColor}{RGB}{0,0,0}

\path[draw=drawColor,line width= 0.5pt,line join=round] ( 34.42, 28.73) --
	( 34.42,113.63);

\path[draw=drawColor,line width= 0.5pt,line join=round] ( 35.55,111.66) --
	( 34.42,113.63) --
	( 33.28,111.66);
\end{scope}
\begin{scope}
\path[clip] (  0.00,  0.00) rectangle (209.58,115.63);
\definecolor{drawColor}{gray}{0.30}

\node[text=drawColor,anchor=base east,inner sep=0pt, outer sep=0pt, scale=  0.70] at ( 30.37, 31.03) {1e-03};

\node[text=drawColor,anchor=base east,inner sep=0pt, outer sep=0pt, scale=  0.70] at ( 30.37, 49.90) {1e-01};

\node[text=drawColor,anchor=base east,inner sep=0pt, outer sep=0pt, scale=  0.70] at ( 30.37, 68.77) {1e+01};

\node[text=drawColor,anchor=base east,inner sep=0pt, outer sep=0pt, scale=  0.70] at ( 30.37, 87.64) {1e+03};

\node[text=drawColor,anchor=base east,inner sep=0pt, outer sep=0pt, scale=  0.70] at ( 30.37,106.50) {1e+05};
\end{scope}
\begin{scope}
\path[clip] (  0.00,  0.00) rectangle (209.58,115.63);
\definecolor{drawColor}{gray}{0.20}

\path[draw=drawColor,line width= 0.5pt,line join=round] ( 32.17, 33.45) --
	( 34.42, 33.45);

\path[draw=drawColor,line width= 0.5pt,line join=round] ( 32.17, 52.31) --
	( 34.42, 52.31);

\path[draw=drawColor,line width= 0.5pt,line join=round] ( 32.17, 71.18) --
	( 34.42, 71.18);

\path[draw=drawColor,line width= 0.5pt,line join=round] ( 32.17, 90.05) --
	( 34.42, 90.05);

\path[draw=drawColor,line width= 0.5pt,line join=round] ( 32.17,108.92) --
	( 34.42,108.92);
\end{scope}
\begin{scope}
\path[clip] (  0.00,  0.00) rectangle (209.58,115.63);
\definecolor{drawColor}{RGB}{0,0,0}

\path[draw=drawColor,line width= 0.5pt,line join=round] ( 34.42, 28.73) --
	(203.58, 28.73);

\path[draw=drawColor,line width= 0.5pt,line join=round] (201.61, 27.59) --
	(203.58, 28.73) --
	(201.61, 29.87);
\end{scope}
\begin{scope}
\path[clip] (  0.00,  0.00) rectangle (209.58,115.63);
\definecolor{drawColor}{gray}{0.20}

\path[draw=drawColor,line width= 0.5pt,line join=round] ( 49.81, 26.48) --
	( 49.81, 28.73);

\path[draw=drawColor,line width= 0.5pt,line join=round] ( 69.58, 26.48) --
	( 69.58, 28.73);

\path[draw=drawColor,line width= 0.5pt,line join=round] ( 89.35, 26.48) --
	( 89.35, 28.73);

\path[draw=drawColor,line width= 0.5pt,line join=round] (109.12, 26.48) --
	(109.12, 28.73);

\path[draw=drawColor,line width= 0.5pt,line join=round] (128.88, 26.48) --
	(128.88, 28.73);

\path[draw=drawColor,line width= 0.5pt,line join=round] (148.65, 26.48) --
	(148.65, 28.73);

\path[draw=drawColor,line width= 0.5pt,line join=round] (168.42, 26.48) --
	(168.42, 28.73);

\path[draw=drawColor,line width= 0.5pt,line join=round] (188.18, 26.48) --
	(188.18, 28.73);
\end{scope}
\begin{scope}
\path[clip] (  0.00,  0.00) rectangle (209.58,115.63);
\definecolor{drawColor}{gray}{0.30}

\node[text=drawColor,anchor=base,inner sep=0pt, outer sep=0pt, scale=  0.70] at ( 49.81, 19.86) {2};

\node[text=drawColor,anchor=base,inner sep=0pt, outer sep=0pt, scale=  0.70] at ( 69.58, 19.86) {4};

\node[text=drawColor,anchor=base,inner sep=0pt, outer sep=0pt, scale=  0.70] at ( 89.35, 19.86) {6};

\node[text=drawColor,anchor=base,inner sep=0pt, outer sep=0pt, scale=  0.70] at (109.12, 19.86) {8};

\node[text=drawColor,anchor=base,inner sep=0pt, outer sep=0pt, scale=  0.70] at (128.88, 19.86) {10};

\node[text=drawColor,anchor=base,inner sep=0pt, outer sep=0pt, scale=  0.70] at (148.65, 19.86) {12};

\node[text=drawColor,anchor=base,inner sep=0pt, outer sep=0pt, scale=  0.70] at (168.42, 19.86) {14};

\node[text=drawColor,anchor=base,inner sep=0pt, outer sep=0pt, scale=  0.70] at (188.18, 19.86) {16};
\end{scope}
\begin{scope}
\path[clip] (  0.00,  0.00) rectangle (209.58,115.63);
\definecolor{drawColor}{RGB}{0,0,0}

\node[text=drawColor,anchor=base,inner sep=0pt, outer sep=0pt, scale=  0.80] at (119.00, 10.74) {$n$};
\end{scope}
\begin{scope}
\path[clip] (  0.00,  0.00) rectangle (209.58,115.63);
\definecolor{drawColor}{RGB}{0,0,0}

\node[text=drawColor,rotate= 90.00,anchor=base,inner sep=0pt, outer sep=0pt, scale=  0.80] at (  7.51, 71.18) {time (ms)};
\end{scope}
\begin{scope}
\path[clip] (  0.00,  0.00) rectangle (209.58,115.63);

\path[] ( 76.97,  1.00) rectangle (161.03,  7.18);
\end{scope}
\begin{scope}
\path[clip] (  0.00,  0.00) rectangle (209.58,115.63);
\definecolor{fillColor}{RGB}{78,155,133}

\path[fill=fillColor] ( 77.68,  1.71) rectangle ( 87.09,  6.47);
\end{scope}
\begin{scope}
\path[clip] (  0.00,  0.00) rectangle (209.58,115.63);
\definecolor{fillColor}{RGB}{78,155,133}

\path[fill=fillColor] ( 77.68,  1.71) rectangle ( 87.09,  6.47);
\end{scope}
\begin{scope}
\path[clip] (  0.00,  0.00) rectangle (209.58,115.63);
\definecolor{fillColor}{RGB}{78,155,133}

\path[fill=fillColor] ( 77.68,  1.71) rectangle ( 87.09,  6.47);
\end{scope}
\begin{scope}
\path[clip] (  0.00,  0.00) rectangle (209.58,115.63);
\definecolor{fillColor}{RGB}{37,122,164}

\path[fill=fillColor] (128.77,  1.71) rectangle (138.19,  6.47);
\end{scope}
\begin{scope}
\path[clip] (  0.00,  0.00) rectangle (209.58,115.63);
\definecolor{fillColor}{RGB}{37,122,164}

\path[fill=fillColor] (128.77,  1.71) rectangle (138.19,  6.47);
\end{scope}
\begin{scope}
\path[clip] (  0.00,  0.00) rectangle (209.58,115.63);
\definecolor{fillColor}{RGB}{37,122,164}

\path[fill=fillColor] (128.77,  1.71) rectangle (138.19,  6.47);
\end{scope}
\begin{scope}
\path[clip] (  0.00,  0.00) rectangle (209.58,115.63);
\definecolor{drawColor}{RGB}{0,0,0}

\node[text=drawColor,anchor=base west,inner sep=0pt, outer sep=0pt, scale=  0.70] at ( 87.81,  1.68) {Candidates};
\end{scope}
\begin{scope}
\path[clip] (  0.00,  0.00) rectangle (209.58,115.63);
\definecolor{drawColor}{RGB}{0,0,0}

\node[text=drawColor,anchor=base west,inner sep=0pt, outer sep=0pt, scale=  0.70] at (138.90,  1.68) {Total};
\end{scope}
\end{tikzpicture}

%% file: files/decor_revised_plot_k=2_times.tex
\begin{tikzpicture}[x=1pt,y=1pt]
\definecolor{fillColor}{RGB}{255,255,255}
\path[use as bounding box,fill=fillColor,fill opacity=0.00] (0,0) rectangle (209.58,115.63);
\begin{scope}
\path[clip] (  0.00,  0.00) rectangle (209.58,115.63);
\definecolor{drawColor}{RGB}{255,255,255}
\definecolor{fillColor}{RGB}{255,255,255}

\path[draw=drawColor,line width= 0.5pt,line join=round,line cap=round,fill=fillColor] (  0.00,  0.00) rectangle (209.58,115.63);
\end{scope}
\begin{scope}
\path[clip] ( 34.42, 28.73) rectangle (203.58,113.63);
\definecolor{fillColor}{RGB}{255,255,255}

\path[fill=fillColor] ( 34.42, 28.73) rectangle (203.58,113.63);
\definecolor{drawColor}{RGB}{247,192,26}

\path[draw=drawColor,line width= 0.5pt,line join=round] ( 42.11, 41.79) --
	( 64.07, 46.82) --
	( 86.04, 53.25) --
	(108.01, 59.64) --
	(129.98, 67.07) --
	(151.95, 75.52) --
	(173.92, 85.31) --
	(195.89, 95.79);
\definecolor{drawColor}{RGB}{78,155,133}

\path[draw=drawColor,line width= 0.5pt,dash pattern=on 4pt off 4pt ,line join=round] ( 42.11, 44.56) --
	( 64.07, 49.39) --
	( 86.04, 55.55) --
	(108.01, 61.69) --
	(129.98, 68.76) --
	(151.95, 76.60) --
	(173.92, 85.83) --
	(195.89, 95.96);
\definecolor{drawColor}{RGB}{37,122,164}

\path[draw=drawColor,line width= 0.5pt,dash pattern=on 1pt off 3pt ,line join=round] ( 42.11, 41.24) --
	( 64.07, 53.15) --
	( 86.04, 62.50) --
	(108.01, 71.01) --
	(129.98, 79.24) --
	(151.95, 87.16) --
	(173.92, 94.90) --
	(195.89,102.88);
\definecolor{drawColor}{RGB}{86,51,94}

\path[draw=drawColor,line width= 0.5pt,dash pattern=on 7pt off 3pt ,line join=round] ( 42.11, 41.19) --
	( 64.07, 54.82) --
	( 86.04, 67.48) --
	(108.01, 79.76) --
	(129.98, 91.66) --
	(151.95,103.59);
\definecolor{fillColor}{RGB}{78,155,133}

\path[fill=fillColor] (108.01, 64.07) --
	(110.08, 60.50) --
	(105.95, 60.50) --
	cycle;

\path[draw=drawColor,line width= 0.4pt,line join=round,line cap=round] (150.42,102.05) -- (153.49,105.12);

\path[draw=drawColor,line width= 0.4pt,line join=round,line cap=round] (150.42,105.12) -- (153.49,102.05);

\path[draw=drawColor,line width= 0.4pt,line join=round,line cap=round] (149.79,103.59) -- (154.12,103.59);

\path[draw=drawColor,line width= 0.4pt,line join=round,line cap=round] (151.95,101.42) -- (151.95,105.75);

\path[fill=fillColor] ( 64.07, 51.78) --
	( 66.14, 48.20) --
	( 62.01, 48.20) --
	cycle;

\path[fill=fillColor] (173.92, 88.21) --
	(175.99, 84.63) --
	(171.86, 84.63) --
	cycle;

\path[draw=drawColor,line width= 0.4pt,line join=round,line cap=round] ( 62.54, 53.29) -- ( 65.61, 56.35);

\path[draw=drawColor,line width= 0.4pt,line join=round,line cap=round] ( 62.54, 56.35) -- ( 65.61, 53.29);

\path[draw=drawColor,line width= 0.4pt,line join=round,line cap=round] ( 61.91, 54.82) -- ( 66.24, 54.82);

\path[draw=drawColor,line width= 0.4pt,line join=round,line cap=round] ( 64.07, 52.65) -- ( 64.07, 56.99);
\definecolor{fillColor}{RGB}{37,122,164}

\path[fill=fillColor] ( 40.57, 39.71) --
	( 43.64, 39.71) --
	( 43.64, 42.77) --
	( 40.57, 42.77) --
	cycle;

\path[draw=drawColor,line width= 0.4pt,line join=round,line cap=round] (106.48, 78.23) -- (109.55, 81.30);

\path[draw=drawColor,line width= 0.4pt,line join=round,line cap=round] (106.48, 81.30) -- (109.55, 78.23);

\path[draw=drawColor,line width= 0.4pt,line join=round,line cap=round] (105.85, 79.76) -- (110.18, 79.76);

\path[draw=drawColor,line width= 0.4pt,line join=round,line cap=round] (108.01, 77.59) -- (108.01, 81.93);

\path[fill=fillColor] (128.45, 77.71) --
	(131.52, 77.71) --
	(131.52, 80.78) --
	(128.45, 80.78) --
	cycle;
\definecolor{fillColor}{RGB}{78,155,133}

\path[fill=fillColor] ( 42.11, 46.94) --
	( 44.17, 43.36) --
	( 40.04, 43.36) --
	cycle;
\definecolor{drawColor}{RGB}{247,192,26}
\definecolor{fillColor}{RGB}{247,192,26}

\path[draw=drawColor,line width= 0.4pt,line join=round,line cap=round,fill=fillColor] (108.01, 59.64) circle (  1.53);
\definecolor{fillColor}{RGB}{37,122,164}

\path[fill=fillColor] (194.36,101.35) --
	(197.43,101.35) --
	(197.43,104.41) --
	(194.36,104.41) --
	cycle;

\path[fill=fillColor] ( 62.54, 51.61) --
	( 65.61, 51.61) --
	( 65.61, 54.68) --
	( 62.54, 54.68) --
	cycle;
\definecolor{fillColor}{RGB}{247,192,26}

\path[draw=drawColor,line width= 0.4pt,line join=round,line cap=round,fill=fillColor] ( 86.04, 53.25) circle (  1.53);
\definecolor{fillColor}{RGB}{37,122,164}

\path[fill=fillColor] (150.42, 85.62) --
	(153.49, 85.62) --
	(153.49, 88.69) --
	(150.42, 88.69) --
	cycle;
\definecolor{fillColor}{RGB}{78,155,133}

\path[fill=fillColor] (151.95, 78.98) --
	(154.02, 75.40) --
	(149.89, 75.40) --
	cycle;

\path[fill=fillColor] ( 86.04, 57.93) --
	( 88.11, 54.35) --
	( 83.98, 54.35) --
	cycle;

\path[fill=fillColor] (195.89, 98.35) --
	(197.96, 94.77) --
	(193.83, 94.77) --
	cycle;
\definecolor{fillColor}{RGB}{247,192,26}

\path[draw=drawColor,line width= 0.4pt,line join=round,line cap=round,fill=fillColor] (173.92, 85.31) circle (  1.53);
\definecolor{fillColor}{RGB}{37,122,164}

\path[fill=fillColor] (172.39, 93.36) --
	(175.46, 93.36) --
	(175.46, 96.43) --
	(172.39, 96.43) --
	cycle;
\definecolor{fillColor}{RGB}{247,192,26}

\path[draw=drawColor,line width= 0.4pt,line join=round,line cap=round,fill=fillColor] (129.98, 67.07) circle (  1.53);
\definecolor{fillColor}{RGB}{37,122,164}

\path[fill=fillColor] ( 84.51, 60.97) --
	( 87.58, 60.97) --
	( 87.58, 64.04) --
	( 84.51, 64.04) --
	cycle;
\definecolor{drawColor}{RGB}{86,51,94}

\path[draw=drawColor,line width= 0.4pt,line join=round,line cap=round] ( 84.51, 65.94) -- ( 87.58, 69.01);

\path[draw=drawColor,line width= 0.4pt,line join=round,line cap=round] ( 84.51, 69.01) -- ( 87.58, 65.94);

\path[draw=drawColor,line width= 0.4pt,line join=round,line cap=round] ( 83.88, 67.48) -- ( 88.21, 67.48);

\path[draw=drawColor,line width= 0.4pt,line join=round,line cap=round] ( 86.04, 65.31) -- ( 86.04, 69.65);
\definecolor{drawColor}{RGB}{247,192,26}
\definecolor{fillColor}{RGB}{247,192,26}

\path[draw=drawColor,line width= 0.4pt,line join=round,line cap=round,fill=fillColor] (151.95, 75.52) circle (  1.53);

\path[draw=drawColor,line width= 0.4pt,line join=round,line cap=round,fill=fillColor] (195.89, 95.79) circle (  1.53);
\definecolor{drawColor}{RGB}{86,51,94}

\path[draw=drawColor,line width= 0.4pt,line join=round,line cap=round] (128.45, 90.13) -- (131.52, 93.20);

\path[draw=drawColor,line width= 0.4pt,line join=round,line cap=round] (128.45, 93.20) -- (131.52, 90.13);

\path[draw=drawColor,line width= 0.4pt,line join=round,line cap=round] (127.82, 91.66) -- (132.15, 91.66);

\path[draw=drawColor,line width= 0.4pt,line join=round,line cap=round] (129.98, 89.50) -- (129.98, 93.83);
\definecolor{fillColor}{RGB}{78,155,133}

\path[fill=fillColor] (129.98, 71.14) --
	(132.05, 67.57) --
	(127.92, 67.57) --
	cycle;
\definecolor{fillColor}{RGB}{37,122,164}

\path[fill=fillColor] (106.48, 69.48) --
	(109.55, 69.48) --
	(109.55, 72.54) --
	(106.48, 72.54) --
	cycle;

\path[draw=drawColor,line width= 0.4pt,line join=round,line cap=round] ( 40.57, 39.66) -- ( 43.64, 42.73);

\path[draw=drawColor,line width= 0.4pt,line join=round,line cap=round] ( 40.57, 42.73) -- ( 43.64, 39.66);

\path[draw=drawColor,line width= 0.4pt,line join=round,line cap=round] ( 39.94, 41.19) -- ( 44.27, 41.19);

\path[draw=drawColor,line width= 0.4pt,line join=round,line cap=round] ( 42.11, 39.02) -- ( 42.11, 43.36);
\definecolor{drawColor}{RGB}{247,192,26}
\definecolor{fillColor}{RGB}{247,192,26}

\path[draw=drawColor,line width= 0.4pt,line join=round,line cap=round,fill=fillColor] ( 42.11, 41.79) circle (  1.53);

\path[draw=drawColor,line width= 0.4pt,line join=round,line cap=round,fill=fillColor] ( 64.07, 46.82) circle (  1.53);
\end{scope}
\begin{scope}
\path[clip] (  0.00,  0.00) rectangle (209.58,115.63);
\definecolor{drawColor}{RGB}{0,0,0}

\path[draw=drawColor,line width= 0.5pt,line join=round] ( 34.42, 28.73) --
	( 34.42,113.63);

\path[draw=drawColor,line width= 0.5pt,line join=round] ( 35.55,111.66) --
	( 34.42,113.63) --
	( 33.28,111.66);
\end{scope}
\begin{scope}
\path[clip] (  0.00,  0.00) rectangle (209.58,115.63);
\definecolor{drawColor}{gray}{0.30}

\node[text=drawColor,anchor=base east,inner sep=0pt, outer sep=0pt, scale=  0.70] at ( 30.37, 31.03) {1e-03};

\node[text=drawColor,anchor=base east,inner sep=0pt, outer sep=0pt, scale=  0.70] at ( 30.37, 49.90) {1e-01};

\node[text=drawColor,anchor=base east,inner sep=0pt, outer sep=0pt, scale=  0.70] at ( 30.37, 68.77) {1e+01};

\node[text=drawColor,anchor=base east,inner sep=0pt, outer sep=0pt, scale=  0.70] at ( 30.37, 87.64) {1e+03};

\node[text=drawColor,anchor=base east,inner sep=0pt, outer sep=0pt, scale=  0.70] at ( 30.37,106.50) {1e+05};
\end{scope}
\begin{scope}
\path[clip] (  0.00,  0.00) rectangle (209.58,115.63);
\definecolor{drawColor}{gray}{0.20}

\path[draw=drawColor,line width= 0.5pt,line join=round] ( 32.17, 33.45) --
	( 34.42, 33.45);

\path[draw=drawColor,line width= 0.5pt,line join=round] ( 32.17, 52.31) --
	( 34.42, 52.31);

\path[draw=drawColor,line width= 0.5pt,line join=round] ( 32.17, 71.18) --
	( 34.42, 71.18);

\path[draw=drawColor,line width= 0.5pt,line join=round] ( 32.17, 90.05) --
	( 34.42, 90.05);

\path[draw=drawColor,line width= 0.5pt,line join=round] ( 32.17,108.92) --
	( 34.42,108.92);
\end{scope}
\begin{scope}
\path[clip] (  0.00,  0.00) rectangle (209.58,115.63);
\definecolor{drawColor}{RGB}{0,0,0}

\path[draw=drawColor,line width= 0.5pt,line join=round] ( 34.42, 28.73) --
	(203.58, 28.73);

\path[draw=drawColor,line width= 0.5pt,line join=round] (201.61, 27.59) --
	(203.58, 28.73) --
	(201.61, 29.87);
\end{scope}
\begin{scope}
\path[clip] (  0.00,  0.00) rectangle (209.58,115.63);
\definecolor{drawColor}{gray}{0.20}

\path[draw=drawColor,line width= 0.5pt,line join=round] ( 64.07, 26.48) --
	( 64.07, 28.73);

\path[draw=drawColor,line width= 0.5pt,line join=round] (108.01, 26.48) --
	(108.01, 28.73);

\path[draw=drawColor,line width= 0.5pt,line join=round] (151.95, 26.48) --
	(151.95, 28.73);

\path[draw=drawColor,line width= 0.5pt,line join=round] (195.89, 26.48) --
	(195.89, 28.73);
\end{scope}
\begin{scope}
\path[clip] (  0.00,  0.00) rectangle (209.58,115.63);
\definecolor{drawColor}{gray}{0.30}

\node[text=drawColor,anchor=base,inner sep=0pt, outer sep=0pt, scale=  0.70] at ( 64.07, 19.86) {4};

\node[text=drawColor,anchor=base,inner sep=0pt, outer sep=0pt, scale=  0.70] at (108.01, 19.86) {8};

\node[text=drawColor,anchor=base,inner sep=0pt, outer sep=0pt, scale=  0.70] at (151.95, 19.86) {12};

\node[text=drawColor,anchor=base,inner sep=0pt, outer sep=0pt, scale=  0.70] at (195.89, 19.86) {16};
\end{scope}
\begin{scope}
\path[clip] (  0.00,  0.00) rectangle (209.58,115.63);
\definecolor{drawColor}{RGB}{0,0,0}

\node[text=drawColor,anchor=base,inner sep=0pt, outer sep=0pt, scale=  0.80] at (119.00, 10.74) {$n$};
\end{scope}
\begin{scope}
\path[clip] (  0.00,  0.00) rectangle (209.58,115.63);
\definecolor{drawColor}{RGB}{0,0,0}

\node[text=drawColor,rotate= 90.00,anchor=base,inner sep=0pt, outer sep=0pt, scale=  0.80] at (  7.51, 71.18) {time (ms)};
\end{scope}
\begin{scope}
\path[clip] (  0.00,  0.00) rectangle (209.58,115.63);

\path[] ( 34.42,  1.00) rectangle (210.27,  7.18);
\end{scope}
\begin{scope}
\path[clip] (  0.00,  0.00) rectangle (209.58,115.63);
\definecolor{drawColor}{RGB}{247,192,26}

\path[draw=drawColor,line width= 0.5pt,line join=round] ( 17.50,  4.09) -- ( 26.17,  4.09);
\end{scope}
\begin{scope}
\path[clip] (  0.00,  0.00) rectangle (209.58,115.63);
\definecolor{drawColor}{RGB}{247,192,26}
\definecolor{fillColor}{RGB}{247,192,26}

\path[draw=drawColor,line width= 0.4pt,line join=round,line cap=round,fill=fillColor] ( 21.84,  4.09) circle (  1.53);
\end{scope}
\begin{scope}
\path[clip] (  0.00,  0.00) rectangle (209.58,115.63);
\definecolor{drawColor}{RGB}{78,155,133}

\path[draw=drawColor,line width= 0.5pt,dash pattern=on 4pt off 4pt ,line join=round] ( 60.10,  4.09) -- ( 68.77,  4.09);
\end{scope}
\begin{scope}
\path[clip] (  0.00,  0.00) rectangle (209.58,115.63);
\definecolor{fillColor}{RGB}{78,155,133}

\path[fill=fillColor] ( 64.43,  6.48) --
	( 66.50,  2.90) --
	( 62.37,  2.90) --
	cycle;
\end{scope}
\begin{scope}
\path[clip] (  0.00,  0.00) rectangle (209.58,115.63);
\definecolor{drawColor}{RGB}{37,122,164}

\path[draw=drawColor,line width= 0.5pt,dash pattern=on 1pt off 3pt ,line join=round] (108.14,  4.09) -- (116.81,  4.09);
\end{scope}
\begin{scope}
\path[clip] (  0.00,  0.00) rectangle (209.58,115.63);
\definecolor{fillColor}{RGB}{37,122,164}

\path[fill=fillColor] (110.94,  2.56) --
	(114.01,  2.56) --
	(114.01,  5.62) --
	(110.94,  5.62) --
	cycle;
\end{scope}
\begin{scope}
\path[clip] (  0.00,  0.00) rectangle (209.58,115.63);
\definecolor{drawColor}{RGB}{86,51,94}

\path[draw=drawColor,line width= 0.5pt,dash pattern=on 7pt off 3pt ,line join=round] (158.32,  4.09) -- (166.99,  4.09);
\end{scope}
\begin{scope}
\path[clip] (  0.00,  0.00) rectangle (209.58,115.63);
\definecolor{drawColor}{RGB}{86,51,94}

\path[draw=drawColor,line width= 0.4pt,line join=round,line cap=round] (161.12,  2.56) -- (164.19,  5.62);

\path[draw=drawColor,line width= 0.4pt,line join=round,line cap=round] (161.12,  5.62) -- (164.19,  2.56);

\path[draw=drawColor,line width= 0.4pt,line join=round,line cap=round] (160.49,  4.09) -- (164.82,  4.09);

\path[draw=drawColor,line width= 0.4pt,line join=round,line cap=round] (162.65,  1.92) -- (162.65,  6.26);
\end{scope}
\begin{scope}
\path[clip] (  0.00,  0.00) rectangle (209.58,115.63);
\definecolor{drawColor}{RGB}{0,0,0}

\node[text=drawColor,anchor=base west,inner sep=0pt, outer sep=0pt, scale=  0.70] at ( 27.26,  1.68) {DECOR};
\end{scope}
\begin{scope}
\path[clip] (  0.00,  0.00) rectangle (209.58,115.63);
\definecolor{drawColor}{RGB}{0,0,0}

\node[text=drawColor,anchor=base west,inner sep=0pt, outer sep=0pt, scale=  0.70] at ( 69.85,  1.68) {DECOR+};
\end{scope}
\begin{scope}
\path[clip] (  0.00,  0.00) rectangle (209.58,115.63);
\definecolor{drawColor}{RGB}{0,0,0}

\node[text=drawColor,anchor=base west,inner sep=0pt, outer sep=0pt, scale=  0.70] at (117.90,  1.68) {A-DECOR};
\end{scope}
\begin{scope}
\path[clip] (  0.00,  0.00) rectangle (209.58,115.63);
\definecolor{drawColor}{RGB}{0,0,0}

\node[text=drawColor,anchor=base west,inner sep=0pt, outer sep=0pt, scale=  0.70] at (168.07,  1.68) {Brute-Force};
\end{scope}
\end{tikzpicture}

%% file: files/decor_revised_plot_k=2_phases_decorplus.tex
\begin{tikzpicture}[x=1pt,y=1pt]
\definecolor{fillColor}{RGB}{255,255,255}
\path[use as bounding box,fill=fillColor,fill opacity=0.00] (0,0) rectangle (209.58,115.63);
\begin{scope}
\path[clip] (  0.00,  0.00) rectangle (209.58,115.63);
\definecolor{drawColor}{RGB}{255,255,255}
\definecolor{fillColor}{RGB}{255,255,255}

\path[draw=drawColor,line width= 0.5pt,line join=round,line cap=round,fill=fillColor] (  0.00,  0.00) rectangle (209.58,115.63);
\end{scope}
\begin{scope}
\path[clip] ( 34.42, 28.73) rectangle (203.58,113.63);
\definecolor{fillColor}{RGB}{255,255,255}

\path[fill=fillColor] ( 34.42, 28.73) rectangle (203.58,113.63);
\definecolor{fillColor}{RGB}{78,155,133}

\path[fill=fillColor] (101.41, 28.73) rectangle (105.76, 59.58);

\path[fill=fillColor] ( 61.87, 28.73) rectangle ( 66.22, 46.97);

\path[fill=fillColor] (160.71, 28.73) rectangle (165.06, 85.32);

\path[fill=fillColor] ( 42.11, 28.73) rectangle ( 46.45, 42.20);

\path[fill=fillColor] (140.94, 28.73) rectangle (145.29, 75.54);

\path[fill=fillColor] ( 81.64, 28.73) rectangle ( 85.99, 53.32);

\path[fill=fillColor] (180.48, 28.73) rectangle (184.82, 95.74);

\path[fill=fillColor] (121.17, 28.73) rectangle (125.52, 67.10);
\definecolor{fillColor}{RGB}{247,192,26}

\path[fill=fillColor] (106.94, 28.73) rectangle (111.29, 57.95);

\path[fill=fillColor] ( 67.41, 28.73) rectangle ( 71.76, 46.10);

\path[fill=fillColor] (166.24, 28.73) rectangle (170.59, 77.00);

\path[fill=fillColor] ( 47.64, 28.73) rectangle ( 51.99, 41.16);

\path[fill=fillColor] (146.48, 28.73) rectangle (150.82, 70.52);

\path[fill=fillColor] ( 87.17, 28.73) rectangle ( 91.52, 51.99);

\path[fill=fillColor] (186.01, 28.73) rectangle (190.36, 83.95);

\path[fill=fillColor] (126.71, 28.73) rectangle (131.06, 64.24);
\definecolor{fillColor}{RGB}{37,122,164}

\path[fill=fillColor] (112.48, 28.73) rectangle (116.82, 61.69);

\path[fill=fillColor] ( 72.94, 28.73) rectangle ( 77.29, 49.39);

\path[fill=fillColor] (171.78, 28.73) rectangle (176.13, 85.83);

\path[fill=fillColor] ( 53.17, 28.73) rectangle ( 57.52, 44.56);

\path[fill=fillColor] (152.01, 28.73) rectangle (156.36, 76.60);

\path[fill=fillColor] ( 92.71, 28.73) rectangle ( 97.06, 55.55);

\path[fill=fillColor] (191.54, 28.73) rectangle (195.89, 95.96);

\path[fill=fillColor] (132.24, 28.73) rectangle (136.59, 68.76);
\end{scope}
\begin{scope}
\path[clip] (  0.00,  0.00) rectangle (209.58,115.63);
\definecolor{drawColor}{RGB}{0,0,0}

\path[draw=drawColor,line width= 0.5pt,line join=round] ( 34.42, 28.73) --
	( 34.42,113.63);

\path[draw=drawColor,line width= 0.5pt,line join=round] ( 35.55,111.66) --
	( 34.42,113.63) --
	( 33.28,111.66);
\end{scope}
\begin{scope}
\path[clip] (  0.00,  0.00) rectangle (209.58,115.63);
\definecolor{drawColor}{gray}{0.30}

\node[text=drawColor,anchor=base east,inner sep=0pt, outer sep=0pt, scale=  0.70] at ( 30.37, 31.03) {1e-03};

\node[text=drawColor,anchor=base east,inner sep=0pt, outer sep=0pt, scale=  0.70] at ( 30.37, 49.90) {1e-01};

\node[text=drawColor,anchor=base east,inner sep=0pt, outer sep=0pt, scale=  0.70] at ( 30.37, 68.77) {1e+01};

\node[text=drawColor,anchor=base east,inner sep=0pt, outer sep=0pt, scale=  0.70] at ( 30.37, 87.64) {1e+03};

\node[text=drawColor,anchor=base east,inner sep=0pt, outer sep=0pt, scale=  0.70] at ( 30.37,106.50) {1e+05};
\end{scope}
\begin{scope}
\path[clip] (  0.00,  0.00) rectangle (209.58,115.63);
\definecolor{drawColor}{gray}{0.20}

\path[draw=drawColor,line width= 0.5pt,line join=round] ( 32.17, 33.45) --
	( 34.42, 33.45);

\path[draw=drawColor,line width= 0.5pt,line join=round] ( 32.17, 52.31) --
	( 34.42, 52.31);

\path[draw=drawColor,line width= 0.5pt,line join=round] ( 32.17, 71.18) --
	( 34.42, 71.18);

\path[draw=drawColor,line width= 0.5pt,line join=round] ( 32.17, 90.05) --
	( 34.42, 90.05);

\path[draw=drawColor,line width= 0.5pt,line join=round] ( 32.17,108.92) --
	( 34.42,108.92);
\end{scope}
\begin{scope}
\path[clip] (  0.00,  0.00) rectangle (209.58,115.63);
\definecolor{drawColor}{RGB}{0,0,0}

\path[draw=drawColor,line width= 0.5pt,line join=round] ( 34.42, 28.73) --
	(203.58, 28.73);

\path[draw=drawColor,line width= 0.5pt,line join=round] (201.61, 27.59) --
	(203.58, 28.73) --
	(201.61, 29.87);
\end{scope}
\begin{scope}
\path[clip] (  0.00,  0.00) rectangle (209.58,115.63);
\definecolor{drawColor}{gray}{0.20}

\path[draw=drawColor,line width= 0.5pt,line join=round] ( 49.81, 26.48) --
	( 49.81, 28.73);

\path[draw=drawColor,line width= 0.5pt,line join=round] ( 69.58, 26.48) --
	( 69.58, 28.73);

\path[draw=drawColor,line width= 0.5pt,line join=round] ( 89.35, 26.48) --
	( 89.35, 28.73);

\path[draw=drawColor,line width= 0.5pt,line join=round] (109.12, 26.48) --
	(109.12, 28.73);

\path[draw=drawColor,line width= 0.5pt,line join=round] (128.88, 26.48) --
	(128.88, 28.73);

\path[draw=drawColor,line width= 0.5pt,line join=round] (148.65, 26.48) --
	(148.65, 28.73);

\path[draw=drawColor,line width= 0.5pt,line join=round] (168.42, 26.48) --
	(168.42, 28.73);

\path[draw=drawColor,line width= 0.5pt,line join=round] (188.18, 26.48) --
	(188.18, 28.73);
\end{scope}
\begin{scope}
\path[clip] (  0.00,  0.00) rectangle (209.58,115.63);
\definecolor{drawColor}{gray}{0.30}

\node[text=drawColor,anchor=base,inner sep=0pt, outer sep=0pt, scale=  0.70] at ( 49.81, 19.86) {2};

\node[text=drawColor,anchor=base,inner sep=0pt, outer sep=0pt, scale=  0.70] at ( 69.58, 19.86) {4};

\node[text=drawColor,anchor=base,inner sep=0pt, outer sep=0pt, scale=  0.70] at ( 89.35, 19.86) {6};

\node[text=drawColor,anchor=base,inner sep=0pt, outer sep=0pt, scale=  0.70] at (109.12, 19.86) {8};

\node[text=drawColor,anchor=base,inner sep=0pt, outer sep=0pt, scale=  0.70] at (128.88, 19.86) {10};

\node[text=drawColor,anchor=base,inner sep=0pt, outer sep=0pt, scale=  0.70] at (148.65, 19.86) {12};

\node[text=drawColor,anchor=base,inner sep=0pt, outer sep=0pt, scale=  0.70] at (168.42, 19.86) {14};

\node[text=drawColor,anchor=base,inner sep=0pt, outer sep=0pt, scale=  0.70] at (188.18, 19.86) {16};
\end{scope}
\begin{scope}
\path[clip] (  0.00,  0.00) rectangle (209.58,115.63);
\definecolor{drawColor}{RGB}{0,0,0}

\node[text=drawColor,anchor=base,inner sep=0pt, outer sep=0pt, scale=  0.80] at (119.00, 10.74) {$n$};
\end{scope}
\begin{scope}
\path[clip] (  0.00,  0.00) rectangle (209.58,115.63);
\definecolor{drawColor}{RGB}{0,0,0}

\node[text=drawColor,rotate= 90.00,anchor=base,inner sep=0pt, outer sep=0pt, scale=  0.80] at (  7.51, 71.18) {time (ms)};
\end{scope}
\begin{scope}
\path[clip] (  0.00,  0.00) rectangle (209.58,115.63);

\path[] ( 51.04,  1.00) rectangle (186.96,  7.18);
\end{scope}
\begin{scope}
\path[clip] (  0.00,  0.00) rectangle (209.58,115.63);
\definecolor{fillColor}{RGB}{78,155,133}

\path[fill=fillColor] ( 51.75,  1.71) rectangle ( 61.17,  6.47);
\end{scope}
\begin{scope}
\path[clip] (  0.00,  0.00) rectangle (209.58,115.63);
\definecolor{fillColor}{RGB}{78,155,133}

\path[fill=fillColor] ( 51.75,  1.71) rectangle ( 61.17,  6.47);
\end{scope}
\begin{scope}
\path[clip] (  0.00,  0.00) rectangle (209.58,115.63);
\definecolor{fillColor}{RGB}{78,155,133}

\path[fill=fillColor] ( 51.75,  1.71) rectangle ( 61.17,  6.47);
\end{scope}
\begin{scope}
\path[clip] (  0.00,  0.00) rectangle (209.58,115.63);
\definecolor{fillColor}{RGB}{247,192,26}

\path[fill=fillColor] (102.84,  1.71) rectangle (112.26,  6.47);
\end{scope}
\begin{scope}
\path[clip] (  0.00,  0.00) rectangle (209.58,115.63);
\definecolor{fillColor}{RGB}{247,192,26}

\path[fill=fillColor] (102.84,  1.71) rectangle (112.26,  6.47);
\end{scope}
\begin{scope}
\path[clip] (  0.00,  0.00) rectangle (209.58,115.63);
\definecolor{fillColor}{RGB}{247,192,26}

\path[fill=fillColor] (102.84,  1.71) rectangle (112.26,  6.47);
\end{scope}
\begin{scope}
\path[clip] (  0.00,  0.00) rectangle (209.58,115.63);
\definecolor{fillColor}{RGB}{37,122,164}

\path[fill=fillColor] (154.70,  1.71) rectangle (164.11,  6.47);
\end{scope}
\begin{scope}
\path[clip] (  0.00,  0.00) rectangle (209.58,115.63);
\definecolor{fillColor}{RGB}{37,122,164}

\path[fill=fillColor] (154.70,  1.71) rectangle (164.11,  6.47);
\end{scope}
\begin{scope}
\path[clip] (  0.00,  0.00) rectangle (209.58,115.63);
\definecolor{fillColor}{RGB}{37,122,164}

\path[fill=fillColor] (154.70,  1.71) rectangle (164.11,  6.47);
\end{scope}
\begin{scope}
\path[clip] (  0.00,  0.00) rectangle (209.58,115.63);
\definecolor{drawColor}{RGB}{0,0,0}

\node[text=drawColor,anchor=base west,inner sep=0pt, outer sep=0pt, scale=  0.70] at ( 61.88,  1.68) {Candidates};
\end{scope}
\begin{scope}
\path[clip] (  0.00,  0.00) rectangle (209.58,115.63);
\definecolor{drawColor}{RGB}{0,0,0}

\node[text=drawColor,anchor=base west,inner sep=0pt, outer sep=0pt, scale=  0.70] at (112.97,  1.68) {Verification};
\end{scope}
\begin{scope}
\path[clip] (  0.00,  0.00) rectangle (209.58,115.63);
\definecolor{drawColor}{RGB}{0,0,0}

\node[text=drawColor,anchor=base west,inner sep=0pt, outer sep=0pt, scale=  0.70] at (164.82,  1.68) {Total};
\end{scope}
\end{tikzpicture}

%% file: files/decor_revised_plot_k=log2n_times.tex
\begin{tikzpicture}[x=1pt,y=1pt]
\definecolor{fillColor}{RGB}{255,255,255}
\path[use as bounding box,fill=fillColor,fill opacity=0.00] (0,0) rectangle (209.58,115.63);
\begin{scope}
\path[clip] (  0.00,  0.00) rectangle (209.58,115.63);
\definecolor{drawColor}{RGB}{255,255,255}
\definecolor{fillColor}{RGB}{255,255,255}

\path[draw=drawColor,line width= 0.5pt,line join=round,line cap=round,fill=fillColor] (  0.00,  0.00) rectangle (209.58,115.63);
\end{scope}
\begin{scope}
\path[clip] ( 34.42, 28.73) rectangle (203.58,113.63);
\definecolor{fillColor}{RGB}{255,255,255}

\path[fill=fillColor] ( 34.42, 28.73) rectangle (203.58,113.63);
\definecolor{drawColor}{RGB}{247,192,26}

\path[draw=drawColor,line width= 0.5pt,line join=round] ( 42.11, 38.47) --
	( 67.74, 45.70) --
	( 93.37, 53.53) --
	(119.00, 61.77) --
	(144.63, 70.38) --
	(170.26, 80.69) --
	(195.89, 90.05);
\definecolor{drawColor}{RGB}{78,155,133}

\path[draw=drawColor,line width= 0.5pt,dash pattern=on 4pt off 4pt ,line join=round] ( 42.11, 41.36) --
	( 67.74, 48.29) --
	( 93.37, 55.52) --
	(119.00, 63.22) --
	(144.63, 71.53) --
	(170.26, 81.37) --
	(195.89, 90.46);
\definecolor{drawColor}{RGB}{37,122,164}

\path[draw=drawColor,line width= 0.5pt,dash pattern=on 1pt off 3pt ,line join=round] ( 42.11, 45.59) --
	( 67.74, 56.11) --
	( 93.37, 65.73) --
	(119.00, 75.11) --
	(144.63, 83.75) --
	(170.26, 92.65) --
	(195.89,101.54);
\definecolor{drawColor}{RGB}{86,51,94}

\path[draw=drawColor,line width= 0.5pt,dash pattern=on 7pt off 3pt ,line join=round] ( 42.11, 47.47) --
	( 67.74, 61.71) --
	( 93.37, 74.75) --
	(119.00, 88.66) --
	(144.63,102.13);
\definecolor{fillColor}{RGB}{78,155,133}

\path[fill=fillColor] ( 42.11, 43.75) --
	( 44.17, 40.17) --
	( 40.04, 40.17) --
	cycle;

\path[draw=drawColor,line width= 0.4pt,line join=round,line cap=round] ( 40.57, 45.93) -- ( 43.64, 49.00);

\path[draw=drawColor,line width= 0.4pt,line join=round,line cap=round] ( 40.57, 49.00) -- ( 43.64, 45.93);

\path[draw=drawColor,line width= 0.4pt,line join=round,line cap=round] ( 39.94, 47.47) -- ( 44.27, 47.47);

\path[draw=drawColor,line width= 0.4pt,line join=round,line cap=round] ( 42.11, 45.30) -- ( 42.11, 49.64);
\definecolor{drawColor}{RGB}{247,192,26}
\definecolor{fillColor}{RGB}{247,192,26}

\path[draw=drawColor,line width= 0.4pt,line join=round,line cap=round,fill=fillColor] (170.26, 80.69) circle (  1.53);
\definecolor{drawColor}{RGB}{86,51,94}

\path[draw=drawColor,line width= 0.4pt,line join=round,line cap=round] ( 91.83, 73.22) -- ( 94.90, 76.28);

\path[draw=drawColor,line width= 0.4pt,line join=round,line cap=round] ( 91.83, 76.28) -- ( 94.90, 73.22);

\path[draw=drawColor,line width= 0.4pt,line join=round,line cap=round] ( 91.20, 74.75) -- ( 95.54, 74.75);

\path[draw=drawColor,line width= 0.4pt,line join=round,line cap=round] ( 93.37, 72.58) -- ( 93.37, 76.92);
\definecolor{fillColor}{RGB}{37,122,164}

\path[fill=fillColor] ( 91.83, 64.20) --
	( 94.90, 64.20) --
	( 94.90, 67.26) --
	( 91.83, 67.26) --
	cycle;
\definecolor{fillColor}{RGB}{78,155,133}

\path[fill=fillColor] (119.00, 65.60) --
	(121.06, 62.03) --
	(116.93, 62.03) --
	cycle;

\path[draw=drawColor,line width= 0.4pt,line join=round,line cap=round] (143.10,100.60) -- (146.16,103.67);

\path[draw=drawColor,line width= 0.4pt,line join=round,line cap=round] (143.10,103.67) -- (146.16,100.60);

\path[draw=drawColor,line width= 0.4pt,line join=round,line cap=round] (142.46,102.13) -- (146.80,102.13);

\path[draw=drawColor,line width= 0.4pt,line join=round,line cap=round] (144.63, 99.96) -- (144.63,104.30);
\definecolor{fillColor}{RGB}{37,122,164}

\path[fill=fillColor] ( 40.57, 44.05) --
	( 43.64, 44.05) --
	( 43.64, 47.12) --
	( 40.57, 47.12) --
	cycle;
\definecolor{drawColor}{RGB}{247,192,26}
\definecolor{fillColor}{RGB}{247,192,26}

\path[draw=drawColor,line width= 0.4pt,line join=round,line cap=round,fill=fillColor] ( 67.74, 45.70) circle (  1.53);
\definecolor{fillColor}{RGB}{37,122,164}

\path[fill=fillColor] (143.10, 82.21) --
	(146.16, 82.21) --
	(146.16, 85.28) --
	(143.10, 85.28) --
	cycle;
\definecolor{fillColor}{RGB}{78,155,133}

\path[fill=fillColor] (170.26, 83.76) --
	(172.33, 80.18) --
	(168.20, 80.18) --
	cycle;
\definecolor{fillColor}{RGB}{247,192,26}

\path[draw=drawColor,line width= 0.4pt,line join=round,line cap=round,fill=fillColor] (195.89, 90.05) circle (  1.53);
\definecolor{fillColor}{RGB}{78,155,133}

\path[fill=fillColor] ( 67.74, 50.67) --
	( 69.80, 47.09) --
	( 65.67, 47.09) --
	cycle;
\definecolor{fillColor}{RGB}{37,122,164}

\path[fill=fillColor] (117.47, 73.58) --
	(120.53, 73.58) --
	(120.53, 76.64) --
	(117.47, 76.64) --
	cycle;
\definecolor{fillColor}{RGB}{247,192,26}

\path[draw=drawColor,line width= 0.4pt,line join=round,line cap=round,fill=fillColor] ( 93.37, 53.53) circle (  1.53);
\definecolor{fillColor}{RGB}{37,122,164}

\path[fill=fillColor] (194.36,100.00) --
	(197.43,100.00) --
	(197.43,103.07) --
	(194.36,103.07) --
	cycle;

\path[fill=fillColor] ( 66.20, 54.58) --
	( 69.27, 54.58) --
	( 69.27, 57.64) --
	( 66.20, 57.64) --
	cycle;
\definecolor{drawColor}{RGB}{86,51,94}

\path[draw=drawColor,line width= 0.4pt,line join=round,line cap=round] (117.47, 87.12) -- (120.53, 90.19);

\path[draw=drawColor,line width= 0.4pt,line join=round,line cap=round] (117.47, 90.19) -- (120.53, 87.12);

\path[draw=drawColor,line width= 0.4pt,line join=round,line cap=round] (116.83, 88.66) -- (121.17, 88.66);

\path[draw=drawColor,line width= 0.4pt,line join=round,line cap=round] (119.00, 86.49) -- (119.00, 90.82);

\path[fill=fillColor] (168.73, 91.11) --
	(171.80, 91.11) --
	(171.80, 94.18) --
	(168.73, 94.18) --
	cycle;

\path[draw=drawColor,line width= 0.4pt,line join=round,line cap=round] ( 66.20, 60.18) -- ( 69.27, 63.24);

\path[draw=drawColor,line width= 0.4pt,line join=round,line cap=round] ( 66.20, 63.24) -- ( 69.27, 60.18);

\path[draw=drawColor,line width= 0.4pt,line join=round,line cap=round] ( 65.57, 61.71) -- ( 69.91, 61.71);

\path[draw=drawColor,line width= 0.4pt,line join=round,line cap=round] ( 67.74, 59.54) -- ( 67.74, 63.88);
\definecolor{fillColor}{RGB}{78,155,133}

\path[fill=fillColor] ( 93.37, 57.90) --
	( 95.43, 54.33) --
	( 91.30, 54.33) --
	cycle;

\path[fill=fillColor] (195.89, 92.84) --
	(197.96, 89.27) --
	(193.83, 89.27) --
	cycle;
\definecolor{drawColor}{RGB}{247,192,26}
\definecolor{fillColor}{RGB}{247,192,26}

\path[draw=drawColor,line width= 0.4pt,line join=round,line cap=round,fill=fillColor] (119.00, 61.77) circle (  1.53);
\definecolor{fillColor}{RGB}{78,155,133}

\path[fill=fillColor] (144.63, 73.92) --
	(146.70, 70.34) --
	(142.57, 70.34) --
	cycle;
\definecolor{fillColor}{RGB}{247,192,26}

\path[draw=drawColor,line width= 0.4pt,line join=round,line cap=round,fill=fillColor] ( 42.11, 38.47) circle (  1.53);

\path[draw=drawColor,line width= 0.4pt,line join=round,line cap=round,fill=fillColor] (144.63, 70.38) circle (  1.53);
\end{scope}
\begin{scope}
\path[clip] (  0.00,  0.00) rectangle (209.58,115.63);
\definecolor{drawColor}{RGB}{0,0,0}

\path[draw=drawColor,line width= 0.5pt,line join=round] ( 34.42, 28.73) --
	( 34.42,113.63);

\path[draw=drawColor,line width= 0.5pt,line join=round] ( 35.55,111.66) --
	( 34.42,113.63) --
	( 33.28,111.66);
\end{scope}
\begin{scope}
\path[clip] (  0.00,  0.00) rectangle (209.58,115.63);
\definecolor{drawColor}{gray}{0.30}

\node[text=drawColor,anchor=base east,inner sep=0pt, outer sep=0pt, scale=  0.70] at ( 30.37, 42.24) {1e-01};

\node[text=drawColor,anchor=base east,inner sep=0pt, outer sep=0pt, scale=  0.70] at ( 30.37, 63.46) {1e+01};

\node[text=drawColor,anchor=base east,inner sep=0pt, outer sep=0pt, scale=  0.70] at ( 30.37, 84.69) {1e+03};

\node[text=drawColor,anchor=base east,inner sep=0pt, outer sep=0pt, scale=  0.70] at ( 30.37,105.92) {1e+05};
\end{scope}
\begin{scope}
\path[clip] (  0.00,  0.00) rectangle (209.58,115.63);
\definecolor{drawColor}{gray}{0.20}

\path[draw=drawColor,line width= 0.5pt,line join=round] ( 32.17, 44.65) --
	( 34.42, 44.65);

\path[draw=drawColor,line width= 0.5pt,line join=round] ( 32.17, 65.87) --
	( 34.42, 65.87);

\path[draw=drawColor,line width= 0.5pt,line join=round] ( 32.17, 87.10) --
	( 34.42, 87.10);

\path[draw=drawColor,line width= 0.5pt,line join=round] ( 32.17,108.33) --
	( 34.42,108.33);
\end{scope}
\begin{scope}
\path[clip] (  0.00,  0.00) rectangle (209.58,115.63);
\definecolor{drawColor}{RGB}{0,0,0}

\path[draw=drawColor,line width= 0.5pt,line join=round] ( 34.42, 28.73) --
	(203.58, 28.73);

\path[draw=drawColor,line width= 0.5pt,line join=round] (201.61, 27.59) --
	(203.58, 28.73) --
	(201.61, 29.87);
\end{scope}
\begin{scope}
\path[clip] (  0.00,  0.00) rectangle (209.58,115.63);
\definecolor{drawColor}{gray}{0.20}

\path[draw=drawColor,line width= 0.5pt,line join=round] ( 42.11, 26.48) --
	( 42.11, 28.73);

\path[draw=drawColor,line width= 0.5pt,line join=round] ( 93.37, 26.48) --
	( 93.37, 28.73);

\path[draw=drawColor,line width= 0.5pt,line join=round] (144.63, 26.48) --
	(144.63, 28.73);

\path[draw=drawColor,line width= 0.5pt,line join=round] (195.89, 26.48) --
	(195.89, 28.73);
\end{scope}
\begin{scope}
\path[clip] (  0.00,  0.00) rectangle (209.58,115.63);
\definecolor{drawColor}{gray}{0.30}

\node[text=drawColor,anchor=base,inner sep=0pt, outer sep=0pt, scale=  0.70] at ( 42.11, 19.86) {4};

\node[text=drawColor,anchor=base,inner sep=0pt, outer sep=0pt, scale=  0.70] at ( 93.37, 19.86) {8};

\node[text=drawColor,anchor=base,inner sep=0pt, outer sep=0pt, scale=  0.70] at (144.63, 19.86) {12};

\node[text=drawColor,anchor=base,inner sep=0pt, outer sep=0pt, scale=  0.70] at (195.89, 19.86) {16};
\end{scope}
\begin{scope}
\path[clip] (  0.00,  0.00) rectangle (209.58,115.63);
\definecolor{drawColor}{RGB}{0,0,0}

\node[text=drawColor,anchor=base,inner sep=0pt, outer sep=0pt, scale=  0.80] at (119.00, 10.74) {$n$};
\end{scope}
\begin{scope}
\path[clip] (  0.00,  0.00) rectangle (209.58,115.63);
\definecolor{drawColor}{RGB}{0,0,0}

\node[text=drawColor,rotate= 90.00,anchor=base,inner sep=0pt, outer sep=0pt, scale=  0.80] at (  7.51, 71.18) {time (ms)};
\end{scope}
\begin{scope}
\path[clip] (  0.00,  0.00) rectangle (209.58,115.63);

\path[] ( 34.42,  1.00) rectangle (210.27,  7.18);
\end{scope}
\begin{scope}
\path[clip] (  0.00,  0.00) rectangle (209.58,115.63);
\definecolor{drawColor}{RGB}{247,192,26}

\path[draw=drawColor,line width= 0.5pt,line join=round] ( 17.50,  4.09) -- ( 26.17,  4.09);
\end{scope}
\begin{scope}
\path[clip] (  0.00,  0.00) rectangle (209.58,115.63);
\definecolor{drawColor}{RGB}{247,192,26}
\definecolor{fillColor}{RGB}{247,192,26}

\path[draw=drawColor,line width= 0.4pt,line join=round,line cap=round,fill=fillColor] ( 21.84,  4.09) circle (  1.53);
\end{scope}
\begin{scope}
\path[clip] (  0.00,  0.00) rectangle (209.58,115.63);
\definecolor{drawColor}{RGB}{78,155,133}

\path[draw=drawColor,line width= 0.5pt,dash pattern=on 4pt off 4pt ,line join=round] ( 60.10,  4.09) -- ( 68.77,  4.09);
\end{scope}
\begin{scope}
\path[clip] (  0.00,  0.00) rectangle (209.58,115.63);
\definecolor{fillColor}{RGB}{78,155,133}

\path[fill=fillColor] ( 64.43,  6.48) --
	( 66.50,  2.90) --
	( 62.37,  2.90) --
	cycle;
\end{scope}
\begin{scope}
\path[clip] (  0.00,  0.00) rectangle (209.58,115.63);
\definecolor{drawColor}{RGB}{37,122,164}

\path[draw=drawColor,line width= 0.5pt,dash pattern=on 1pt off 3pt ,line join=round] (108.14,  4.09) -- (116.81,  4.09);
\end{scope}
\begin{scope}
\path[clip] (  0.00,  0.00) rectangle (209.58,115.63);
\definecolor{fillColor}{RGB}{37,122,164}

\path[fill=fillColor] (110.94,  2.56) --
	(114.01,  2.56) --
	(114.01,  5.62) --
	(110.94,  5.62) --
	cycle;
\end{scope}
\begin{scope}
\path[clip] (  0.00,  0.00) rectangle (209.58,115.63);
\definecolor{drawColor}{RGB}{86,51,94}

\path[draw=drawColor,line width= 0.5pt,dash pattern=on 7pt off 3pt ,line join=round] (158.32,  4.09) -- (166.99,  4.09);
\end{scope}
\begin{scope}
\path[clip] (  0.00,  0.00) rectangle (209.58,115.63);
\definecolor{drawColor}{RGB}{86,51,94}

\path[draw=drawColor,line width= 0.4pt,line join=round,line cap=round] (161.12,  2.56) -- (164.19,  5.62);

\path[draw=drawColor,line width= 0.4pt,line join=round,line cap=round] (161.12,  5.62) -- (164.19,  2.56);

\path[draw=drawColor,line width= 0.4pt,line join=round,line cap=round] (160.49,  4.09) -- (164.82,  4.09);

\path[draw=drawColor,line width= 0.4pt,line join=round,line cap=round] (162.65,  1.92) -- (162.65,  6.26);
\end{scope}
\begin{scope}
\path[clip] (  0.00,  0.00) rectangle (209.58,115.63);
\definecolor{drawColor}{RGB}{0,0,0}

\node[text=drawColor,anchor=base west,inner sep=0pt, outer sep=0pt, scale=  0.70] at ( 27.26,  1.68) {DECOR};
\end{scope}
\begin{scope}
\path[clip] (  0.00,  0.00) rectangle (209.58,115.63);
\definecolor{drawColor}{RGB}{0,0,0}

\node[text=drawColor,anchor=base west,inner sep=0pt, outer sep=0pt, scale=  0.70] at ( 69.85,  1.68) {DECOR+};
\end{scope}
\begin{scope}
\path[clip] (  0.00,  0.00) rectangle (209.58,115.63);
\definecolor{drawColor}{RGB}{0,0,0}

\node[text=drawColor,anchor=base west,inner sep=0pt, outer sep=0pt, scale=  0.70] at (117.90,  1.68) {A-DECOR};
\end{scope}
\begin{scope}
\path[clip] (  0.00,  0.00) rectangle (209.58,115.63);
\definecolor{drawColor}{RGB}{0,0,0}

\node[text=drawColor,anchor=base west,inner sep=0pt, outer sep=0pt, scale=  0.70] at (168.07,  1.68) {Brute-Force};
\end{scope}
\end{tikzpicture}

%% file: files/decor_revised_plot_k=log2n_phases_decorplus.tex
\begin{tikzpicture}[x=1pt,y=1pt]
\definecolor{fillColor}{RGB}{255,255,255}
\path[use as bounding box,fill=fillColor,fill opacity=0.00] (0,0) rectangle (209.58,115.63);
\begin{scope}
\path[clip] (  0.00,  0.00) rectangle (209.58,115.63);
\definecolor{drawColor}{RGB}{255,255,255}
\definecolor{fillColor}{RGB}{255,255,255}

\path[draw=drawColor,line width= 0.5pt,line join=round,line cap=round,fill=fillColor] (  0.00,  0.00) rectangle (209.58,115.63);
\end{scope}
\begin{scope}
\path[clip] ( 34.42, 28.73) rectangle (203.58,113.63);
\definecolor{fillColor}{RGB}{255,255,255}

\path[fill=fillColor] ( 34.42, 28.73) rectangle (203.58,113.63);
\definecolor{fillColor}{RGB}{78,155,133}

\path[fill=fillColor] ( 42.11, 28.73) rectangle ( 47.10, 38.63);

\path[fill=fillColor] (110.15, 28.73) rectangle (115.14, 61.63);

\path[fill=fillColor] (155.52, 28.73) rectangle (160.51, 80.75);

\path[fill=fillColor] ( 64.79, 28.73) rectangle ( 69.78, 45.78);

\path[fill=fillColor] ( 87.47, 28.73) rectangle ( 92.46, 53.59);

\path[fill=fillColor] (178.20, 28.73) rectangle (183.19, 90.07);

\path[fill=fillColor] (132.84, 28.73) rectangle (137.83, 70.35);
\definecolor{fillColor}{RGB}{247,192,26}

\path[fill=fillColor] ( 48.46, 28.73) rectangle ( 53.45, 37.65);

\path[fill=fillColor] (116.50, 28.73) rectangle (121.49, 57.55);

\path[fill=fillColor] (161.87, 28.73) rectangle (166.86, 71.86);

\path[fill=fillColor] ( 71.14, 28.73) rectangle ( 76.13, 44.28);

\path[fill=fillColor] ( 93.82, 28.73) rectangle ( 98.81, 50.58);

\path[fill=fillColor] (184.55, 28.73) rectangle (189.54, 78.85);

\path[fill=fillColor] (139.19, 28.73) rectangle (144.18, 64.68);
\definecolor{fillColor}{RGB}{37,122,164}

\path[fill=fillColor] ( 54.81, 28.73) rectangle ( 59.80, 41.36);

\path[fill=fillColor] (122.86, 28.73) rectangle (127.85, 63.22);

\path[fill=fillColor] (168.22, 28.73) rectangle (173.21, 81.37);

\path[fill=fillColor] ( 77.49, 28.73) rectangle ( 82.48, 48.29);

\path[fill=fillColor] (100.17, 28.73) rectangle (105.16, 55.52);

\path[fill=fillColor] (190.90, 28.73) rectangle (195.89, 90.46);

\path[fill=fillColor] (145.54, 28.73) rectangle (150.53, 71.53);
\end{scope}
\begin{scope}
\path[clip] (  0.00,  0.00) rectangle (209.58,115.63);
\definecolor{drawColor}{RGB}{0,0,0}

\path[draw=drawColor,line width= 0.5pt,line join=round] ( 34.42, 28.73) --
	( 34.42,113.63);

\path[draw=drawColor,line width= 0.5pt,line join=round] ( 35.55,111.66) --
	( 34.42,113.63) --
	( 33.28,111.66);
\end{scope}
\begin{scope}
\path[clip] (  0.00,  0.00) rectangle (209.58,115.63);
\definecolor{drawColor}{gray}{0.30}

\node[text=drawColor,anchor=base east,inner sep=0pt, outer sep=0pt, scale=  0.70] at ( 30.37, 42.24) {1e-01};

\node[text=drawColor,anchor=base east,inner sep=0pt, outer sep=0pt, scale=  0.70] at ( 30.37, 63.46) {1e+01};

\node[text=drawColor,anchor=base east,inner sep=0pt, outer sep=0pt, scale=  0.70] at ( 30.37, 84.69) {1e+03};

\node[text=drawColor,anchor=base east,inner sep=0pt, outer sep=0pt, scale=  0.70] at ( 30.37,105.92) {1e+05};
\end{scope}
\begin{scope}
\path[clip] (  0.00,  0.00) rectangle (209.58,115.63);
\definecolor{drawColor}{gray}{0.20}

\path[draw=drawColor,line width= 0.5pt,line join=round] ( 32.17, 44.65) --
	( 34.42, 44.65);

\path[draw=drawColor,line width= 0.5pt,line join=round] ( 32.17, 65.87) --
	( 34.42, 65.87);

\path[draw=drawColor,line width= 0.5pt,line join=round] ( 32.17, 87.10) --
	( 34.42, 87.10);

\path[draw=drawColor,line width= 0.5pt,line join=round] ( 32.17,108.33) --
	( 34.42,108.33);
\end{scope}
\begin{scope}
\path[clip] (  0.00,  0.00) rectangle (209.58,115.63);
\definecolor{drawColor}{RGB}{0,0,0}

\path[draw=drawColor,line width= 0.5pt,line join=round] ( 34.42, 28.73) --
	(203.58, 28.73);

\path[draw=drawColor,line width= 0.5pt,line join=round] (201.61, 27.59) --
	(203.58, 28.73) --
	(201.61, 29.87);
\end{scope}
\begin{scope}
\path[clip] (  0.00,  0.00) rectangle (209.58,115.63);
\definecolor{drawColor}{gray}{0.20}

\path[draw=drawColor,line width= 0.5pt,line join=round] ( 50.95, 26.48) --
	( 50.95, 28.73);

\path[draw=drawColor,line width= 0.5pt,line join=round] ( 73.63, 26.48) --
	( 73.63, 28.73);

\path[draw=drawColor,line width= 0.5pt,line join=round] ( 96.32, 26.48) --
	( 96.32, 28.73);

\path[draw=drawColor,line width= 0.5pt,line join=round] (119.00, 26.48) --
	(119.00, 28.73);

\path[draw=drawColor,line width= 0.5pt,line join=round] (141.68, 26.48) --
	(141.68, 28.73);

\path[draw=drawColor,line width= 0.5pt,line join=round] (164.36, 26.48) --
	(164.36, 28.73);

\path[draw=drawColor,line width= 0.5pt,line join=round] (187.05, 26.48) --
	(187.05, 28.73);
\end{scope}
\begin{scope}
\path[clip] (  0.00,  0.00) rectangle (209.58,115.63);
\definecolor{drawColor}{gray}{0.30}

\node[text=drawColor,anchor=base,inner sep=0pt, outer sep=0pt, scale=  0.70] at ( 50.95, 19.86) {4};

\node[text=drawColor,anchor=base,inner sep=0pt, outer sep=0pt, scale=  0.70] at ( 73.63, 19.86) {6};

\node[text=drawColor,anchor=base,inner sep=0pt, outer sep=0pt, scale=  0.70] at ( 96.32, 19.86) {8};

\node[text=drawColor,anchor=base,inner sep=0pt, outer sep=0pt, scale=  0.70] at (119.00, 19.86) {10};

\node[text=drawColor,anchor=base,inner sep=0pt, outer sep=0pt, scale=  0.70] at (141.68, 19.86) {12};

\node[text=drawColor,anchor=base,inner sep=0pt, outer sep=0pt, scale=  0.70] at (164.36, 19.86) {14};

\node[text=drawColor,anchor=base,inner sep=0pt, outer sep=0pt, scale=  0.70] at (187.05, 19.86) {16};
\end{scope}
\begin{scope}
\path[clip] (  0.00,  0.00) rectangle (209.58,115.63);
\definecolor{drawColor}{RGB}{0,0,0}

\node[text=drawColor,anchor=base,inner sep=0pt, outer sep=0pt, scale=  0.80] at (119.00, 10.74) {$n$};
\end{scope}
\begin{scope}
\path[clip] (  0.00,  0.00) rectangle (209.58,115.63);
\definecolor{drawColor}{RGB}{0,0,0}

\node[text=drawColor,rotate= 90.00,anchor=base,inner sep=0pt, outer sep=0pt, scale=  0.80] at (  7.51, 71.18) {time (ms)};
\end{scope}
\begin{scope}
\path[clip] (  0.00,  0.00) rectangle (209.58,115.63);

\path[] ( 51.04,  1.00) rectangle (186.96,  7.18);
\end{scope}
\begin{scope}
\path[clip] (  0.00,  0.00) rectangle (209.58,115.63);
\definecolor{fillColor}{RGB}{78,155,133}

\path[fill=fillColor] ( 51.75,  1.71) rectangle ( 61.17,  6.47);
\end{scope}
\begin{scope}
\path[clip] (  0.00,  0.00) rectangle (209.58,115.63);
\definecolor{fillColor}{RGB}{78,155,133}

\path[fill=fillColor] ( 51.75,  1.71) rectangle ( 61.17,  6.47);
\end{scope}
\begin{scope}
\path[clip] (  0.00,  0.00) rectangle (209.58,115.63);
\definecolor{fillColor}{RGB}{78,155,133}

\path[fill=fillColor] ( 51.75,  1.71) rectangle ( 61.17,  6.47);
\end{scope}
\begin{scope}
\path[clip] (  0.00,  0.00) rectangle (209.58,115.63);
\definecolor{fillColor}{RGB}{247,192,26}

\path[fill=fillColor] (102.84,  1.71) rectangle (112.26,  6.47);
\end{scope}
\begin{scope}
\path[clip] (  0.00,  0.00) rectangle (209.58,115.63);
\definecolor{fillColor}{RGB}{247,192,26}

\path[fill=fillColor] (102.84,  1.71) rectangle (112.26,  6.47);
\end{scope}
\begin{scope}
\path[clip] (  0.00,  0.00) rectangle (209.58,115.63);
\definecolor{fillColor}{RGB}{247,192,26}

\path[fill=fillColor] (102.84,  1.71) rectangle (112.26,  6.47);
\end{scope}
\begin{scope}
\path[clip] (  0.00,  0.00) rectangle (209.58,115.63);
\definecolor{fillColor}{RGB}{37,122,164}

\path[fill=fillColor] (154.70,  1.71) rectangle (164.11,  6.47);
\end{scope}
\begin{scope}
\path[clip] (  0.00,  0.00) rectangle (209.58,115.63);
\definecolor{fillColor}{RGB}{37,122,164}

\path[fill=fillColor] (154.70,  1.71) rectangle (164.11,  6.47);
\end{scope}
\begin{scope}
\path[clip] (  0.00,  0.00) rectangle (209.58,115.63);
\definecolor{fillColor}{RGB}{37,122,164}

\path[fill=fillColor] (154.70,  1.71) rectangle (164.11,  6.47);
\end{scope}
\begin{scope}
\path[clip] (  0.00,  0.00) rectangle (209.58,115.63);
\definecolor{drawColor}{RGB}{0,0,0}

\node[text=drawColor,anchor=base west,inner sep=0pt, outer sep=0pt, scale=  0.70] at ( 61.88,  1.68) {Candidates};
\end{scope}
\begin{scope}
\path[clip] (  0.00,  0.00) rectangle (209.58,115.63);
\definecolor{drawColor}{RGB}{0,0,0}

\node[text=drawColor,anchor=base west,inner sep=0pt, outer sep=0pt, scale=  0.70] at (112.97,  1.68) {Verification};
\end{scope}
\begin{scope}
\path[clip] (  0.00,  0.00) rectangle (209.58,115.63);
\definecolor{drawColor}{RGB}{0,0,0}

\node[text=drawColor,anchor=base west,inner sep=0pt, outer sep=0pt, scale=  0.70] at (164.82,  1.68) {Total};
\end{scope}
\end{tikzpicture}

%% file: files/decor_revised_plot_k=ndiv2_times.tex
\begin{tikzpicture}[x=1pt,y=1pt]
\definecolor{fillColor}{RGB}{255,255,255}
\path[use as bounding box,fill=fillColor,fill opacity=0.00] (0,0) rectangle (209.58,115.63);
\begin{scope}
\path[clip] (  0.00,  0.00) rectangle (209.58,115.63);
\definecolor{drawColor}{RGB}{255,255,255}
\definecolor{fillColor}{RGB}{255,255,255}

\path[draw=drawColor,line width= 0.5pt,line join=round,line cap=round,fill=fillColor] (  0.00,  0.00) rectangle (209.58,115.63);
\end{scope}
\begin{scope}
\path[clip] ( 34.42, 28.73) rectangle (203.58,113.63);
\definecolor{fillColor}{RGB}{255,255,255}

\path[fill=fillColor] ( 34.42, 28.73) rectangle (203.58,113.63);
\definecolor{drawColor}{RGB}{247,192,26}

\path[draw=drawColor,line width= 0.5pt,line join=round] ( 42.11, 38.47) --
	( 67.74, 46.55) --
	( 93.37, 53.88) --
	(119.00, 61.84) --
	(144.63, 69.26) --
	(170.26, 77.58) --
	(195.89, 84.86);
\definecolor{drawColor}{RGB}{78,155,133}

\path[draw=drawColor,line width= 0.5pt,dash pattern=on 4pt off 4pt ,line join=round] ( 42.11, 41.36) --
	( 67.74, 48.55) --
	( 93.37, 55.72) --
	(119.00, 63.32) --
	(144.63, 70.58) --
	(170.26, 78.56) --
	(195.89, 85.74);
\definecolor{drawColor}{RGB}{37,122,164}

\path[draw=drawColor,line width= 0.5pt,dash pattern=on 1pt off 3pt ,line join=round] ( 42.11, 45.59) --
	( 67.74, 56.13) --
	( 93.37, 65.73) --
	(119.00, 75.33) --
	(144.63, 83.93) --
	(170.26, 92.69) --
	(195.89,101.94);
\definecolor{drawColor}{RGB}{86,51,94}

\path[draw=drawColor,line width= 0.5pt,dash pattern=on 7pt off 3pt ,line join=round] ( 42.11, 47.47) --
	( 67.74, 59.27) --
	( 93.37, 73.39) --
	(119.00, 86.69) --
	(144.63, 99.38);
\definecolor{fillColor}{RGB}{78,155,133}

\path[fill=fillColor] ( 42.11, 43.75) --
	( 44.17, 40.17) --
	( 40.04, 40.17) --
	cycle;

\path[draw=drawColor,line width= 0.4pt,line join=round,line cap=round] ( 40.57, 45.93) -- ( 43.64, 49.00);

\path[draw=drawColor,line width= 0.4pt,line join=round,line cap=round] ( 40.57, 49.00) -- ( 43.64, 45.93);

\path[draw=drawColor,line width= 0.4pt,line join=round,line cap=round] ( 39.94, 47.47) -- ( 44.27, 47.47);

\path[draw=drawColor,line width= 0.4pt,line join=round,line cap=round] ( 42.11, 45.30) -- ( 42.11, 49.64);
\definecolor{fillColor}{RGB}{37,122,164}

\path[fill=fillColor] (168.73, 91.16) --
	(171.80, 91.16) --
	(171.80, 94.23) --
	(168.73, 94.23) --
	cycle;

\path[draw=drawColor,line width= 0.4pt,line join=round,line cap=round] (117.47, 85.15) -- (120.53, 88.22);

\path[draw=drawColor,line width= 0.4pt,line join=round,line cap=round] (117.47, 88.22) -- (120.53, 85.15);

\path[draw=drawColor,line width= 0.4pt,line join=round,line cap=round] (116.83, 86.69) -- (121.17, 86.69);

\path[draw=drawColor,line width= 0.4pt,line join=round,line cap=round] (119.00, 84.52) -- (119.00, 88.86);

\path[fill=fillColor] ( 40.57, 44.05) --
	( 43.64, 44.05) --
	( 43.64, 47.12) --
	( 40.57, 47.12) --
	cycle;
\definecolor{fillColor}{RGB}{78,155,133}

\path[fill=fillColor] (170.26, 80.95) --
	(172.33, 77.37) --
	(168.20, 77.37) --
	cycle;
\definecolor{fillColor}{RGB}{37,122,164}

\path[fill=fillColor] ( 66.20, 54.60) --
	( 69.27, 54.60) --
	( 69.27, 57.67) --
	( 66.20, 57.67) --
	cycle;
\definecolor{fillColor}{RGB}{78,155,133}

\path[fill=fillColor] (144.63, 72.96) --
	(146.70, 69.39) --
	(142.57, 69.39) --
	cycle;

\path[draw=drawColor,line width= 0.4pt,line join=round,line cap=round] ( 66.20, 57.74) -- ( 69.27, 60.81);

\path[draw=drawColor,line width= 0.4pt,line join=round,line cap=round] ( 66.20, 60.81) -- ( 69.27, 57.74);

\path[draw=drawColor,line width= 0.4pt,line join=round,line cap=round] ( 65.57, 59.27) -- ( 69.91, 59.27);

\path[draw=drawColor,line width= 0.4pt,line join=round,line cap=round] ( 67.74, 57.11) -- ( 67.74, 61.44);

\path[fill=fillColor] ( 93.37, 58.10) --
	( 95.43, 54.52) --
	( 91.30, 54.52) --
	cycle;
\definecolor{fillColor}{RGB}{37,122,164}

\path[fill=fillColor] (143.10, 82.40) --
	(146.16, 82.40) --
	(146.16, 85.46) --
	(143.10, 85.46) --
	cycle;
\definecolor{fillColor}{RGB}{78,155,133}

\path[fill=fillColor] (195.89, 88.12) --
	(197.96, 84.55) --
	(193.83, 84.55) --
	cycle;

\path[draw=drawColor,line width= 0.4pt,line join=round,line cap=round] (143.10, 97.85) -- (146.16,100.91);

\path[draw=drawColor,line width= 0.4pt,line join=round,line cap=round] (143.10,100.91) -- (146.16, 97.85);

\path[draw=drawColor,line width= 0.4pt,line join=round,line cap=round] (142.46, 99.38) -- (146.80, 99.38);

\path[draw=drawColor,line width= 0.4pt,line join=round,line cap=round] (144.63, 97.21) -- (144.63,101.55);

\path[fill=fillColor] ( 67.74, 50.94) --
	( 69.80, 47.36) --
	( 65.67, 47.36) --
	cycle;

\path[draw=drawColor,line width= 0.4pt,line join=round,line cap=round] ( 91.83, 71.86) -- ( 94.90, 74.92);

\path[draw=drawColor,line width= 0.4pt,line join=round,line cap=round] ( 91.83, 74.92) -- ( 94.90, 71.86);

\path[draw=drawColor,line width= 0.4pt,line join=round,line cap=round] ( 91.20, 73.39) -- ( 95.54, 73.39);

\path[draw=drawColor,line width= 0.4pt,line join=round,line cap=round] ( 93.37, 71.22) -- ( 93.37, 75.56);
\definecolor{drawColor}{RGB}{247,192,26}
\definecolor{fillColor}{RGB}{247,192,26}

\path[draw=drawColor,line width= 0.4pt,line join=round,line cap=round,fill=fillColor] ( 93.37, 53.88) circle (  1.53);

\path[draw=drawColor,line width= 0.4pt,line join=round,line cap=round,fill=fillColor] (119.00, 61.84) circle (  1.53);

\path[draw=drawColor,line width= 0.4pt,line join=round,line cap=round,fill=fillColor] (170.26, 77.58) circle (  1.53);

\path[draw=drawColor,line width= 0.4pt,line join=round,line cap=round,fill=fillColor] (144.63, 69.26) circle (  1.53);

\path[draw=drawColor,line width= 0.4pt,line join=round,line cap=round,fill=fillColor] (195.89, 84.86) circle (  1.53);

\path[draw=drawColor,line width= 0.4pt,line join=round,line cap=round,fill=fillColor] ( 67.74, 46.55) circle (  1.53);
\definecolor{fillColor}{RGB}{78,155,133}

\path[fill=fillColor] (119.00, 65.71) --
	(121.06, 62.13) --
	(116.93, 62.13) --
	cycle;
\definecolor{fillColor}{RGB}{37,122,164}

\path[fill=fillColor] (117.47, 73.79) --
	(120.53, 73.79) --
	(120.53, 76.86) --
	(117.47, 76.86) --
	cycle;

\path[fill=fillColor] ( 91.83, 64.20) --
	( 94.90, 64.20) --
	( 94.90, 67.27) --
	( 91.83, 67.27) --
	cycle;

\path[fill=fillColor] (194.36,100.40) --
	(197.43,100.40) --
	(197.43,103.47) --
	(194.36,103.47) --
	cycle;
\definecolor{fillColor}{RGB}{247,192,26}

\path[draw=drawColor,line width= 0.4pt,line join=round,line cap=round,fill=fillColor] ( 42.11, 38.47) circle (  1.53);
\end{scope}
\begin{scope}
\path[clip] (  0.00,  0.00) rectangle (209.58,115.63);
\definecolor{drawColor}{RGB}{0,0,0}

\path[draw=drawColor,line width= 0.5pt,line join=round] ( 34.42, 28.73) --
	( 34.42,113.63);

\path[draw=drawColor,line width= 0.5pt,line join=round] ( 35.55,111.66) --
	( 34.42,113.63) --
	( 33.28,111.66);
\end{scope}
\begin{scope}
\path[clip] (  0.00,  0.00) rectangle (209.58,115.63);
\definecolor{drawColor}{gray}{0.30}

\node[text=drawColor,anchor=base east,inner sep=0pt, outer sep=0pt, scale=  0.70] at ( 30.37, 42.24) {1e-01};

\node[text=drawColor,anchor=base east,inner sep=0pt, outer sep=0pt, scale=  0.70] at ( 30.37, 63.46) {1e+01};

\node[text=drawColor,anchor=base east,inner sep=0pt, outer sep=0pt, scale=  0.70] at ( 30.37, 84.69) {1e+03};

\node[text=drawColor,anchor=base east,inner sep=0pt, outer sep=0pt, scale=  0.70] at ( 30.37,105.92) {1e+05};
\end{scope}
\begin{scope}
\path[clip] (  0.00,  0.00) rectangle (209.58,115.63);
\definecolor{drawColor}{gray}{0.20}

\path[draw=drawColor,line width= 0.5pt,line join=round] ( 32.17, 44.65) --
	( 34.42, 44.65);

\path[draw=drawColor,line width= 0.5pt,line join=round] ( 32.17, 65.87) --
	( 34.42, 65.87);

\path[draw=drawColor,line width= 0.5pt,line join=round] ( 32.17, 87.10) --
	( 34.42, 87.10);

\path[draw=drawColor,line width= 0.5pt,line join=round] ( 32.17,108.33) --
	( 34.42,108.33);
\end{scope}
\begin{scope}
\path[clip] (  0.00,  0.00) rectangle (209.58,115.63);
\definecolor{drawColor}{RGB}{0,0,0}

\path[draw=drawColor,line width= 0.5pt,line join=round] ( 34.42, 28.73) --
	(203.58, 28.73);

\path[draw=drawColor,line width= 0.5pt,line join=round] (201.61, 27.59) --
	(203.58, 28.73) --
	(201.61, 29.87);
\end{scope}
\begin{scope}
\path[clip] (  0.00,  0.00) rectangle (209.58,115.63);
\definecolor{drawColor}{gray}{0.20}

\path[draw=drawColor,line width= 0.5pt,line join=round] ( 42.11, 26.48) --
	( 42.11, 28.73);

\path[draw=drawColor,line width= 0.5pt,line join=round] ( 93.37, 26.48) --
	( 93.37, 28.73);

\path[draw=drawColor,line width= 0.5pt,line join=round] (144.63, 26.48) --
	(144.63, 28.73);

\path[draw=drawColor,line width= 0.5pt,line join=round] (195.89, 26.48) --
	(195.89, 28.73);
\end{scope}
\begin{scope}
\path[clip] (  0.00,  0.00) rectangle (209.58,115.63);
\definecolor{drawColor}{gray}{0.30}

\node[text=drawColor,anchor=base,inner sep=0pt, outer sep=0pt, scale=  0.70] at ( 42.11, 19.86) {4};

\node[text=drawColor,anchor=base,inner sep=0pt, outer sep=0pt, scale=  0.70] at ( 93.37, 19.86) {8};

\node[text=drawColor,anchor=base,inner sep=0pt, outer sep=0pt, scale=  0.70] at (144.63, 19.86) {12};

\node[text=drawColor,anchor=base,inner sep=0pt, outer sep=0pt, scale=  0.70] at (195.89, 19.86) {16};
\end{scope}
\begin{scope}
\path[clip] (  0.00,  0.00) rectangle (209.58,115.63);
\definecolor{drawColor}{RGB}{0,0,0}

\node[text=drawColor,anchor=base,inner sep=0pt, outer sep=0pt, scale=  0.80] at (119.00, 10.74) {$n$};
\end{scope}
\begin{scope}
\path[clip] (  0.00,  0.00) rectangle (209.58,115.63);
\definecolor{drawColor}{RGB}{0,0,0}

\node[text=drawColor,rotate= 90.00,anchor=base,inner sep=0pt, outer sep=0pt, scale=  0.80] at (  7.51, 71.18) {time (ms)};
\end{scope}
\begin{scope}
\path[clip] (  0.00,  0.00) rectangle (209.58,115.63);

\path[] ( 34.42,  1.00) rectangle (210.27,  7.18);
\end{scope}
\begin{scope}
\path[clip] (  0.00,  0.00) rectangle (209.58,115.63);
\definecolor{drawColor}{RGB}{247,192,26}

\path[draw=drawColor,line width= 0.5pt,line join=round] ( 17.50,  4.09) -- ( 26.17,  4.09);
\end{scope}
\begin{scope}
\path[clip] (  0.00,  0.00) rectangle (209.58,115.63);
\definecolor{drawColor}{RGB}{247,192,26}
\definecolor{fillColor}{RGB}{247,192,26}

\path[draw=drawColor,line width= 0.4pt,line join=round,line cap=round,fill=fillColor] ( 21.84,  4.09) circle (  1.53);
\end{scope}
\begin{scope}
\path[clip] (  0.00,  0.00) rectangle (209.58,115.63);
\definecolor{drawColor}{RGB}{78,155,133}

\path[draw=drawColor,line width= 0.5pt,dash pattern=on 4pt off 4pt ,line join=round] ( 60.10,  4.09) -- ( 68.77,  4.09);
\end{scope}
\begin{scope}
\path[clip] (  0.00,  0.00) rectangle (209.58,115.63);
\definecolor{fillColor}{RGB}{78,155,133}

\path[fill=fillColor] ( 64.43,  6.48) --
	( 66.50,  2.90) --
	( 62.37,  2.90) --
	cycle;
\end{scope}
\begin{scope}
\path[clip] (  0.00,  0.00) rectangle (209.58,115.63);
\definecolor{drawColor}{RGB}{37,122,164}

\path[draw=drawColor,line width= 0.5pt,dash pattern=on 1pt off 3pt ,line join=round] (108.14,  4.09) -- (116.81,  4.09);
\end{scope}
\begin{scope}
\path[clip] (  0.00,  0.00) rectangle (209.58,115.63);
\definecolor{fillColor}{RGB}{37,122,164}

\path[fill=fillColor] (110.94,  2.56) --
	(114.01,  2.56) --
	(114.01,  5.62) --
	(110.94,  5.62) --
	cycle;
\end{scope}
\begin{scope}
\path[clip] (  0.00,  0.00) rectangle (209.58,115.63);
\definecolor{drawColor}{RGB}{86,51,94}

\path[draw=drawColor,line width= 0.5pt,dash pattern=on 7pt off 3pt ,line join=round] (158.32,  4.09) -- (166.99,  4.09);
\end{scope}
\begin{scope}
\path[clip] (  0.00,  0.00) rectangle (209.58,115.63);
\definecolor{drawColor}{RGB}{86,51,94}

\path[draw=drawColor,line width= 0.4pt,line join=round,line cap=round] (161.12,  2.56) -- (164.19,  5.62);

\path[draw=drawColor,line width= 0.4pt,line join=round,line cap=round] (161.12,  5.62) -- (164.19,  2.56);

\path[draw=drawColor,line width= 0.4pt,line join=round,line cap=round] (160.49,  4.09) -- (164.82,  4.09);

\path[draw=drawColor,line width= 0.4pt,line join=round,line cap=round] (162.65,  1.92) -- (162.65,  6.26);
\end{scope}
\begin{scope}
\path[clip] (  0.00,  0.00) rectangle (209.58,115.63);
\definecolor{drawColor}{RGB}{0,0,0}

\node[text=drawColor,anchor=base west,inner sep=0pt, outer sep=0pt, scale=  0.70] at ( 27.26,  1.68) {DECOR};
\end{scope}
\begin{scope}
\path[clip] (  0.00,  0.00) rectangle (209.58,115.63);
\definecolor{drawColor}{RGB}{0,0,0}

\node[text=drawColor,anchor=base west,inner sep=0pt, outer sep=0pt, scale=  0.70] at ( 69.85,  1.68) {DECOR+};
\end{scope}
\begin{scope}
\path[clip] (  0.00,  0.00) rectangle (209.58,115.63);
\definecolor{drawColor}{RGB}{0,0,0}

\node[text=drawColor,anchor=base west,inner sep=0pt, outer sep=0pt, scale=  0.70] at (117.90,  1.68) {A-DECOR};
\end{scope}
\begin{scope}
\path[clip] (  0.00,  0.00) rectangle (209.58,115.63);
\definecolor{drawColor}{RGB}{0,0,0}

\node[text=drawColor,anchor=base west,inner sep=0pt, outer sep=0pt, scale=  0.70] at (168.07,  1.68) {Brute-Force};
\end{scope}
\end{tikzpicture}

%% file: files/decor_revised_plot_k=ndiv2_phases_decorplus.tex
\begin{tikzpicture}[x=1pt,y=1pt]
\definecolor{fillColor}{RGB}{255,255,255}
\path[use as bounding box,fill=fillColor,fill opacity=0.00] (0,0) rectangle (209.58,115.63);
\begin{scope}
\path[clip] (  0.00,  0.00) rectangle (209.58,115.63);
\definecolor{drawColor}{RGB}{255,255,255}
\definecolor{fillColor}{RGB}{255,255,255}

\path[draw=drawColor,line width= 0.5pt,line join=round,line cap=round,fill=fillColor] (  0.00,  0.00) rectangle (209.58,115.63);
\end{scope}
\begin{scope}
\path[clip] ( 34.42, 28.73) rectangle (203.58,113.63);
\definecolor{fillColor}{RGB}{255,255,255}

\path[fill=fillColor] ( 34.42, 28.73) rectangle (203.58,113.63);
\definecolor{fillColor}{RGB}{78,155,133}

\path[fill=fillColor] ( 42.11, 28.73) rectangle ( 47.10, 38.63);

\path[fill=fillColor] (155.52, 28.73) rectangle (160.51, 77.61);

\path[fill=fillColor] (132.84, 28.73) rectangle (137.83, 69.30);

\path[fill=fillColor] ( 87.47, 28.73) rectangle ( 92.46, 53.90);

\path[fill=fillColor] (178.20, 28.73) rectangle (183.19, 84.84);

\path[fill=fillColor] ( 64.79, 28.73) rectangle ( 69.78, 46.55);

\path[fill=fillColor] (110.15, 28.73) rectangle (115.14, 61.87);
\definecolor{fillColor}{RGB}{247,192,26}

\path[fill=fillColor] ( 48.46, 28.73) rectangle ( 53.45, 37.65);

\path[fill=fillColor] (161.87, 28.73) rectangle (166.86, 70.81);

\path[fill=fillColor] (139.19, 28.73) rectangle (144.18, 64.04);

\path[fill=fillColor] ( 93.82, 28.73) rectangle ( 98.81, 50.55);

\path[fill=fillColor] (184.55, 28.73) rectangle (189.54, 77.75);

\path[fill=fillColor] ( 71.14, 28.73) rectangle ( 76.13, 43.74);

\path[fill=fillColor] (116.50, 28.73) rectangle (121.49, 57.29);
\definecolor{fillColor}{RGB}{37,122,164}

\path[fill=fillColor] ( 54.81, 28.73) rectangle ( 59.80, 41.36);

\path[fill=fillColor] (168.22, 28.73) rectangle (173.21, 78.56);

\path[fill=fillColor] (145.54, 28.73) rectangle (150.53, 70.58);

\path[fill=fillColor] (100.17, 28.73) rectangle (105.16, 55.72);

\path[fill=fillColor] (190.90, 28.73) rectangle (195.89, 85.74);

\path[fill=fillColor] ( 77.49, 28.73) rectangle ( 82.48, 48.55);

\path[fill=fillColor] (122.86, 28.73) rectangle (127.85, 63.32);
\end{scope}
\begin{scope}
\path[clip] (  0.00,  0.00) rectangle (209.58,115.63);
\definecolor{drawColor}{RGB}{0,0,0}

\path[draw=drawColor,line width= 0.5pt,line join=round] ( 34.42, 28.73) --
	( 34.42,113.63);

\path[draw=drawColor,line width= 0.5pt,line join=round] ( 35.55,111.66) --
	( 34.42,113.63) --
	( 33.28,111.66);
\end{scope}
\begin{scope}
\path[clip] (  0.00,  0.00) rectangle (209.58,115.63);
\definecolor{drawColor}{gray}{0.30}

\node[text=drawColor,anchor=base east,inner sep=0pt, outer sep=0pt, scale=  0.70] at ( 30.37, 42.24) {1e-01};

\node[text=drawColor,anchor=base east,inner sep=0pt, outer sep=0pt, scale=  0.70] at ( 30.37, 63.46) {1e+01};

\node[text=drawColor,anchor=base east,inner sep=0pt, outer sep=0pt, scale=  0.70] at ( 30.37, 84.69) {1e+03};

\node[text=drawColor,anchor=base east,inner sep=0pt, outer sep=0pt, scale=  0.70] at ( 30.37,105.92) {1e+05};
\end{scope}
\begin{scope}
\path[clip] (  0.00,  0.00) rectangle (209.58,115.63);
\definecolor{drawColor}{gray}{0.20}

\path[draw=drawColor,line width= 0.5pt,line join=round] ( 32.17, 44.65) --
	( 34.42, 44.65);

\path[draw=drawColor,line width= 0.5pt,line join=round] ( 32.17, 65.87) --
	( 34.42, 65.87);

\path[draw=drawColor,line width= 0.5pt,line join=round] ( 32.17, 87.10) --
	( 34.42, 87.10);

\path[draw=drawColor,line width= 0.5pt,line join=round] ( 32.17,108.33) --
	( 34.42,108.33);
\end{scope}
\begin{scope}
\path[clip] (  0.00,  0.00) rectangle (209.58,115.63);
\definecolor{drawColor}{RGB}{0,0,0}

\path[draw=drawColor,line width= 0.5pt,line join=round] ( 34.42, 28.73) --
	(203.58, 28.73);

\path[draw=drawColor,line width= 0.5pt,line join=round] (201.61, 27.59) --
	(203.58, 28.73) --
	(201.61, 29.87);
\end{scope}
\begin{scope}
\path[clip] (  0.00,  0.00) rectangle (209.58,115.63);
\definecolor{drawColor}{gray}{0.20}

\path[draw=drawColor,line width= 0.5pt,line join=round] ( 50.95, 26.48) --
	( 50.95, 28.73);

\path[draw=drawColor,line width= 0.5pt,line join=round] ( 73.63, 26.48) --
	( 73.63, 28.73);

\path[draw=drawColor,line width= 0.5pt,line join=round] ( 96.32, 26.48) --
	( 96.32, 28.73);

\path[draw=drawColor,line width= 0.5pt,line join=round] (119.00, 26.48) --
	(119.00, 28.73);

\path[draw=drawColor,line width= 0.5pt,line join=round] (141.68, 26.48) --
	(141.68, 28.73);

\path[draw=drawColor,line width= 0.5pt,line join=round] (164.36, 26.48) --
	(164.36, 28.73);

\path[draw=drawColor,line width= 0.5pt,line join=round] (187.05, 26.48) --
	(187.05, 28.73);
\end{scope}
\begin{scope}
\path[clip] (  0.00,  0.00) rectangle (209.58,115.63);
\definecolor{drawColor}{gray}{0.30}

\node[text=drawColor,anchor=base,inner sep=0pt, outer sep=0pt, scale=  0.70] at ( 50.95, 19.86) {4};

\node[text=drawColor,anchor=base,inner sep=0pt, outer sep=0pt, scale=  0.70] at ( 73.63, 19.86) {6};

\node[text=drawColor,anchor=base,inner sep=0pt, outer sep=0pt, scale=  0.70] at ( 96.32, 19.86) {8};

\node[text=drawColor,anchor=base,inner sep=0pt, outer sep=0pt, scale=  0.70] at (119.00, 19.86) {10};

\node[text=drawColor,anchor=base,inner sep=0pt, outer sep=0pt, scale=  0.70] at (141.68, 19.86) {12};

\node[text=drawColor,anchor=base,inner sep=0pt, outer sep=0pt, scale=  0.70] at (164.36, 19.86) {14};

\node[text=drawColor,anchor=base,inner sep=0pt, outer sep=0pt, scale=  0.70] at (187.05, 19.86) {16};
\end{scope}
\begin{scope}
\path[clip] (  0.00,  0.00) rectangle (209.58,115.63);
\definecolor{drawColor}{RGB}{0,0,0}

\node[text=drawColor,anchor=base,inner sep=0pt, outer sep=0pt, scale=  0.80] at (119.00, 10.74) {$n$};
\end{scope}
\begin{scope}
\path[clip] (  0.00,  0.00) rectangle (209.58,115.63);
\definecolor{drawColor}{RGB}{0,0,0}

\node[text=drawColor,rotate= 90.00,anchor=base,inner sep=0pt, outer sep=0pt, scale=  0.80] at (  7.51, 71.18) {time (ms)};
\end{scope}
\begin{scope}
\path[clip] (  0.00,  0.00) rectangle (209.58,115.63);

\path[] ( 51.04,  1.00) rectangle (186.96,  7.18);
\end{scope}
\begin{scope}
\path[clip] (  0.00,  0.00) rectangle (209.58,115.63);
\definecolor{fillColor}{RGB}{78,155,133}

\path[fill=fillColor] ( 51.75,  1.71) rectangle ( 61.17,  6.47);
\end{scope}
\begin{scope}
\path[clip] (  0.00,  0.00) rectangle (209.58,115.63);
\definecolor{fillColor}{RGB}{78,155,133}

\path[fill=fillColor] ( 51.75,  1.71) rectangle ( 61.17,  6.47);
\end{scope}
\begin{scope}
\path[clip] (  0.00,  0.00) rectangle (209.58,115.63);
\definecolor{fillColor}{RGB}{78,155,133}

\path[fill=fillColor] ( 51.75,  1.71) rectangle ( 61.17,  6.47);
\end{scope}
\begin{scope}
\path[clip] (  0.00,  0.00) rectangle (209.58,115.63);
\definecolor{fillColor}{RGB}{247,192,26}

\path[fill=fillColor] (102.84,  1.71) rectangle (112.26,  6.47);
\end{scope}
\begin{scope}
\path[clip] (  0.00,  0.00) rectangle (209.58,115.63);
\definecolor{fillColor}{RGB}{247,192,26}

\path[fill=fillColor] (102.84,  1.71) rectangle (112.26,  6.47);
\end{scope}
\begin{scope}
\path[clip] (  0.00,  0.00) rectangle (209.58,115.63);
\definecolor{fillColor}{RGB}{247,192,26}

\path[fill=fillColor] (102.84,  1.71) rectangle (112.26,  6.47);
\end{scope}
\begin{scope}
\path[clip] (  0.00,  0.00) rectangle (209.58,115.63);
\definecolor{fillColor}{RGB}{37,122,164}

\path[fill=fillColor] (154.70,  1.71) rectangle (164.11,  6.47);
\end{scope}
\begin{scope}
\path[clip] (  0.00,  0.00) rectangle (209.58,115.63);
\definecolor{fillColor}{RGB}{37,122,164}

\path[fill=fillColor] (154.70,  1.71) rectangle (164.11,  6.47);
\end{scope}
\begin{scope}
\path[clip] (  0.00,  0.00) rectangle (209.58,115.63);
\definecolor{fillColor}{RGB}{37,122,164}

\path[fill=fillColor] (154.70,  1.71) rectangle (164.11,  6.47);
\end{scope}
\begin{scope}
\path[clip] (  0.00,  0.00) rectangle (209.58,115.63);
\definecolor{drawColor}{RGB}{0,0,0}

\node[text=drawColor,anchor=base west,inner sep=0pt, outer sep=0pt, scale=  0.70] at ( 61.88,  1.68) {Candidates};
\end{scope}
\begin{scope}
\path[clip] (  0.00,  0.00) rectangle (209.58,115.63);
\definecolor{drawColor}{RGB}{0,0,0}

\node[text=drawColor,anchor=base west,inner sep=0pt, outer sep=0pt, scale=  0.70] at (112.97,  1.68) {Verification};
\end{scope}
\begin{scope}
\path[clip] (  0.00,  0.00) rectangle (209.58,115.63);
\definecolor{drawColor}{RGB}{0,0,0}

\node[text=drawColor,anchor=base west,inner sep=0pt, outer sep=0pt, scale=  0.70] at (164.82,  1.68) {Total};
\end{scope}
\end{tikzpicture}

%% file: files/decor_revised_plot_k=nsub1_times.tex
\begin{tikzpicture}[x=1pt,y=1pt]
\definecolor{fillColor}{RGB}{255,255,255}
\path[use as bounding box,fill=fillColor,fill opacity=0.00] (0,0) rectangle (209.58,115.63);
\begin{scope}
\path[clip] (  0.00,  0.00) rectangle (209.58,115.63);
\definecolor{drawColor}{RGB}{255,255,255}
\definecolor{fillColor}{RGB}{255,255,255}

\path[draw=drawColor,line width= 0.5pt,line join=round,line cap=round,fill=fillColor] (  0.00,  0.00) rectangle (209.58,115.63);
\end{scope}
\begin{scope}
\path[clip] ( 34.42, 28.73) rectangle (203.58,113.63);
\definecolor{fillColor}{RGB}{255,255,255}

\path[fill=fillColor] ( 34.42, 28.73) rectangle (203.58,113.63);
\definecolor{drawColor}{RGB}{247,192,26}

\path[draw=drawColor,line width= 0.5pt,line join=round] ( 42.11, 39.14) --
	( 67.74, 47.09) --
	( 93.37, 54.10) --
	(119.00, 62.03) --
	(144.63, 68.99) --
	(170.26, 77.04) --
	(195.89, 84.18);
\definecolor{drawColor}{RGB}{78,155,133}

\path[draw=drawColor,line width= 0.5pt,dash pattern=on 4pt off 4pt ,line join=round] ( 42.11, 41.63) --
	( 67.74, 48.95) --
	( 93.37, 55.86) --
	(119.00, 63.38) --
	(144.63, 70.28) --
	(170.26, 78.07) --
	(195.89, 85.18);
\definecolor{drawColor}{RGB}{37,122,164}

\path[draw=drawColor,line width= 0.5pt,dash pattern=on 1pt off 3pt ,line join=round] ( 42.11, 45.73) --
	( 67.74, 56.43) --
	( 93.37, 66.22) --
	(119.00, 75.47) --
	(144.63, 84.29) --
	(170.26, 93.67) --
	(195.89,102.78);
\definecolor{drawColor}{RGB}{86,51,94}

\path[draw=drawColor,line width= 0.5pt,dash pattern=on 7pt off 3pt ,line join=round] ( 42.11, 43.40) --
	( 67.74, 48.51) --
	( 93.37, 59.18) --
	(119.00, 61.90) --
	(144.63, 74.21) --
	(170.26, 81.93) --
	(195.89, 89.91);
\definecolor{fillColor}{RGB}{78,155,133}

\path[fill=fillColor] ( 67.74, 51.34) --
	( 69.80, 47.76) --
	( 65.67, 47.76) --
	cycle;
\definecolor{fillColor}{RGB}{37,122,164}

\path[fill=fillColor] ( 91.83, 64.68) --
	( 94.90, 64.68) --
	( 94.90, 67.75) --
	( 91.83, 67.75) --
	cycle;
\definecolor{drawColor}{RGB}{247,192,26}
\definecolor{fillColor}{RGB}{247,192,26}

\path[draw=drawColor,line width= 0.4pt,line join=round,line cap=round,fill=fillColor] (144.63, 68.99) circle (  1.53);

\path[draw=drawColor,line width= 0.4pt,line join=round,line cap=round,fill=fillColor] (170.26, 77.04) circle (  1.53);
\definecolor{fillColor}{RGB}{78,155,133}

\path[fill=fillColor] ( 42.11, 44.01) --
	( 44.17, 40.43) --
	( 40.04, 40.43) --
	cycle;
\definecolor{fillColor}{RGB}{247,192,26}

\path[draw=drawColor,line width= 0.4pt,line join=round,line cap=round,fill=fillColor] ( 67.74, 47.09) circle (  1.53);
\definecolor{fillColor}{RGB}{78,155,133}

\path[fill=fillColor] (170.26, 80.46) --
	(172.33, 76.88) --
	(168.20, 76.88) --
	cycle;
\definecolor{drawColor}{RGB}{86,51,94}

\path[draw=drawColor,line width= 0.4pt,line join=round,line cap=round] (168.73, 80.39) -- (171.80, 83.46);

\path[draw=drawColor,line width= 0.4pt,line join=round,line cap=round] (168.73, 83.46) -- (171.80, 80.39);

\path[draw=drawColor,line width= 0.4pt,line join=round,line cap=round] (168.09, 81.93) -- (172.43, 81.93);

\path[draw=drawColor,line width= 0.4pt,line join=round,line cap=round] (170.26, 79.76) -- (170.26, 84.10);

\path[draw=drawColor,line width= 0.4pt,line join=round,line cap=round] (143.10, 72.67) -- (146.16, 75.74);

\path[draw=drawColor,line width= 0.4pt,line join=round,line cap=round] (143.10, 75.74) -- (146.16, 72.67);

\path[draw=drawColor,line width= 0.4pt,line join=round,line cap=round] (142.46, 74.21) -- (146.80, 74.21);

\path[draw=drawColor,line width= 0.4pt,line join=round,line cap=round] (144.63, 72.04) -- (144.63, 76.38);

\path[fill=fillColor] (195.89, 87.56) --
	(197.96, 83.99) --
	(193.83, 83.99) --
	cycle;

\path[fill=fillColor] (119.00, 65.76) --
	(121.06, 62.18) --
	(116.93, 62.18) --
	cycle;
\definecolor{fillColor}{RGB}{37,122,164}

\path[fill=fillColor] ( 66.20, 54.90) --
	( 69.27, 54.90) --
	( 69.27, 57.97) --
	( 66.20, 57.97) --
	cycle;

\path[draw=drawColor,line width= 0.4pt,line join=round,line cap=round] (194.36, 88.38) -- (197.43, 91.44);

\path[draw=drawColor,line width= 0.4pt,line join=round,line cap=round] (194.36, 91.44) -- (197.43, 88.38);

\path[draw=drawColor,line width= 0.4pt,line join=round,line cap=round] (193.72, 89.91) -- (198.06, 89.91);

\path[draw=drawColor,line width= 0.4pt,line join=round,line cap=round] (195.89, 87.74) -- (195.89, 92.08);
\definecolor{fillColor}{RGB}{78,155,133}

\path[fill=fillColor] (144.63, 72.67) --
	(146.70, 69.09) --
	(142.57, 69.09) --
	cycle;
\definecolor{fillColor}{RGB}{37,122,164}

\path[fill=fillColor] ( 40.57, 44.19) --
	( 43.64, 44.19) --
	( 43.64, 47.26) --
	( 40.57, 47.26) --
	cycle;
\definecolor{drawColor}{RGB}{247,192,26}
\definecolor{fillColor}{RGB}{247,192,26}

\path[draw=drawColor,line width= 0.4pt,line join=round,line cap=round,fill=fillColor] ( 42.11, 39.14) circle (  1.53);
\definecolor{fillColor}{RGB}{37,122,164}

\path[fill=fillColor] (194.36,101.24) --
	(197.43,101.24) --
	(197.43,104.31) --
	(194.36,104.31) --
	cycle;
\definecolor{drawColor}{RGB}{86,51,94}

\path[draw=drawColor,line width= 0.4pt,line join=round,line cap=round] ( 66.20, 46.97) -- ( 69.27, 50.04);

\path[draw=drawColor,line width= 0.4pt,line join=round,line cap=round] ( 66.20, 50.04) -- ( 69.27, 46.97);

\path[draw=drawColor,line width= 0.4pt,line join=round,line cap=round] ( 65.57, 48.51) -- ( 69.91, 48.51);

\path[draw=drawColor,line width= 0.4pt,line join=round,line cap=round] ( 67.74, 46.34) -- ( 67.74, 50.68);

\path[draw=drawColor,line width= 0.4pt,line join=round,line cap=round] ( 91.83, 57.65) -- ( 94.90, 60.71);

\path[draw=drawColor,line width= 0.4pt,line join=round,line cap=round] ( 91.83, 60.71) -- ( 94.90, 57.65);

\path[draw=drawColor,line width= 0.4pt,line join=round,line cap=round] ( 91.20, 59.18) -- ( 95.54, 59.18);

\path[draw=drawColor,line width= 0.4pt,line join=round,line cap=round] ( 93.37, 57.01) -- ( 93.37, 61.35);
\definecolor{drawColor}{RGB}{247,192,26}
\definecolor{fillColor}{RGB}{247,192,26}

\path[draw=drawColor,line width= 0.4pt,line join=round,line cap=round,fill=fillColor] ( 93.37, 54.10) circle (  1.53);
\definecolor{fillColor}{RGB}{37,122,164}

\path[fill=fillColor] (168.73, 92.14) --
	(171.80, 92.14) --
	(171.80, 95.21) --
	(168.73, 95.21) --
	cycle;
\definecolor{drawColor}{RGB}{86,51,94}

\path[draw=drawColor,line width= 0.4pt,line join=round,line cap=round] ( 40.57, 41.87) -- ( 43.64, 44.94);

\path[draw=drawColor,line width= 0.4pt,line join=round,line cap=round] ( 40.57, 44.94) -- ( 43.64, 41.87);

\path[draw=drawColor,line width= 0.4pt,line join=round,line cap=round] ( 39.94, 43.40) -- ( 44.27, 43.40);

\path[draw=drawColor,line width= 0.4pt,line join=round,line cap=round] ( 42.11, 41.23) -- ( 42.11, 45.57);
\definecolor{drawColor}{RGB}{247,192,26}
\definecolor{fillColor}{RGB}{247,192,26}

\path[draw=drawColor,line width= 0.4pt,line join=round,line cap=round,fill=fillColor] (119.00, 62.03) circle (  1.53);
\definecolor{drawColor}{RGB}{86,51,94}

\path[draw=drawColor,line width= 0.4pt,line join=round,line cap=round] (117.47, 60.37) -- (120.53, 63.44);

\path[draw=drawColor,line width= 0.4pt,line join=round,line cap=round] (117.47, 63.44) -- (120.53, 60.37);

\path[draw=drawColor,line width= 0.4pt,line join=round,line cap=round] (116.83, 61.90) -- (121.17, 61.90);

\path[draw=drawColor,line width= 0.4pt,line join=round,line cap=round] (119.00, 59.73) -- (119.00, 64.07);
\definecolor{fillColor}{RGB}{37,122,164}

\path[fill=fillColor] (143.10, 82.76) --
	(146.16, 82.76) --
	(146.16, 85.83) --
	(143.10, 85.83) --
	cycle;

\path[fill=fillColor] (117.47, 73.94) --
	(120.53, 73.94) --
	(120.53, 77.00) --
	(117.47, 77.00) --
	cycle;
\definecolor{drawColor}{RGB}{247,192,26}
\definecolor{fillColor}{RGB}{247,192,26}

\path[draw=drawColor,line width= 0.4pt,line join=round,line cap=round,fill=fillColor] (195.89, 84.18) circle (  1.53);
\definecolor{fillColor}{RGB}{78,155,133}

\path[fill=fillColor] ( 93.37, 58.24) --
	( 95.43, 54.67) --
	( 91.30, 54.67) --
	cycle;
\end{scope}
\begin{scope}
\path[clip] (  0.00,  0.00) rectangle (209.58,115.63);
\definecolor{drawColor}{RGB}{0,0,0}

\path[draw=drawColor,line width= 0.5pt,line join=round] ( 34.42, 28.73) --
	( 34.42,113.63);

\path[draw=drawColor,line width= 0.5pt,line join=round] ( 35.55,111.66) --
	( 34.42,113.63) --
	( 33.28,111.66);
\end{scope}
\begin{scope}
\path[clip] (  0.00,  0.00) rectangle (209.58,115.63);
\definecolor{drawColor}{gray}{0.30}

\node[text=drawColor,anchor=base east,inner sep=0pt, outer sep=0pt, scale=  0.70] at ( 30.37, 42.24) {1e-01};

\node[text=drawColor,anchor=base east,inner sep=0pt, outer sep=0pt, scale=  0.70] at ( 30.37, 63.46) {1e+01};

\node[text=drawColor,anchor=base east,inner sep=0pt, outer sep=0pt, scale=  0.70] at ( 30.37, 84.69) {1e+03};

\node[text=drawColor,anchor=base east,inner sep=0pt, outer sep=0pt, scale=  0.70] at ( 30.37,105.92) {1e+05};
\end{scope}
\begin{scope}
\path[clip] (  0.00,  0.00) rectangle (209.58,115.63);
\definecolor{drawColor}{gray}{0.20}

\path[draw=drawColor,line width= 0.5pt,line join=round] ( 32.17, 44.65) --
	( 34.42, 44.65);

\path[draw=drawColor,line width= 0.5pt,line join=round] ( 32.17, 65.87) --
	( 34.42, 65.87);

\path[draw=drawColor,line width= 0.5pt,line join=round] ( 32.17, 87.10) --
	( 34.42, 87.10);

\path[draw=drawColor,line width= 0.5pt,line join=round] ( 32.17,108.33) --
	( 34.42,108.33);
\end{scope}
\begin{scope}
\path[clip] (  0.00,  0.00) rectangle (209.58,115.63);
\definecolor{drawColor}{RGB}{0,0,0}

\path[draw=drawColor,line width= 0.5pt,line join=round] ( 34.42, 28.73) --
	(203.58, 28.73);

\path[draw=drawColor,line width= 0.5pt,line join=round] (201.61, 27.59) --
	(203.58, 28.73) --
	(201.61, 29.87);
\end{scope}
\begin{scope}
\path[clip] (  0.00,  0.00) rectangle (209.58,115.63);
\definecolor{drawColor}{gray}{0.20}

\path[draw=drawColor,line width= 0.5pt,line join=round] ( 42.11, 26.48) --
	( 42.11, 28.73);

\path[draw=drawColor,line width= 0.5pt,line join=round] ( 93.37, 26.48) --
	( 93.37, 28.73);

\path[draw=drawColor,line width= 0.5pt,line join=round] (144.63, 26.48) --
	(144.63, 28.73);

\path[draw=drawColor,line width= 0.5pt,line join=round] (195.89, 26.48) --
	(195.89, 28.73);
\end{scope}
\begin{scope}
\path[clip] (  0.00,  0.00) rectangle (209.58,115.63);
\definecolor{drawColor}{gray}{0.30}

\node[text=drawColor,anchor=base,inner sep=0pt, outer sep=0pt, scale=  0.70] at ( 42.11, 19.86) {4};

\node[text=drawColor,anchor=base,inner sep=0pt, outer sep=0pt, scale=  0.70] at ( 93.37, 19.86) {8};

\node[text=drawColor,anchor=base,inner sep=0pt, outer sep=0pt, scale=  0.70] at (144.63, 19.86) {12};

\node[text=drawColor,anchor=base,inner sep=0pt, outer sep=0pt, scale=  0.70] at (195.89, 19.86) {16};
\end{scope}
\begin{scope}
\path[clip] (  0.00,  0.00) rectangle (209.58,115.63);
\definecolor{drawColor}{RGB}{0,0,0}

\node[text=drawColor,anchor=base,inner sep=0pt, outer sep=0pt, scale=  0.80] at (119.00, 10.74) {$n$};
\end{scope}
\begin{scope}
\path[clip] (  0.00,  0.00) rectangle (209.58,115.63);
\definecolor{drawColor}{RGB}{0,0,0}

\node[text=drawColor,rotate= 90.00,anchor=base,inner sep=0pt, outer sep=0pt, scale=  0.80] at (  7.51, 71.18) {time (ms)};
\end{scope}
\begin{scope}
\path[clip] (  0.00,  0.00) rectangle (209.58,115.63);

\path[] ( 34.42,  1.00) rectangle (210.27,  7.18);
\end{scope}
\begin{scope}
\path[clip] (  0.00,  0.00) rectangle (209.58,115.63);
\definecolor{drawColor}{RGB}{247,192,26}

\path[draw=drawColor,line width= 0.5pt,line join=round] ( 17.50,  4.09) -- ( 26.17,  4.09);
\end{scope}
\begin{scope}
\path[clip] (  0.00,  0.00) rectangle (209.58,115.63);
\definecolor{drawColor}{RGB}{247,192,26}
\definecolor{fillColor}{RGB}{247,192,26}

\path[draw=drawColor,line width= 0.4pt,line join=round,line cap=round,fill=fillColor] ( 21.84,  4.09) circle (  1.53);
\end{scope}
\begin{scope}
\path[clip] (  0.00,  0.00) rectangle (209.58,115.63);
\definecolor{drawColor}{RGB}{78,155,133}

\path[draw=drawColor,line width= 0.5pt,dash pattern=on 4pt off 4pt ,line join=round] ( 60.10,  4.09) -- ( 68.77,  4.09);
\end{scope}
\begin{scope}
\path[clip] (  0.00,  0.00) rectangle (209.58,115.63);
\definecolor{fillColor}{RGB}{78,155,133}

\path[fill=fillColor] ( 64.43,  6.48) --
	( 66.50,  2.90) --
	( 62.37,  2.90) --
	cycle;
\end{scope}
\begin{scope}
\path[clip] (  0.00,  0.00) rectangle (209.58,115.63);
\definecolor{drawColor}{RGB}{37,122,164}

\path[draw=drawColor,line width= 0.5pt,dash pattern=on 1pt off 3pt ,line join=round] (108.14,  4.09) -- (116.81,  4.09);
\end{scope}
\begin{scope}
\path[clip] (  0.00,  0.00) rectangle (209.58,115.63);
\definecolor{fillColor}{RGB}{37,122,164}

\path[fill=fillColor] (110.94,  2.56) --
	(114.01,  2.56) --
	(114.01,  5.62) --
	(110.94,  5.62) --
	cycle;
\end{scope}
\begin{scope}
\path[clip] (  0.00,  0.00) rectangle (209.58,115.63);
\definecolor{drawColor}{RGB}{86,51,94}

\path[draw=drawColor,line width= 0.5pt,dash pattern=on 7pt off 3pt ,line join=round] (158.32,  4.09) -- (166.99,  4.09);
\end{scope}
\begin{scope}
\path[clip] (  0.00,  0.00) rectangle (209.58,115.63);
\definecolor{drawColor}{RGB}{86,51,94}

\path[draw=drawColor,line width= 0.4pt,line join=round,line cap=round] (161.12,  2.56) -- (164.19,  5.62);

\path[draw=drawColor,line width= 0.4pt,line join=round,line cap=round] (161.12,  5.62) -- (164.19,  2.56);

\path[draw=drawColor,line width= 0.4pt,line join=round,line cap=round] (160.49,  4.09) -- (164.82,  4.09);

\path[draw=drawColor,line width= 0.4pt,line join=round,line cap=round] (162.65,  1.92) -- (162.65,  6.26);
\end{scope}
\begin{scope}
\path[clip] (  0.00,  0.00) rectangle (209.58,115.63);
\definecolor{drawColor}{RGB}{0,0,0}

\node[text=drawColor,anchor=base west,inner sep=0pt, outer sep=0pt, scale=  0.70] at ( 27.26,  1.68) {DECOR};
\end{scope}
\begin{scope}
\path[clip] (  0.00,  0.00) rectangle (209.58,115.63);
\definecolor{drawColor}{RGB}{0,0,0}

\node[text=drawColor,anchor=base west,inner sep=0pt, outer sep=0pt, scale=  0.70] at ( 69.85,  1.68) {DECOR+};
\end{scope}
\begin{scope}
\path[clip] (  0.00,  0.00) rectangle (209.58,115.63);
\definecolor{drawColor}{RGB}{0,0,0}

\node[text=drawColor,anchor=base west,inner sep=0pt, outer sep=0pt, scale=  0.70] at (117.90,  1.68) {A-DECOR};
\end{scope}
\begin{scope}
\path[clip] (  0.00,  0.00) rectangle (209.58,115.63);
\definecolor{drawColor}{RGB}{0,0,0}

\node[text=drawColor,anchor=base west,inner sep=0pt, outer sep=0pt, scale=  0.70] at (168.07,  1.68) {Brute-Force};
\end{scope}
\end{tikzpicture}

%% file: files/decor_revised_plot_k=nsub1_phases_decorplus.tex
\begin{tikzpicture}[x=1pt,y=1pt]
\definecolor{fillColor}{RGB}{255,255,255}
\path[use as bounding box,fill=fillColor,fill opacity=0.00] (0,0) rectangle (209.58,115.63);
\begin{scope}
\path[clip] (  0.00,  0.00) rectangle (209.58,115.63);
\definecolor{drawColor}{RGB}{255,255,255}
\definecolor{fillColor}{RGB}{255,255,255}

\path[draw=drawColor,line width= 0.5pt,line join=round,line cap=round,fill=fillColor] (  0.00,  0.00) rectangle (209.58,115.63);
\end{scope}
\begin{scope}
\path[clip] ( 34.42, 28.73) rectangle (203.58,113.63);
\definecolor{fillColor}{RGB}{255,255,255}

\path[fill=fillColor] ( 34.42, 28.73) rectangle (203.58,113.63);
\definecolor{fillColor}{RGB}{78,155,133}

\path[fill=fillColor] ( 64.79, 28.73) rectangle ( 69.78, 47.31);

\path[fill=fillColor] ( 42.11, 28.73) rectangle ( 47.10, 39.24);

\path[fill=fillColor] (155.52, 28.73) rectangle (160.51, 77.04);

\path[fill=fillColor] (178.20, 28.73) rectangle (183.19, 84.13);

\path[fill=fillColor] (110.15, 28.73) rectangle (115.14, 62.06);

\path[fill=fillColor] (132.84, 28.73) rectangle (137.83, 69.01);

\path[fill=fillColor] ( 87.47, 28.73) rectangle ( 92.46, 54.24);
\definecolor{fillColor}{RGB}{247,192,26}

\path[fill=fillColor] ( 71.14, 28.73) rectangle ( 76.13, 43.39);

\path[fill=fillColor] ( 48.46, 28.73) rectangle ( 53.45, 37.46);

\path[fill=fillColor] (161.87, 28.73) rectangle (166.86, 70.70);

\path[fill=fillColor] (184.55, 28.73) rectangle (189.54, 77.84);

\path[fill=fillColor] (116.50, 28.73) rectangle (121.49, 56.95);

\path[fill=fillColor] (139.19, 28.73) rectangle (144.18, 63.75);

\path[fill=fillColor] ( 93.82, 28.73) rectangle ( 98.81, 50.24);
\definecolor{fillColor}{RGB}{37,122,164}

\path[fill=fillColor] ( 77.49, 28.73) rectangle ( 82.48, 48.95);

\path[fill=fillColor] ( 54.81, 28.73) rectangle ( 59.80, 41.63);

\path[fill=fillColor] (168.22, 28.73) rectangle (173.21, 78.07);

\path[fill=fillColor] (190.90, 28.73) rectangle (195.89, 85.18);

\path[fill=fillColor] (122.86, 28.73) rectangle (127.85, 63.38);

\path[fill=fillColor] (145.54, 28.73) rectangle (150.53, 70.28);

\path[fill=fillColor] (100.17, 28.73) rectangle (105.16, 55.86);
\end{scope}
\begin{scope}
\path[clip] (  0.00,  0.00) rectangle (209.58,115.63);
\definecolor{drawColor}{RGB}{0,0,0}

\path[draw=drawColor,line width= 0.5pt,line join=round] ( 34.42, 28.73) --
	( 34.42,113.63);

\path[draw=drawColor,line width= 0.5pt,line join=round] ( 35.55,111.66) --
	( 34.42,113.63) --
	( 33.28,111.66);
\end{scope}
\begin{scope}
\path[clip] (  0.00,  0.00) rectangle (209.58,115.63);
\definecolor{drawColor}{gray}{0.30}

\node[text=drawColor,anchor=base east,inner sep=0pt, outer sep=0pt, scale=  0.70] at ( 30.37, 42.24) {1e-01};

\node[text=drawColor,anchor=base east,inner sep=0pt, outer sep=0pt, scale=  0.70] at ( 30.37, 63.46) {1e+01};

\node[text=drawColor,anchor=base east,inner sep=0pt, outer sep=0pt, scale=  0.70] at ( 30.37, 84.69) {1e+03};

\node[text=drawColor,anchor=base east,inner sep=0pt, outer sep=0pt, scale=  0.70] at ( 30.37,105.92) {1e+05};
\end{scope}
\begin{scope}
\path[clip] (  0.00,  0.00) rectangle (209.58,115.63);
\definecolor{drawColor}{gray}{0.20}

\path[draw=drawColor,line width= 0.5pt,line join=round] ( 32.17, 44.65) --
	( 34.42, 44.65);

\path[draw=drawColor,line width= 0.5pt,line join=round] ( 32.17, 65.87) --
	( 34.42, 65.87);

\path[draw=drawColor,line width= 0.5pt,line join=round] ( 32.17, 87.10) --
	( 34.42, 87.10);

\path[draw=drawColor,line width= 0.5pt,line join=round] ( 32.17,108.33) --
	( 34.42,108.33);
\end{scope}
\begin{scope}
\path[clip] (  0.00,  0.00) rectangle (209.58,115.63);
\definecolor{drawColor}{RGB}{0,0,0}

\path[draw=drawColor,line width= 0.5pt,line join=round] ( 34.42, 28.73) --
	(203.58, 28.73);

\path[draw=drawColor,line width= 0.5pt,line join=round] (201.61, 27.59) --
	(203.58, 28.73) --
	(201.61, 29.87);
\end{scope}
\begin{scope}
\path[clip] (  0.00,  0.00) rectangle (209.58,115.63);
\definecolor{drawColor}{gray}{0.20}

\path[draw=drawColor,line width= 0.5pt,line join=round] ( 50.95, 26.48) --
	( 50.95, 28.73);

\path[draw=drawColor,line width= 0.5pt,line join=round] ( 73.63, 26.48) --
	( 73.63, 28.73);

\path[draw=drawColor,line width= 0.5pt,line join=round] ( 96.32, 26.48) --
	( 96.32, 28.73);

\path[draw=drawColor,line width= 0.5pt,line join=round] (119.00, 26.48) --
	(119.00, 28.73);

\path[draw=drawColor,line width= 0.5pt,line join=round] (141.68, 26.48) --
	(141.68, 28.73);

\path[draw=drawColor,line width= 0.5pt,line join=round] (164.36, 26.48) --
	(164.36, 28.73);

\path[draw=drawColor,line width= 0.5pt,line join=round] (187.05, 26.48) --
	(187.05, 28.73);
\end{scope}
\begin{scope}
\path[clip] (  0.00,  0.00) rectangle (209.58,115.63);
\definecolor{drawColor}{gray}{0.30}

\node[text=drawColor,anchor=base,inner sep=0pt, outer sep=0pt, scale=  0.70] at ( 50.95, 19.86) {4};

\node[text=drawColor,anchor=base,inner sep=0pt, outer sep=0pt, scale=  0.70] at ( 73.63, 19.86) {6};

\node[text=drawColor,anchor=base,inner sep=0pt, outer sep=0pt, scale=  0.70] at ( 96.32, 19.86) {8};

\node[text=drawColor,anchor=base,inner sep=0pt, outer sep=0pt, scale=  0.70] at (119.00, 19.86) {10};

\node[text=drawColor,anchor=base,inner sep=0pt, outer sep=0pt, scale=  0.70] at (141.68, 19.86) {12};

\node[text=drawColor,anchor=base,inner sep=0pt, outer sep=0pt, scale=  0.70] at (164.36, 19.86) {14};

\node[text=drawColor,anchor=base,inner sep=0pt, outer sep=0pt, scale=  0.70] at (187.05, 19.86) {16};
\end{scope}
\begin{scope}
\path[clip] (  0.00,  0.00) rectangle (209.58,115.63);
\definecolor{drawColor}{RGB}{0,0,0}

\node[text=drawColor,anchor=base,inner sep=0pt, outer sep=0pt, scale=  0.80] at (119.00, 10.74) {$n$};
\end{scope}
\begin{scope}
\path[clip] (  0.00,  0.00) rectangle (209.58,115.63);
\definecolor{drawColor}{RGB}{0,0,0}

\node[text=drawColor,rotate= 90.00,anchor=base,inner sep=0pt, outer sep=0pt, scale=  0.80] at (  7.51, 71.18) {time (ms)};
\end{scope}
\begin{scope}
\path[clip] (  0.00,  0.00) rectangle (209.58,115.63);

\path[] ( 51.04,  1.00) rectangle (186.96,  7.18);
\end{scope}
\begin{scope}
\path[clip] (  0.00,  0.00) rectangle (209.58,115.63);
\definecolor{fillColor}{RGB}{78,155,133}

\path[fill=fillColor] ( 51.75,  1.71) rectangle ( 61.17,  6.47);
\end{scope}
\begin{scope}
\path[clip] (  0.00,  0.00) rectangle (209.58,115.63);
\definecolor{fillColor}{RGB}{78,155,133}

\path[fill=fillColor] ( 51.75,  1.71) rectangle ( 61.17,  6.47);
\end{scope}
\begin{scope}
\path[clip] (  0.00,  0.00) rectangle (209.58,115.63);
\definecolor{fillColor}{RGB}{78,155,133}

\path[fill=fillColor] ( 51.75,  1.71) rectangle ( 61.17,  6.47);
\end{scope}
\begin{scope}
\path[clip] (  0.00,  0.00) rectangle (209.58,115.63);
\definecolor{fillColor}{RGB}{247,192,26}

\path[fill=fillColor] (102.84,  1.71) rectangle (112.26,  6.47);
\end{scope}
\begin{scope}
\path[clip] (  0.00,  0.00) rectangle (209.58,115.63);
\definecolor{fillColor}{RGB}{247,192,26}

\path[fill=fillColor] (102.84,  1.71) rectangle (112.26,  6.47);
\end{scope}
\begin{scope}
\path[clip] (  0.00,  0.00) rectangle (209.58,115.63);
\definecolor{fillColor}{RGB}{247,192,26}

\path[fill=fillColor] (102.84,  1.71) rectangle (112.26,  6.47);
\end{scope}
\begin{scope}
\path[clip] (  0.00,  0.00) rectangle (209.58,115.63);
\definecolor{fillColor}{RGB}{37,122,164}

\path[fill=fillColor] (154.70,  1.71) rectangle (164.11,  6.47);
\end{scope}
\begin{scope}
\path[clip] (  0.00,  0.00) rectangle (209.58,115.63);
\definecolor{fillColor}{RGB}{37,122,164}

\path[fill=fillColor] (154.70,  1.71) rectangle (164.11,  6.47);
\end{scope}
\begin{scope}
\path[clip] (  0.00,  0.00) rectangle (209.58,115.63);
\definecolor{fillColor}{RGB}{37,122,164}

\path[fill=fillColor] (154.70,  1.71) rectangle (164.11,  6.47);
\end{scope}
\begin{scope}
\path[clip] (  0.00,  0.00) rectangle (209.58,115.63);
\definecolor{drawColor}{RGB}{0,0,0}

\node[text=drawColor,anchor=base west,inner sep=0pt, outer sep=0pt, scale=  0.70] at ( 61.88,  1.68) {Candidates};
\end{scope}
\begin{scope}
\path[clip] (  0.00,  0.00) rectangle (209.58,115.63);
\definecolor{drawColor}{RGB}{0,0,0}

\node[text=drawColor,anchor=base west,inner sep=0pt, outer sep=0pt, scale=  0.70] at (112.97,  1.68) {Verification};
\end{scope}
\begin{scope}
\path[clip] (  0.00,  0.00) rectangle (209.58,115.63);
\definecolor{drawColor}{RGB}{0,0,0}

\node[text=drawColor,anchor=base west,inner sep=0pt, outer sep=0pt, scale=  0.70] at (164.82,  1.68) {Total};
\end{scope}
\end{tikzpicture}

%% file: files/decor_revised_plot_k=n_times.tex
\begin{tikzpicture}[x=1pt,y=1pt]
\definecolor{fillColor}{RGB}{255,255,255}
\path[use as bounding box,fill=fillColor,fill opacity=0.00] (0,0) rectangle (209.58,115.63);
\begin{scope}
\path[clip] (  0.00,  0.00) rectangle (209.58,115.63);
\definecolor{drawColor}{RGB}{255,255,255}
\definecolor{fillColor}{RGB}{255,255,255}

\path[draw=drawColor,line width= 0.5pt,line join=round,line cap=round,fill=fillColor] (  0.00,  0.00) rectangle (209.58,115.63);
\end{scope}
\begin{scope}
\path[clip] ( 34.42, 28.73) rectangle (203.58,113.63);
\definecolor{fillColor}{RGB}{255,255,255}

\path[fill=fillColor] ( 34.42, 28.73) rectangle (203.58,113.63);
\definecolor{drawColor}{RGB}{247,192,26}

\path[draw=drawColor,line width= 0.5pt,line join=round] ( 42.11, 41.79) --
	( 64.07, 47.55) --
	( 86.04, 54.59) --
	(108.01, 60.70) --
	(129.98, 67.59) --
	(151.95, 73.99) --
	(173.92, 81.13) --
	(195.89, 87.41);
\definecolor{drawColor}{RGB}{78,155,133}

\path[draw=drawColor,line width= 0.5pt,dash pattern=on 4pt off 4pt ,line join=round] ( 42.11, 44.56) --
	( 64.07, 49.58) --
	( 86.04, 55.89) --
	(108.01, 61.98) --
	(129.98, 68.71) --
	(151.95, 75.04) --
	(173.92, 82.13) --
	(195.89, 88.30);
\definecolor{drawColor}{RGB}{37,122,164}

\path[draw=drawColor,line width= 0.5pt,dash pattern=on 1pt off 3pt ,line join=round] ( 42.11, 41.24) --
	( 64.07, 53.59) --
	( 86.04, 63.06) --
	(108.01, 71.65) --
	(129.98, 79.88) --
	(151.95, 87.67) --
	(173.92, 95.61) --
	(195.89,104.09);
\definecolor{drawColor}{RGB}{86,51,94}

\path[draw=drawColor,line width= 0.5pt,dash pattern=on 7pt off 3pt ,line join=round] ( 42.11, 41.19) --
	( 64.07, 45.33) --
	( 86.04, 50.70) --
	(108.01, 56.74) --
	(129.98, 63.04) --
	(151.95, 69.02) --
	(173.92, 75.06) --
	(195.89, 81.42);
\definecolor{fillColor}{RGB}{78,155,133}

\path[fill=fillColor] ( 64.07, 51.97) --
	( 66.14, 48.39) --
	( 62.01, 48.39) --
	cycle;
\definecolor{drawColor}{RGB}{247,192,26}
\definecolor{fillColor}{RGB}{247,192,26}

\path[draw=drawColor,line width= 0.4pt,line join=round,line cap=round,fill=fillColor] (129.98, 67.59) circle (  1.53);
\definecolor{fillColor}{RGB}{37,122,164}

\path[fill=fillColor] ( 40.57, 39.71) --
	( 43.64, 39.71) --
	( 43.64, 42.77) --
	( 40.57, 42.77) --
	cycle;

\path[fill=fillColor] ( 62.54, 52.06) --
	( 65.61, 52.06) --
	( 65.61, 55.12) --
	( 62.54, 55.12) --
	cycle;
\definecolor{drawColor}{RGB}{86,51,94}

\path[draw=drawColor,line width= 0.4pt,line join=round,line cap=round] ( 84.51, 49.16) -- ( 87.58, 52.23);

\path[draw=drawColor,line width= 0.4pt,line join=round,line cap=round] ( 84.51, 52.23) -- ( 87.58, 49.16);

\path[draw=drawColor,line width= 0.4pt,line join=round,line cap=round] ( 83.88, 50.70) -- ( 88.21, 50.70);

\path[draw=drawColor,line width= 0.4pt,line join=round,line cap=round] ( 86.04, 48.53) -- ( 86.04, 52.87);
\definecolor{fillColor}{RGB}{78,155,133}

\path[fill=fillColor] (108.01, 64.36) --
	(110.08, 60.79) --
	(105.95, 60.79) --
	cycle;
\definecolor{drawColor}{RGB}{247,192,26}
\definecolor{fillColor}{RGB}{247,192,26}

\path[draw=drawColor,line width= 0.4pt,line join=round,line cap=round,fill=fillColor] (195.89, 87.41) circle (  1.53);
\definecolor{fillColor}{RGB}{78,155,133}

\path[fill=fillColor] ( 42.11, 46.94) --
	( 44.17, 43.36) --
	( 40.04, 43.36) --
	cycle;
\definecolor{drawColor}{RGB}{86,51,94}

\path[draw=drawColor,line width= 0.4pt,line join=round,line cap=round] (150.42, 67.48) -- (153.49, 70.55);

\path[draw=drawColor,line width= 0.4pt,line join=round,line cap=round] (150.42, 70.55) -- (153.49, 67.48);

\path[draw=drawColor,line width= 0.4pt,line join=round,line cap=round] (149.79, 69.02) -- (154.12, 69.02);

\path[draw=drawColor,line width= 0.4pt,line join=round,line cap=round] (151.95, 66.85) -- (151.95, 71.19);

\path[draw=drawColor,line width= 0.4pt,line join=round,line cap=round] (128.45, 61.51) -- (131.52, 64.57);

\path[draw=drawColor,line width= 0.4pt,line join=round,line cap=round] (128.45, 64.57) -- (131.52, 61.51);

\path[draw=drawColor,line width= 0.4pt,line join=round,line cap=round] (127.82, 63.04) -- (132.15, 63.04);

\path[draw=drawColor,line width= 0.4pt,line join=round,line cap=round] (129.98, 60.87) -- (129.98, 65.21);
\definecolor{drawColor}{RGB}{247,192,26}
\definecolor{fillColor}{RGB}{247,192,26}

\path[draw=drawColor,line width= 0.4pt,line join=round,line cap=round,fill=fillColor] (151.95, 73.99) circle (  1.53);
\definecolor{fillColor}{RGB}{37,122,164}

\path[fill=fillColor] (106.48, 70.12) --
	(109.55, 70.12) --
	(109.55, 73.19) --
	(106.48, 73.19) --
	cycle;

\path[fill=fillColor] ( 84.51, 61.53) --
	( 87.58, 61.53) --
	( 87.58, 64.60) --
	( 84.51, 64.60) --
	cycle;
\definecolor{drawColor}{RGB}{86,51,94}

\path[draw=drawColor,line width= 0.4pt,line join=round,line cap=round] (106.48, 55.21) -- (109.55, 58.28);

\path[draw=drawColor,line width= 0.4pt,line join=round,line cap=round] (106.48, 58.28) -- (109.55, 55.21);

\path[draw=drawColor,line width= 0.4pt,line join=round,line cap=round] (105.85, 56.74) -- (110.18, 56.74);

\path[draw=drawColor,line width= 0.4pt,line join=round,line cap=round] (108.01, 54.58) -- (108.01, 58.91);
\definecolor{drawColor}{RGB}{247,192,26}
\definecolor{fillColor}{RGB}{247,192,26}

\path[draw=drawColor,line width= 0.4pt,line join=round,line cap=round,fill=fillColor] (108.01, 60.70) circle (  1.53);
\definecolor{fillColor}{RGB}{78,155,133}

\path[fill=fillColor] (173.92, 84.52) --
	(175.99, 80.94) --
	(171.86, 80.94) --
	cycle;

\path[fill=fillColor] ( 86.04, 58.27) --
	( 88.11, 54.70) --
	( 83.98, 54.70) --
	cycle;
\definecolor{fillColor}{RGB}{37,122,164}

\path[fill=fillColor] (194.36,102.55) --
	(197.43,102.55) --
	(197.43,105.62) --
	(194.36,105.62) --
	cycle;
\definecolor{fillColor}{RGB}{247,192,26}

\path[draw=drawColor,line width= 0.4pt,line join=round,line cap=round,fill=fillColor] ( 86.04, 54.59) circle (  1.53);
\definecolor{drawColor}{RGB}{86,51,94}

\path[draw=drawColor,line width= 0.4pt,line join=round,line cap=round] (172.39, 73.53) -- (175.46, 76.59);

\path[draw=drawColor,line width= 0.4pt,line join=round,line cap=round] (172.39, 76.59) -- (175.46, 73.53);

\path[draw=drawColor,line width= 0.4pt,line join=round,line cap=round] (171.75, 75.06) -- (176.09, 75.06);

\path[draw=drawColor,line width= 0.4pt,line join=round,line cap=round] (173.92, 72.89) -- (173.92, 77.23);
\definecolor{fillColor}{RGB}{37,122,164}

\path[fill=fillColor] (128.45, 78.34) --
	(131.52, 78.34) --
	(131.52, 81.41) --
	(128.45, 81.41) --
	cycle;
\definecolor{fillColor}{RGB}{78,155,133}

\path[fill=fillColor] (129.98, 71.10) --
	(132.05, 67.52) --
	(127.92, 67.52) --
	cycle;

\path[fill=fillColor] (151.95, 77.43) --
	(154.02, 73.85) --
	(149.89, 73.85) --
	cycle;

\path[draw=drawColor,line width= 0.4pt,line join=round,line cap=round] (194.36, 79.88) -- (197.43, 82.95);

\path[draw=drawColor,line width= 0.4pt,line join=round,line cap=round] (194.36, 82.95) -- (197.43, 79.88);

\path[draw=drawColor,line width= 0.4pt,line join=round,line cap=round] (193.72, 81.42) -- (198.06, 81.42);

\path[draw=drawColor,line width= 0.4pt,line join=round,line cap=round] (195.89, 79.25) -- (195.89, 83.59);
\definecolor{drawColor}{RGB}{247,192,26}
\definecolor{fillColor}{RGB}{247,192,26}

\path[draw=drawColor,line width= 0.4pt,line join=round,line cap=round,fill=fillColor] (173.92, 81.13) circle (  1.53);

\path[draw=drawColor,line width= 0.4pt,line join=round,line cap=round,fill=fillColor] ( 64.07, 47.55) circle (  1.53);
\definecolor{fillColor}{RGB}{37,122,164}

\path[fill=fillColor] (150.42, 86.13) --
	(153.49, 86.13) --
	(153.49, 89.20) --
	(150.42, 89.20) --
	cycle;
\definecolor{drawColor}{RGB}{86,51,94}

\path[draw=drawColor,line width= 0.4pt,line join=round,line cap=round] ( 62.54, 43.80) -- ( 65.61, 46.86);

\path[draw=drawColor,line width= 0.4pt,line join=round,line cap=round] ( 62.54, 46.86) -- ( 65.61, 43.80);

\path[draw=drawColor,line width= 0.4pt,line join=round,line cap=round] ( 61.91, 45.33) -- ( 66.24, 45.33);

\path[draw=drawColor,line width= 0.4pt,line join=round,line cap=round] ( 64.07, 43.16) -- ( 64.07, 47.50);
\definecolor{fillColor}{RGB}{78,155,133}

\path[fill=fillColor] (195.89, 90.69) --
	(197.96, 87.11) --
	(193.83, 87.11) --
	cycle;

\path[draw=drawColor,line width= 0.4pt,line join=round,line cap=round] ( 40.57, 39.66) -- ( 43.64, 42.73);

\path[draw=drawColor,line width= 0.4pt,line join=round,line cap=round] ( 40.57, 42.73) -- ( 43.64, 39.66);

\path[draw=drawColor,line width= 0.4pt,line join=round,line cap=round] ( 39.94, 41.19) -- ( 44.27, 41.19);

\path[draw=drawColor,line width= 0.4pt,line join=round,line cap=round] ( 42.11, 39.02) -- ( 42.11, 43.36);
\definecolor{fillColor}{RGB}{37,122,164}

\path[fill=fillColor] (172.39, 94.08) --
	(175.46, 94.08) --
	(175.46, 97.15) --
	(172.39, 97.15) --
	cycle;
\definecolor{drawColor}{RGB}{247,192,26}
\definecolor{fillColor}{RGB}{247,192,26}

\path[draw=drawColor,line width= 0.4pt,line join=round,line cap=round,fill=fillColor] ( 42.11, 41.79) circle (  1.53);
\end{scope}
\begin{scope}
\path[clip] (  0.00,  0.00) rectangle (209.58,115.63);
\definecolor{drawColor}{RGB}{0,0,0}

\path[draw=drawColor,line width= 0.5pt,line join=round] ( 34.42, 28.73) --
	( 34.42,113.63);

\path[draw=drawColor,line width= 0.5pt,line join=round] ( 35.55,111.66) --
	( 34.42,113.63) --
	( 33.28,111.66);
\end{scope}
\begin{scope}
\path[clip] (  0.00,  0.00) rectangle (209.58,115.63);
\definecolor{drawColor}{gray}{0.30}

\node[text=drawColor,anchor=base east,inner sep=0pt, outer sep=0pt, scale=  0.70] at ( 30.37, 31.03) {1e-03};

\node[text=drawColor,anchor=base east,inner sep=0pt, outer sep=0pt, scale=  0.70] at ( 30.37, 49.90) {1e-01};

\node[text=drawColor,anchor=base east,inner sep=0pt, outer sep=0pt, scale=  0.70] at ( 30.37, 68.77) {1e+01};

\node[text=drawColor,anchor=base east,inner sep=0pt, outer sep=0pt, scale=  0.70] at ( 30.37, 87.64) {1e+03};

\node[text=drawColor,anchor=base east,inner sep=0pt, outer sep=0pt, scale=  0.70] at ( 30.37,106.50) {1e+05};
\end{scope}
\begin{scope}
\path[clip] (  0.00,  0.00) rectangle (209.58,115.63);
\definecolor{drawColor}{gray}{0.20}

\path[draw=drawColor,line width= 0.5pt,line join=round] ( 32.17, 33.45) --
	( 34.42, 33.45);

\path[draw=drawColor,line width= 0.5pt,line join=round] ( 32.17, 52.31) --
	( 34.42, 52.31);

\path[draw=drawColor,line width= 0.5pt,line join=round] ( 32.17, 71.18) --
	( 34.42, 71.18);

\path[draw=drawColor,line width= 0.5pt,line join=round] ( 32.17, 90.05) --
	( 34.42, 90.05);

\path[draw=drawColor,line width= 0.5pt,line join=round] ( 32.17,108.92) --
	( 34.42,108.92);
\end{scope}
\begin{scope}
\path[clip] (  0.00,  0.00) rectangle (209.58,115.63);
\definecolor{drawColor}{RGB}{0,0,0}

\path[draw=drawColor,line width= 0.5pt,line join=round] ( 34.42, 28.73) --
	(203.58, 28.73);

\path[draw=drawColor,line width= 0.5pt,line join=round] (201.61, 27.59) --
	(203.58, 28.73) --
	(201.61, 29.87);
\end{scope}
\begin{scope}
\path[clip] (  0.00,  0.00) rectangle (209.58,115.63);
\definecolor{drawColor}{gray}{0.20}

\path[draw=drawColor,line width= 0.5pt,line join=round] ( 64.07, 26.48) --
	( 64.07, 28.73);

\path[draw=drawColor,line width= 0.5pt,line join=round] (108.01, 26.48) --
	(108.01, 28.73);

\path[draw=drawColor,line width= 0.5pt,line join=round] (151.95, 26.48) --
	(151.95, 28.73);

\path[draw=drawColor,line width= 0.5pt,line join=round] (195.89, 26.48) --
	(195.89, 28.73);
\end{scope}
\begin{scope}
\path[clip] (  0.00,  0.00) rectangle (209.58,115.63);
\definecolor{drawColor}{gray}{0.30}

\node[text=drawColor,anchor=base,inner sep=0pt, outer sep=0pt, scale=  0.70] at ( 64.07, 19.86) {4};

\node[text=drawColor,anchor=base,inner sep=0pt, outer sep=0pt, scale=  0.70] at (108.01, 19.86) {8};

\node[text=drawColor,anchor=base,inner sep=0pt, outer sep=0pt, scale=  0.70] at (151.95, 19.86) {12};

\node[text=drawColor,anchor=base,inner sep=0pt, outer sep=0pt, scale=  0.70] at (195.89, 19.86) {16};
\end{scope}
\begin{scope}
\path[clip] (  0.00,  0.00) rectangle (209.58,115.63);
\definecolor{drawColor}{RGB}{0,0,0}

\node[text=drawColor,anchor=base,inner sep=0pt, outer sep=0pt, scale=  0.80] at (119.00, 10.74) {$n$};
\end{scope}
\begin{scope}
\path[clip] (  0.00,  0.00) rectangle (209.58,115.63);
\definecolor{drawColor}{RGB}{0,0,0}

\node[text=drawColor,rotate= 90.00,anchor=base,inner sep=0pt, outer sep=0pt, scale=  0.80] at (  7.51, 71.18) {time (ms)};
\end{scope}
\begin{scope}
\path[clip] (  0.00,  0.00) rectangle (209.58,115.63);

\path[] ( 34.42,  1.00) rectangle (210.27,  7.18);
\end{scope}
\begin{scope}
\path[clip] (  0.00,  0.00) rectangle (209.58,115.63);
\definecolor{drawColor}{RGB}{247,192,26}

\path[draw=drawColor,line width= 0.5pt,line join=round] ( 17.50,  4.09) -- ( 26.17,  4.09);
\end{scope}
\begin{scope}
\path[clip] (  0.00,  0.00) rectangle (209.58,115.63);
\definecolor{drawColor}{RGB}{247,192,26}
\definecolor{fillColor}{RGB}{247,192,26}

\path[draw=drawColor,line width= 0.4pt,line join=round,line cap=round,fill=fillColor] ( 21.84,  4.09) circle (  1.53);
\end{scope}
\begin{scope}
\path[clip] (  0.00,  0.00) rectangle (209.58,115.63);
\definecolor{drawColor}{RGB}{78,155,133}

\path[draw=drawColor,line width= 0.5pt,dash pattern=on 4pt off 4pt ,line join=round] ( 60.10,  4.09) -- ( 68.77,  4.09);
\end{scope}
\begin{scope}
\path[clip] (  0.00,  0.00) rectangle (209.58,115.63);
\definecolor{fillColor}{RGB}{78,155,133}

\path[fill=fillColor] ( 64.43,  6.48) --
	( 66.50,  2.90) --
	( 62.37,  2.90) --
	cycle;
\end{scope}
\begin{scope}
\path[clip] (  0.00,  0.00) rectangle (209.58,115.63);
\definecolor{drawColor}{RGB}{37,122,164}

\path[draw=drawColor,line width= 0.5pt,dash pattern=on 1pt off 3pt ,line join=round] (108.14,  4.09) -- (116.81,  4.09);
\end{scope}
\begin{scope}
\path[clip] (  0.00,  0.00) rectangle (209.58,115.63);
\definecolor{fillColor}{RGB}{37,122,164}

\path[fill=fillColor] (110.94,  2.56) --
	(114.01,  2.56) --
	(114.01,  5.62) --
	(110.94,  5.62) --
	cycle;
\end{scope}
\begin{scope}
\path[clip] (  0.00,  0.00) rectangle (209.58,115.63);
\definecolor{drawColor}{RGB}{86,51,94}

\path[draw=drawColor,line width= 0.5pt,dash pattern=on 7pt off 3pt ,line join=round] (158.32,  4.09) -- (166.99,  4.09);
\end{scope}
\begin{scope}
\path[clip] (  0.00,  0.00) rectangle (209.58,115.63);
\definecolor{drawColor}{RGB}{86,51,94}

\path[draw=drawColor,line width= 0.4pt,line join=round,line cap=round] (161.12,  2.56) -- (164.19,  5.62);

\path[draw=drawColor,line width= 0.4pt,line join=round,line cap=round] (161.12,  5.62) -- (164.19,  2.56);

\path[draw=drawColor,line width= 0.4pt,line join=round,line cap=round] (160.49,  4.09) -- (164.82,  4.09);

\path[draw=drawColor,line width= 0.4pt,line join=round,line cap=round] (162.65,  1.92) -- (162.65,  6.26);
\end{scope}
\begin{scope}
\path[clip] (  0.00,  0.00) rectangle (209.58,115.63);
\definecolor{drawColor}{RGB}{0,0,0}

\node[text=drawColor,anchor=base west,inner sep=0pt, outer sep=0pt, scale=  0.70] at ( 27.26,  1.68) {DECOR};
\end{scope}
\begin{scope}
\path[clip] (  0.00,  0.00) rectangle (209.58,115.63);
\definecolor{drawColor}{RGB}{0,0,0}

\node[text=drawColor,anchor=base west,inner sep=0pt, outer sep=0pt, scale=  0.70] at ( 69.85,  1.68) {DECOR+};
\end{scope}
\begin{scope}
\path[clip] (  0.00,  0.00) rectangle (209.58,115.63);
\definecolor{drawColor}{RGB}{0,0,0}

\node[text=drawColor,anchor=base west,inner sep=0pt, outer sep=0pt, scale=  0.70] at (117.90,  1.68) {A-DECOR};
\end{scope}
\begin{scope}
\path[clip] (  0.00,  0.00) rectangle (209.58,115.63);
\definecolor{drawColor}{RGB}{0,0,0}

\node[text=drawColor,anchor=base west,inner sep=0pt, outer sep=0pt, scale=  0.70] at (168.07,  1.68) {Brute-Force};
\end{scope}
\end{tikzpicture}

%% file: files/decor_revised_plot_k=n_phases_decorplus.tex
\begin{tikzpicture}[x=1pt,y=1pt]
\definecolor{fillColor}{RGB}{255,255,255}
\path[use as bounding box,fill=fillColor,fill opacity=0.00] (0,0) rectangle (209.58,115.63);
\begin{scope}
\path[clip] (  0.00,  0.00) rectangle (209.58,115.63);
\definecolor{drawColor}{RGB}{255,255,255}
\definecolor{fillColor}{RGB}{255,255,255}

\path[draw=drawColor,line width= 0.5pt,line join=round,line cap=round,fill=fillColor] (  0.00,  0.00) rectangle (209.58,115.63);
\end{scope}
\begin{scope}
\path[clip] ( 34.42, 28.73) rectangle (203.58,113.63);
\definecolor{fillColor}{RGB}{255,255,255}

\path[fill=fillColor] ( 34.42, 28.73) rectangle (203.58,113.63);
\definecolor{fillColor}{RGB}{78,155,133}

\path[fill=fillColor] ( 61.87, 28.73) rectangle ( 66.22, 47.88);

\path[fill=fillColor] (101.41, 28.73) rectangle (105.76, 60.67);

\path[fill=fillColor] ( 42.11, 28.73) rectangle ( 46.45, 42.20);

\path[fill=fillColor] (160.71, 28.73) rectangle (165.06, 81.34);

\path[fill=fillColor] ( 81.64, 28.73) rectangle ( 85.99, 54.60);

\path[fill=fillColor] (121.17, 28.73) rectangle (125.52, 67.59);

\path[fill=fillColor] (140.94, 28.73) rectangle (145.29, 73.94);

\path[fill=fillColor] (180.48, 28.73) rectangle (184.82, 87.47);
\definecolor{fillColor}{RGB}{247,192,26}

\path[fill=fillColor] ( 67.41, 28.73) rectangle ( 71.76, 45.17);

\path[fill=fillColor] (106.94, 28.73) rectangle (111.29, 56.66);

\path[fill=fillColor] ( 47.64, 28.73) rectangle ( 51.99, 41.16);

\path[fill=fillColor] (166.24, 28.73) rectangle (170.59, 75.02);

\path[fill=fillColor] ( 87.17, 28.73) rectangle ( 91.52, 50.53);

\path[fill=fillColor] (126.71, 28.73) rectangle (131.06, 62.86);

\path[fill=fillColor] (146.48, 28.73) rectangle (150.82, 69.12);

\path[fill=fillColor] (186.01, 28.73) rectangle (190.36, 81.36);
\definecolor{fillColor}{RGB}{37,122,164}

\path[fill=fillColor] ( 72.94, 28.73) rectangle ( 77.29, 49.58);

\path[fill=fillColor] (112.48, 28.73) rectangle (116.82, 61.98);

\path[fill=fillColor] ( 53.17, 28.73) rectangle ( 57.52, 44.56);

\path[fill=fillColor] (171.78, 28.73) rectangle (176.13, 82.13);

\path[fill=fillColor] ( 92.71, 28.73) rectangle ( 97.06, 55.89);

\path[fill=fillColor] (132.24, 28.73) rectangle (136.59, 68.71);

\path[fill=fillColor] (152.01, 28.73) rectangle (156.36, 75.04);

\path[fill=fillColor] (191.54, 28.73) rectangle (195.89, 88.30);
\end{scope}
\begin{scope}
\path[clip] (  0.00,  0.00) rectangle (209.58,115.63);
\definecolor{drawColor}{RGB}{0,0,0}

\path[draw=drawColor,line width= 0.5pt,line join=round] ( 34.42, 28.73) --
	( 34.42,113.63);

\path[draw=drawColor,line width= 0.5pt,line join=round] ( 35.55,111.66) --
	( 34.42,113.63) --
	( 33.28,111.66);
\end{scope}
\begin{scope}
\path[clip] (  0.00,  0.00) rectangle (209.58,115.63);
\definecolor{drawColor}{gray}{0.30}

\node[text=drawColor,anchor=base east,inner sep=0pt, outer sep=0pt, scale=  0.70] at ( 30.37, 31.03) {1e-03};

\node[text=drawColor,anchor=base east,inner sep=0pt, outer sep=0pt, scale=  0.70] at ( 30.37, 49.90) {1e-01};

\node[text=drawColor,anchor=base east,inner sep=0pt, outer sep=0pt, scale=  0.70] at ( 30.37, 68.77) {1e+01};

\node[text=drawColor,anchor=base east,inner sep=0pt, outer sep=0pt, scale=  0.70] at ( 30.37, 87.64) {1e+03};

\node[text=drawColor,anchor=base east,inner sep=0pt, outer sep=0pt, scale=  0.70] at ( 30.37,106.50) {1e+05};
\end{scope}
\begin{scope}
\path[clip] (  0.00,  0.00) rectangle (209.58,115.63);
\definecolor{drawColor}{gray}{0.20}

\path[draw=drawColor,line width= 0.5pt,line join=round] ( 32.17, 33.45) --
	( 34.42, 33.45);

\path[draw=drawColor,line width= 0.5pt,line join=round] ( 32.17, 52.31) --
	( 34.42, 52.31);

\path[draw=drawColor,line width= 0.5pt,line join=round] ( 32.17, 71.18) --
	( 34.42, 71.18);

\path[draw=drawColor,line width= 0.5pt,line join=round] ( 32.17, 90.05) --
	( 34.42, 90.05);

\path[draw=drawColor,line width= 0.5pt,line join=round] ( 32.17,108.92) --
	( 34.42,108.92);
\end{scope}
\begin{scope}
\path[clip] (  0.00,  0.00) rectangle (209.58,115.63);
\definecolor{drawColor}{RGB}{0,0,0}

\path[draw=drawColor,line width= 0.5pt,line join=round] ( 34.42, 28.73) --
	(203.58, 28.73);

\path[draw=drawColor,line width= 0.5pt,line join=round] (201.61, 27.59) --
	(203.58, 28.73) --
	(201.61, 29.87);
\end{scope}
\begin{scope}
\path[clip] (  0.00,  0.00) rectangle (209.58,115.63);
\definecolor{drawColor}{gray}{0.20}

\path[draw=drawColor,line width= 0.5pt,line join=round] ( 49.81, 26.48) --
	( 49.81, 28.73);

\path[draw=drawColor,line width= 0.5pt,line join=round] ( 69.58, 26.48) --
	( 69.58, 28.73);

\path[draw=drawColor,line width= 0.5pt,line join=round] ( 89.35, 26.48) --
	( 89.35, 28.73);

\path[draw=drawColor,line width= 0.5pt,line join=round] (109.12, 26.48) --
	(109.12, 28.73);

\path[draw=drawColor,line width= 0.5pt,line join=round] (128.88, 26.48) --
	(128.88, 28.73);

\path[draw=drawColor,line width= 0.5pt,line join=round] (148.65, 26.48) --
	(148.65, 28.73);

\path[draw=drawColor,line width= 0.5pt,line join=round] (168.42, 26.48) --
	(168.42, 28.73);

\path[draw=drawColor,line width= 0.5pt,line join=round] (188.18, 26.48) --
	(188.18, 28.73);
\end{scope}
\begin{scope}
\path[clip] (  0.00,  0.00) rectangle (209.58,115.63);
\definecolor{drawColor}{gray}{0.30}

\node[text=drawColor,anchor=base,inner sep=0pt, outer sep=0pt, scale=  0.70] at ( 49.81, 19.86) {2};

\node[text=drawColor,anchor=base,inner sep=0pt, outer sep=0pt, scale=  0.70] at ( 69.58, 19.86) {4};

\node[text=drawColor,anchor=base,inner sep=0pt, outer sep=0pt, scale=  0.70] at ( 89.35, 19.86) {6};

\node[text=drawColor,anchor=base,inner sep=0pt, outer sep=0pt, scale=  0.70] at (109.12, 19.86) {8};

\node[text=drawColor,anchor=base,inner sep=0pt, outer sep=0pt, scale=  0.70] at (128.88, 19.86) {10};

\node[text=drawColor,anchor=base,inner sep=0pt, outer sep=0pt, scale=  0.70] at (148.65, 19.86) {12};

\node[text=drawColor,anchor=base,inner sep=0pt, outer sep=0pt, scale=  0.70] at (168.42, 19.86) {14};

\node[text=drawColor,anchor=base,inner sep=0pt, outer sep=0pt, scale=  0.70] at (188.18, 19.86) {16};
\end{scope}
\begin{scope}
\path[clip] (  0.00,  0.00) rectangle (209.58,115.63);
\definecolor{drawColor}{RGB}{0,0,0}

\node[text=drawColor,anchor=base,inner sep=0pt, outer sep=0pt, scale=  0.80] at (119.00, 10.74) {$n$};
\end{scope}
\begin{scope}
\path[clip] (  0.00,  0.00) rectangle (209.58,115.63);
\definecolor{drawColor}{RGB}{0,0,0}

\node[text=drawColor,rotate= 90.00,anchor=base,inner sep=0pt, outer sep=0pt, scale=  0.80] at (  7.51, 71.18) {time (ms)};
\end{scope}
\begin{scope}
\path[clip] (  0.00,  0.00) rectangle (209.58,115.63);

\path[] ( 51.04,  1.00) rectangle (186.96,  7.18);
\end{scope}
\begin{scope}
\path[clip] (  0.00,  0.00) rectangle (209.58,115.63);
\definecolor{fillColor}{RGB}{78,155,133}

\path[fill=fillColor] ( 51.75,  1.71) rectangle ( 61.17,  6.47);
\end{scope}
\begin{scope}
\path[clip] (  0.00,  0.00) rectangle (209.58,115.63);
\definecolor{fillColor}{RGB}{78,155,133}

\path[fill=fillColor] ( 51.75,  1.71) rectangle ( 61.17,  6.47);
\end{scope}
\begin{scope}
\path[clip] (  0.00,  0.00) rectangle (209.58,115.63);
\definecolor{fillColor}{RGB}{78,155,133}

\path[fill=fillColor] ( 51.75,  1.71) rectangle ( 61.17,  6.47);
\end{scope}
\begin{scope}
\path[clip] (  0.00,  0.00) rectangle (209.58,115.63);
\definecolor{fillColor}{RGB}{247,192,26}

\path[fill=fillColor] (102.84,  1.71) rectangle (112.26,  6.47);
\end{scope}
\begin{scope}
\path[clip] (  0.00,  0.00) rectangle (209.58,115.63);
\definecolor{fillColor}{RGB}{247,192,26}

\path[fill=fillColor] (102.84,  1.71) rectangle (112.26,  6.47);
\end{scope}
\begin{scope}
\path[clip] (  0.00,  0.00) rectangle (209.58,115.63);
\definecolor{fillColor}{RGB}{247,192,26}

\path[fill=fillColor] (102.84,  1.71) rectangle (112.26,  6.47);
\end{scope}
\begin{scope}
\path[clip] (  0.00,  0.00) rectangle (209.58,115.63);
\definecolor{fillColor}{RGB}{37,122,164}

\path[fill=fillColor] (154.70,  1.71) rectangle (164.11,  6.47);
\end{scope}
\begin{scope}
\path[clip] (  0.00,  0.00) rectangle (209.58,115.63);
\definecolor{fillColor}{RGB}{37,122,164}

\path[fill=fillColor] (154.70,  1.71) rectangle (164.11,  6.47);
\end{scope}
\begin{scope}
\path[clip] (  0.00,  0.00) rectangle (209.58,115.63);
\definecolor{fillColor}{RGB}{37,122,164}

\path[fill=fillColor] (154.70,  1.71) rectangle (164.11,  6.47);
\end{scope}
\begin{scope}
\path[clip] (  0.00,  0.00) rectangle (209.58,115.63);
\definecolor{drawColor}{RGB}{0,0,0}

\node[text=drawColor,anchor=base west,inner sep=0pt, outer sep=0pt, scale=  0.70] at ( 61.88,  1.68) {Candidates};
\end{scope}
\begin{scope}
\path[clip] (  0.00,  0.00) rectangle (209.58,115.63);
\definecolor{drawColor}{RGB}{0,0,0}

\node[text=drawColor,anchor=base west,inner sep=0pt, outer sep=0pt, scale=  0.70] at (112.97,  1.68) {Verification};
\end{scope}
\begin{scope}
\path[clip] (  0.00,  0.00) rectangle (209.58,115.63);
\definecolor{drawColor}{RGB}{0,0,0}

\node[text=drawColor,anchor=base west,inner sep=0pt, outer sep=0pt, scale=  0.70] at (164.82,  1.68) {Total};
\end{scope}
\end{tikzpicture}

%% file: files/decor_revised_plot_groups_vary_g.tex
\begin{tikzpicture}[x=1pt,y=1pt]
\definecolor{fillColor}{RGB}{255,255,255}
\path[use as bounding box,fill=fillColor,fill opacity=0.00] (0,0) rectangle (209.58,115.63);
\begin{scope}
\path[clip] (  0.00,  0.00) rectangle (209.58,115.63);
\definecolor{drawColor}{RGB}{255,255,255}
\definecolor{fillColor}{RGB}{255,255,255}

\path[draw=drawColor,line width= 0.5pt,line join=round,line cap=round,fill=fillColor] (  0.00,  0.00) rectangle (209.58,115.63);
\end{scope}
\begin{scope}
\path[clip] ( 34.42, 28.73) rectangle (203.58,113.63);
\definecolor{fillColor}{RGB}{255,255,255}

\path[fill=fillColor] ( 34.42, 28.73) rectangle (203.58,113.63);
\definecolor{drawColor}{RGB}{247,192,26}

\path[draw=drawColor,line width= 0.5pt,line join=round] ( 42.11, 39.35) --
	( 67.74, 47.56) --
	( 93.37, 54.40) --
	(119.00, 62.23) --
	(144.63, 69.82) --
	(170.26, 78.83) --
	(195.89, 87.13);
\definecolor{drawColor}{RGB}{78,155,133}

\path[draw=drawColor,line width= 0.5pt,dash pattern=on 4pt off 4pt ,line join=round] ( 42.11, 41.94) --
	( 67.74, 49.79) --
	( 93.37, 56.16) --
	(119.00, 63.88) --
	(144.63, 71.23) --
	(170.26, 79.87) --
	(195.89, 88.06);
\definecolor{drawColor}{RGB}{37,122,164}

\path[draw=drawColor,line width= 0.5pt,dash pattern=on 1pt off 3pt ,line join=round] ( 42.11, 45.71) --
	( 67.74, 56.07) --
	( 93.37, 65.74) --
	(119.00, 75.13) --
	(144.63, 83.84) --
	(170.26, 92.64) --
	(195.89,101.62);
\definecolor{drawColor}{RGB}{86,51,94}

\path[draw=drawColor,line width= 0.5pt,dash pattern=on 7pt off 3pt ,line join=round] ( 42.11, 46.18) --
	( 67.74, 60.93) --
	( 93.37, 75.08) --
	(119.00, 88.77) --
	(144.63,102.14);
\definecolor{drawColor}{RGB}{247,192,26}
\definecolor{fillColor}{RGB}{247,192,26}

\path[draw=drawColor,line width= 0.4pt,line join=round,line cap=round,fill=fillColor] (170.26, 78.83) circle (  1.53);
\definecolor{fillColor}{RGB}{37,122,164}

\path[fill=fillColor] ( 91.83, 64.20) --
	( 94.90, 64.20) --
	( 94.90, 67.27) --
	( 91.83, 67.27) --
	cycle;
\definecolor{fillColor}{RGB}{247,192,26}

\path[draw=drawColor,line width= 0.4pt,line join=round,line cap=round,fill=fillColor] (195.89, 87.13) circle (  1.53);
\definecolor{drawColor}{RGB}{86,51,94}

\path[draw=drawColor,line width= 0.4pt,line join=round,line cap=round] ( 66.20, 59.39) -- ( 69.27, 62.46);

\path[draw=drawColor,line width= 0.4pt,line join=round,line cap=round] ( 66.20, 62.46) -- ( 69.27, 59.39);

\path[draw=drawColor,line width= 0.4pt,line join=round,line cap=round] ( 65.57, 60.93) -- ( 69.91, 60.93);

\path[draw=drawColor,line width= 0.4pt,line join=round,line cap=round] ( 67.74, 58.76) -- ( 67.74, 63.09);
\definecolor{fillColor}{RGB}{78,155,133}

\path[fill=fillColor] (144.63, 73.62) --
	(146.70, 70.04) --
	(142.57, 70.04) --
	cycle;

\path[draw=drawColor,line width= 0.4pt,line join=round,line cap=round] ( 40.57, 44.64) -- ( 43.64, 47.71);

\path[draw=drawColor,line width= 0.4pt,line join=round,line cap=round] ( 40.57, 47.71) -- ( 43.64, 44.64);

\path[draw=drawColor,line width= 0.4pt,line join=round,line cap=round] ( 39.94, 46.18) -- ( 44.27, 46.18);

\path[draw=drawColor,line width= 0.4pt,line join=round,line cap=round] ( 42.11, 44.01) -- ( 42.11, 48.34);
\definecolor{fillColor}{RGB}{37,122,164}

\path[fill=fillColor] ( 66.20, 54.53) --
	( 69.27, 54.53) --
	( 69.27, 57.60) --
	( 66.20, 57.60) --
	cycle;
\definecolor{drawColor}{RGB}{247,192,26}
\definecolor{fillColor}{RGB}{247,192,26}

\path[draw=drawColor,line width= 0.4pt,line join=round,line cap=round,fill=fillColor] ( 93.37, 54.40) circle (  1.53);
\definecolor{fillColor}{RGB}{37,122,164}

\path[fill=fillColor] (117.47, 73.59) --
	(120.53, 73.59) --
	(120.53, 76.66) --
	(117.47, 76.66) --
	cycle;
\definecolor{fillColor}{RGB}{247,192,26}

\path[draw=drawColor,line width= 0.4pt,line join=round,line cap=round,fill=fillColor] ( 42.11, 39.35) circle (  1.53);
\definecolor{fillColor}{RGB}{37,122,164}

\path[fill=fillColor] (194.36,100.09) --
	(197.43,100.09) --
	(197.43,103.16) --
	(194.36,103.16) --
	cycle;
\definecolor{fillColor}{RGB}{78,155,133}

\path[fill=fillColor] (119.00, 66.27) --
	(121.06, 62.69) --
	(116.93, 62.69) --
	cycle;
\definecolor{drawColor}{RGB}{86,51,94}

\path[draw=drawColor,line width= 0.4pt,line join=round,line cap=round] (117.47, 87.23) -- (120.53, 90.30);

\path[draw=drawColor,line width= 0.4pt,line join=round,line cap=round] (117.47, 90.30) -- (120.53, 87.23);

\path[draw=drawColor,line width= 0.4pt,line join=round,line cap=round] (116.83, 88.77) -- (121.17, 88.77);

\path[draw=drawColor,line width= 0.4pt,line join=round,line cap=round] (119.00, 86.60) -- (119.00, 90.94);
\definecolor{drawColor}{RGB}{247,192,26}
\definecolor{fillColor}{RGB}{247,192,26}

\path[draw=drawColor,line width= 0.4pt,line join=round,line cap=round,fill=fillColor] (144.63, 69.82) circle (  1.53);
\definecolor{drawColor}{RGB}{86,51,94}

\path[draw=drawColor,line width= 0.4pt,line join=round,line cap=round] (143.10,100.61) -- (146.16,103.68);

\path[draw=drawColor,line width= 0.4pt,line join=round,line cap=round] (143.10,103.68) -- (146.16,100.61);

\path[draw=drawColor,line width= 0.4pt,line join=round,line cap=round] (142.46,102.14) -- (146.80,102.14);

\path[draw=drawColor,line width= 0.4pt,line join=round,line cap=round] (144.63, 99.97) -- (144.63,104.31);

\path[draw=drawColor,line width= 0.4pt,line join=round,line cap=round] ( 91.83, 73.55) -- ( 94.90, 76.62);

\path[draw=drawColor,line width= 0.4pt,line join=round,line cap=round] ( 91.83, 76.62) -- ( 94.90, 73.55);

\path[draw=drawColor,line width= 0.4pt,line join=round,line cap=round] ( 91.20, 75.08) -- ( 95.54, 75.08);

\path[draw=drawColor,line width= 0.4pt,line join=round,line cap=round] ( 93.37, 72.91) -- ( 93.37, 77.25);
\definecolor{fillColor}{RGB}{78,155,133}

\path[fill=fillColor] (195.89, 90.45) --
	(197.96, 86.87) --
	(193.83, 86.87) --
	cycle;
\definecolor{fillColor}{RGB}{37,122,164}

\path[fill=fillColor] (143.10, 82.31) --
	(146.16, 82.31) --
	(146.16, 85.38) --
	(143.10, 85.38) --
	cycle;
\definecolor{drawColor}{RGB}{247,192,26}
\definecolor{fillColor}{RGB}{247,192,26}

\path[draw=drawColor,line width= 0.4pt,line join=round,line cap=round,fill=fillColor] (119.00, 62.23) circle (  1.53);
\definecolor{fillColor}{RGB}{37,122,164}

\path[fill=fillColor] ( 40.57, 44.18) --
	( 43.64, 44.18) --
	( 43.64, 47.24) --
	( 40.57, 47.24) --
	cycle;
\definecolor{fillColor}{RGB}{78,155,133}

\path[fill=fillColor] ( 93.37, 58.55) --
	( 95.43, 54.97) --
	( 91.30, 54.97) --
	cycle;

\path[fill=fillColor] ( 42.11, 44.33) --
	( 44.17, 40.75) --
	( 40.04, 40.75) --
	cycle;
\definecolor{fillColor}{RGB}{37,122,164}

\path[fill=fillColor] (168.73, 91.10) --
	(171.80, 91.10) --
	(171.80, 94.17) --
	(168.73, 94.17) --
	cycle;
\definecolor{fillColor}{RGB}{247,192,26}

\path[draw=drawColor,line width= 0.4pt,line join=round,line cap=round,fill=fillColor] ( 67.74, 47.56) circle (  1.53);
\definecolor{fillColor}{RGB}{78,155,133}

\path[fill=fillColor] ( 67.74, 52.18) --
	( 69.80, 48.60) --
	( 65.67, 48.60) --
	cycle;

\path[fill=fillColor] (170.26, 82.25) --
	(172.33, 78.68) --
	(168.20, 78.68) --
	cycle;
\end{scope}
\begin{scope}
\path[clip] (  0.00,  0.00) rectangle (209.58,115.63);
\definecolor{drawColor}{RGB}{0,0,0}

\path[draw=drawColor,line width= 0.5pt,line join=round] ( 34.42, 28.73) --
	( 34.42,113.63);

\path[draw=drawColor,line width= 0.5pt,line join=round] ( 35.55,111.66) --
	( 34.42,113.63) --
	( 33.28,111.66);
\end{scope}
\begin{scope}
\path[clip] (  0.00,  0.00) rectangle (209.58,115.63);
\definecolor{drawColor}{gray}{0.30}

\node[text=drawColor,anchor=base east,inner sep=0pt, outer sep=0pt, scale=  0.70] at ( 30.37, 42.24) {1e-01};

\node[text=drawColor,anchor=base east,inner sep=0pt, outer sep=0pt, scale=  0.70] at ( 30.37, 63.46) {1e+01};

\node[text=drawColor,anchor=base east,inner sep=0pt, outer sep=0pt, scale=  0.70] at ( 30.37, 84.69) {1e+03};

\node[text=drawColor,anchor=base east,inner sep=0pt, outer sep=0pt, scale=  0.70] at ( 30.37,105.92) {1e+05};
\end{scope}
\begin{scope}
\path[clip] (  0.00,  0.00) rectangle (209.58,115.63);
\definecolor{drawColor}{gray}{0.20}

\path[draw=drawColor,line width= 0.5pt,line join=round] ( 32.17, 44.65) --
	( 34.42, 44.65);

\path[draw=drawColor,line width= 0.5pt,line join=round] ( 32.17, 65.87) --
	( 34.42, 65.87);

\path[draw=drawColor,line width= 0.5pt,line join=round] ( 32.17, 87.10) --
	( 34.42, 87.10);

\path[draw=drawColor,line width= 0.5pt,line join=round] ( 32.17,108.33) --
	( 34.42,108.33);
\end{scope}
\begin{scope}
\path[clip] (  0.00,  0.00) rectangle (209.58,115.63);
\definecolor{drawColor}{RGB}{0,0,0}

\path[draw=drawColor,line width= 0.5pt,line join=round] ( 34.42, 28.73) --
	(203.58, 28.73);

\path[draw=drawColor,line width= 0.5pt,line join=round] (201.61, 27.59) --
	(203.58, 28.73) --
	(201.61, 29.87);
\end{scope}
\begin{scope}
\path[clip] (  0.00,  0.00) rectangle (209.58,115.63);
\definecolor{drawColor}{gray}{0.20}

\path[draw=drawColor,line width= 0.5pt,line join=round] ( 42.11, 26.48) --
	( 42.11, 28.73);

\path[draw=drawColor,line width= 0.5pt,line join=round] ( 67.74, 26.48) --
	( 67.74, 28.73);

\path[draw=drawColor,line width= 0.5pt,line join=round] ( 93.37, 26.48) --
	( 93.37, 28.73);

\path[draw=drawColor,line width= 0.5pt,line join=round] (119.00, 26.48) --
	(119.00, 28.73);

\path[draw=drawColor,line width= 0.5pt,line join=round] (144.63, 26.48) --
	(144.63, 28.73);

\path[draw=drawColor,line width= 0.5pt,line join=round] (170.26, 26.48) --
	(170.26, 28.73);

\path[draw=drawColor,line width= 0.5pt,line join=round] (195.89, 26.48) --
	(195.89, 28.73);
\end{scope}
\begin{scope}
\path[clip] (  0.00,  0.00) rectangle (209.58,115.63);
\definecolor{drawColor}{gray}{0.30}

\node[text=drawColor,anchor=base,inner sep=0pt, outer sep=0pt, scale=  0.70] at ( 42.11, 19.86) {2};

\node[text=drawColor,anchor=base,inner sep=0pt, outer sep=0pt, scale=  0.70] at ( 67.74, 19.86) {3};

\node[text=drawColor,anchor=base,inner sep=0pt, outer sep=0pt, scale=  0.70] at ( 93.37, 19.86) {4};

\node[text=drawColor,anchor=base,inner sep=0pt, outer sep=0pt, scale=  0.70] at (119.00, 19.86) {5};

\node[text=drawColor,anchor=base,inner sep=0pt, outer sep=0pt, scale=  0.70] at (144.63, 19.86) {6};

\node[text=drawColor,anchor=base,inner sep=0pt, outer sep=0pt, scale=  0.70] at (170.26, 19.86) {7};

\node[text=drawColor,anchor=base,inner sep=0pt, outer sep=0pt, scale=  0.70] at (195.89, 19.86) {8};
\end{scope}
\begin{scope}
\path[clip] (  0.00,  0.00) rectangle (209.58,115.63);
\definecolor{drawColor}{RGB}{0,0,0}

\node[text=drawColor,anchor=base,inner sep=0pt, outer sep=0pt, scale=  0.80] at (119.00, 10.74) {$g$};
\end{scope}
\begin{scope}
\path[clip] (  0.00,  0.00) rectangle (209.58,115.63);
\definecolor{drawColor}{RGB}{0,0,0}

\node[text=drawColor,rotate= 90.00,anchor=base,inner sep=0pt, outer sep=0pt, scale=  0.80] at (  7.51, 71.18) {time (ms)};
\end{scope}
\begin{scope}
\path[clip] (  0.00,  0.00) rectangle (209.58,115.63);

\path[] ( 34.42,  1.00) rectangle (210.27,  7.18);
\end{scope}
\begin{scope}
\path[clip] (  0.00,  0.00) rectangle (209.58,115.63);
\definecolor{drawColor}{RGB}{247,192,26}

\path[draw=drawColor,line width= 0.5pt,line join=round] ( 17.50,  4.09) -- ( 26.17,  4.09);
\end{scope}
\begin{scope}
\path[clip] (  0.00,  0.00) rectangle (209.58,115.63);
\definecolor{drawColor}{RGB}{247,192,26}
\definecolor{fillColor}{RGB}{247,192,26}

\path[draw=drawColor,line width= 0.4pt,line join=round,line cap=round,fill=fillColor] ( 21.84,  4.09) circle (  1.53);
\end{scope}
\begin{scope}
\path[clip] (  0.00,  0.00) rectangle (209.58,115.63);
\definecolor{drawColor}{RGB}{78,155,133}

\path[draw=drawColor,line width= 0.5pt,dash pattern=on 4pt off 4pt ,line join=round] ( 60.10,  4.09) -- ( 68.77,  4.09);
\end{scope}
\begin{scope}
\path[clip] (  0.00,  0.00) rectangle (209.58,115.63);
\definecolor{fillColor}{RGB}{78,155,133}

\path[fill=fillColor] ( 64.43,  6.48) --
	( 66.50,  2.90) --
	( 62.37,  2.90) --
	cycle;
\end{scope}
\begin{scope}
\path[clip] (  0.00,  0.00) rectangle (209.58,115.63);
\definecolor{drawColor}{RGB}{37,122,164}

\path[draw=drawColor,line width= 0.5pt,dash pattern=on 1pt off 3pt ,line join=round] (108.14,  4.09) -- (116.81,  4.09);
\end{scope}
\begin{scope}
\path[clip] (  0.00,  0.00) rectangle (209.58,115.63);
\definecolor{fillColor}{RGB}{37,122,164}

\path[fill=fillColor] (110.94,  2.56) --
	(114.01,  2.56) --
	(114.01,  5.62) --
	(110.94,  5.62) --
	cycle;
\end{scope}
\begin{scope}
\path[clip] (  0.00,  0.00) rectangle (209.58,115.63);
\definecolor{drawColor}{RGB}{86,51,94}

\path[draw=drawColor,line width= 0.5pt,dash pattern=on 7pt off 3pt ,line join=round] (158.32,  4.09) -- (166.99,  4.09);
\end{scope}
\begin{scope}
\path[clip] (  0.00,  0.00) rectangle (209.58,115.63);
\definecolor{drawColor}{RGB}{86,51,94}

\path[draw=drawColor,line width= 0.4pt,line join=round,line cap=round] (161.12,  2.56) -- (164.19,  5.62);

\path[draw=drawColor,line width= 0.4pt,line join=round,line cap=round] (161.12,  5.62) -- (164.19,  2.56);

\path[draw=drawColor,line width= 0.4pt,line join=round,line cap=round] (160.49,  4.09) -- (164.82,  4.09);

\path[draw=drawColor,line width= 0.4pt,line join=round,line cap=round] (162.65,  1.92) -- (162.65,  6.26);
\end{scope}
\begin{scope}
\path[clip] (  0.00,  0.00) rectangle (209.58,115.63);
\definecolor{drawColor}{RGB}{0,0,0}

\node[text=drawColor,anchor=base west,inner sep=0pt, outer sep=0pt, scale=  0.70] at ( 27.26,  1.68) {DECOR};
\end{scope}
\begin{scope}
\path[clip] (  0.00,  0.00) rectangle (209.58,115.63);
\definecolor{drawColor}{RGB}{0,0,0}

\node[text=drawColor,anchor=base west,inner sep=0pt, outer sep=0pt, scale=  0.70] at ( 69.85,  1.68) {DECOR+};
\end{scope}
\begin{scope}
\path[clip] (  0.00,  0.00) rectangle (209.58,115.63);
\definecolor{drawColor}{RGB}{0,0,0}

\node[text=drawColor,anchor=base west,inner sep=0pt, outer sep=0pt, scale=  0.70] at (117.90,  1.68) {A-DECOR};
\end{scope}
\begin{scope}
\path[clip] (  0.00,  0.00) rectangle (209.58,115.63);
\definecolor{drawColor}{RGB}{0,0,0}

\node[text=drawColor,anchor=base west,inner sep=0pt, outer sep=0pt, scale=  0.70] at (168.07,  1.68) {Brute-Force};
\end{scope}
\end{tikzpicture}

%% file: files/decor_revised_plot_groups_vary_s.tex
\begin{tikzpicture}[x=1pt,y=1pt]
\definecolor{fillColor}{RGB}{255,255,255}
\path[use as bounding box,fill=fillColor,fill opacity=0.00] (0,0) rectangle (209.58,115.63);
\begin{scope}
\path[clip] (  0.00,  0.00) rectangle (209.58,115.63);
\definecolor{drawColor}{RGB}{255,255,255}
\definecolor{fillColor}{RGB}{255,255,255}

\path[draw=drawColor,line width= 0.5pt,line join=round,line cap=round,fill=fillColor] (  0.00,  0.00) rectangle (209.58,115.63);
\end{scope}
\begin{scope}
\path[clip] ( 34.42, 28.73) rectangle (203.58,113.63);
\definecolor{fillColor}{RGB}{255,255,255}

\path[fill=fillColor] ( 34.42, 28.73) rectangle (203.58,113.63);
\definecolor{drawColor}{RGB}{247,192,26}

\path[draw=drawColor,line width= 0.5pt,line join=round] ( 42.11, 39.35) --
	( 67.74, 47.31) --
	( 93.37, 54.14) --
	(119.00, 61.87) --
	(144.63, 69.24) --
	(170.26, 77.09) --
	(195.89, 83.91);
\definecolor{drawColor}{RGB}{78,155,133}

\path[draw=drawColor,line width= 0.5pt,dash pattern=on 4pt off 4pt ,line join=round] ( 42.11, 41.94) --
	( 67.74, 49.08) --
	( 93.37, 55.85) --
	(119.00, 63.62) --
	(144.63, 70.55) --
	(170.26, 78.07) --
	(195.89, 85.02);
\definecolor{drawColor}{RGB}{37,122,164}

\path[draw=drawColor,line width= 0.5pt,dash pattern=on 1pt off 3pt ,line join=round] ( 42.11, 45.71) --
	( 67.74, 56.39) --
	( 93.37, 66.24) --
	(119.00, 75.42) --
	(144.63, 84.21) --
	(170.26, 92.90) --
	(195.89,101.97);
\definecolor{drawColor}{RGB}{86,51,94}

\path[draw=drawColor,line width= 0.5pt,dash pattern=on 7pt off 3pt ,line join=round] ( 42.11, 46.18) --
	( 67.74, 59.11) --
	( 93.37, 70.93) --
	(119.00, 84.34) --
	(144.63, 98.42);
\definecolor{fillColor}{RGB}{78,155,133}

\path[fill=fillColor] (119.00, 66.00) --
	(121.06, 62.42) --
	(116.93, 62.42) --
	cycle;
\definecolor{drawColor}{RGB}{247,192,26}
\definecolor{fillColor}{RGB}{247,192,26}

\path[draw=drawColor,line width= 0.4pt,line join=round,line cap=round,fill=fillColor] ( 93.37, 54.14) circle (  1.53);
\definecolor{fillColor}{RGB}{37,122,164}

\path[fill=fillColor] ( 66.20, 54.85) --
	( 69.27, 54.85) --
	( 69.27, 57.92) --
	( 66.20, 57.92) --
	cycle;
\definecolor{drawColor}{RGB}{86,51,94}

\path[draw=drawColor,line width= 0.4pt,line join=round,line cap=round] ( 91.83, 69.40) -- ( 94.90, 72.47);

\path[draw=drawColor,line width= 0.4pt,line join=round,line cap=round] ( 91.83, 72.47) -- ( 94.90, 69.40);

\path[draw=drawColor,line width= 0.4pt,line join=round,line cap=round] ( 91.20, 70.93) -- ( 95.54, 70.93);

\path[draw=drawColor,line width= 0.4pt,line join=round,line cap=round] ( 93.37, 68.77) -- ( 93.37, 73.10);

\path[draw=drawColor,line width= 0.4pt,line join=round,line cap=round] ( 66.20, 57.57) -- ( 69.27, 60.64);

\path[draw=drawColor,line width= 0.4pt,line join=round,line cap=round] ( 66.20, 60.64) -- ( 69.27, 57.57);

\path[draw=drawColor,line width= 0.4pt,line join=round,line cap=round] ( 65.57, 59.11) -- ( 69.91, 59.11);

\path[draw=drawColor,line width= 0.4pt,line join=round,line cap=round] ( 67.74, 56.94) -- ( 67.74, 61.28);
\definecolor{drawColor}{RGB}{247,192,26}
\definecolor{fillColor}{RGB}{247,192,26}

\path[draw=drawColor,line width= 0.4pt,line join=round,line cap=round,fill=fillColor] (195.89, 83.91) circle (  1.53);
\definecolor{drawColor}{RGB}{86,51,94}

\path[draw=drawColor,line width= 0.4pt,line join=round,line cap=round] (143.10, 96.88) -- (146.16, 99.95);

\path[draw=drawColor,line width= 0.4pt,line join=round,line cap=round] (143.10, 99.95) -- (146.16, 96.88);

\path[draw=drawColor,line width= 0.4pt,line join=round,line cap=round] (142.46, 98.42) -- (146.80, 98.42);

\path[draw=drawColor,line width= 0.4pt,line join=round,line cap=round] (144.63, 96.25) -- (144.63,100.59);

\path[draw=drawColor,line width= 0.4pt,line join=round,line cap=round] ( 40.57, 44.64) -- ( 43.64, 47.71);

\path[draw=drawColor,line width= 0.4pt,line join=round,line cap=round] ( 40.57, 47.71) -- ( 43.64, 44.64);

\path[draw=drawColor,line width= 0.4pt,line join=round,line cap=round] ( 39.94, 46.18) -- ( 44.27, 46.18);

\path[draw=drawColor,line width= 0.4pt,line join=round,line cap=round] ( 42.11, 44.01) -- ( 42.11, 48.34);
\definecolor{fillColor}{RGB}{37,122,164}

\path[fill=fillColor] (143.10, 82.67) --
	(146.16, 82.67) --
	(146.16, 85.74) --
	(143.10, 85.74) --
	cycle;
\definecolor{drawColor}{RGB}{247,192,26}
\definecolor{fillColor}{RGB}{247,192,26}

\path[draw=drawColor,line width= 0.4pt,line join=round,line cap=round,fill=fillColor] ( 42.11, 39.35) circle (  1.53);
\definecolor{fillColor}{RGB}{78,155,133}

\path[fill=fillColor] (170.26, 80.46) --
	(172.33, 76.88) --
	(168.20, 76.88) --
	cycle;
\definecolor{fillColor}{RGB}{37,122,164}

\path[fill=fillColor] (194.36,100.44) --
	(197.43,100.44) --
	(197.43,103.50) --
	(194.36,103.50) --
	cycle;
\definecolor{fillColor}{RGB}{247,192,26}

\path[draw=drawColor,line width= 0.4pt,line join=round,line cap=round,fill=fillColor] (119.00, 61.87) circle (  1.53);
\definecolor{fillColor}{RGB}{78,155,133}

\path[fill=fillColor] (144.63, 72.94) --
	(146.70, 69.36) --
	(142.57, 69.36) --
	cycle;
\definecolor{fillColor}{RGB}{247,192,26}

\path[draw=drawColor,line width= 0.4pt,line join=round,line cap=round,fill=fillColor] ( 67.74, 47.31) circle (  1.53);
\definecolor{fillColor}{RGB}{37,122,164}

\path[fill=fillColor] ( 40.57, 44.18) --
	( 43.64, 44.18) --
	( 43.64, 47.24) --
	( 40.57, 47.24) --
	cycle;
\definecolor{drawColor}{RGB}{86,51,94}

\path[draw=drawColor,line width= 0.4pt,line join=round,line cap=round] (117.47, 82.81) -- (120.53, 85.87);

\path[draw=drawColor,line width= 0.4pt,line join=round,line cap=round] (117.47, 85.87) -- (120.53, 82.81);

\path[draw=drawColor,line width= 0.4pt,line join=round,line cap=round] (116.83, 84.34) -- (121.17, 84.34);

\path[draw=drawColor,line width= 0.4pt,line join=round,line cap=round] (119.00, 82.17) -- (119.00, 86.51);

\path[fill=fillColor] (168.73, 91.36) --
	(171.80, 91.36) --
	(171.80, 94.43) --
	(168.73, 94.43) --
	cycle;
\definecolor{fillColor}{RGB}{78,155,133}

\path[fill=fillColor] ( 93.37, 58.24) --
	( 95.43, 54.66) --
	( 91.30, 54.66) --
	cycle;

\path[fill=fillColor] ( 42.11, 44.33) --
	( 44.17, 40.75) --
	( 40.04, 40.75) --
	cycle;
\definecolor{drawColor}{RGB}{247,192,26}
\definecolor{fillColor}{RGB}{247,192,26}

\path[draw=drawColor,line width= 0.4pt,line join=round,line cap=round,fill=fillColor] (144.63, 69.24) circle (  1.53);
\definecolor{fillColor}{RGB}{37,122,164}

\path[fill=fillColor] ( 91.83, 64.70) --
	( 94.90, 64.70) --
	( 94.90, 67.77) --
	( 91.83, 67.77) --
	cycle;

\path[fill=fillColor] (117.47, 73.88) --
	(120.53, 73.88) --
	(120.53, 76.95) --
	(117.47, 76.95) --
	cycle;
\definecolor{fillColor}{RGB}{78,155,133}

\path[fill=fillColor] (195.89, 87.40) --
	(197.96, 83.83) --
	(193.83, 83.83) --
	cycle;
\definecolor{fillColor}{RGB}{247,192,26}

\path[draw=drawColor,line width= 0.4pt,line join=round,line cap=round,fill=fillColor] (170.26, 77.09) circle (  1.53);
\definecolor{fillColor}{RGB}{78,155,133}

\path[fill=fillColor] ( 67.74, 51.46) --
	( 69.80, 47.89) --
	( 65.67, 47.89) --
	cycle;
\end{scope}
\begin{scope}
\path[clip] (  0.00,  0.00) rectangle (209.58,115.63);
\definecolor{drawColor}{RGB}{0,0,0}

\path[draw=drawColor,line width= 0.5pt,line join=round] ( 34.42, 28.73) --
	( 34.42,113.63);

\path[draw=drawColor,line width= 0.5pt,line join=round] ( 35.55,111.66) --
	( 34.42,113.63) --
	( 33.28,111.66);
\end{scope}
\begin{scope}
\path[clip] (  0.00,  0.00) rectangle (209.58,115.63);
\definecolor{drawColor}{gray}{0.30}

\node[text=drawColor,anchor=base east,inner sep=0pt, outer sep=0pt, scale=  0.70] at ( 30.37, 42.24) {1e-01};

\node[text=drawColor,anchor=base east,inner sep=0pt, outer sep=0pt, scale=  0.70] at ( 30.37, 63.46) {1e+01};

\node[text=drawColor,anchor=base east,inner sep=0pt, outer sep=0pt, scale=  0.70] at ( 30.37, 84.69) {1e+03};

\node[text=drawColor,anchor=base east,inner sep=0pt, outer sep=0pt, scale=  0.70] at ( 30.37,105.92) {1e+05};
\end{scope}
\begin{scope}
\path[clip] (  0.00,  0.00) rectangle (209.58,115.63);
\definecolor{drawColor}{gray}{0.20}

\path[draw=drawColor,line width= 0.5pt,line join=round] ( 32.17, 44.65) --
	( 34.42, 44.65);

\path[draw=drawColor,line width= 0.5pt,line join=round] ( 32.17, 65.87) --
	( 34.42, 65.87);

\path[draw=drawColor,line width= 0.5pt,line join=round] ( 32.17, 87.10) --
	( 34.42, 87.10);

\path[draw=drawColor,line width= 0.5pt,line join=round] ( 32.17,108.33) --
	( 34.42,108.33);
\end{scope}
\begin{scope}
\path[clip] (  0.00,  0.00) rectangle (209.58,115.63);
\definecolor{drawColor}{RGB}{0,0,0}

\path[draw=drawColor,line width= 0.5pt,line join=round] ( 34.42, 28.73) --
	(203.58, 28.73);

\path[draw=drawColor,line width= 0.5pt,line join=round] (201.61, 27.59) --
	(203.58, 28.73) --
	(201.61, 29.87);
\end{scope}
\begin{scope}
\path[clip] (  0.00,  0.00) rectangle (209.58,115.63);
\definecolor{drawColor}{gray}{0.20}

\path[draw=drawColor,line width= 0.5pt,line join=round] ( 42.11, 26.48) --
	( 42.11, 28.73);

\path[draw=drawColor,line width= 0.5pt,line join=round] ( 67.74, 26.48) --
	( 67.74, 28.73);

\path[draw=drawColor,line width= 0.5pt,line join=round] ( 93.37, 26.48) --
	( 93.37, 28.73);

\path[draw=drawColor,line width= 0.5pt,line join=round] (119.00, 26.48) --
	(119.00, 28.73);

\path[draw=drawColor,line width= 0.5pt,line join=round] (144.63, 26.48) --
	(144.63, 28.73);

\path[draw=drawColor,line width= 0.5pt,line join=round] (170.26, 26.48) --
	(170.26, 28.73);

\path[draw=drawColor,line width= 0.5pt,line join=round] (195.89, 26.48) --
	(195.89, 28.73);
\end{scope}
\begin{scope}
\path[clip] (  0.00,  0.00) rectangle (209.58,115.63);
\definecolor{drawColor}{gray}{0.30}

\node[text=drawColor,anchor=base,inner sep=0pt, outer sep=0pt, scale=  0.70] at ( 42.11, 19.86) {2};

\node[text=drawColor,anchor=base,inner sep=0pt, outer sep=0pt, scale=  0.70] at ( 67.74, 19.86) {3};

\node[text=drawColor,anchor=base,inner sep=0pt, outer sep=0pt, scale=  0.70] at ( 93.37, 19.86) {4};

\node[text=drawColor,anchor=base,inner sep=0pt, outer sep=0pt, scale=  0.70] at (119.00, 19.86) {5};

\node[text=drawColor,anchor=base,inner sep=0pt, outer sep=0pt, scale=  0.70] at (144.63, 19.86) {6};

\node[text=drawColor,anchor=base,inner sep=0pt, outer sep=0pt, scale=  0.70] at (170.26, 19.86) {7};

\node[text=drawColor,anchor=base,inner sep=0pt, outer sep=0pt, scale=  0.70] at (195.89, 19.86) {8};
\end{scope}
\begin{scope}
\path[clip] (  0.00,  0.00) rectangle (209.58,115.63);
\definecolor{drawColor}{RGB}{0,0,0}

\node[text=drawColor,anchor=base,inner sep=0pt, outer sep=0pt, scale=  0.80] at (119.00, 10.74) {$s$};
\end{scope}
\begin{scope}
\path[clip] (  0.00,  0.00) rectangle (209.58,115.63);
\definecolor{drawColor}{RGB}{0,0,0}

\node[text=drawColor,rotate= 90.00,anchor=base,inner sep=0pt, outer sep=0pt, scale=  0.80] at (  7.51, 71.18) {time (ms)};
\end{scope}
\begin{scope}
\path[clip] (  0.00,  0.00) rectangle (209.58,115.63);

\path[] ( 34.42,  1.00) rectangle (210.27,  7.18);
\end{scope}
\begin{scope}
\path[clip] (  0.00,  0.00) rectangle (209.58,115.63);
\definecolor{drawColor}{RGB}{247,192,26}

\path[draw=drawColor,line width= 0.5pt,line join=round] ( 17.50,  4.09) -- ( 26.17,  4.09);
\end{scope}
\begin{scope}
\path[clip] (  0.00,  0.00) rectangle (209.58,115.63);
\definecolor{drawColor}{RGB}{247,192,26}
\definecolor{fillColor}{RGB}{247,192,26}

\path[draw=drawColor,line width= 0.4pt,line join=round,line cap=round,fill=fillColor] ( 21.84,  4.09) circle (  1.53);
\end{scope}
\begin{scope}
\path[clip] (  0.00,  0.00) rectangle (209.58,115.63);
\definecolor{drawColor}{RGB}{78,155,133}

\path[draw=drawColor,line width= 0.5pt,dash pattern=on 4pt off 4pt ,line join=round] ( 60.10,  4.09) -- ( 68.77,  4.09);
\end{scope}
\begin{scope}
\path[clip] (  0.00,  0.00) rectangle (209.58,115.63);
\definecolor{fillColor}{RGB}{78,155,133}

\path[fill=fillColor] ( 64.43,  6.48) --
	( 66.50,  2.90) --
	( 62.37,  2.90) --
	cycle;
\end{scope}
\begin{scope}
\path[clip] (  0.00,  0.00) rectangle (209.58,115.63);
\definecolor{drawColor}{RGB}{37,122,164}

\path[draw=drawColor,line width= 0.5pt,dash pattern=on 1pt off 3pt ,line join=round] (108.14,  4.09) -- (116.81,  4.09);
\end{scope}
\begin{scope}
\path[clip] (  0.00,  0.00) rectangle (209.58,115.63);
\definecolor{fillColor}{RGB}{37,122,164}

\path[fill=fillColor] (110.94,  2.56) --
	(114.01,  2.56) --
	(114.01,  5.62) --
	(110.94,  5.62) --
	cycle;
\end{scope}
\begin{scope}
\path[clip] (  0.00,  0.00) rectangle (209.58,115.63);
\definecolor{drawColor}{RGB}{86,51,94}

\path[draw=drawColor,line width= 0.5pt,dash pattern=on 7pt off 3pt ,line join=round] (158.32,  4.09) -- (166.99,  4.09);
\end{scope}
\begin{scope}
\path[clip] (  0.00,  0.00) rectangle (209.58,115.63);
\definecolor{drawColor}{RGB}{86,51,94}

\path[draw=drawColor,line width= 0.4pt,line join=round,line cap=round] (161.12,  2.56) -- (164.19,  5.62);

\path[draw=drawColor,line width= 0.4pt,line join=round,line cap=round] (161.12,  5.62) -- (164.19,  2.56);

\path[draw=drawColor,line width= 0.4pt,line join=round,line cap=round] (160.49,  4.09) -- (164.82,  4.09);

\path[draw=drawColor,line width= 0.4pt,line join=round,line cap=round] (162.65,  1.92) -- (162.65,  6.26);
\end{scope}
\begin{scope}
\path[clip] (  0.00,  0.00) rectangle (209.58,115.63);
\definecolor{drawColor}{RGB}{0,0,0}

\node[text=drawColor,anchor=base west,inner sep=0pt, outer sep=0pt, scale=  0.70] at ( 27.26,  1.68) {DECOR};
\end{scope}
\begin{scope}
\path[clip] (  0.00,  0.00) rectangle (209.58,115.63);
\definecolor{drawColor}{RGB}{0,0,0}

\node[text=drawColor,anchor=base west,inner sep=0pt, outer sep=0pt, scale=  0.70] at ( 69.85,  1.68) {DECOR+};
\end{scope}
\begin{scope}
\path[clip] (  0.00,  0.00) rectangle (209.58,115.63);
\definecolor{drawColor}{RGB}{0,0,0}

\node[text=drawColor,anchor=base west,inner sep=0pt, outer sep=0pt, scale=  0.70] at (117.90,  1.68) {A-DECOR};
\end{scope}
\begin{scope}
\path[clip] (  0.00,  0.00) rectangle (209.58,115.63);
\definecolor{drawColor}{RGB}{0,0,0}

\node[text=drawColor,anchor=base west,inner sep=0pt, outer sep=0pt, scale=  0.70] at (168.07,  1.68) {Brute-Force};
\end{scope}
\end{tikzpicture}

%% file: files/decor_revised_plot_avg_times_apriori.tex
\begin{tikzpicture}[x=1pt,y=1pt]
\definecolor{fillColor}{RGB}{255,255,255}
\path[use as bounding box,fill=fillColor,fill opacity=0.00] (0,0) rectangle (209.58,115.63);
\begin{scope}
\path[clip] (  0.00,  0.00) rectangle (209.58,115.63);
\definecolor{drawColor}{RGB}{255,255,255}
\definecolor{fillColor}{RGB}{255,255,255}

\path[draw=drawColor,line width= 0.5pt,line join=round,line cap=round,fill=fillColor] (  0.00,  0.00) rectangle (209.58,115.63);
\end{scope}
\begin{scope}
\path[clip] ( 34.42, 47.23) rectangle (203.58,113.63);
\definecolor{fillColor}{RGB}{255,255,255}

\path[fill=fillColor] ( 34.42, 47.23) rectangle (203.58,113.63);
\definecolor{drawColor}{RGB}{247,192,26}

\path[draw=drawColor,line width= 0.5pt,line join=round] ( 42.11, 56.93) --
	( 64.07, 66.40) --
	( 86.04, 73.77) --
	(108.01, 80.50) --
	(129.98, 86.95) --
	(151.95, 93.03) --
	(173.92, 99.21) --
	(195.89,105.53);
\definecolor{drawColor}{RGB}{78,155,133}

\path[draw=drawColor,line width= 0.5pt,dash pattern=on 4pt off 4pt ,line join=round] ( 42.11, 57.07) --
	( 64.07, 66.46) --
	( 86.04, 73.82) --
	(108.01, 80.50) --
	(129.98, 86.91) --
	(151.95, 93.05) --
	(173.92, 99.18) --
	(195.89,105.54);
\definecolor{drawColor}{RGB}{247,192,26}
\definecolor{fillColor}{RGB}{247,192,26}

\path[draw=drawColor,line width= 0.4pt,line join=round,line cap=round,fill=fillColor] (173.92, 99.21) circle (  1.53);

\path[draw=drawColor,line width= 0.4pt,line join=round,line cap=round,fill=fillColor] ( 64.07, 66.40) circle (  1.53);

\path[draw=drawColor,line width= 0.4pt,line join=round,line cap=round,fill=fillColor] (151.95, 93.03) circle (  1.53);

\path[draw=drawColor,line width= 0.4pt,line join=round,line cap=round,fill=fillColor] (129.98, 86.95) circle (  1.53);

\path[draw=drawColor,line width= 0.4pt,line join=round,line cap=round,fill=fillColor] ( 42.11, 56.93) circle (  1.53);

\path[draw=drawColor,line width= 0.4pt,line join=round,line cap=round,fill=fillColor] ( 86.04, 73.77) circle (  1.53);

\path[draw=drawColor,line width= 0.4pt,line join=round,line cap=round,fill=fillColor] (195.89,105.53) circle (  1.53);

\path[draw=drawColor,line width= 0.4pt,line join=round,line cap=round,fill=fillColor] (108.01, 80.50) circle (  1.53);
\definecolor{fillColor}{RGB}{78,155,133}

\path[fill=fillColor] (173.92,101.56) --
	(175.99, 97.98) --
	(171.86, 97.98) --
	cycle;

\path[fill=fillColor] ( 64.07, 68.85) --
	( 66.14, 65.27) --
	( 62.01, 65.27) --
	cycle;

\path[fill=fillColor] (151.95, 95.44) --
	(154.02, 91.86) --
	(149.89, 91.86) --
	cycle;

\path[fill=fillColor] (129.98, 89.30) --
	(132.05, 85.72) --
	(127.92, 85.72) --
	cycle;

\path[fill=fillColor] ( 42.11, 59.46) --
	( 44.17, 55.88) --
	( 40.04, 55.88) --
	cycle;

\path[fill=fillColor] ( 86.04, 76.20) --
	( 88.11, 72.62) --
	( 83.98, 72.62) --
	cycle;

\path[fill=fillColor] (195.89,107.93) --
	(197.96,104.35) --
	(193.83,104.35) --
	cycle;

\path[fill=fillColor] (108.01, 82.89) --
	(110.08, 79.31) --
	(105.95, 79.31) --
	cycle;
\end{scope}
\begin{scope}
\path[clip] (  0.00,  0.00) rectangle (209.58,115.63);
\definecolor{drawColor}{RGB}{0,0,0}

\path[draw=drawColor,line width= 0.5pt,line join=round] ( 34.42, 47.23) --
	( 34.42,113.63);

\path[draw=drawColor,line width= 0.5pt,line join=round] ( 35.55,111.66) --
	( 34.42,113.63) --
	( 33.28,111.66);
\end{scope}
\begin{scope}
\path[clip] (  0.00,  0.00) rectangle (209.58,115.63);
\definecolor{drawColor}{gray}{0.30}

\node[text=drawColor,anchor=base east,inner sep=0pt, outer sep=0pt, scale=  0.70] at ( 30.37, 48.51) {1e-03};

\node[text=drawColor,anchor=base east,inner sep=0pt, outer sep=0pt, scale=  0.70] at ( 30.37, 63.26) {1e-01};

\node[text=drawColor,anchor=base east,inner sep=0pt, outer sep=0pt, scale=  0.70] at ( 30.37, 78.02) {1e+01};

\node[text=drawColor,anchor=base east,inner sep=0pt, outer sep=0pt, scale=  0.70] at ( 30.37, 92.78) {1e+03};

\node[text=drawColor,anchor=base east,inner sep=0pt, outer sep=0pt, scale=  0.70] at ( 30.37,107.53) {1e+05};
\end{scope}
\begin{scope}
\path[clip] (  0.00,  0.00) rectangle (209.58,115.63);
\definecolor{drawColor}{gray}{0.20}

\path[draw=drawColor,line width= 0.5pt,line join=round] ( 32.17, 50.92) --
	( 34.42, 50.92);

\path[draw=drawColor,line width= 0.5pt,line join=round] ( 32.17, 65.67) --
	( 34.42, 65.67);

\path[draw=drawColor,line width= 0.5pt,line join=round] ( 32.17, 80.43) --
	( 34.42, 80.43);

\path[draw=drawColor,line width= 0.5pt,line join=round] ( 32.17, 95.19) --
	( 34.42, 95.19);

\path[draw=drawColor,line width= 0.5pt,line join=round] ( 32.17,109.94) --
	( 34.42,109.94);
\end{scope}
\begin{scope}
\path[clip] (  0.00,  0.00) rectangle (209.58,115.63);
\definecolor{drawColor}{RGB}{0,0,0}

\path[draw=drawColor,line width= 0.5pt,line join=round] ( 34.42, 47.23) --
	(203.58, 47.23);

\path[draw=drawColor,line width= 0.5pt,line join=round] (201.61, 46.09) --
	(203.58, 47.23) --
	(201.61, 48.37);
\end{scope}
\begin{scope}
\path[clip] (  0.00,  0.00) rectangle (209.58,115.63);
\definecolor{drawColor}{gray}{0.20}

\path[draw=drawColor,line width= 0.5pt,line join=round] ( 64.07, 44.98) --
	( 64.07, 47.23);

\path[draw=drawColor,line width= 0.5pt,line join=round] (108.01, 44.98) --
	(108.01, 47.23);

\path[draw=drawColor,line width= 0.5pt,line join=round] (151.95, 44.98) --
	(151.95, 47.23);

\path[draw=drawColor,line width= 0.5pt,line join=round] (195.89, 44.98) --
	(195.89, 47.23);
\end{scope}
\begin{scope}
\path[clip] (  0.00,  0.00) rectangle (209.58,115.63);
\definecolor{drawColor}{gray}{0.30}

\node[text=drawColor,anchor=base,inner sep=0pt, outer sep=0pt, scale=  0.70] at ( 64.07, 38.36) {4};

\node[text=drawColor,anchor=base,inner sep=0pt, outer sep=0pt, scale=  0.70] at (108.01, 38.36) {8};

\node[text=drawColor,anchor=base,inner sep=0pt, outer sep=0pt, scale=  0.70] at (151.95, 38.36) {12};

\node[text=drawColor,anchor=base,inner sep=0pt, outer sep=0pt, scale=  0.70] at (195.89, 38.36) {16};
\end{scope}
\begin{scope}
\path[clip] (  0.00,  0.00) rectangle (209.58,115.63);
\definecolor{drawColor}{RGB}{0,0,0}

\node[text=drawColor,anchor=base,inner sep=0pt, outer sep=0pt, scale=  0.80] at (119.00, 29.24) {$n$};
\end{scope}
\begin{scope}
\path[clip] (  0.00,  0.00) rectangle (209.58,115.63);
\definecolor{drawColor}{RGB}{0,0,0}

\node[text=drawColor,rotate= 90.00,anchor=base,inner sep=0pt, outer sep=0pt, scale=  0.80] at (  7.51, 80.43) {time (ms)};
\end{scope}
\begin{scope}
\path[clip] (  0.00,  0.00) rectangle (209.58,115.63);

\path[] ( 66.39, 19.50) rectangle (171.61, 25.68);
\end{scope}
\begin{scope}
\path[clip] (  0.00,  0.00) rectangle (209.58,115.63);
\definecolor{drawColor}{RGB}{247,192,26}

\path[draw=drawColor,line width= 0.5pt,line join=round] ( 67.47, 22.59) -- ( 76.15, 22.59);
\end{scope}
\begin{scope}
\path[clip] (  0.00,  0.00) rectangle (209.58,115.63);
\definecolor{drawColor}{RGB}{247,192,26}
\definecolor{fillColor}{RGB}{247,192,26}

\path[draw=drawColor,line width= 0.4pt,line join=round,line cap=round,fill=fillColor] ( 71.81, 22.59) circle (  1.53);
\end{scope}
\begin{scope}
\path[clip] (  0.00,  0.00) rectangle (209.58,115.63);
\definecolor{drawColor}{RGB}{78,155,133}

\path[draw=drawColor,line width= 0.5pt,dash pattern=on 4pt off 4pt ,line join=round] (117.65, 22.59) -- (126.33, 22.59);
\end{scope}
\begin{scope}
\path[clip] (  0.00,  0.00) rectangle (209.58,115.63);
\definecolor{fillColor}{RGB}{78,155,133}

\path[fill=fillColor] (121.99, 24.98) --
	(124.06, 21.40) --
	(119.92, 21.40) --
	cycle;
\end{scope}
\begin{scope}
\path[clip] (  0.00,  0.00) rectangle (209.58,115.63);
\definecolor{drawColor}{RGB}{0,0,0}

\node[text=drawColor,anchor=base west,inner sep=0pt, outer sep=0pt, scale=  0.70] at ( 77.23, 20.18) {A-DECOR};
\end{scope}
\begin{scope}
\path[clip] (  0.00,  0.00) rectangle (209.58,115.63);
\definecolor{drawColor}{RGB}{0,0,0}

\node[text=drawColor,anchor=base west,inner sep=0pt, outer sep=0pt, scale=  0.70] at (127.41, 20.18) {CC-DECOR};
\end{scope}
\end{tikzpicture}

%% file: files/decor_revised_plot_avg_times_heuristics.tex
\begin{tikzpicture}[x=1pt,y=1pt]
\definecolor{fillColor}{RGB}{255,255,255}
\path[use as bounding box,fill=fillColor,fill opacity=0.00] (0,0) rectangle (209.58,115.63);
\begin{scope}
\path[clip] (  0.00,  0.00) rectangle (209.58,115.63);
\definecolor{drawColor}{RGB}{255,255,255}
\definecolor{fillColor}{RGB}{255,255,255}

\path[draw=drawColor,line width= 0.5pt,line join=round,line cap=round,fill=fillColor] (  0.00,  0.00) rectangle (209.58,115.63);
\end{scope}
\begin{scope}
\path[clip] ( 34.42, 47.09) rectangle (203.58,113.63);
\definecolor{fillColor}{RGB}{255,255,255}

\path[fill=fillColor] ( 34.42, 47.09) rectangle (203.58,113.63);
\definecolor{drawColor}{RGB}{247,192,26}

\path[draw=drawColor,line width= 0.5pt,line join=round] ( 42.11, 58.34) --
	( 64.07, 62.98) --
	( 86.04, 68.18) --
	(108.01, 72.96) --
	(129.98, 78.29) --
	(151.95, 83.64) --
	(173.92, 89.69) --
	(195.89, 95.80);
\definecolor{drawColor}{RGB}{78,155,133}

\path[draw=drawColor,line width= 0.5pt,dash pattern=on 4pt off 4pt ,line join=round] ( 42.11, 58.42) --
	( 64.07, 62.99) --
	( 86.04, 68.16) --
	(108.01, 72.95) --
	(129.98, 78.26) --
	(151.95, 83.62) --
	(173.92, 89.65) --
	(195.89, 95.80);
\definecolor{drawColor}{RGB}{37,122,164}

\path[draw=drawColor,line width= 0.5pt,dash pattern=on 1pt off 3pt ,line join=round] ( 42.11, 58.49) --
	( 64.07, 63.14) --
	( 86.04, 68.17) --
	(108.01, 72.96) --
	(129.98, 78.28) --
	(151.95, 83.64) --
	(173.92, 89.66) --
	(195.89, 95.80);
\definecolor{drawColor}{RGB}{86,51,94}

\path[draw=drawColor,line width= 0.5pt,dash pattern=on 7pt off 3pt ,line join=round] ( 42.11, 58.56) --
	( 64.07, 63.03) --
	( 86.04, 68.18) --
	(108.01, 72.98) --
	(129.98, 78.29) --
	(151.95, 83.64) --
	(173.92, 89.68) --
	(195.89, 95.81);
\definecolor{drawColor}{RGB}{126,41,84}

\path[draw=drawColor,line width= 0.5pt,dash pattern=on 2pt off 2pt on 6pt off 2pt ,line join=round] ( 42.11, 58.47) --
	( 64.07, 63.01) --
	( 86.04, 68.15) --
	(108.01, 72.98) --
	(129.98, 78.29) --
	(151.95, 83.66) --
	(173.92, 89.68) --
	(195.89, 95.82);
\definecolor{drawColor}{RGB}{247,192,26}
\definecolor{fillColor}{RGB}{247,192,26}

\path[draw=drawColor,line width= 0.4pt,line join=round,line cap=round,fill=fillColor] (173.92, 89.69) circle (  1.53);

\path[draw=drawColor,line width= 0.4pt,line join=round,line cap=round,fill=fillColor] ( 64.07, 62.98) circle (  1.53);

\path[draw=drawColor,line width= 0.4pt,line join=round,line cap=round,fill=fillColor] (151.95, 83.64) circle (  1.53);

\path[draw=drawColor,line width= 0.4pt,line join=round,line cap=round,fill=fillColor] (129.98, 78.29) circle (  1.53);

\path[draw=drawColor,line width= 0.4pt,line join=round,line cap=round,fill=fillColor] ( 42.11, 58.34) circle (  1.53);

\path[draw=drawColor,line width= 0.4pt,line join=round,line cap=round,fill=fillColor] ( 86.04, 68.18) circle (  1.53);

\path[draw=drawColor,line width= 0.4pt,line join=round,line cap=round,fill=fillColor] (195.89, 95.80) circle (  1.53);

\path[draw=drawColor,line width= 0.4pt,line join=round,line cap=round,fill=fillColor] (108.01, 72.96) circle (  1.53);
\definecolor{drawColor}{RGB}{126,41,84}

\path[draw=drawColor,line width= 0.4pt,line join=round,line cap=round] (173.92, 87.76) --
	(175.85, 89.68) --
	(173.92, 91.60) --
	(172.00, 89.68) --
	cycle;

\path[draw=drawColor,line width= 0.4pt,line join=round,line cap=round] ( 64.07, 61.09) --
	( 66.00, 63.01) --
	( 64.07, 64.93) --
	( 62.15, 63.01) --
	cycle;

\path[draw=drawColor,line width= 0.4pt,line join=round,line cap=round] (151.95, 81.73) --
	(153.88, 83.66) --
	(151.95, 85.58) --
	(150.03, 83.66) --
	cycle;

\path[draw=drawColor,line width= 0.4pt,line join=round,line cap=round] (129.98, 76.37) --
	(131.91, 78.29) --
	(129.98, 80.21) --
	(128.06, 78.29) --
	cycle;

\path[draw=drawColor,line width= 0.4pt,line join=round,line cap=round] ( 42.11, 56.54) --
	( 44.03, 58.47) --
	( 42.11, 60.39) --
	( 40.18, 58.47) --
	cycle;

\path[draw=drawColor,line width= 0.4pt,line join=round,line cap=round] ( 86.04, 66.22) --
	( 87.97, 68.15) --
	( 86.04, 70.07) --
	( 84.12, 68.15) --
	cycle;

\path[draw=drawColor,line width= 0.4pt,line join=round,line cap=round] (195.89, 93.90) --
	(197.82, 95.82) --
	(195.89, 97.74) --
	(193.97, 95.82) --
	cycle;

\path[draw=drawColor,line width= 0.4pt,line join=round,line cap=round] (108.01, 71.05) --
	(109.94, 72.98) --
	(108.01, 74.90) --
	(106.09, 72.98) --
	cycle;
\definecolor{drawColor}{RGB}{86,51,94}

\path[draw=drawColor,line width= 0.4pt,line join=round,line cap=round] (172.39, 88.15) -- (175.46, 91.22);

\path[draw=drawColor,line width= 0.4pt,line join=round,line cap=round] (172.39, 91.22) -- (175.46, 88.15);

\path[draw=drawColor,line width= 0.4pt,line join=round,line cap=round] (171.75, 89.68) -- (176.09, 89.68);

\path[draw=drawColor,line width= 0.4pt,line join=round,line cap=round] (173.92, 87.52) -- (173.92, 91.85);

\path[draw=drawColor,line width= 0.4pt,line join=round,line cap=round] ( 62.54, 61.50) -- ( 65.61, 64.56);

\path[draw=drawColor,line width= 0.4pt,line join=round,line cap=round] ( 62.54, 64.56) -- ( 65.61, 61.50);

\path[draw=drawColor,line width= 0.4pt,line join=round,line cap=round] ( 61.91, 63.03) -- ( 66.24, 63.03);

\path[draw=drawColor,line width= 0.4pt,line join=round,line cap=round] ( 64.07, 60.86) -- ( 64.07, 65.20);

\path[draw=drawColor,line width= 0.4pt,line join=round,line cap=round] (150.42, 82.11) -- (153.49, 85.18);

\path[draw=drawColor,line width= 0.4pt,line join=round,line cap=round] (150.42, 85.18) -- (153.49, 82.11);

\path[draw=drawColor,line width= 0.4pt,line join=round,line cap=round] (149.79, 83.64) -- (154.12, 83.64);

\path[draw=drawColor,line width= 0.4pt,line join=round,line cap=round] (151.95, 81.47) -- (151.95, 85.81);

\path[draw=drawColor,line width= 0.4pt,line join=round,line cap=round] (128.45, 76.75) -- (131.52, 79.82);

\path[draw=drawColor,line width= 0.4pt,line join=round,line cap=round] (128.45, 79.82) -- (131.52, 76.75);

\path[draw=drawColor,line width= 0.4pt,line join=round,line cap=round] (127.82, 78.29) -- (132.15, 78.29);

\path[draw=drawColor,line width= 0.4pt,line join=round,line cap=round] (129.98, 76.12) -- (129.98, 80.46);

\path[draw=drawColor,line width= 0.4pt,line join=round,line cap=round] ( 40.57, 57.03) -- ( 43.64, 60.10);

\path[draw=drawColor,line width= 0.4pt,line join=round,line cap=round] ( 40.57, 60.10) -- ( 43.64, 57.03);

\path[draw=drawColor,line width= 0.4pt,line join=round,line cap=round] ( 39.94, 58.56) -- ( 44.27, 58.56);

\path[draw=drawColor,line width= 0.4pt,line join=round,line cap=round] ( 42.11, 56.39) -- ( 42.11, 60.73);

\path[draw=drawColor,line width= 0.4pt,line join=round,line cap=round] ( 84.51, 66.65) -- ( 87.58, 69.71);

\path[draw=drawColor,line width= 0.4pt,line join=round,line cap=round] ( 84.51, 69.71) -- ( 87.58, 66.65);

\path[draw=drawColor,line width= 0.4pt,line join=round,line cap=round] ( 83.88, 68.18) -- ( 88.21, 68.18);

\path[draw=drawColor,line width= 0.4pt,line join=round,line cap=round] ( 86.04, 66.01) -- ( 86.04, 70.35);

\path[draw=drawColor,line width= 0.4pt,line join=round,line cap=round] (194.36, 94.28) -- (197.43, 97.34);

\path[draw=drawColor,line width= 0.4pt,line join=round,line cap=round] (194.36, 97.34) -- (197.43, 94.28);

\path[draw=drawColor,line width= 0.4pt,line join=round,line cap=round] (193.72, 95.81) -- (198.06, 95.81);

\path[draw=drawColor,line width= 0.4pt,line join=round,line cap=round] (195.89, 93.64) -- (195.89, 97.98);

\path[draw=drawColor,line width= 0.4pt,line join=round,line cap=round] (106.48, 71.44) -- (109.55, 74.51);

\path[draw=drawColor,line width= 0.4pt,line join=round,line cap=round] (106.48, 74.51) -- (109.55, 71.44);

\path[draw=drawColor,line width= 0.4pt,line join=round,line cap=round] (105.85, 72.98) -- (110.18, 72.98);

\path[draw=drawColor,line width= 0.4pt,line join=round,line cap=round] (108.01, 70.81) -- (108.01, 75.14);
\definecolor{fillColor}{RGB}{78,155,133}

\path[fill=fillColor] (173.92, 92.04) --
	(175.99, 88.46) --
	(171.86, 88.46) --
	cycle;

\path[fill=fillColor] ( 64.07, 65.38) --
	( 66.14, 61.80) --
	( 62.01, 61.80) --
	cycle;

\path[fill=fillColor] (151.95, 86.00) --
	(154.02, 82.43) --
	(149.89, 82.43) --
	cycle;

\path[fill=fillColor] (129.98, 80.65) --
	(132.05, 77.07) --
	(127.92, 77.07) --
	cycle;

\path[fill=fillColor] ( 42.11, 60.81) --
	( 44.17, 57.23) --
	( 40.04, 57.23) --
	cycle;

\path[fill=fillColor] ( 86.04, 70.55) --
	( 88.11, 66.97) --
	( 83.98, 66.97) --
	cycle;

\path[fill=fillColor] (195.89, 98.18) --
	(197.96, 94.61) --
	(193.83, 94.61) --
	cycle;

\path[fill=fillColor] (108.01, 75.34) --
	(110.08, 71.76) --
	(105.95, 71.76) --
	cycle;
\definecolor{fillColor}{RGB}{37,122,164}

\path[fill=fillColor] (172.39, 88.13) --
	(175.46, 88.13) --
	(175.46, 91.19) --
	(172.39, 91.19) --
	cycle;

\path[fill=fillColor] ( 62.54, 61.61) --
	( 65.61, 61.61) --
	( 65.61, 64.67) --
	( 62.54, 64.67) --
	cycle;

\path[fill=fillColor] (150.42, 82.11) --
	(153.49, 82.11) --
	(153.49, 85.18) --
	(150.42, 85.18) --
	cycle;

\path[fill=fillColor] (128.45, 76.74) --
	(131.52, 76.74) --
	(131.52, 79.81) --
	(128.45, 79.81) --
	cycle;

\path[fill=fillColor] ( 40.57, 56.96) --
	( 43.64, 56.96) --
	( 43.64, 60.02) --
	( 40.57, 60.02) --
	cycle;

\path[fill=fillColor] ( 84.51, 66.63) --
	( 87.58, 66.63) --
	( 87.58, 69.70) --
	( 84.51, 69.70) --
	cycle;

\path[fill=fillColor] (194.36, 94.27) --
	(197.43, 94.27) --
	(197.43, 97.34) --
	(194.36, 97.34) --
	cycle;

\path[fill=fillColor] (106.48, 71.43) --
	(109.55, 71.43) --
	(109.55, 74.49) --
	(106.48, 74.49) --
	cycle;
\end{scope}
\begin{scope}
\path[clip] (  0.00,  0.00) rectangle (209.58,115.63);
\definecolor{drawColor}{RGB}{0,0,0}

\path[draw=drawColor,line width= 0.5pt,line join=round] ( 34.42, 47.09) --
	( 34.42,113.63);

\path[draw=drawColor,line width= 0.5pt,line join=round] ( 35.55,111.66) --
	( 34.42,113.63) --
	( 33.28,111.66);
\end{scope}
\begin{scope}
\path[clip] (  0.00,  0.00) rectangle (209.58,115.63);
\definecolor{drawColor}{gray}{0.30}

\node[text=drawColor,anchor=base east,inner sep=0pt, outer sep=0pt, scale=  0.70] at ( 30.37, 48.38) {1e-03};

\node[text=drawColor,anchor=base east,inner sep=0pt, outer sep=0pt, scale=  0.70] at ( 30.37, 63.16) {1e-01};

\node[text=drawColor,anchor=base east,inner sep=0pt, outer sep=0pt, scale=  0.70] at ( 30.37, 77.95) {1e+01};

\node[text=drawColor,anchor=base east,inner sep=0pt, outer sep=0pt, scale=  0.70] at ( 30.37, 92.74) {1e+03};

\node[text=drawColor,anchor=base east,inner sep=0pt, outer sep=0pt, scale=  0.70] at ( 30.37,107.52) {1e+05};
\end{scope}
\begin{scope}
\path[clip] (  0.00,  0.00) rectangle (209.58,115.63);
\definecolor{drawColor}{gray}{0.20}

\path[draw=drawColor,line width= 0.5pt,line join=round] ( 32.17, 50.79) --
	( 34.42, 50.79);

\path[draw=drawColor,line width= 0.5pt,line join=round] ( 32.17, 65.58) --
	( 34.42, 65.58);

\path[draw=drawColor,line width= 0.5pt,line join=round] ( 32.17, 80.36) --
	( 34.42, 80.36);

\path[draw=drawColor,line width= 0.5pt,line join=round] ( 32.17, 95.15) --
	( 34.42, 95.15);

\path[draw=drawColor,line width= 0.5pt,line join=round] ( 32.17,109.94) --
	( 34.42,109.94);
\end{scope}
\begin{scope}
\path[clip] (  0.00,  0.00) rectangle (209.58,115.63);
\definecolor{drawColor}{RGB}{0,0,0}

\path[draw=drawColor,line width= 0.5pt,line join=round] ( 34.42, 47.09) --
	(203.58, 47.09);

\path[draw=drawColor,line width= 0.5pt,line join=round] (201.61, 45.95) --
	(203.58, 47.09) --
	(201.61, 48.23);
\end{scope}
\begin{scope}
\path[clip] (  0.00,  0.00) rectangle (209.58,115.63);
\definecolor{drawColor}{gray}{0.20}

\path[draw=drawColor,line width= 0.5pt,line join=round] ( 64.07, 44.84) --
	( 64.07, 47.09);

\path[draw=drawColor,line width= 0.5pt,line join=round] (108.01, 44.84) --
	(108.01, 47.09);

\path[draw=drawColor,line width= 0.5pt,line join=round] (151.95, 44.84) --
	(151.95, 47.09);

\path[draw=drawColor,line width= 0.5pt,line join=round] (195.89, 44.84) --
	(195.89, 47.09);
\end{scope}
\begin{scope}
\path[clip] (  0.00,  0.00) rectangle (209.58,115.63);
\definecolor{drawColor}{gray}{0.30}

\node[text=drawColor,anchor=base,inner sep=0pt, outer sep=0pt, scale=  0.70] at ( 64.07, 38.22) {4};

\node[text=drawColor,anchor=base,inner sep=0pt, outer sep=0pt, scale=  0.70] at (108.01, 38.22) {8};

\node[text=drawColor,anchor=base,inner sep=0pt, outer sep=0pt, scale=  0.70] at (151.95, 38.22) {12};

\node[text=drawColor,anchor=base,inner sep=0pt, outer sep=0pt, scale=  0.70] at (195.89, 38.22) {16};
\end{scope}
\begin{scope}
\path[clip] (  0.00,  0.00) rectangle (209.58,115.63);
\definecolor{drawColor}{RGB}{0,0,0}

\node[text=drawColor,anchor=base,inner sep=0pt, outer sep=0pt, scale=  0.80] at (119.00, 29.10) {$n$};
\end{scope}
\begin{scope}
\path[clip] (  0.00,  0.00) rectangle (209.58,115.63);
\definecolor{drawColor}{RGB}{0,0,0}

\node[text=drawColor,rotate= 90.00,anchor=base,inner sep=0pt, outer sep=0pt, scale=  0.80] at (  7.51, 80.36) {time (ms)};
\end{scope}
\begin{scope}
\path[clip] (  0.00,  0.00) rectangle (209.58,115.63);

\path[] ( 44.52,  1.00) rectangle (193.48, 25.55);
\end{scope}
\begin{scope}
\path[clip] (  0.00,  0.00) rectangle (209.58,115.63);
\definecolor{drawColor}{RGB}{247,192,26}

\path[draw=drawColor,line width= 0.5pt,line join=round] ( 45.60, 22.45) -- ( 54.28, 22.45);
\end{scope}
\begin{scope}
\path[clip] (  0.00,  0.00) rectangle (209.58,115.63);
\definecolor{drawColor}{RGB}{247,192,26}
\definecolor{fillColor}{RGB}{247,192,26}

\path[draw=drawColor,line width= 0.4pt,line join=round,line cap=round,fill=fillColor] ( 49.94, 22.45) circle (  1.53);
\end{scope}
\begin{scope}
\path[clip] (  0.00,  0.00) rectangle (209.58,115.63);
\definecolor{drawColor}{RGB}{78,155,133}

\path[draw=drawColor,line width= 0.5pt,dash pattern=on 4pt off 4pt ,line join=round] (121.35, 22.45) -- (130.02, 22.45);
\end{scope}
\begin{scope}
\path[clip] (  0.00,  0.00) rectangle (209.58,115.63);
\definecolor{fillColor}{RGB}{78,155,133}

\path[fill=fillColor] (125.68, 24.84) --
	(127.75, 21.26) --
	(123.62, 21.26) --
	cycle;
\end{scope}
\begin{scope}
\path[clip] (  0.00,  0.00) rectangle (209.58,115.63);
\definecolor{drawColor}{RGB}{37,122,164}

\path[draw=drawColor,line width= 0.5pt,dash pattern=on 1pt off 3pt ,line join=round] ( 45.60, 13.27) -- ( 54.28, 13.27);
\end{scope}
\begin{scope}
\path[clip] (  0.00,  0.00) rectangle (209.58,115.63);
\definecolor{fillColor}{RGB}{37,122,164}

\path[fill=fillColor] ( 48.41, 11.74) --
	( 51.47, 11.74) --
	( 51.47, 14.81) --
	( 48.41, 14.81) --
	cycle;
\end{scope}
\begin{scope}
\path[clip] (  0.00,  0.00) rectangle (209.58,115.63);
\definecolor{drawColor}{RGB}{86,51,94}

\path[draw=drawColor,line width= 0.5pt,dash pattern=on 7pt off 3pt ,line join=round] (121.35, 13.27) -- (130.02, 13.27);
\end{scope}
\begin{scope}
\path[clip] (  0.00,  0.00) rectangle (209.58,115.63);
\definecolor{drawColor}{RGB}{86,51,94}

\path[draw=drawColor,line width= 0.4pt,line join=round,line cap=round] (124.15, 11.74) -- (127.22, 14.81);

\path[draw=drawColor,line width= 0.4pt,line join=round,line cap=round] (124.15, 14.81) -- (127.22, 11.74);

\path[draw=drawColor,line width= 0.4pt,line join=round,line cap=round] (123.51, 13.27) -- (127.85, 13.27);

\path[draw=drawColor,line width= 0.4pt,line join=round,line cap=round] (125.68, 11.10) -- (125.68, 15.44);
\end{scope}
\begin{scope}
\path[clip] (  0.00,  0.00) rectangle (209.58,115.63);
\definecolor{drawColor}{RGB}{126,41,84}

\path[draw=drawColor,line width= 0.5pt,dash pattern=on 2pt off 2pt on 6pt off 2pt ,line join=round] ( 45.60,  4.09) -- ( 54.28,  4.09);
\end{scope}
\begin{scope}
\path[clip] (  0.00,  0.00) rectangle (209.58,115.63);
\definecolor{drawColor}{RGB}{126,41,84}

\path[draw=drawColor,line width= 0.4pt,line join=round,line cap=round] ( 49.94,  2.17) --
	( 51.86,  4.09) --
	( 49.94,  6.01) --
	( 48.02,  4.09) --
	cycle;
\end{scope}
\begin{scope}
\path[clip] (  0.00,  0.00) rectangle (209.58,115.63);
\definecolor{drawColor}{RGB}{0,0,0}

\node[text=drawColor,anchor=base west,inner sep=0pt, outer sep=0pt, scale=  0.70] at ( 55.36, 20.04) {DECOR+};
\end{scope}
\begin{scope}
\path[clip] (  0.00,  0.00) rectangle (209.58,115.63);
\definecolor{drawColor}{RGB}{0,0,0}

\node[text=drawColor,anchor=base west,inner sep=0pt, outer sep=0pt, scale=  0.70] at (131.10, 20.04) {DECOR+ (SBF)};
\end{scope}
\begin{scope}
\path[clip] (  0.00,  0.00) rectangle (209.58,115.63);
\definecolor{drawColor}{RGB}{0,0,0}

\node[text=drawColor,anchor=base west,inner sep=0pt, outer sep=0pt, scale=  0.70] at ( 55.36, 10.86) {DECOR+ (LGF)};
\end{scope}
\begin{scope}
\path[clip] (  0.00,  0.00) rectangle (209.58,115.63);
\definecolor{drawColor}{RGB}{0,0,0}

\node[text=drawColor,anchor=base west,inner sep=0pt, outer sep=0pt, scale=  0.70] at (131.10, 10.86) {DECOR+ (SCSF)};
\end{scope}
\begin{scope}
\path[clip] (  0.00,  0.00) rectangle (209.58,115.63);
\definecolor{drawColor}{RGB}{0,0,0}

\node[text=drawColor,anchor=base west,inner sep=0pt, outer sep=0pt, scale=  0.70] at ( 55.36,  1.68) {DECOR+ (SMCF)};
\end{scope}
\end{tikzpicture}